\documentclass[12pt,paper=letter]{scrreprt}

\RequirePackage{lastpage}
\RequirePackage{multicol}
\RequirePackage{array}
\RequirePackage{trimspaces}
\RequirePackage[bottom,multiple]{footmisc}
\PassOptionsToPackage{dvipsnames}{xcolor}
\RequirePackage{pdfpages}
\RequirePackage{xr}
\RequirePackage{fancyhdr}
\RequirePackage[footnotesize,bf]{caption}
\RequirePackage{emptypage}
\RequirePackage{url}
\RequirePackage{graphicx}
\RequirePackage{lmodern}
\RequirePackage[T1]{fontenc}
\RequirePackage{microtype}
\RequirePackage{natbib}
\RequirePackage{amsthm}
\newtheorem{theorem}{Theorem}[chapter]
\newtheorem{lemma}[theorem]{Lemma}
\newtheorem{corollary}[theorem]{Corollary}
\newtheorem{proposition}{Proposition}[chapter]
\newtheorem{remark}{Remark}[chapter]
\newtheorem{definition}{Definition}[chapter]
\newtheorem{example}{Example}[chapter]

\usepackage{acro}
\usepackage{bm}
\usepackage[utf8]{inputenc}
\usepackage{mwe}
\usepackage{graphicx}
\usepackage{amsfonts,amsmath,amssymb}
\usepackage{enumitem}
\usepackage{mathtools}
\usepackage{algpseudocode,algorithm}
\usepackage{nicefrac}
\usepackage{color}
\usepackage[colorlinks,citecolor=blue,urlcolor=blue,linkcolor=blue,pdfborder={0 0 0}]{hyperref}
\usepackage{url}
\usepackage[scr = boondoxo, scrscaled = 1]{mathalfa}
\usepackage[capitalize]{cleveref}

\crefname{equation}{}{}
\crefname{nlem}{Lemma}{Lemmas}
\crefname{assumption}{Assumption}{Assumptions}
\let\oldproposition\proposition
\renewcommand{\proposition}{%
  \crefalias{theorem}{nprop}%
  \oldproposition}
  \let\oldremark\remark
\renewcommand{\remark}{%
  \crefalias{theorem}{nrem}%
  \oldremark}
\let\oldcorollary\corollary
\renewcommand{\corollary}{%
  \crefalias{theorem}{ncor}%
  \oldcorollary}
\let\oldlemma\lemma
\renewcommand{\lemma}{%
  \crefalias{theorem}{nlem}%
  \oldlemma}
 \let\olddefinition\definition
 \renewcommand{\definition}{%
   \crefalias{theorem}{ndef}%
   \olddefinition}
\usepackage{longtable} %
\usepackage{fullpage}

\newcommand{\eps}{\epsilon}

\newcommand{\dee}{\mathop{\mathrm{d}\!}}
\newcommand{\dno}{\mathrm{d}}
\newcommand{\dt}{\,\dee t}
\newcommand{\ddt}{\frac{\dno}{\dno t}}

\newcommand{\dx}{\,\dee x}
\newcommand{\dy}{\,\dee y}
\newcommand{\dz}{\,\dee z}

\newcommand{\dr}{\,\dee r}

\newcommand{\domega}{\,\dee \omega}

\def\balign#1\ealign{\begin{align}#1\end{align}}
\def\baligns#1\ealigns{\begin{align*}#1\end{align*}}
\def\balignat#1\ealign{\begin{alignat}#1\end{alignat}}
\def\balignats#1\ealigns{\begin{alignat*}#1\end{alignat*}}
\def\bitemize#1\eitemize{\begin{itemize}#1\end{itemize}}
\def\benumerate#1\eenumerate{\begin{enumerate}#1\end{enumerate}}

\newenvironment{talign*}
 {\csname align*\endcsname}
 {\endalign}
\newenvironment{talign}
 {\csname align\endcsname}
 {\endalign}

\def\balignst#1\ealignst{\begin{talign*}#1\end{talign*}}
\def\balignt#1\ealignt{\begin{talign}#1\end{talign}}
\newcommand{\qtext}[1]{\quad\text{#1}\quad}

\def\Ito{It\^o\xspace}

\def\mbb#1{\mathbb{#1}}

\newcommand{\sss}[1]{\scriptscriptstyle{#1}}

\def\reals{\mathbb{R}} %
\def\R{\mathbb{R}}
\newcommand{\RR}{\mathbb{R}}
\newcommand{\Rd}{{\R}^d}

\def\naturals{\mathbb{N}} %

\def\complex{\mathbb{C}} %

\def\<{\left\langle} %
\def\>{\right\rangle}

\def\iff{\Leftrightarrow}

\def\defeq{\triangleq} %
\def\defn{\defeq} %
\def\defined{\defeq} %

\newcommand{\textfrac}[2]{{\textstyle\frac{#1}{#2}}}

\newcommand{\psdge}{\succcurlyeq}

\newcommand\ci{\perp\!\!\!\perp}

\def\norm#1{\left\|{#1}\right\|} %
\newcommand{\dualnorm}[1]{\norm{#1}^*} %
\newcommand{\onenorm}[1]{\norm{#1}_1} %
\newcommand{\twonorm}[1]{\norm{#1}_2} %
\newcommand{\infnorm}[1]{\norm{#1}_{\infty}} %
\def\staticnorm#1{\|{#1}\|} %
\newcommand{\inner}[2]{\langle{#1},{#2}\rangle} %
\def\indic#1{\mbb{I}\left[{#1}\right]} %
\def\polylog{\operatorname{polylog}}
\def\E{\mbb{E}} %
\def\Earg#1{\E\left[{#1}\right]}
\def\Esub#1{\E_{#1}}

\def\P{\mbb{P}} %

\newcommand{\Gsn}{\mathcal{N}}

\newcommand{\grad}{\nabla} %
\newcommand{\Hess}{\nabla^2} %

\newcommand{\toas}{\stackrel{a.s.}{\to}}

\def\KL#1#2{\textnormal{KL}({#1}\Vert{#2})}

\newcommand{\iid}{\textrm{i.i.d.}\xspace}
\newcommand{\dist}{\sim}
\newcommand{\distiid}{\overset{\textrm{\tiny\iid}}{\dist}}

\providecommand{\argmax}{\mathop\mathrm{arg max}} %
\providecommand{\argmin}{\mathop\mathrm{arg min}}

\def\supp#1{\mathrm{supp}({#1})}

\newtheorem{assumption}{Assumption}[chapter]

\newcommand{\Div}[1]{\grad\cdot{#1}} %

\newcommand{\normal}{\mathrm{n}} %

\newcommand{\evalat}[2]{{#1}\big|_{#2}}

\newcommand{\Pset}[0]{\mathcal{P}} %
\newcommand{\p}{p} %
\newcommand{\xset}[0]{\mathcal{X}} %
\newcommand{\gset}[0]{\mathcal{G}} %
\newcommand{\steinset}[0]{\steinsetarg{}} %
\newcommand{\steinsetarg}[1]{\gset_{\norm{\cdot}_{#1}}} %
\newcommand{\gsteinset}[2]{\gset_{\norm{\cdot}_{#1},Q,{#2}}} %
\newcommand{\sset}{\sigma} %
\newcommand{\equisubs}[2]{{[#1]\choose #2}}

\newcommand{\operator}{\mathcal{T}} %
\newcommand{\opsub}[1]{\mathcal{T}_{#1}} %
\newcommand{\opsubarg}[3]{(\opsub{#1}{#2})({#3})} %

\newcommand{\diffusion}[1]{\mathcal{T}{#1}} %

\newcommand{\TP}{\mathcal{T}_P}
\newcommand{\TPpsi}{\mathcal{T}_{P,\psi}}

\newcommand{\AP}{\mathcal{A}_P}

\newcommand{\nbr}[1][x]{\mathcal{N}_{#1}} %

\newcommand{\langevin}{\mathcal{T}_P} %
\newcommand{\ssdn}[1]{\mathcal{SS}({#1},\operator{},\gset_n)} %
\newcommand{\langssdn}[1]{\mathcal{SS}({#1},\langevin{},\gset_n)} %
\newcommand{\stein}[3]{\mathcal{S}({#1},{#2},{#3})} %
\newcommand{\opstein}[2]{\stein{#1}{\operator{}}{#2}} %
\newcommand{\langstein}[2]{\stein{#1}{\langevin{}}{#2}} %
\newcommand{\diffstein}[2]{\stein{#1}{\diffusion{}}{#2}} %
\newcommand{\pvar}[0]{z}%
\newcommand{\PVAR}[0]{\MakeUppercase{\pvar}} %

\newcommand{\process}[2]{\PVAR_{#1,#2}} %

\newcommand{\sr}{s_r} %

\newcommand{\KSD}{\ensuremath{\operatorname{KSD}}\xspace}
\renewcommand{\k}{k}
  
\newcommand{\kb}{\k} %
\newcommand{\Kb}{K} %

\newcommand{\ks}{\k_{\p}}%

\newcommand{\embedtozero}[1][\Kb]{\Pset_{\!\sss{#1,0}}}

\newcommand{\dd}{\mathrm{d}}

\def\F{\mathcal{F}}

\def\H{\mathcal{H}}

\newcommand{\Dset}{\mathscr{D}}
\newcommand{\DL}[1]{\Dset^{#1}_{\sss{L^1}}}

\newcommand{\HK}{\H_{\k}}
\newcommand{\Hks}{\H_{\ks}}

\newcommand{\tilt}{a}
\newcommand{\growth}{\theta}

\def\Id{\mathrm{Id}} %

\newcommand{\EE}{\mathbb{E}}	%

\newcommand{\mcC}{\mathcal{C}}

\newcommand{\convP}{\overset{P}{\to}}

\newcommand{\Lp}[1][p]{\mathcal{L}^{#1}}

\newcommand{\FTop}[0]{\mathscr{F}} %

\newcommand{\feat}{\Phi} %
\newcommand{\statfeat}{F}  %
\newcommand{\statfeatscalar}{f}  %
\newcommand{\statkernel}{\Psi} %
\newcommand{\basekernel}{k}  %

\newcommand{\FSD}{\ensuremath{\operatorname{\mathrm{\Phi}SD}}\xspace}

\newcommand{\RFSD}{\ensuremath{\operatorname{R\mathrm{\Phi}SD}}\xspace}

\newcommand{\isdist}{\nu}

\newcommand{\approxdist}[1][]{Q_{N}}

\newcommand{\approxdensity}[1][]{q_{}}
\newcommand{\targetdist}{P}

\newcommand{\bnprop}{\begin{proposition}}
\newcommand{\enprop}{\end{proposition}}
\newcommand{\bexa}{\begin{example}}
\newcommand{\eexa}{\end{example}}

\newcommand{\bnrmk}{\begin{remark}}
\newcommand{\enrmk}{\end{remark}}

\newcommand{\bnumdefn}{\begin{definition}}
\newcommand{\enumdefn}{\end{definition}}

\DeclareAcronym{cd}{short = CD, long = contrastive divergence}
\DeclareAcronym{clt}{short = CLT, long = central limit theorem}
\DeclareAcronym{elbo}{short = ELBO, long = evidence lower bound}
\DeclareAcronym{gan}{short = GAN, long = generative adversarial network}
\DeclareAcronym{ksd}{short = KSD, long = kernel Stein discrepancy}
\DeclareAcronym{map}{short = MAP, long = maximum a posteriori}
\DeclareAcronym{mle}{short = MLE, long = maximum likelihood estimation}
\DeclareAcronym{mcmc}{short = MCMC, long = Markov chain Monte Carlo}
\DeclareAcronym{mmd}{short = MMD, long = maximum mean discrepancy}
\DeclareAcronym{rkhs}{short = RKHS, long = reproducing kernel Hilbert space}
\DeclareAcronym{sde}{short = SDE, long = stochastic differential equation}
\DeclareAcronym{ssd}{short = SSD, long = stochastic Stein discrepancy}
\DeclareAcronym{svgd}{short = SVGD, long = Stein variational gradient descent}
\DeclareAcronym{ula}{short = ULA, long = unadjusted Langevin algorithm}
\DeclareAcronym{vae}{short = VAE, long = variational autoencoder}

\newcommand{\vecf}{v}
\newcommand{\vv}[1]{#1}
\newcommand{\df}{\mathrm{d}}
\def\bb#1\ee{\begin{align*}#1\end{align*}}
\def\bbb#1\eee{\begin{align}#1\end{align}}

\newcommand{\ff}{\phi}

\newcommand{\vx}{x}

\newcommand{\KLL}{\mathrm{KL}}
\newcommand{\vtheta}{\theta}

\usepackage{authblk}

\author[1]{Qiang Liu}
\author[2]{Lester Mackey}
\author[3,4]{Chris Oates}
\affil[1]{University of Texas at Austin}
\affil[2]{Microsoft Research}
\affil[3]{Newcastle University}
\affil[4]{The Alan Turing Institute}
\title{Probabilistic Inference and Learning with Stein's Method}

\begin{document}

\maketitle

\begin{abstract}
This monograph provides a rigorous overview of theoretical and methodological aspects of probabilistic inference and learning with Stein's method.
Recipes are provided for constructing Stein discrepancies from Stein operators and Stein sets, and properties of these discrepancies such as computability, separation, convergence detection, and convergence control are discussed.
Further, the connection between Stein operators and Stein variational gradient descent is set out in detail.
The main definitions and results are precisely stated, and references to all proofs are provided.    
\end{abstract}

\clearpage

\tableofcontents

\clearpage

\chapter{Introduction}
\label{chap: intro}

For over 50 years, Stein's method has provided a powerful \emph{theoretical} tool in probability theory, enabling explicit upper bounds on the convergence of random variables to their limiting distributions  \citep{stein1972bound}.
Several excellent references summarize this line of research and we refer the interested reader to \citet{barbour2005introduction,chen2010normal,Ross2011,LeyReSw2017}.
Our focus is rather on Stein's method as a \emph{methodological} tool that gives rise to new and powerful algorithms for probabilistic inference and learning. 
Our exposition is motivated by the introduction of computable Stein discrepancies in \citet{gorham2015measuring}, but the applications we explore have roots dating back to the early work of \citet{stein1986approximate,stein2004use}. 
A recent high-level survey of these developments can be found in \citet{anastasiou2021stein}, but otherwise the main results are distributed across the literature.
The present text therefore aims to provide a singular reference for rigorous definitions and results relevant to probabilistic inference and learning with Stein’s method.

\section{Stein's Method in a Nutshell}

Suppose $P$ is a probability distribution of interest supported on an appropriate space $\mathcal{X}$; we are especially interested in settings where most expectations under $P$ are not easily computed. 
Stein's method can be thought of as a recipe to measure how well integration under $P$ is approximated by integration under a given surrogate distribution $Q$.
For example, $P$ may represent the intractable posterior distribution arising in a Bayesian analysis, while $Q$ may represent a \emph{sample approximation} $Q = \frac{1}{n}\sum_{i=1}^n\delta_{x_i}$ with discrete points $\{x_i\}_{i=1}^n \subset \mathcal{X}$. 
To measure how well $Q$ approximates $P$, one would ideally compute an \emph{integral probability metric} \citep{muller1997integral} that measures the maximum discrepancy 
\begin{align*}
\sup_{f\in\mathcal{F}} \left|\E_Q[f(X)]-\E_P[f(X)]\right|
=
\sup_{f\in\mathcal{F}}\left|\frac{1}{n}\sum_{i=1}^nf(x_i) -\int f(x) \; \mathrm{d}P(x)\right|
\end{align*}
between sample and target expectations over a class of test functions $\mathcal{F}$. 
However, most standard integral probability metrics (including the Wasserstein metrics, Dudley metric, and maximum mean discrepancies described in \cref{sec: divergences}) require explicit integration under $P$, rendering them incomputable for our target distributions of interest. 
To address this challenge, we will adapt an ingenious idea that Charles M. Stein (1920-2016) introduced to quantify approximation error in the \ac{clt}. 
It is important to note that our presentation is adapted for \emph{methodological} applications of Stein's method, and the aspects that we emphasize are rather different from the presentations found in the aforementioned references on \emph{theoretical} applications of Stein's method.

At the heart of \emph{Stein's method} are a \emph{Stein set} $\mathcal{G}$ and a \emph{Stein operator} $\mathcal{T}_P : \mathcal{G} \rightarrow \mathcal{L}^1(P)$, which together generate mean-zero expectations under $P$:
\begin{align}
    \int (\mathcal{T}_P g)(x) \; \mathrm{d}P(x) = 0 
    \qtext{for all}
    g\in\mathcal{G}. \label{eq: stein identity}
\end{align}
Here $\mathcal{L}^1(P)$ denotes the set of functions that are integrable with respect to $P$, to ensure the integral in \eqref{eq: stein identity} is well-defined.
For instance, if $P$ is the standard normal distribution in dimension $d = 1$, and $\mathcal{G}$ is the set of all monomials 
$g(x) = x^m$ with $m \in \naturals$,
then 
\begin{align}
(\mathcal{T}_P g)(x) \defeq g'(x) - x g(x) \label{eq: intro stein op}
\end{align}
defines a valid Stein operator. Indeed we can explicitly compute the expectation 
\begin{align*}
    \int (\mathcal{T}_P g)(x) \; \mathrm{d}P(x) 
    & = \int m x^{m-1} - x^{m+1} \; \mathrm{d}P(x) \\
    & = \left\{ \begin{array}{ll} 0 & m \; \text{is even} \\ m (m-2)!! - m!! & m \; \text{is odd} \end{array} \right\} = 0
\end{align*}
where $n!!$ denotes the double factorial, that is, the product of all numbers from 1 to $n$ that have the same parity as $n$.

Given any valid combination of Stein set $\mathcal{G}$ and Stein operator $\mathcal{T}_P$, one can construct a \emph{Stein discrepancy}
\begin{align}
    \mathcal{S}(Q,\mathcal{T}_P,\mathcal{G}) \defeq \sup_{\substack{g \in \mathcal{G}}}
    \left|\int (\mathcal{T}_P g)(x) \; \mathrm{d}Q(x)\right|  \label{eq: first stein discrep def}
\end{align}
to measure the extent to which $Q$ differs from $P$.
The topology induced by a Stein discrepancy depends critically on the choice of $\mathcal{G}$ and $\mathcal{T}_P$, and an important challenge is to understand how to construct Stein discrepancies \eqref{eq: first stein discrep def} that are practically useful. 
To be slightly more precise, one can ask whether $\mathcal{G}$ and $\mathcal{T}_P$ can be chosen so that some (or all) of the following desiderata hold:
\begin{itemize}
    \item \textbf{Separation}:  $\mathcal{S}(Q,\mathcal{T}_P , \mathcal{G}) = 0$ if and only if $Q$ and $P$ are equal (i.e., the Stein discrepancy is a valid statistical divergence).
    \item \textbf{Convergence Detection}: $\mathcal{S}(Q_n,\mathcal{T}_P , \mathcal{G}) \rightarrow 0$ whenever $Q_n$ converge to $P$, in a sense to be specified.  
    \item \textbf{Convergence Control}: $\mathcal{S}(Q_n,\mathcal{T}_P , \mathcal{G}) \nrightarrow 0$ whenever $Q_n$ fails to converge to $P$, in a sense to be specified.
    \item \textbf{Computability}:  For any distribution $Q_n = \sum_{i=1}^n w_i \delta_{x_i}$ with finite support set $\{x_i\}_{i=1}^n \subset \mathcal{X}$, the Stein discrepancy $\mathcal{S}(Q_n,\mathcal{T}_P,\mathcal{G})$ can be explicitly computed.
\end{itemize}

\noindent Several positive and negative results concerning the properties of Stein discrepancies have appeared over the last decade, but they are somewhat scattered in the probability, statistics, and machine learning literature. 
One aim of the present manuscript is to gather together rigorous definitions and results in a singular reference text, with references to proofs of all results provided.

\section{Illustrating Stein's Method as a Methodological Tool}
\label{sec: Stein as a tool}

As a first illustration of the methodological potential of Stein's method, consider the standard setting of Bayesian posterior inference on $\mathcal{X} = \mathbb{R}^d$  \citep{GelmanCaStDuVeRu2014}.  
Given a prior density $\pi$ and a likelihood $\mathcal{L}$, the posterior distribution $P$ has a density $p$ that can be characterized up to a normalizing constant via Bayes' theorem:
\begin{align}
    p(x) \propto \pi(x) \mathcal{L}(x).
\end{align}
Unfortunately, in most cases, the normalizing constant, also called the \emph{marginal likelihood}, 
\begin{align}
    \int \pi(x) \mathcal{L}(x) \; \mathrm{d}x \label{eq: marginal likelihood}
\end{align}
is unavailable in closed form, and exact integration under $p$ is intractable for most functions of interest.  The same form of intractability naturally arises in the settings of maximum likelihood estimation \citep{geyer1991markov}, post-selection inference \citep{tian2016magic}, and probabilistic inference more generally \citep{neal1993probabilistic}. 
To approximate $P$ and its expectations, a plethora of practical options are now available including \ac{mcmc} \citep{brooks2011handbook}, approximate \ac{mcmc} \citep[see, e.g.,][]{welling2011bayesian}, variational inference \citep{blei2017variational}, and  quadrature. However, until recently, practitioners lacked tools suitable for measuring and comparing the quality of such diverse and potentially inconsistent sample approximations. 

To fill this gap, \citet{gorham2015measuring} introduced the notion of a Stein discrepancy and demonstrated its promise as a practical sample quality measure.
They began by deriving a multivariate and $P$-targeted generalization of Stein's original operator for the standard normal \cref{eq: intro stein op}: 
\begin{align}\label{eq:langevin-stein}
    (\mathcal{T}_P g)(x) \defeq (\nabla \cdot g)(x) + g(x) \cdot (\nabla \log p)(x).
\end{align}
This \emph{Langevin} Stein operator generates mean-zero functions under $P$ under mild conditions (see \cref{chap: Stein operators}) and 
can be computed even when the normalizing constant \eqref{eq: marginal likelihood} is unknown. 
Indeed, the operator \cref{eq:langevin-stein} depends on $P$ only via the 
\emph{Stein score}, 
\begin{align*}
    s_p(x) & \defeq (\nabla \log p)(x)
    = (\nabla \log \pi)(x) + (\nabla \log \mathcal{L})(x),
\end{align*}
which is a function only of the prior and likelihood derivatives. 
Remarkably, when coupled with a suitable Stein set $\mathcal{G}$ of vector fields $g : \mathbb{R}^d \rightarrow \mathbb{R}^d$, the resulting Stein discrepancy simultaneously enjoys all of the aforementioned properties of separation, convergence detection, convergence control, and computability (see \cref{chap: Stein discrepancies}).

Stein discrepancies are now widely used to measure sample quality, but the methodological implications of Stein's method are broader, giving rise to new effective solutions for tasks such as gradient estimation, training generative models, goodness-of-fit testing, bias correction, sample quality improvement, and sampling; some of the diverse problems in probabilistic inference and learning that are now tackled using Stein's method are outlined in \cref{sec:ksd_application}.

\section{Outline of the Monograph}

This monograph aims for a systematic and rigorous presentation of probabilistic inference and learning using Stein's method.
Throughout we will assume the reader has taken undergraduate-level courses in real analysis, but we will not assume familiarity with probability, measure theory, or functional analysis; the necessary concepts will introduced in \cref{chap: background}.
An advanced reader may prefer to initially skip \cref{chap: background} and return to it only when required.
Stein operators are discussed in detail in \cref{chap: Stein operators} and their associated Stein discrepancies are discussed in \cref{chap: Stein discrepancies}.
Recent developments in Stein dynamics are covered in \cref{chap: Stein dynamics}.
Although it is not possible to cover \emph{all} of the recent and emerging applications of Stein's method for probabilistic inference and learning, we have selected some notable examples to present in \cref{sec:ksd_application}.

\section{Acknowledgments}

QL was supported by NSF CAREER 1846421, the Institute for Foundations of Machine Learning (IFML), and Office of Naval Research.  
CJO was supported by EPSRC EP/W019590/1, the Alan Turing Institute, and a Philip Leverhulme Prize PLP-2023-004.

\section{Notation}

In the following table, the generic function $f$ is scalar-valued, $g$ is vector-valued, and $A$ and $B$ are matrix-valued.

\begin{longtable*}{p{.20\textwidth} p{.70\textwidth} } 
$\defeq$ & defined as being equal \\
$[n]$ & $\{1  , \dots , n \}$ \\
$\binom{[n]}{m}$ & $\{\sset\subseteq [n] : |\sset| = m\}$, all size $m$ subsets of $[n]$ \\
$\|\cdot\|$ & a generic norm on $\mathbb{R}^d$ \\
$\mathbb{I}_S$ & indicator function $\mathbb{I}_S(x) = 1$ if $x \in S$, else $\mathbb{I}_S(x) = 0$ \\
$\mathcal{P}$ & the set of all probability measures on a measurable space that will be clear from context \\
$\|\cdot\|_p$ & the norm $\|x\|_p \defeq (|x_1|^p + \dots + |x_d|^p)^{1/p}$ on $\mathbb{R}^d$, $p \in [1,\infty)$ \\
$\|\cdot\|_\infty$ & the norm $\|x\|_\infty \defeq \max\{ |x_1| , \dots , |x_d| \}$ on $\mathbb{R}^d$ \\
$\|\cdot\|_F$ & the Frobenius norm $\|A\|_F \defeq (\sum_{i,j} A_{i,j}^2)^{1/2}$ \\
$\|\cdot\|^*$ & the dual norm $\|x\|^* \defeq \sup_{y \in \mathbb{R}^d, \|y\| = 1} \langle x , y \rangle$ on $\mathbb{R}^d$ \\
& \emph{or} $\|A\|^* \defeq  \sup_{y \in \mathbb{R}^d , \|y\| = 1} \|A y\|^*$ on $\mathbb{R}^{d \times d}$ \\
$\|\cdot\|_{\mathrm{op}}$ & the operator norm $\|L\|_{\mathrm{op}} \defeq \sup_{\|x\|_{\mathcal{X}} \leq 1} |L(x)|_{\mathcal{Y}}$, for a continuous linear operator $L : \mathcal{X} \rightarrow \mathcal{Y}$ from a normed space $(\mathcal{X},\|\cdot\|_{\mathcal{X}})$ to a semi-normed space $(\mathcal{Y},|\cdot|_{\mathcal{Y}})$ \\
$a \cdot b$ & the dot product between $a,b \in \mathbb{R}^d$ \\
$A:B$ & the double dot product $\mathrm{tr}(AB^\top)$ \\
$e_j$ & the $j$th basis vector of $\mathbb{R}^d$ \\
$\partial_j f$ & the $j$th partial derivative $x \mapsto \partial_{x_j} f(x)$ \\
$\partial^\alpha f$ & the mixed partial derivative $\partial_1^{\alpha_1} \dots \partial_d^{\alpha_d} f $ \\
$\nabla f$ & the gradient $[\nabla f]_i \defeq \partial_i f$  \\
$\nabla g$ & the gradient $[\nabla g]_{i,j} \defeq \partial_j g_i$ \\
$\nabla \cdot g$ & the divergence $\sum_i \partial_i g_i$ \\
$\nabla \cdot A$ & the vector divergence $[\nabla \cdot A]_j \defeq \nabla \cdot (A^\top e_j)$ \\
$\Delta f$ & the Laplacian $\nabla \cdot \nabla f$  \\
$(\nabla^k f)(x)$ & the $k$th order gradient of $f$ at $x \in \mathbb{R}^d$ \\
&  i.e. a function $(v_1 , \dots , v_k) \mapsto (\nabla^k f)(x)[v_1 , \dots , v_k]$ \\
$C(\mathbb{R}^d)$ & the set of continuous $f : \mathbb{R}^d \rightarrow \mathbb{R}$ \\
$C(\mathbb{R}^d,\mathbb{R}^p)$ & the set of continuous $g : \mathbb{R}^d \rightarrow \mathbb{R}^p$ \\
$C_0(\R^d)$ & the set of continuous $f : \mathbb{R}^d \rightarrow \mathbb{R}$ vanishing at infinity on $\mathbb{R}^d$ \\
$C^m(\mathbb{R}^d)$ & the set of $f : \mathbb{R}^d \rightarrow \mathbb{R}$ with $\nabla^k f$ continuous for all $k \leq m$ \\
$C^{(m,m)}(\mathbb{R}^d)$ & the set of matrix-valued $A$ such that $\partial_x^\alpha \partial_y^\beta A(x,y)$ is continuous for multi-indices $\alpha$, $\beta$ satisfying $\|\alpha\|_1, \|\beta\|_1 \leq m$ \\
$C_b^m(\mathbb{R}^d)$ & the set of $m$ times continuously differentiable $f : \mathbb{R}^d \rightarrow \mathbb{R}$ for which derivatives of all orders $\leq m$ are bounded \\
$C^{(m,m)}_b(\mathbb{R}^d)$ & the set of matrix-valued $A$ such that $\partial_x^\alpha \partial_y^\beta A(x,y)$ is bounded and continuous for multi-indices $\alpha$, $\beta$ satisfying $\|\alpha\|_1, \|\beta\|_1 \leq m$ \\
$M_0(g)$ & $\sup_{x} \|g(x)\|_{\mathrm{op}}$ \\
$M_k(g)$ & the $k$th order Lipschitz constant \\
 & $$\sup_{x \neq y} \frac{ \|(\nabla^{k-1} g)(x) - (\nabla^{k-1} g)(y)\|_{\mathrm{op}}}{\|x-y\|_2}$$ \\
$F_k(A)$ & $\sup_{x, \|v_1\|_2 = 1 , \dots , \|v_k\|_2 = 1} \| (\nabla^k A)(x)[v_1 , \dots , v_k] \|_F$ \\
$M_1^*(A)$ & the dual Lipschitz constant \\
 & $$\sup_{x \neq y} \frac{ \| A(x) - A(y)\|_{\mathrm{op}}^*}{\|x-y\|_2}$$ \\
$\FTop(f)$ & the generalized Fourier transform of $f : \mathbb{R}^d \rightarrow \mathbb{C}$
\end{longtable*}

\section{Acronyms}

The following acronyms are used:%

\bigskip

\acsetup{format/short=\normalfont,format/replace=false}
\printacronyms[template=tabular,heading = none]
\chapter{Background}
\label{chap: background}

This monograph assumes familiarity with real analysis but not with any probability, measure theory, or functional analysis; the necessary concepts will be introduced.
An experienced reader may wish to skip this Chapter and return to it only when required.

\section{Probability and Measure Theory}

Though familiarity with measure theory is not required to understand most of the results that we present in Chapters \ref{chap: Stein operators}, \ref{chap: Stein discrepancies}, and \ref{chap: Stein dynamics}, we briefly introduce the core concepts so that our main definitions can be precisely stated.

\subsection{Measures}

Our starting point is a (non-empty) set $\Omega$.

\begin{definition}[$\sigma$-algebra]
A collection $\mathcal{S}$ of subsets of $\Omega$ for which
\begin{enumerate}
\item $\Omega \in \mathcal{S}$
\item $\Omega \setminus S \in \mathcal{S}$ whenever $S \in \mathcal{S}$
\item $S_1 \cup S_2 \cup \dots \in \mathcal{S}$ whenever $S_1, S_2, \dots \in \mathcal{S}$
\end{enumerate}
is called a \emph{$\sigma$-algebra} of $\Omega$.
\end{definition}

\noindent The pair $(\Omega,\mathcal{S})$ is called a \emph{measurable space}.
If the $\sigma$-algebra is clear from context, we may refer to the measurable space using just $\Omega$.

\begin{definition}[Measure]
Let $\mathcal{S}$ be a $\sigma$-algebra on $\Omega$.
A map $\mu : \mathcal{S} \rightarrow [0,\infty]$ is called a \emph{measure} if
\begin{enumerate}
\item $\mu(\emptyset) = 0$
\item $\mu(S_1 \cup S_2 \cup \dots) = \mu(S_1) + \mu(S_2) + \dots$ whenever the sets $S_1,S_2,\dots$ are pairwise disjoint.
\end{enumerate}
\end{definition}

\noindent The triple $(\Omega,\mathcal{S},\mu)$ is called a \emph{measure space}.
If the $\sigma$-algebra and measure are clear from context, we may refer to the measure space using just $\Omega$.
A set $S \in \mathcal{S}$ is called a \emph{null set} if $\mu(S) = 0$.
A measure $\mu$ is called a \emph{probability measure} if, in addition, $\mu(\Omega) = 1$, and we call $(\Omega,\mathcal{S},\mu)$ a \emph{probability space}.

\begin{example}[Atomic measure]
    \label{ex: atomic measure}
    Let $(\Omega,\mathcal{S})$ be a measurable space and fix $\omega \in \Omega$.
    Then 
    \begin{align*}
        \delta_\omega(S) \defeq \left\{ \begin{array}{ll} 1 & \omega \in S \\ 0 & \omega \notin S \end{array} \right.
    \end{align*}
    defines a probability measure on $(\Omega,\mathcal{S})$ called an \emph{atomic measure} at $\omega \in \Omega$.
\end{example}

\noindent The $\sigma$-algebra determines which events $S \subset \Omega$ we are allowed to measure.
For example, a probability measure $\mu$ constructed with respect to the \emph{trivial} $\sigma$-algebra $\mathcal{S} = \{\emptyset,\Omega\}$ contains only the vacuous information that $\mu(\emptyset) = 0$ and $\mu(\Omega) = 1$.

\begin{example}[Discrete distributions]
\label{def: pmf}
Let $\Omega$ be a countable set, and let $\mathcal{S}$ be the $\sigma$-algebra consisting of all subsets of $\Omega$.
Then a probability measure is uniquely determined by the values $\mu(\omega)$ for each $\omega \in \Omega$, and we call $\omega \mapsto \mu(\omega)$ the \emph{probability mass function}.
\end{example}

If $\Omega$ carries additional mathematical structure, this may entail a natural choice for a $\sigma$-algebra.
Here we will consider $\Omega$ to be a topological space that is \emph{Hausdorff}, which we recall means that any two distinct elements $\omega_1, \omega_2 \in \Omega$ can be separated by open sets $\Omega_1, \Omega_2 \subset \Omega$ with $\omega \in \Omega_1$, $\omega_2 \in \Omega_2$ and $\Omega_1 \cap \Omega_2 = \emptyset$.
Such a space $\Omega$ is further said to be \emph{locally compact} if, for each $\omega_1 \in \Omega$, there exists a compact set $\Omega_1 \subset \Omega$ with $\omega_1 \in \Omega_1$.
It is straightforward to check that examples of locally compact Hausdorff spaces include $\Omega = \mathbb{R}^d$ and $\Omega = [0,1]^d$.

\begin{example}[Borel measures]
Let $\Omega$ be a locally compact Hausdorff space, and let $\mathcal{B}_\Omega$ be the smallest $\sigma$-algebra that contains the open sets of $\Omega$, called the \emph{Borel} $\sigma$-algebra.
A measure $\mu : \mathcal{S} \rightarrow [0,1]$ is called a \emph{Borel} measure.
\end{example}

\noindent The Borel $\sigma$-algebra can be insufficient, since not every subset of a set of Borel measure 0 is a Borel measurable set.
In the context of $\Omega = \mathbb{R}^d$, the Lebesgue measure \emph{completes} the Borel measure and will form the basis of Lebesgue integration in \Cref{subsec: lebesgue spaces}.
For a box $B = (a_1,b_1) \times \dots \times (a_d,b_d) \subset \mathbb{R}^d$, let $\text{vol}(B) \defeq (b_1 - a_1) \dots (b_d - a_d)$ denote the volume of the box.

\begin{example}[Lebesgue measures]
\label{ex: lebesgue measure}
For a subset $S \subset \Omega = \mathbb{R}^d$, the \emph{Lebesgue outer measure} is defined as
$$
\lambda^*(S) \defeq \inf \left\{ \sum_{i=1}^\infty \text{vol}(B_i) : (B_i) \; \text{a sequence of boxes with} \; S \subset \cup_i B_i \right\}.
$$
The elements of the \emph{Lebesgue} $\sigma$-algebra $\mathcal{L}_\Omega$ are the sets $S$ such that the Carath\'{e}odory criterion, $\lambda^*(A \cap S) + \lambda^*(A \cap S^c) = \lambda^*(A)$ for all $A \subset \mathbb{R}^d$, is satisfied.
For each $S \in \mathcal{L}_\Omega$, the \emph{Lebesgue measure} $\lambda(S)$ is defined as $\lambda(S) = \lambda^*(S)$.
\end{example}

\begin{definition}[Measurable function]
A function $f : \Omega_1 \rightarrow \Omega_2$ between measurable spaces $(\Omega_1,\mathcal{S}_1)$ and $(\Omega_2,\mathcal{S}_2)$ is said to be \emph{measurable} if $f^{-1}(S) = \{\omega \in \Omega_1 : f(\omega) \in S\} \in \mathcal{S}_2$ whenever $S \in \mathcal{S}_1$.
\end{definition}

\subsection{Lebesgue Spaces}
\label{subsec: lebesgue spaces}

Let $(\Omega,\mathcal{S},\mu)$ be a measure space.
A function $\phi : \Omega \rightarrow [0,\infty]$ is called \emph{simple} if there exist $a_1,\dots,a_n \in [0,\infty]$, $S_1,\dots,S_n \in \mathcal{S}$, such that
$$
\phi(\omega) = \sum_{i=1}^n a_i \mathrm{1}_{S_i}(\omega)
$$
for all $\omega \in \Omega$, and we define an integral
$$
\int \phi \; \mathrm{d}\mu \defeq \sum_{i=1}^n a_i \mu(S_i) .
$$

\begin{definition}[Lebesgue integral]
For a positive measurable function $f : \Omega \rightarrow [0,\infty]$, mapping from $(\Omega,\mathcal{S})$ to $(\mathbb{R} , \mathcal{B}_{\mathbb{R}})$, the \emph{Lebesgue integral} is defined as
$$
\int f \; \mathrm{d} \mu \defeq \sup\left\{ \int \phi \; \mathrm{d}\mu : \phi \; \text{is a simple function and} \; \phi \leq f \right\} ,
$$
while if $f$ also takes negative values then we define its Lebesgue integral as
$$
\int f \; \mathrm{d}\mu \defeq \int f_+ \; \mathrm{d}\mu - \int f_- \; \mathrm{d}\mu
$$
where $f_\pm = f \cdot 1_{S_\pm}$ and $S_\pm = f^{-1}([0,\pm \infty])$, whenever at least one of these integrals is finite.
\end{definition}

The nomenclature of the Lebesgue integral derives from its original construction, involving Lebesgue measure $\lambda$ from \Cref{ex: lebesgue measure}, but here we present it for general measures $\mu$.
In the particular case of Lebesgue measure $\lambda$ on $\Omega = \mathbb{R}^d$, we follow standard convention and often write $\int f(x) \; \mathrm{d}x$ as an alternative notation for $\int f \; \mathrm{d}\lambda$.
It will sometimes be convenient to emphasize the argument of integration, in which case we write $\int f(\omega) \; \mathrm{d}\mu(\omega)$.
On the other hand, the shorthand
\begin{align}
\mu(f) \defeq \int f \; \mathrm{d}\mu \label{eq: int short}
\end{align}
will be used extensively in \Cref{chap: Stein operators}.
It will also be convenient to extend the definition of the Lebesgue integral to functions $f : \Omega \rightarrow \mathbb{R}^d$, which we achieve by applying the above construction to each coordinate function $f_i$, $i = 1,\dots,d$.

For $\mu$ and $\nu$ probability measures on a measurable space $(\Omega,\mathcal{S})$, we say that $\mu$ is \emph{absolutely continuous} with respect to $\nu$ (written $\mu \ll \nu$) if $\mu(S) = 0$ for every set $S \in \mathcal{S}$ for which $\nu(S) = 0$.

\begin{theorem}[Radon--Nikodym]
\label{thm: radon nikodym}
    Let $\mu$ and $\nu$ be probability measures on a measurable space $(\Omega,\mathcal{S})$ with $\mu \ll \nu$.
    Then there exists a measurable function $p : \Omega \rightarrow [0,\infty)$ such that $\mu(S) = \int p(\omega) 1_S(\omega) \; \mathrm{d}\nu(\omega)$ for any $S \in \mathcal{S}$.
\end{theorem}

\noindent The function $p$ appearing in \Cref{thm: radon nikodym} is called a \emph{Radon--Nikodym derivative} of $\mu$ with respect to $\nu$ and is denoted $\mathrm{d}\mu / \mathrm{d} \nu$. 
In the particular case where $\nu$ is the Lebesgue measure $\lambda$ from \Cref{ex: lebesgue measure}, we recover the familiar concept of a probability density function:

\begin{definition}[Continuous distributions]
\label{def: cts disn}
Let $\Omega = \mathbb{R}^d$.
Let $\mu$ be a probability measure such that $\mu(S) = \int p(\omega) 1_S(\omega) \; \mathrm{d} \lambda(\omega)$ for some measurable function $p : \mathbb{R}^d \rightarrow [0,\infty)$, where $\lambda$ is the Lebesgue measure from \Cref{ex: lebesgue measure}.
Then $p = \mathrm{d}\mu / \mathrm{d}\lambda$ is known as a \emph{probability density function} for $\mu$ with respect to $\lambda$, or simply a \emph{(Lebesgue) density} for $\mu$. 
\end{definition}

Let $1 \leq s < \infty$.
Consider the real vector space $\mathcal{L}^s(\mu)$ of functions $f : \Omega \rightarrow \mathbb{R}$ for which the Lebesgue integral
$$
|f|_{\mathcal{L}^s(\mu)} \defeq \left( \int |f|^s \; \mathrm{d}\mu \right)^{1/s} < \infty
$$
is well-defined.
The space $\mathcal{L}^s(\mu)$ is a seminormed vector space when equipped with $|\cdot|_{\mathcal{L}^s(\mu)}$ but not a normed space because $|f|_{\mathcal{L}^s(\mu)} = 0$ only implies that $f$ is non-zero on a $\mu$-null set and not that $f = 0$ on $\Omega$.
(An elementary discussion of norms on spaces of functions can be found in \Cref{sec: functional analysis}.)
A useful convention used in this monograph is that, for a vector- or matrix-valued function $f$, we will use the shorthand $f \in \mathcal{L}^s(\mu)$ to denote that all components of $f$ are in $\mathcal{L}^s(\mu)$.
In situations where the measure $\mu$ on $\Omega$ is unambiguous we may write $\mathcal{L}^s(\Omega)$.
The Lebesgue spaces are obtained by identification of functions that agree up to a Lebesgue null set:

\begin{definition}[Lebesgue spaces]
\label{def: Lebesgue space}
Let $1 \leq s < \infty$.
The \emph{Lebesgue space} $L^s(\mu)$ is the vector space whose elements are the equivalence classes $[f] \defeq \{g \in \mathcal{L}^s(\mu) : g \sim f\}$ of $\mathcal{L}^s(\mu)$ under the equivalence relation that $f \sim g$ if and only if $|f-g|_{\mathcal{L}^s(\mu)} = 0$.
The space $L^s(\mu)$ becomes a normed vector space when equipped with 
$$
\|[f]\|_{L^s(\mu)} \defeq |g|_{\mathcal{L}^s(\mu)} , \quad g \in [f] ,
$$
the latter being identical for all $g$ and thus well-defined. 
\end{definition}

\noindent In the case where $\mu$ is the Lebesgue measure $\lambda$ on $\mathbb{R}^d$, it is common to write $\mathcal{L}^s(\lambda)$ as $\mathcal{L}^s(\mathbb{R}^d)$ and $L^s(\lambda)$ as $L^s(\mathbb{R}^d)$.

\subsection{Random Variables and Independence}

An (almost) equivalent, and often more intuitive, representation of probability measures is through the lens of random variables, which are now precisely defined:

\begin{definition}[Random variable]
Let $(\Omega,\mathcal{S},\mu)$ be a probability space, called the \emph{sample space}, and let $(\mathcal{X},\mathcal{S}_{\mathcal{X}})$ be a measurable space, called the \emph{state space}.
A \emph{random variable} with these sample and state spaces is a measurable function $X : \Omega \rightarrow \mathcal{X}$.
\end{definition}

\noindent A random variable $X$ is associated with a probability measure $P$ on the state space $\mathcal{X}$, called the \emph{law} of $X$, defined as $P(S) = \mu(X^{-1}(S))$ for each measurable set $S \in \mathcal{S}_{\mathcal{X}}$, and we write $X \sim P$ as shorthand.
The \emph{expectation} of a random variable $X$ is defined as the integral $\mu(X)$, and is conventionally denoted $\mathbb{E}[X]$.
Since the expectation is fully determined by the law, it is common to also write $\mathbb{E}_{Y \sim P}[Y]$ where $Y$ is understood to be any random variable with law $P$.
The set $\mathcal{S}_X$ whose elements are the sets $X^{-1}(S)$, $S \in \mathcal{S}_{\mathcal{X}}$, itself satisfies the axioms of a $\sigma$-algebra and is therefore called the $\sigma$-algebra \emph{generated} by $X$.

\begin{definition}[Independence]
Let $(\Omega,\mathcal{S},\mu)$ be a probability space and let each $\mathcal{X}_i$ be a measurable space, with index $i$ running over a possibly uncountably infinite set $I$.
A collection $(X_i)_{i \in I}$ of random variables $X_i : \Omega \rightarrow \mathcal{X}_i$ are said to be \emph{independent} if
$$
\mu\left( \cap_{i \in I} S_i \right) = \prod_{i \in I} \mu(S_i)
$$
for all $(S_i)_{i \in I}$ with $S_i \in \mathcal{S}_{\mathcal{X}_i}$ for each $i \in I$.
\end{definition}

\begin{definition}[Conditional expectation]
Let $X : \Omega \rightarrow \mathbb{R}^d$ and $Y : \Omega \rightarrow \mathcal{Y}$ be random variables on a common probability space $(\Omega,\mathcal{S},\mu)$.
A \emph{conditional expectation} of $X$ given $Y$ is a measurable function, denoted $\mathbb{E}[X|Y]$, from $(\Omega,\mathcal{S}_Y)$ to $\mathbb{R}^d$ for which 
$$
\int \mathbb{E}[X|Y] \cdot 1_S \; \mathrm{d}\mu = \int X \cdot 1_S \; \mathrm{d}\mu
$$
for all $S \in \mathcal{S}_Y$.
\end{definition}

\noindent The \emph{conditional probability} of an event $S \in \mathcal{S}_{X}$ given $Y$ is $\mu(S | Y) = \mathbb{E}[\mathrm{1}_S | Y]$.

\begin{definition}[Conditional independence]
Let $(\Omega,\mathcal{S},\mu)$ be a probability space and let $\mathcal{X}_i$ and $\mathcal{Y}$ be measurable spaces, with index $i$ running over a possibly uncountably infinite set $I$.
A collection $(X_i)_{i \in I}$ of random variables $X_i : \Omega \rightarrow \mathcal{X}_i$ are said to be \emph{conditionally independent} given a random variable $Y : \Omega \rightarrow \mathcal{X}$ if, $\mu$-almost surely,
$$
\mu\left( \cap_{i \in I} S_i \mid Y \right) = \prod_{i \in I} \mu(S_i | Y)
$$
for all $(S_i)_{i \in I}$ with $S_i \in \mathcal{S}_{X_i}$ for each $i \in I$.
\end{definition}

\noindent The shorthand $X_i \ci  X_j | Y$ is used to denote the statement that the random variables $X_i$ and $X_j$ are conditionally independent given $Y$.

\subsection{Markov Processes}
\label{subsec: Markov processes}

In dealing with collections of random variables $(X_i)_{i \in I}$, the structure of the index set $I$, together with the conditional independence relationships among the random variables, are used to signify different settings of interest.

\begin{definition}[Stochastic process]
Let $\Omega$ be a probability space and $\mathcal{X}$ be a measurable space.
A \emph{stochastic process} is a collection $(X_t)_{t \in T}$ of random variables $X_t : \Omega \rightarrow \mathcal{X}$ with index $t$ running over a set $T$.
\end{definition}

\noindent In this monograph we will primarily be concerned with two cases; (1) the index set $T$ is discrete, e.g. $t =  0, 1,2,\dots$, with the index $t$ interpreted as a time associated with the random variable $X_t$; in this case the stochastic process is said to be a \emph{discrete time} process; (2) there is a continuous temporal index, e.g. $t \in [0,\infty)$; in this case the stochastic processes is said to be a \emph{continuous time} process.
In both cases the index set $T$ is ordered.

\begin{definition}[Markov process]
A stochastic process $(X_t)_{t \in T}$ with the property that
$$
X_r \ci X_t | X_s
$$
for all $r < s < t$ is said to be a \emph{Markov process}.
\end{definition}

\noindent Intuitively, this property means that the future states of the stochastic process are conditionally independent of the past states given the current state, and thus knowledge of the current state is sufficient for predicting future states.

The law associated to each random variable $X_t$ is denoted $P_t$, which we recall is defined as $P_t(S) \defeq \mu(X_t^{-1}(S))$.

\begin{definition}[Markov kernel]
\label{def: Markov kernel}
A Markov process is \emph{time-homogeneous} if, for all $t \geq 0$, there exists a \emph{Markov kernel} $K_t : \mathcal{X} \times \mathcal{S}_{\mathcal{X}} \rightarrow [0,\infty]$ such that
\begin{enumerate}
\item $x \mapsto K_t(x,S)$ is measurable for each $S \in \mathcal{S}_{\mathcal{X}}$
\item $S \mapsto K_t(x,S)$ is a probability measure for each $x \in \mathcal{X}$
\item it holds that
$$
P_{s+t}(S) = \int K_t(x,S) \; \mathrm{d}P_s(x)
$$
for all $s \geq 0$ and $S \in \mathcal{X}_{\mathcal{S}}$.
\end{enumerate}
\end{definition}

\noindent The Markov processes that we consider in this monograph are all time-homogeneous, and are thus completely characterised by their initial distribution together with their Markov kernel.
Often we will implicitly assume time-homogeneity when describing a time-homogeneous Markov process in terms of its Markov kernel.

\begin{definition}[Invariance]
\label{def: invariant}
A Markov kernel $K_t$ is said to leave a probability measure $P$ \emph{invariant} if 
\begin{align*}
P(S) = \int K_t(x,S) \; \mathrm{d}P(x) 
\end{align*}
for all $t > 0$ and all $S \in \mathcal{S}_{\mathcal{X}}$.
\end{definition}

\begin{example}[Overdamped Langevin diffusion I]
\label{ex: overdamped Langevin}
Let $P$ be a continuous distribution on $\mathbb{R}^d$ with positive and differentiable density $p$ on $\mathbb{R}^d$.
The continuous time process $(X_t)_{t \geq 0}$ defined by the \ac{sde}
\begin{align*}
\mathrm{d}X_t = \frac{1}{2} (\nabla \log p)(X_t) \; \mathrm{d}t + \mathrm{d}W_t & & t > 0 \\
X_0 \sim P_0 & &
\end{align*}
is a time-homogeneous Markov process, with a Markov kernel for which $P$ is invariant, called the \emph{overdamped Langevin diffusion} with \emph{initial distribution} $P_0$.
Here $\nabla \log p : \mathbb{R}^d \rightarrow \mathbb{R}^d$ denotes the gradient of $\log p : \mathbb{R}^d \rightarrow \mathbb{R}$ and $(W_t)_{t \geq 0}$ denotes a standard $d$-dimensional Wiener process on $\mathbb{R}^d$; see Chapter 5 of \citet{oksendal2013stochastic}.
\end{example}

\begin{definition}[Generator of a Markov process]
\label{def: generator}
Let $(X_t)_{t \geq 0}$ be a continuous time Markov process with Markov kernel $K_t$.
If it exists, the \emph{infinitesimal generator} of the process is defined as the operator
$$
(A f)(x) \defeq \lim_{t \rightarrow 0} \; \frac{1}{t} \left\{ \int f \; \mathrm{d} K_t(x,\cdot) - f(x) \right\}
$$
acting on sufficiently regular functions $f : \mathcal{X} \rightarrow \mathbb{R}$.
\end{definition}

\begin{example}[Overdamped Langevin diffusion II]
\label{ex: overdamped II}
The infinitesimal generator of the overdamped Langevin diffusion from \Cref{ex: overdamped Langevin} is the differential operator $(A f)(x) = (\Delta f)(x) + (\nabla f)(x) \cdot (\nabla \log p)(x)$; see Section 7.3 of \citet{oksendal2013stochastic}.
\end{example}

\begin{definition}[Ergodic process]
A stochastic process $(X_t)_{t \in T}$ is said to be \emph{ergodic} if the associated measures $P_t$ converge to a limit $P$ as $t \rightarrow \infty$, where the sense of convergence is to be specified.
\end{definition}

\noindent Some different senses of convergence are discussed in \Cref{sec: divergences}.
One may also consider refined notions of ergodicity in which, for example, the rate of convergence is specified.
Further background on ergodicity of Markov processes can be found in \citet{gallegos2024equivalences}.

\section{Functional Analysis}
\label{sec: functional analysis}

Functional analysis concerns the mathematical properties of function spaces.
In this section we let $\mathcal{X}$ be a set.

\begin{definition}[Function space]
A \emph{function space} $\mathcal{F}$, in this book, is a vector space (over the reals) whose elements are functions of the form $f : \mathcal{X} \rightarrow \mathbb{R}^d$, for some $d \in \mathbb{N}$.
\end{definition}

The most common function spaces that we will encounter consist of functions that are scalar-valued (i.e. $d = 1$), but we will also require vector-valued functions in \Cref{sec: ksds} and examples of such function spaces are discussed in \Cref{subsec: vvrkhs}.

\begin{example}[Polynomial space I] \label{ex: poly space}
A familiar function space is the space of polynomials of fixed maximal order on $\mathcal{X} = \mathbb{R}$: this is the set $\mathcal{F}_p \defeq \{f(x) = a_0 + a_1 x + \dots + a_p x^p \; : \; a_0, a_1,\dots,a_p \in \mathbb{R}\}$ equipped with pointwise addition and scalar multiplication, meaning that $(\alpha f + \beta g)(x) = \alpha f(x) + \beta g(x)$ for all $f,g \in \mathcal{F}_p$, $\alpha,\beta \in \mathbb{R}$ and $x \in \mathbb{R}$.
\end{example}

Function spaces can often be endowed with additional mathematical structure that is useful for theory and computation, and in particular we will refer to a \emph{normed space} in the context of a function space when the function space is equipped with a norm $\|\cdot\|_{\mathcal{F}} : \mathcal{F} \rightarrow [0,\infty)$, and an \emph{inner product space} when the function space is equipped with an inner product $\langle \cdot , \cdot \rangle_{\mathcal{F}} : \mathcal{F} \times \mathcal{F} \rightarrow \mathbb{R}$.
All inner products give rise to an associated \emph{induced} norm, defined as $\|f\|_{\mathcal{F}} = \sqrt{\langle f , f \rangle_{\mathcal{F}}}$, but not all norms are induced by an inner product.

\begin{example}[Polynomial space II] \label{ex: inner prod poly}
The polynomial space $\mathcal{F}_p$ from \Cref{ex: poly space} is an inner product space when equipped with $\langle f , g \rangle_{\mathcal{F}_p} = a_0 b_0 + \dots + a_p b_p$ where $f(x) = a_0 + \dots + a_p x^p$ and $g(x) = b_0 + \dots + b_p x^p$.
\end{example}

\begin{example}[Sup and bounded Lipschitz norms]
\label{ex: bl norm}
    Let $\rho$ be a metric on $\mathcal{X}$.
    For a function space $\mathcal{F}$ whose elements are continuous and bounded functions $f : \mathcal{X} \rightarrow \mathbb{R}^d$, the \emph{sup norm} and \emph{Lipschitz semi-norm}, are defined, respectively, as
    \begin{align*}
        \|f\|_\infty \defeq \sup_{x \in \mathcal{X}} \|f(x)\| , \qquad 
        |f|_{\text{Lip},\rho} \defeq \sup_{x \neq y, \; x,y \in \mathcal{X}} \frac{\|f(x) - f(y)\|}{\rho(x,y)} .
    \end{align*}
The \emph{bounded Lipschitz norm} is defined as $\|f\|_{\text{BL},\rho} \defeq \|f\|_\infty + |f|_{\text{Lip},\rho}$.    
\end{example}

\noindent In settings where the metric $\rho$ is unambiguous, we will simply write $|f|_{\text{Lip}}$ for $|f|_{\text{Lip},\rho}$.

A sequence $(f_i)_{i \in \mathbb{N}}$ of elements in a normed space $\mathcal{F}$ is said to be \emph{Cauchy} if for all $\epsilon > 0$ there exists $n \in \mathbb{N}$ such that $\|f_i - f_j \|_{\mathcal{F}} < \epsilon$ for all $i,j > n$.
A normed space $\mathcal{F}$ is said to be \emph{complete} if all Cauchy sequences in $\mathcal{F}$ also have a limit in $\mathcal{F}$ (meaning that there exists $f \in \mathcal{F}$ such that $\|f_i - f\|_{\mathcal{F}} \rightarrow 0$ as $i \rightarrow \infty$).

\begin{definition}[Hilbert space]
An inner product space that is complete (with respect to the norm induced by the inner product) is called a \emph{Hilbert space}.
\end{definition}

\begin{example}[Polynomial space III] \label{ex: inner prod poly}
The polynomial space $\mathcal{F}_p$ from \Cref{ex: poly space} is Hilbert.
Indeed, the map $\varphi : \mathbb{R}^{p+1} \rightarrow \mathcal{F}_p$ that sends a coefficient vector $a \in \mathbb{R}^{p+1}$ to the corresponding polynomial $f(x) = a_0 + \dots + a_p x^p$ is an isometry of $\mathbb{R}^{p+1}$ and $\mathcal{F}_p$, under which the completeness of $\mathbb{R}^{p+1}$ is preserved.
\end{example}

\begin{example}[Lebesgue space $L^2(\mu)$]
    \label{ex: Leb is Hilb}
    The Lebesgue spaces $L^p(\mu)$, introduced in \Cref{def: Lebesgue space}, are strictly speaking \emph{not} function spaces because their elements are not functions \emph{per se}, but rather equivalence classes of functions.
    Nevertheless, in the case $p = 2$ the Lebesgue space admits a natural inner product
    \begin{align*}
        \langle [f] , [g] \rangle_{L^2(\mu)} = \int f(\omega) g(\omega) \; \mathrm{d}\mu(\omega) ,
    \end{align*}
    for which the induced norm is $\|\cdot\|_{L^2(\mu)}$, and it can be verified that $L^2(\mu)$ equipped with this inner product is Hilbert.
\end{example}

One of the most important mathematical results on Hilbert spaces is the \emph{Riesz representer theorem}:

\begin{theorem}[Riesz representation] \label{thm: Riesz}
Let $\mathcal{H}$ be a Hilbert space and $L : \mathcal{H} \rightarrow \mathbb{R}$ be a continuous linear functional, meaning that $L$ is linear and that there exists a constant $C$ such that $|L(f)| \leq C \|f\|_{\mathcal{H}}$ for all $f \in \mathcal{H}$.
Then $L(f) = \langle f , g \rangle_{\mathcal{H}}$ for some $g \in \mathcal{H}$.
\end{theorem}

\noindent The element $g$ in \Cref{thm: Riesz} is called the \emph{Riesz representer} of $L$ in $\mathcal{H}$.

\subsection{Reproducing Kernel Hilbert Spaces}

The Lebesgue space $L^2(\mu)$ from \Cref{ex: Leb is Hilb} is Hilbert, but its elements are not functions that can be pointwise evaluated.
For the most part in this monograph, we will want to exploit the convenient mathematical structure of Hilbert spaces whilst excluding certain Hilbert spaces, such as $L^2(\mu)$, whose elements are more abstract.
To proceed, we will restrict attention to Hilbert spaces $\mathcal{H}$ for which pointwise evaluation $f \mapsto f(x)$ is a continuous linear functional on $\mathcal{H}$, so that individual function values $f(x)$ are well-defined for all elements $f \in \mathcal{H}$.

\begin{definition}[Reproducing kernel Hilbert space] \label{def: rkhs}
Let $\mathcal{X}$ be a set.
A \emph{reproducing kernel Hilbert space} is a Hilbert space $\mathcal{H}$ of real-valued functions on $\mathcal{X}$, such that for each $x \in \mathcal{X}$, the map $f \mapsto f(x)$ is a continuous linear functional on $\mathcal{H}$.
\end{definition}

From \Cref{thm: Riesz}, we deduce that $f \mapsto f(x)$ has a Riesz representer which we denote $k_x \in \mathcal{H}$.
The \emph{reproducing kernel} of the \ac{rkhs} is defined as a bivariate function $k : \mathcal{X} \times \mathcal{X} \rightarrow \mathbb{R}$ with $k(x,y) = \langle k_x,k_y \rangle_{\mathcal{H}}$.
This construction implies that $k_x(y) = \langle k_y , k_x \rangle_{\mathcal{H}} = k(y,x)$ and we will therefore use the notation $k_x$ and $k(\cdot,x)$ interchangeably in the text.
It is straightforward to verify that a reproducing kernel $k$ is symmetric and positive semi-definite and is therefore an instance of a \emph{kernel}:

\begin{definition}[Kernel] \label{def: kernel}
A bivariate function $k : \mathcal{X} \times \mathcal{X} \rightarrow \mathbb{R}$ is said to be a \emph{kernel} on a set $\mathcal{X}$ if it is
\begin{enumerate}
\item \emph{symmetric:} $k(x,y) = k(y,x)$ for all $x,y \in \mathcal{X}$
\item \emph{positive semi-definite:} for all $w_1,\dots,w_n \in \mathbb{R}$, $x_1,\dots,x_n \in \mathcal{X}$, and $n \in \mathbb{N}$, it holds that 
$$
\sum_{i=1}^n \sum_{j = 1}^n w_i w_j k(x_i , x_j) \geq 0 .
$$
\end{enumerate}
\end{definition}

\noindent This second property of a kernel is closely related to a familiar concept from linear algebra:
A matrix $K \in \mathbb{R}^{n \times n}$ is said to be \emph{positive semi-definite} if $w^\top K w \geq 0$ for all $w \in \mathbb{R}^n$, and we write $K \succeq 0$.
If the inequality is strict for all $w \neq 0$, we say that the matrix $K$ is \emph{positive definite}, written $K \succ 0$.
Thus the second property in \Cref{def: kernel} states that all matrices $K$ of the form $K_{i,j} = k(x_i, x_j)$ are positive semi-definite for all choices of $x_1,\dots,x_n \in \mathcal{X}$ and $n \in \mathbb{N}$.
If all such matrices are (strictly) positive definite then we say that $k$ is a \emph{positive definite} kernel.

A fundamental result in functional analysis is that \acp{rkhs} are completely characterized by their kernel: 

\begin{theorem}[Existence and uniqueness]
\label{thm: existence rkhs}
For all kernels $k : \mathcal{X} \times \mathcal{X} \rightarrow \mathbb{R}$, there exists a unique \ac{rkhs} for which $k$ is a reproducing kernel, denoted $\mathcal{H}_k$.
This Hilbert space is characterized as the unique Hilbert space for which
\begin{enumerate}
    \item $k_x \in \mathcal{H}_k$ for all $x \in \mathcal{X}$,
    \item $f(x) = \langle f , k_x \rangle_{\mathcal{H}_k}$ for all $x \in \mathcal{X}$, $f \in \mathcal{H}_k$,
\end{enumerate}
with the latter called the \emph{reproducing property} of the kernel.
\end{theorem}

\noindent The proof of \Cref{thm: existence rkhs} is beyond the scope of this monograph; see Section 4.2 of \citet{steinwart2008support}.
The proof itself is constructive, and demonstrates that the elements of $\mathcal{H}_k$ are (limits of) finite linear combinations of the Riesz representers;
\begin{align*}
    \mathcal{H}_k \defeq \overline{ \mathrm{span}\left\{ k_x \; | \; x \in \mathcal{X} \right\} }
\end{align*}
where the bar notation denotes the completion taken with respect to the inner product 
\begin{align*}
    \left( \sum_{i=1}^n u_i k_{x_i} , \sum_{j=1}^m v_j k_{y_j} \right) \mapsto \sum_{i=1}^n \sum_{j=1}^m u_i v_j k(x_i,y_j) .
\end{align*}

One route to establishing whether a general bivariate function $k$ is or is not a kernel is to consider its Taylor series; we will not discuss this further, but see e.g. Lemma 4.8 of \citet{steinwart2008support}.
Common examples of kernels on $\mathcal{X} = \mathbb{R}^d$ include the following:

\begin{example}[Gaussian kernel] \label{ex: Gauss kernel}
The \emph{Gaussian} kernel has the form
$k(x,y) = \exp(- \|x - y\|^2 / 2 \ell^2)$, for some $\ell > 0$.
\end{example}

\begin{example}[Inverse multi-quadric kernel] \label{ex: IMQ kernel}
The \emph{inverse multi-quadric} kernel has the form
$k(x,y) = (\ell^2 + \|x - y\|^2)^{-\beta}$, for some $\ell > 0$, $\beta \in (0,\infty)$.
\end{example}

The Hilbert spaces $\mathcal{H}_k$ reproduced by the Gaussian and inverse multi-quadric kernels are "small", in the sense that their elements are functions whose derivatives of all orders exist.
To reproduce larger Hilbert spaces -- for example, containing functions with only a finite number of derivatives -- we can employ alternative kernels, such as the Mat\'{e}rn kernel:

\begin{example}[Mat\'{e}rn kernel] \label{ex: Matern kernel}
The Mat\'{e}rn kernel has the form
$$
k(x,y) = \exp\left( - \frac{\|x-y\|}{\ell} \right) \frac{s!}{(2s)!} \sum_{i=0}^s \frac{(s+i)!}{i! (s-i)!} \left( 2 \frac{\|x-y\|}{\ell} \right)^{s-i} 
$$
for some $\ell > 0$, $s \in \mathbb{N}_0$.
\end{example}

\noindent Elements of the Hilbert space reproduced by the Mat\'{e}rn kernel have partial derivatives up to order $s$ existing in the weak sense, as elements of $L^2(\mathbb{R})$.
The parameter $\ell$, appearing in all of these kernels, is called a \emph{bandwidth} or a \emph{length-scale} of the kernel.

Each of the kernels in \Cref{ex: Gauss kernel,ex: IMQ kernel,ex: Matern kernel} had the property that the value of $k(x,y)$ depended on $x$ and $y$ only through the difference $x - y$; such kernels are called \emph{translation-invariant}.
A famous theorem of Bochner characterizes all continuous translation-invariant kernels on $\mathbb{R}^d$:

\begin{theorem}[Bochner's Theorem; Theorem 6.6 of \citealp{wendland2004scattered}]
\label{thm: bochner}
    A continuous function $k : \mathbb{R}^d \times \mathbb{R}^d \rightarrow \mathbb{R}$ is a translation-invariant kernel if and only if there exists a finite non-negative (Borel) measure $\mu$ on $\mathbb{R}^d$, called the \emph{spectral measure}, such that
    \begin{align}
        k(x,y) = \frac{1}{(2 \pi)^d} \int e^{-\mathrm{i}\langle x-y , \omega \rangle} \; \mathrm{d}\mu(\omega) \label{eq: Bochner thm}
    \end{align}
    for all $x,y \in \mathbb{R}^d$.
\end{theorem}

\noindent That is, all translation-invariant kernels $k$ can be viewed as the Fourier transform of a certain non-negative measure $\mu$.  
Further, it can be shown that a continuous, translation-invariant kernel $k$ is positive definite whenever the \emph{carrier} of the spectral measure $\mu$ in \eqref{eq: Bochner thm}, defined as
\begin{align*}
    \mathrm{carrier}(\mu) \defeq  \mathbb{R}^d \setminus \{ U : U \; \text{is open and} \; \mu(U) = 0 \} ,
\end{align*}
contains an open subset; see \citet{chang1996strictly} and Theorem 6.8 of \citet{wendland2004scattered}.

Access to an explicit formula for a reproducing kernel allows us to perform calculations in \acp{rkhs}, for example to calculate norms:
\begin{align*}
\left\| \sum_{i=1}^n w_i k(\cdot,x_i) \right\|_{\mathcal{H}_k} & = \sqrt{ \left\langle \sum_{i=1}^n w_i k(\cdot,x_i) , \sum_{j=1}^n w_j k(\cdot,x_j) \right\rangle_{\mathcal{H}_k} } \\
& = \sqrt{ \sum_{i=1}^n \sum_{j=1}^n w_i w_j \langle k(\cdot, x_i) , k(\cdot x_j) \rangle_{\mathcal{H}_k} } \\
& = \sqrt{ \sum_{i=1}^n \sum_{j=1}^n w_i w_j k(x_i,x_j) } .
\end{align*}
This is quite a remarkable feature of \acp{rkhs}, since in general it is difficult to obtain an explicit characterization of the elements of these spaces directly from the kernel.
Some notable exceptions, where both the \ac{rkhs} and the kernel are explicit, include:

\begin{example}[Polynomial space IV] 
\label{ex: poly space kernel}
The polynomial space $\mathcal{F}_p$ from \Cref{ex: poly space} is an \ac{rkhs} with kernel $k(x,y) = 1 + (xy) + \dots + (xy)^p$.
\end{example}

\begin{example}[Gaussian kernel II] 
The Gaussian kernel $k$ from \Cref{ex: Gauss kernel} reproduces (in dimension $d=1$ for simplicity) the Hilbert space
\begin{align*}
    \mathcal{H}_k = \left\{ f = \sum_{i=0}^\infty c_i \phi_i \; : \; \|f\|_{\mathcal{H}_k}^2 \defeq  \sum_{i=1}^\infty c_i^2 < \infty \right\} ,
\end{align*}
where 
\begin{align*}
    \phi_i(x) \defeq  \frac{1}{\ell^i \sqrt{i!}} x^i \exp\left( - \frac{x^2}{2\ell^2} \right) .
\end{align*}
See \citet{steinwart2006explicit}.
\end{example}

\subsection{Kernel Mean Embedding}
\label{subsec: kernel mean embed}

The convenient mathematical structure of \acp{rkhs} stands in contrast to the limited structures that are available on sets of probability measures, which are not vector spaces in general.
Nevertheless, it is possible to embed (sufficiently regular) probability measures into \acp{rkhs}, in order that their mathematical structure can be exploited.
In this section, $k : \mathcal{X} \times \mathcal{X} \rightarrow \mathbb{R}$ is a positive definite kernel defined for a measurable space $\mathcal{X}$.
It is assumed that $k_x$ is measurable for all $x \in \mathcal{X}$, which is equivalent to assuming that the elements of $\mathcal{H}_k$ are measurable \citep[Lemma 4.24]{steinwart2008support}.
Let $\mathcal{P}_{\mathcal{H}_k}$ be the set of probability distributions $\mu$ on $\mathcal{X}$ for which the map $\mathrm{I}_\mu : \mathcal{H}_k \rightarrow \mathbb{R}$ given by $\mathrm{I}_\mu (f) = \int f \; \mathrm{d}\mu$ is a continuous linear functional.

\begin{definition}[Kernel mean embedding]
\label{def: kernel mean element}
For $\mu \in \mathcal{P}_{\mathcal{H}_k}$, the Riesz representer of $\mathrm{I}_\mu$ is called the \emph{kernel mean element}, denoted $\phi_\mu \in \mathcal{H}_k$.
The map
\begin{eqnarray*}
    \phi : \mathcal{P}_k & \rightarrow & \mathcal{H}(k) \\
    \mu & \mapsto & \phi_\mu
\end{eqnarray*}
is called the \emph{kernel mean embedding} of $\mathcal{P}_{\mathcal{H}_k}$ into $\mathcal{H}_k$.
\end{definition}

This definition is well-defined as a consequence of \Cref{thm: Riesz}.
The kernel mean embedding enables us to immediately make use of the Hilbert space structure of $\mathcal{H}_k$ to perform analysis and computation on $\mathcal{P}_{\mathcal{H}_k}$.
For example we may measure the dissimilarity between a pair of measures $\mu$ and $\nu$ in terms of the magnitude of the vector $\phi_\mu - \phi_\nu$ that connects their kernel mean elements $\phi_\mu$ and $\phi_\nu$; we will return to this idea in detail in \Cref{subsec: mmd}.

To calculate the kernel mean element associated to a probability measure $\mu$, we can take $f = k_x$ in the Riesz representation statement $\mathrm{I}_\mu(f) = \langle f , \phi_\mu \rangle_{\mathcal{H}_k}$ to see immediately that 
\begin{align}
    \phi_\mu(x) = \int k(x,y) \; \mathrm{d}\mu(y).  \label{eq: kernel mean element}
\end{align}
For certain pairs of probability measure $\mu$ and kernel $k$, such as the Gaussian distribution paired with the Gaussian kernel, it is possible to analytically calculate the kernel mean element \citep{briol2025dictionary}.
Outside of these tractable cases, numerical methods can be used to approximate the kernel mean element.
For example, if $\int k(x,x) \; \mathrm{d}\mu(x) < \infty$ then the \emph{Monte Carlo} estimator
\begin{align}
    \phi_\mu^n(\cdot) \defeq  \frac{1}{n} \sum_{i=1}^n k(\cdot, Y_i), \qquad Y_1,\dots,Y_n \stackrel{\mathrm{iid}}{\sim} \mu \label{eq: approx kernel mean}
\end{align}
provides a strongly consistent approximation to $\phi_\mu$, in the sense that $\|\phi_\mu^n - \phi_\mu\|_{\mathcal{H}_k} \rightarrow 0$ almost surely as $n \rightarrow \infty$.
It is useful also to note that $\phi_\mu^n$ can be interpreted as the \emph{exact} kernel mean element associated to the approximating measure $\mu_n = \frac{1}{n} \sum_{i=1}^n \delta_{Y_i}$.
See \Cref{subsec: mmd} for further discussion on Monte Carlo estimation of the kernel mean element.
Several more sophisticated approaches to approximating a kernel mean element have been developed, under names such as quasi-Monte Carlo, determinantal point processes, gradient flows, and kernel herding; we do not attempt to survey these in detail here, but a subset of these algorithms with particular relevance to Stein's method are described in \Cref{sec:improving}.

In general it is not possible to embed all probability measures into the same Hilbert space, and for a given kernel $k$ the set $\mathcal{P}_{\mathcal{H}_k}$ determines which measures can be safely embedded.
Sufficient conditions for $\mu \in \mathcal{P}_{\mathcal{H}_k}$ will now be discussed.
Recall that $\mathcal{L}^p(\mu)$ denotes the vector space of functions $f : \mathcal{X} \rightarrow \mathbb{R}$ for which the semi-norm $|f|_{\mathcal{L}^p(\mu)} = \int |f|^p \; \mathrm{d}\mu$ is finite, from \Cref{subsec: lebesgue spaces}.
Our presentation here follows \citet[][Appendix C]{barp2022targeted}.

\begin{definition}[Scalarly integrable]
A map $\Phi : \mathcal{X} \rightarrow \mathcal{H}_k$ is said to be \emph{scalarly} $\mu$-\emph{integrable} if $\{ x \mapsto \langle h , \Phi(x) \rangle_{\mathcal{H}_k} : h \in \mathcal{H}_k \} \subset \mathcal{L}^1(\mu)$.
\end{definition}

\noindent The argument used in the following proof can be traced back to \citet{dunford1937integration}, with our account based on \citet[][Lemma 2.1.1]{schwabik2005topics}.

\begin{proposition} \label{prop: Dunford}
If $\Phi : \mathcal{X} \rightarrow \mathcal{H}_k$ is scalarly $\mu$-integrable, then 
\begin{eqnarray*}
\mathrm{I}_\Phi : \mathcal{H}_k & \rightarrow & \mathbb{R} \\
h & \mapsto & \int \langle h , \Phi(x) \rangle_{\mathcal{H}_k} \; \mathrm{d}\mu(x) 
\end{eqnarray*}
is a continuous linear functional.
\end{proposition}
\begin{proof}
First we claim that the graph of the linear map
\begin{eqnarray*}
T : \mathcal{H}_k & \rightarrow & \mathcal{L}^1(\mu) \\
h & \mapsto & ( x \mapsto \langle h , \Phi(x) \rangle_{\mathcal{H}_k} ) .
\end{eqnarray*}
is closed.
To see this, let $h_n \rightarrow h$ in $\mathcal{H}_k$ and suppose that $T(h_n) \rightarrow g$ in $\mathcal{L}^1(\mu)$.
The claim is that $g$ and $T(h)$ are equal in $\mathcal{L}^1(\mu)$.
Since every sequence converging in $\mathcal{L}^1(\mu)$ has an almost surely converging subsequence, there is a subsequence $(h_{n_i})_{i \in \mathbb{N}}$ such that
$$
\langle h_{n_i} , \Phi(x) \rangle_{\mathcal{H}_k} \rightarrow g(x)
$$
for $\mu$-almost all $x \in \mathcal{X}$.
Since $\langle h_n , \Phi(x) \rangle_{\mathcal{H}_k} \rightarrow \langle h , \Phi(x) \rangle_{\mathcal{H}_k}$ for all $x \in \mathcal{X}$, it follows that $g(x) = \langle h , \Phi(x) \rangle_{\mathcal{H}_k} = T(h)(x)$ for $\mu$-almost all $x \in \mathcal{X}$.
Thus $g$ and $T(h)$ are equal in $\mathcal{L}^1(\mu)$ and the graph of $T$ is indeed closed.
The conditions of the closed graph theorem, which states that a linear map $T$ between Banach spaces is bounded if and only if its graph is closed, have now been verified.
Thus $T$ is continuous and 
$$
\|T\|_{\mathrm{op}} = \sup_{\|h\|_{\mathcal{H}_k} \leq 1} \int |T(h)| \; \mathrm{d}\mu < \infty .
$$
This allows us to conclude that $\mathrm{I}_\Phi$ is a continuous linear functional, since
\begin{align*}
|\mathrm{I}_\Phi(h)|
= \left| \int \langle h , \Phi(x) \rangle_{\mathcal{H}_k} \; \mathrm{d}\mu(x) \right| 
& \leq \int | \langle h , \Phi(x) \rangle_{\mathcal{H}_k} | \; \mathrm{d}\mu(x) \\
& = \int |T(h)| \; \mathrm{d}P\mu
\leq \|T\|_{\mathrm{op}} \|h\|_{\mathcal{H}_k} ,
\end{align*}
as required.
\end{proof}

These results furnish the following elegant characterization of the set $\mathcal{P}_{\mathcal{H}_k}$ of probability measures that can be safely embedded into $\mathcal{H}_k$:

\begin{proposition}[Characterization of $\mathcal{P}_k$]
\label{cor: embeddability in general}
$\mu \in \mathcal{P}_{\mathcal{H}_k}$ if and only if $\mathcal{H}_k \subset \mathcal{L}^1(\mu)$.
\end{proposition}
\begin{proof}
Take $\Phi(x) = k(\cdot,x)$ to be the \emph{canonical feature map}, so that $k$ being measurable implies the scalar functions $x \mapsto \langle h , \Phi(x) \rangle_{\mathcal{H}_k} = h(x)$ are measurable and $\mathrm{I}_\mu = \mathrm{I}_\Phi$.
If $\mathcal{H}_k \subset \mathcal{L}^1(\mu)$ then $x \mapsto k(\cdot,x)$ is scalarly $\mu$-integrable, and \Cref{prop: Dunford} shows that $\mathrm{I}_\mu$ is a continuous linear functional.
Conversely, if $\mathcal{H}_k \not\subset \mathcal{L}^1(\mu)$ then it is clear that $\mathrm{I}_\mu$ is not continuous and thus $\mu \notin \mathcal{P}_{\mathcal{H}_k}$.
\end{proof}

\noindent The question of which measures can be embedded then reduces to whether $\mathcal{H}_k \subset \mathcal{L}^1(\mu)$.
A well-known sufficient condition is provided in the following result:

\begin{proposition}
\label{prop: sufficient for embed}
If $\int \sqrt{k(x,x)} \mathrm{d}\mu(x) < \infty$ then $\mathcal{H}_k \subset \mathcal{L}^1(\mu)$.
\end{proposition}
\begin{proof}
For $h \in \mathcal{H}_k$, from the reproducing property and Cauchy--Schwarz,
\begin{align*}
\int |h(x)| \; \mathrm{d}\mu(x) = \int |\langle h , k(\cdot,x) \rangle | \; \mathrm{d}\mu(x) \leq \|h\|_{\mathcal{H}_k} \int \|k(\cdot,x)\|_{\mathcal{H}_k} \; \mathrm{d}\mu(x) ,
\end{align*}
where the reproducing property again yields $\|k(\cdot,x)\|_{\mathcal{H}_k} = \sqrt{k(x,x)}$, as required.
\end{proof}

\noindent One can weaken the above integrability condition under mild assumptions on $k$ and $\mathcal{X}$: 

\begin{proposition} \label{prop: double integral condition}
If $k$ is continuous, $\mathcal{X}$ is separable, and 
$$
\iint |k(x,y)| \; \mathrm{d}\mu(x) \mathrm{d}\mu(y) < \infty, 
$$
then $\mathcal{H}_k \subset \mathcal{L}^1(\mu)$.
\end{proposition}
\begin{proof}
Since $k$ is measurable and real-valued, $k$ is \emph{strongly measurable} in the sense of \citet[][Section 3.1]{carmeli2006vector}.
Further, since $k$ is strongly measurable and $\iint |k(x,y)| \; \mathrm{d}\mu(x) \mathrm{d}\mu(y) < \infty$, then $k$ is \emph{$\infty$-bounded} in the sense of \citet[][Definition 4.1]{carmeli2006vector}; see \citet[][Corollary 4.3]{carmeli2006vector}.
Since $k$ is continuous and $\mathcal{X}$ is separable, it follows that $\mathcal{H}_k$ is separable \citep[][Corollary 5.2]{carmeli2006vector}.
Since $\mathcal{H}_k$ is separable, $k$ being $\infty$-bounded is equivalent to $\mathcal{H}_k \subset \mathcal{L}^1(\mu)$ \citep[][Proposition 4.4]{carmeli2006vector}.
\end{proof}

\subsection{Vector-Valued Reproducing Kernel Hilbert Spaces}
\label{subsec: vvrkhs}

The development of multivariate Stein discrepancies later in this monograph relies on a generalization of \acp{rkhs} to spaces of functions that are vector-valued.
To generalize the concept of a bounded linear functional, which was used to define (scalar-valued) \acp{rkhs} in \Cref{def: rkhs}, we say that a map $L : \mathcal{H} \rightarrow \mathbb{R}^d$ acting on a Hilbert space $\mathcal{H}$ is a \emph{continuous linear operator} if $L$ is linear and there exists a constant $C$ such that $\|L(f)\| \leq C \|f\|_{\mathcal{H}}$ for all $f \in \mathcal{H}$.

\begin{definition}[Vector-valued reproducing kernel Hilbert space] \label{def: vvrkhs}
Let $\mathcal{X}$ be a set.
A \emph{vector-valued reproducing kernel Hilbert space} is a Hilbert space $\mathcal{H}$ of vector-valued functions on $\mathcal{X}$, such that for each $x \in \mathcal{X}$, the map $f \mapsto f(x)$ is a continuous linear operator on $\mathcal{H}$.
\end{definition}

If $L : \mathcal{H} \rightarrow \mathbb{R}^d$ is a continuous linear operator, then the components $L_i$ of $L$ are each continuous linear functionals and the Riesz representation theorem can be applied.
It follows that, if $\mathcal{H}$ is a vector-valued \ac{rkhs} whose elements are functions of the form $f : \mathcal{X} \rightarrow \mathbb{R}^d$, then for each $x \in \mathcal{X}$ and $i \in \{1,\dots,d\}$
there exists a Riesz representer $K_{x,i} \in \mathcal{H}$ for the functional $f \mapsto f_i(x)$, meaning that $f_i(x) = \langle f , K_{x,i} \rangle_{\mathcal{H}}$ for all $f \in \mathcal{H}$.
The \emph{reproducing kernel} of a vector-valued \ac{rkhs} is defined as a bivariate function $K : \mathcal{X} \times \mathcal{X} \rightarrow \mathbb{R}^{d \times d}$ with $K_{i,j}(x,y) = \langle K_{x,i} , K_{y,j} \rangle_{\mathcal{H}}$.
A reproducing kernel $K$ can be verified to be an instance of a \emph{matrix-valued kernel}:

\begin{definition}[Matrix-valued kernel] \label{def: kernel}
A bivariate function $K : \mathcal{X} \times \mathcal{X} \rightarrow \mathbb{R}^{d \times d}$, for some $d \in \mathbb{N}$, is said to be a \emph{kernel} on a set $\mathcal{X}$ if it is
\begin{enumerate}
    \item \emph{transpose-symmetric}; $K(x,y) = K(y,x)^\top$ for all $x,y \in \mathcal{X}$
    \item \emph{positive semi-definite}; 
    \begin{align}
    \sum_{i=1}^n \sum_{j=1}^n \langle u_i , K(x_i,x_j) u_j \rangle \geq 0 \label{eq: PSD VV}
    \end{align}
    for all $x_1, \dots, x_n \in \mathcal{X}$, all $u_1,\dots,u_n \in \mathbb{R}^d$, and all $n \in \mathbb{N}$. 
\end{enumerate} 
\end{definition}

\noindent For clarity we emphasize that $\langle \cdot , \cdot \rangle$, appearing in \eqref{eq: PSD VV} is the usual Euclidean inner product on $\mathbb{R}^d$.

Let $K_x \defeq K(\cdot,x)$, so that $K_x : \mathcal{X} \rightarrow \mathbb{R}^{d \times d}$ is matrix-valued.
In analogy with the case of a scalar-valued kernel, there is a unique Hilbert space reproduced by a given kernel:

\begin{theorem}[Existence and uniqueness]
\label{thm: existence vvrkhs}
For all kernels $K : \mathcal{X} \times \mathcal{X} \rightarrow \mathbb{R}^{d \times d}$, there exists a unique Hilbert space for which $K$ is a reproducing kernel, denoted $\mathcal{H}_K$.
This Hilbert space is characterized as the unique Hilbert space for which
\begin{enumerate}
    \item $K_x u \in \mathcal{H}_K$ for all $x \in \mathcal{X}$, $u \in \mathbb{R}^d$,
    \item $\langle f , K_x u \rangle_{\mathcal{H}_K} = \langle f(x) , u \rangle$ for all $f \in \mathcal{H}_K$, $x \in \mathcal{X}$, $u \in \mathbb{R}^d$,
\end{enumerate}
with the latter called the \emph{reproducing property} of the kernel.
\end{theorem}

\noindent The proof of \Cref{thm: existence vvrkhs} can be found in \citet[Proposition 2.1]{carmeli2006vector}.
The elements of $\mathcal{H}_K$ are (limits of) finite linear combinations of vector-valued functions of the form $K_x u$, where $u \in \mathbb{R}^d$; that is,
\begin{align*}
    \mathcal{H}_K \defeq \overline{\mathrm{span} \{ K_x u : x \in \mathcal{X}, u \in \mathbb{R}^d \} }
\end{align*}
where here the completion is taken with respect to the inner product 
\begin{align*}
    \left( \sum_{i=1}^n K_{x_i} u_i , \sum_{j=1}^m K_{y_j} v_j \right) \mapsto \sum_{i=1}^n \sum_{j=1}^m \langle u_i , K(x_i , y_j) v_j \rangle .
\end{align*}

The simplest examples of matrix-valued kernels are those constructed from scalar-valued kernels, for which the corresponding vector-valued \acp{rkhs} can be explicitly characterized:

\begin{example}[Diagonal matrix-valued kernel]
    Let $k_i : \mathcal{X} \times \mathcal{X} \rightarrow \mathbb{R}$ be a (scalar-valued) kernel for each $i \in \{1,\dots,d\}$.
    Then $K(x,y) \defeq  \mathrm{diag}(k_1(x,y),\dots,k_d(x,y))$ is a matrix-valued kernel.
    The associated vector-valued Hilbert space has inner product
    $$
    \langle f , g \rangle_{\mathcal{H}_K} = \sum_{i=1}^d \langle f_i , g_i \rangle_{\mathcal{H}_{k_i}}.
    $$
\end{example}

\begin{example}[Tensor product matrix-valued kernel]
    Let $k : \mathcal{X} \times \mathcal{X} \rightarrow \mathbb{R}$ be a (scalar-valued) kernel and let $B \in \mathbb{R}^{d \times d}$ be a symmetric positive definite matrix.
    Then $K(x,y) \defeq  k(x,y) B$ is a matrix-valued kernel.
    The associated vector-valued Hilbert space has inner product
    $$
    \langle f , g \rangle_{\mathcal{H}_K} = \sum_{i=1}^d \sum_{j=1}^d B_{i,j} \langle f_i , g_j \rangle_{\mathcal{H}_k}.
    $$
\end{example}

\section{Statistical Divergences}
\label{sec: divergences}

This section introduces tools that can be used to quantify the extent to which two probability distributions differ.
The abstract formulation is that of a \emph{statistical divergence}:

\begin{definition}[Statistical divergence]
\label{def: divergence}
    Let $\mathcal{P}$ be the set of probability measures on a common measurable space $(\Omega,\mathcal{S})$.
    A map $D : \mathcal{P} \times \mathcal{P} \rightarrow [0,\infty]$ is called a \emph{statistical divergence} if $D(\mu,\nu) \geq 0$ with equality if and only if $\mu = \nu$.
\end{definition}

Any metric on $\mathcal{P}$ is automatically a statistical divergence, but in general a statistical divergence need not be symmetric or satisfy a triangle inequality.
In what follows we review several existing statistical divergences, noting that explicit computation of these divergences may be difficult or impossible in the context of the problems that motivate the use of Stein's method.

\subsection{Classical Divergences}

Perhaps the most classical statistical divergence is the Kullback--Leibler divergence.
Recall that the concept of a Radon--Nikodym derivative introduced in \Cref{subsec: lebesgue spaces}.

\begin{definition}[Kullback--Leibler divergence]
\label{def: KL divergence}
    Let $\mu$ and $\nu$ be probability measures on a measurable space $(\Omega,\mathcal{S})$ with $\mu \ll \nu$.
    The \emph{Kullback--Leibler divergence} is defined as
    \begin{align*}
        \mathrm{KL}(\mu || \nu) \defeq \int \log \left( \frac{\mathrm{d}\mu}{\mathrm{d}\nu} \right) \; \mathrm{d} \mu
    \end{align*}
    with the convention that $\mathrm{KL}(\mu || \nu)$ is infinite whenever the integral is ill-defined.
\end{definition}

The Kullback--Leibler divergence is indeed a statistical divergence, since $\log(x) \leq x - 1$ implies $$
\mathrm{KL}(\mu || \nu) = - \int \log \left(\frac{\mathrm{d}\nu}{\mathrm{d}\mu} \right) \; \mathrm{d}\mu \geq - \int \frac{\mathrm{d}\nu}{\mathrm{d}\mu} - 1 \; \mathrm{d}\mu = 0
$$ 
with equality if and only if $(\mathrm{d}\mu / \mathrm{d}\nu)(\omega) = 1$ for $\nu$-almost all $\omega \in \Omega$.
The Kullback--Leibler divergence is central to the field of information theory, where it is known as the \emph{relative entropy} from $\nu$ to $\mu$.
Compared to other statistical divergences (see below), the Kullback--Leibler divergence is a strong notion of divergence, requiring that both $\mu$ and $\nu$ have the same support.
For example, if $x_n \rightarrow x$ with $x_n \neq x$ for all $n \in \mathbb{N}$, $\mu_n = \delta_{x_n}$ and $\nu = \delta_x$, then $\mathrm{KL}(\mu_n || \nu) \nrightarrow 0$.
Here $\delta_x$ denotes the atomic measure at $x$, defined in \Cref{ex: atomic measure}.

An alternative and often more useful notion is that of \emph{weak convergence}.
There are several equivalent formulations of weak convergence, but from the point of view of statistical divergences we can metrize weak convergence using the Dudley metric \citep[][Chapter 11]{dudley2018real}:

\begin{definition}[Dudley metric]
    \label{def: dudley metric}
    Let $(\mathcal{X},\rho)$ be a metric space equipped with the Borel $\sigma$-algebra.
    Let $\|\cdot\|_{\text{BL},\rho}$ denote the bounded Lipschitz norm from \Cref{ex: bl norm}.
    For $\mu$ and $\nu$ (Borel) probability distributions on $\mathcal{X}$, the \emph{Dudley} (or \emph{bounded Lipschitz}) \emph{metric} is defined as
    \begin{align*}
    \mathrm{BL}_\rho(\mu,\nu) \defeq \sup_{\|f\|_{\text{BL},\rho} \leq 1} \mu(f) - \nu(f) ,
    \end{align*}
    where we recall the shorthand $\mu(f) = \int f \; \mathrm{d}\mu$ introduced in \eqref{eq: int short}.
\end{definition}

\noindent A sequence $(\mu_n)_{n \in \mathbb{N}}$ of probability distributions is said to \emph{converge in distribution} to a a probability distribution $\nu$ if $\mathrm{BL}_\rho(\mu_n , \nu) \rightarrow 0$, and we use the shorthand $P_n \stackrel{d}{\rightarrow} P$.
Equivalently, a sequence of random variables $X_n : \Omega \rightarrow \mathcal{X}$ defined on a common probability space $\Omega$ with $X_n \sim P_n$, is said to \emph{converge weakly} to a random variable $X : \Omega \rightarrow \mathcal{X}$ with $X \sim P$ if $P_n \stackrel{d}{\rightarrow} P$.
It is straightforward to verify that $\mathrm{BL}_\rho$ is a metric on the set of all (Borel) probability distributions on $\mathcal{X}$, and in particular it is a statistical divergence.

Computation of the Kullback--Leibler divergence and the Dudley metric requires detailed information about both the distributions $\mu$ and $\nu$, either through their Radon--Nikodym derivative or their generalized moments. 
However, in our motivating context it is the case that one of the distributions involved is implicitly defined via an intractable normalization constant.
In such cases alternative statistical divergences are needed.

\subsection{Wasserstein Metric}
\label{subsec: wass metric}

An important statistical divergence is given by the \emph{Wasserstein metric}.
For $\mu$ and $\nu$ probability distributions on a metric space $(\mathcal{X},\rho)$, a \emph{coupling} of $\mu$ and $\nu$ is a probability distribution on $\mathcal{X} \times \mathcal{X}$ whose marginals are $\mu(S) = \gamma(S \times \mathcal{X})$ and $\nu(S) = \gamma(\mathcal{X} \times S)$.
Let $\Gamma(\mu,\nu)$ denote the set of all possible couplings of $\mu$ and $\nu$.
Let $\mathcal{P}_s(\mathcal{X})$ denote the set of probability distributions $\mu$ on a measurable space $\mathcal{X}$ for which the $s$th moment $\int \rho(x,x_0)^s \; \mathrm{d}\mu(x)$ is finite, for any (and therefore all) $x_0 \in \mathcal{X}$.

\begin{definition}[$s$-Wasserstein metric]
    \label{def: Wasserstein}
    Let $\rho$ be a metric on a Radon\footnote{All separable complete metric spaces are Radon.} space $\mathcal{X}$.
    For $\mu , \nu \in \mathcal{P}_s(\mathcal{X})$, the $s$-\emph{Wasserstein metric} is defined as
    \begin{align*}
        \mathrm{W}_\rho^s(\mu,\nu) \defeq \left( \inf_{\gamma \in \Gamma(\mu,\nu)} \int \rho(x,y)^s \; \mathrm{d}\gamma(x,y) \right)^{1/s} .
    \end{align*}
\end{definition}

\noindent It can be proven that \Cref{def: Wasserstein} defines a metric on the set $\mathcal{P}_s(\mathcal{X})$, and in particular a statistical divergence \citep{clement2008elementary}.
The $s$-Wasserstein metric is routinely used in applications of optimal transport since it is naturally adapted to the metric structure of $\mathcal{X}$; for example, $\mathrm{W}_\rho^s(\delta_x, \delta_y) = \rho(x,y)$ for all $x,y \in \mathcal{X}$.
Convergence of a sequence $(\mu_n)_{n \in \mathbb{N}}$ to a limit $\nu$ in the sense of $\mathrm{W}_\rho^s(\mu_n,\nu) \rightarrow 0$ is equivalent to the statement that both $\mu_n \stackrel{d}{\rightarrow} \nu$ and $\int \rho(x,x_0)^s \; \mathrm{d}\mu_n(x) \rightarrow \int \rho(x,x_0)^s \; \mathrm{d}\nu$ for some $x_0 \in \mathcal{X}$.
Thus in general the $s$-Wasserstein metric induces a stronger topology compared to the Dudley metric from \Cref{def: dudley metric}.
This fact is especially clear in the case $s = 1$, where the 1-Wasserstein metric has the dual representation
\begin{align}
\mathrm{W}_\rho^1(\mu,\nu) = \sup_{|f|_{\text{Lip},\rho} \leq 1} \mu(f) - \nu(f) , \label{eq: 1-Wass}
\end{align}
where $|\cdot|_{\text{Lip},\rho}$ is the Lipschitz semi-norm from \Cref{ex: bl norm}.
Thus, compared to the Dudley metric, the 1-Wasserstein metric involves a supremum over a larger set of test functions $f$ by relaxing the boundedness requirement.
Of course, since continuous functions are automatically bounded on compact subsets of $\mathbb{R}^d$, the topology induced by the 1-Wasserstein metric coincides with that of the Dudley metric when $\mathcal{X}$ is compact.
In settings where the metric $\rho$ is unambiguous, we will simply write $\mathrm{W}^s$ for $\mathrm{W}_\rho^s$.

In addition to the standard 1-Wasserstein metric, we also introduce a \emph{tilted} counterpart which, given a \emph{tilting function} $t : \mathbb{R}^d \rightarrow (0 , \infty)$ sets\footnote{To be clear, here $t f$ denotes the pointwise product $(tf)(x) \defined t(x) f(x)$.}
\begin{align*}
\mathrm{W}_\rho^{1,t}(\mu,\nu) \defined \sup_{|f|_{\mathrm{Lip},\rho} \leq 1} \mu(t f) - \nu(t f) . 
\end{align*}
There is no general dominance relation between tilted Wasserstein distances; the topologies they induce are different and depend on the tilting function $t$.
Note that the standard 1-Wasserstein distance \eqref{eq: 1-Wass} is recovered when $t(x) \equiv 1$. 
As with the standard Wasserstein distances, the tilted Wasserstein distances induces much weaker topologies than, for example, divergences such as Kullback–-Leibler or Hellinger, since they do not require absolute continuity of measures.

Though the Wasserstein metrics are elegant and natural, they are not well-suited to computation in our motivating context.
Indeed, computing $\mathrm{W}_\rho^s(\mu,\nu)$ requires integration under $\mu$ and $\nu$, which is typically intractable for at least one of the distributions involved.
Numerical approximations schemes can sometimes be used, but these encounter prohibitive \emph{sample complexity}:
Suppose that $\mu$ and $\nu$ are defined on $\mathbb{R}^d$, and that we have access to each distribution only via a (possibly large) number $n$ of independent samples $X_1,\dots,X_n \sim \mu$ and $Y_1,\dots ,Y_n \sim \nu$.
A natural approach to approximate $\mathrm{W}_\rho^s(\mu,\nu)$ is to use the plug-in estimator $\mathrm{W}_\rho^s(\mu_n,\nu_n)$ where $\mu_n = \frac{1}{n} \sum_{i=1}^n \delta_{X_i}$ and $\nu_n = \frac{1}{n} \sum_{i=1}^n \delta_{Y_i}$.
However, this approximation incurs an error of size $O_P(n^{1/d})$.
That is, the difficulty of computing Wasserstein metrics increases exponentially in the dimension $d$ of the domain on which the distributions are defined; see \citet{sriperumbudur2012empirical}.
Furthermore, since one of our main motivations was to use a discrepancy to assess the quality of approximate sampling algorithms, requiring an exact sample to carry out this assessment leads to a chicken-and-egg problem.

\subsection{Maximum Mean Discrepancy}
\label{subsec: mmd}

The maximum mean discrepancy, introduced in this section and based on the kernel mean embedding introduced in \Cref{subsec: kernel mean embed}, offers an alternative to the 1-Wasserstein metric for which the sample complexity is essentially dimension-independent.
Recall that the kernel mean element corresponding to a distribution $\mu$ is denoted $\phi_\mu$; see \Cref{def: kernel mean element}.

\begin{definition}[Maximum mean discrepancy]
Let $k : \mathcal{X} \times \mathcal{X} \rightarrow \mathbb{R}$ be a positive definite kernel and let $\mu, \nu \in \mathcal{P}_{\mathcal{H}_k}$.
The \emph{maximum mean discrepancy} between two distributions $\mu, \nu \in \mathcal{P}_{\mathcal{H}_k}$ is $\mathrm{MMD}_k(\mu,\nu) \defeq \|\phi_\mu - \phi_\nu \|_{\mathcal{H}_k}$.
\end{definition}

To understand why the \ac{mmd} enjoys favorable sample complexity compared to the Wasserstein metrics, we first derive an equivalent form that is more algebraically explicit.
Indeed, 
\begin{align}
\mathrm{MMD}_k^2(\mu,\nu) = \|\phi_\mu - \phi_\nu \|_{\mathcal{H}_k}^2 & = \langle \phi_\mu - \phi_\nu , \phi_\mu - \phi_\nu \rangle_{\mathcal{H}_k} \label{eq: mmd explicit} \\
& = \langle \phi_\mu , \phi_\mu \rangle_{\mathcal{H}_k} - 2 \langle \phi_\mu , \phi_\nu \rangle_{\mathcal{H}_k} + \langle \phi_\nu , \phi_\nu \rangle_{\mathcal{H}_k} \nonumber
\end{align}
where, considering for example the term $\langle \phi_\mu , \phi_\nu \rangle_{\mathcal{H}_k}$, we have from the fact that $\phi_\nu$ is the Riesz representer of $\mathrm{I}_\nu$ and \eqref{eq: kernel mean element} that 
\begin{align}
\langle \phi_\nu , \phi_\mu \rangle_{\mathcal{H}_k} = \mathrm{I}_\nu(\phi_\mu) & = \int \phi_\mu(y) \; \mathrm{d}\nu(y) \nonumber \\
& = \iint k(x,y) \; \mathrm{d}\mu(x) \mathrm{d}\nu(y). \label{eq: kernel mean inner prod}
\end{align}
Proceeding similarly with all three terms in \eqref{eq: mmd explicit} results in the explicit expression
\begin{align} 
\mathrm{MMD}_k^2(\mu,\nu) & = \iint k(x,y) \mathrm{d}\mu(x) \mathrm{d}\mu(y) - 2 \iint k(x,y) \mathrm{d}\mu(x) \mathrm{d}\nu(y) \nonumber \\
& \hspace{30pt} + \iint k(x,y) \mathrm{d}\nu(x) \mathrm{d}\nu(y) . \label{eq: computable mmd}
\end{align}
In settings where the integrals in \eqref{eq: computable mmd} cannot be exactly computed, but where it is nevertheless possible to sample from $\mu$ and $\mu$, the \ac{mmd} could be approximated by replacing the intractable kernel mean elements $\phi_\mu$ and $\phi_\nu$ with their respective Monte Carlo estimators $\phi_\mu^n$ and $\phi_\nu^n$, respectively based on independent samples $X_1,\dots,X_n \sim \mu$ and $Y_1,\dots,Y_n \sim \nu$, as defined in \eqref{eq: approx kernel mean}.
From our observation following \eqref{eq: approx kernel mean} that $\phi_\mu^n$ can be interpreted as the \emph{exact} kernel mean element corresponding to the approximating measure $\mu_n = \frac{1}{n} \sum_{i=1}^n \delta_{X_i}$, it follows that the plug-in estimator $\|\phi_\mu^n - \phi_\nu^n \|_{\mathcal{H}_k}$ carries the dual interpretation of computing the \emph{exact} $\mathrm{MMD}_k(\mu_n,\nu_n)$ between the two approximating measures $\mu_n = \frac{1}{n} \sum_{i=1}^n \delta_{X_i}$ and $\nu_n = \frac{1}{n} \sum_{j=1}^n \delta_{Y_j}$.
From this perspective, it is trivial to obtain conditions under which the approximation to \ac{mmd} obtained in this manner is consistent.
Indeed, from \eqref{eq: computable mmd} we have that
\begin{align}
    \mathrm{MMD}_k^2(\mu_n,\nu_n) & = \frac{1}{n^2} \sum_{i=1}^n \sum_{i'=1}^n k(X_i,X_{i'}) - 2 \frac{1}{n^2} \sum_{i=1}^n \sum_{j=1}^n k(X_i,Y_j) \nonumber \\
    & \qquad + \frac{1}{n^2} \sum_{j=1}^n \sum_{j'=1}^n k(Y_j,Y_{j'}) \label{eq: MC MMD}
\end{align}
where each of these three terms represents a Monte Carlo approximation of the corresponding integral in \eqref{eq: computable mmd} whose error converges at $O_P(n^{-1/2})$.
As such, the sample complexity of \ac{mmd} is essentially dimension-independent.
Furthermore, the apparent $O(n^2)$ computational complexity associated with the Monte Carlo estimator \eqref{eq: MC MMD} can be mitigated by noting that computation is "embarrassingly parallel".

Though the \ac{mmd} can have computational advantages compared with the Wasserstein metrics, they are similarly poorly-suited to computation in our motivating context.
Indeed, approximating $\mathrm{MMD}_k(\mu,\nu)$ using exact samples from $\mu$ and $\nu$ is problematic since one of our main motivations was to use a discrepancy to assess the quality of approximate sampling algorithms, leading again to a chicken-and-egg problem.
Nevertheless, it is instructive to understand the appealing theoretical properties of \ac{mmd}, as we will later seek to construct Stein discrepancies that enjoy similar appeal.

The \ac{mmd} is also known as the \emph{worst case cubature error} due to its dual interpretation as
\begin{align}
    \mathrm{MMD}_k(\mu,\nu) = \sup_{\|f\|_{\mathcal{H}_k} \leq 1} \mu(f) - \nu(f) .  \label{eq: MMD as WCE}
\end{align}
From this perspective, it is clear that the kernel $k$ determines the topology induced by \ac{mmd}, through determining the regularity of functions contained in the unit ball of $\mathcal{H}_k$.
In particular, if $k(x,y) = \varphi(x-y)$ for a Lipschitz function $\varphi$ with Lipschitz constant $L_\varphi$, then $\{f : \|f\|_{\mathcal{H}_k} \leq 1 \} \subset \{f : |f|_{\text{Lip}} \leq L_\varphi \}$, in which case the topology of \ac{mmd} is weaker than or equal to that of the 1-Wasserstein distance, since $\mathrm{MMD}_k(\mu,\nu) \leq L_\varphi \mathrm{W}^1(\mu,\nu)$. 
A kernel that reproduces a sufficiently large Hilbert space will, intuitively, contain enough test functions to be capable of distinguishing between different probability measures:

\begin{definition}[Characteristic kernel]
\label{def: char kernel}
A kernel $k$ is said to be \emph{characteristic} if $\mathrm{MMD}_k(\mu,\nu) = 0$ implies $\mu = \nu$ for all $\mu,\nu \in \mathcal{P}_{\mathcal{H}_k}$.
\end{definition}

\begin{example}[Polynomial kernel is not characteristic] \label{ex: poly not char}
From \Cref{ex: poly space kernel} we have the kernel $k(x,y) = 1 + (xy) + \dots + (xy)^p$, which reproduces a Hilbert space whose elements are the polynomials of degree at most $p$ on the domain $\mathcal{X} = \mathbb{R}$.
Thus $\mathrm{MMD}_k(\mu,\nu) = 0$ if and only if the moments $\int x^i \; \mathrm{d}\mu(x)$ and $\int x^i \; \mathrm{d}\nu(x)$ are identical for $i = 1,\dots,p$.
In particular, $k$ is \emph{not} a characteristic kernel.
\end{example}

\begin{example}
The Gaussian kernel $k(x,y) = \exp(-\|x-y\|^2 / (2\ell^2))$ is a characteristic kernel on $\mathcal{X} = \mathbb{R}^d$, for all $\ell > 0$.
\end{example}

The characteristic property says nothing about "small" values of $\mathrm{MMD}_k(\mu,\nu)$, only about the case when $\mathrm{MMD}_k(\mu,\nu)$ is exactly 0.
Thus characteristicness on its own does not provide strong justification for using a kernel $k$ to measure the discrepancy between $\mu$ and $\nu$.
For this reason we now introduce a stronger property, called (ironically, in this context) \textit{weak convergence control}.

\begin{definition}[Weak convergence control]
A kernel $k$ is said to have \emph{weak convergence control} if $\mathrm{MMD}_k(\mu_n,\mu) \rightarrow 0$ implies $\mu_n \stackrel{d}{\rightarrow} \mu$.
\end{definition}

\noindent Convergence control justifies the use of \ac{mmd} as an optimization criterion for the purposes of quantization and more general distributional approximation.
See \citet{muandet2016kernel} for a general discussion of applications in which \ac{mmd} is often used.

\begin{remark}
Perhaps surprisingly, for a compact Hausdorff space $\mathcal{X}$, a bounded and measurable characteristic kernel $k$ is guaranteed to have weak convergence control.
This equivalence no longer holds when the domain $\mathcal{X}$ is non-compact, and a bounded and characteristic kernel can fail to have weak convergence control; see \citet{simon2020metrizing}.
Clearly a kernel that is not characteristic fails to have weak convergence control.
\end{remark}

\begin{example}
The Gaussian kernel $k(x,y) = \exp(-\|x-y\|^2 / (2\ell^2))$ controls weak convergence of probability distributions on $\mathcal{X} = [0,1]^d$.
It can also be shown that the Gaussian kernel controls weak convergence on $\mathcal{X} = \mathbb{R}^d$; this can be deduced from e.g. Theorem 7 of \citet{simon2020metrizing} and the general results in \citet{sriperumbudur2011universality}.
\end{example}

\chapter{Stein Operators}
\label{chap: Stein operators}

This chapter introduces the concept of a Stein operator.
Informally, a Stein operator is a mapping that generates functions that integrate to zero under a target probability measure $P$.
Stein operators are the key ingredient that we will use to construct Stein discrepancies in \Cref{chap: Stein discrepancies}.
This chapter will make extensive use of the shorthand notation $P(f) = \int f \; \mathrm{d}P$, introduced in \Cref{eq: int short}.

\begin{definition}[Stein operator and Stein set]
    \label{def: stein op}
    Let $P$ be a probability distribution on a measurable space $\xset$.
    We call a set $\gset$ a \emph{Stein set} and a linear map $\TP : \gset \rightarrow \Lp[1](P)$ a \emph{Stein operator} for $P$ if 
    \begin{align}
    P(\TP g) = 0
    \qtext{for all}
    g \in \gset.
    \end{align}
\end{definition}

\noindent The integral $P(\TP g)$ appearing in \Cref{def: stein op} can be viewed as the expectation of a random variable $(\TP g)(X)$ when $X \sim P$, and in this sense $\TP$ is a Stein operator when $(\TP g)(X)$ has mean zero for all elements $g$ in the stein set $\gset$.

There are numerous constructions for Stein operators, but, from both a practical and theoretical standpoint, 
not all Stein operators are equal.
For example, at one extreme we have the always-zero Stein operator 
\begin{align}
(\TP g)(x) = 0,
\end{align}
which is simple to compute but provides absolutely no information about the target distribution $P$. 
At the other extreme, we have the mean-recentering Stein operator
\begin{align}
    (\TP g)(x) = g(x) - P(g) \label{eq: triv SO}
\end{align}
with Stein set $\gset = \Lp[1](P)$.
This operator is ideal from the standpoint of theoretical analysis but is only computable when integration under $P$ is tractable. %
Sadly, for most of our applications of interest, exact integration under $P$ is decidedly intractable.
In fact, for many applications, our ultimate goal is to approximate intractable integrals under $P$ or to assess approximations thereof.

Remarkably, there are computable Stein operators that yield useful discrepancies without requiring explicit integration under $P$.  %
For example, \citet{stein1986approximate} showed that if $P$ is a univariate probability distribution on $\R$ with a positive and differentiable probability density function $p$ (see \Cref{def: cts disn}) then, for all differentiable functions $g$ in an appropriate Stein set $\gset$, 
\begin{align}
(\TP g)(x) = g’(x) + g(x) \frac{\mathrm{d}}{\mathrm{d}x} \log p(x)     \label{eq: density method}
\end{align}
is a Stein operator for $P$.
Indeed, integrating by parts
\begin{align}
    P(\TP g) = \int \TP g \; \mathrm{d}P & = \int \left[ g’(x) + g(x) \frac{\mathrm{d}}{\mathrm{d}x} \log p(x)  \right] \; p(x) \dx \nonumber \\
    & = \int \left[ g’(x) + \frac{g(x)p'(x)}{p(x)}  \right] \; p(x) \dx \label{eq: uni int by parts} \\
    & = \int g'(x) p(x) + g(x) p'(x) \dx 
    = \int (gp)'(x) \dx \nonumber
\end{align}
which from the fundamental theorem of calculus will be equal to zero provided that $g(x) p(x) \rightarrow 0$ as $|x| \rightarrow \infty$.

The operator \cref{eq: density method} is called Stein's \emph{density method} operator, as it exchanges explicit integration under $P$ \cref{eq: triv SO} for evaluation of the log density derivative $\frac{\mathrm{d}}{\mathrm{d}x} \log p(x)$. 
Importantly, this log density derivative is often computable even when exact integration under $p$ is intractable. 
For example, if $P$ is a posterior distribution arising from a Bayesian statistical analysis (c.f. \Cref{sec: Stein as a tool}), then Bayes' theorem implies that $p(x) \propto \pi(x) \mathcal{L}(x)$ where $\pi(x)$ is the density of the prior distribution and $\mathcal{L}$ is the likelihood.
These two quantities are usually explicit and are sufficient for use of the density method since
\begin{align*}
    \frac{\mathrm{d}}{\mathrm{d}x} \log p(x) = \frac{\mathrm{d}}{\mathrm{d}x} \log \pi(x) + \frac{\mathrm{d}}{\mathrm{d}x} \log \mathcal{L}(x) ,
\end{align*}
which can in principle be computed provided that both the prior and the likelihood can each be differentiated.
Indeed, the fact that gradients of the log posterior density can be computed without access to the intractable marginal likelihood underpins a wide range of approximate sampling algorithms used in Bayesian statistical contexts \citep{brooks2011handbook}.

In this Chapter, we will explore similarly practical Stein operators for a wide variety of target distributions $P$.
Our starting point is a canonical Stein operator for distributions on $\reals^d$, called the Langevin Stein operator, which we present in \Cref{sec: Lang Stein}.
The Langevin Stein operator is a special case of a more general construction called a diffusion Stein operator, which we present in \Cref{sec: diffusion operators}.
Often we encounter distributions whose domain is a subset $\xset \subset \reals^d$, for which special considerations are required when designing a Stein operator; for these situations we present the mirrored Stein operator in \Cref{sec: mirrored operators}.
On the other hand, we might seek to avoid computation of gradients of the density of the target distribution $P$, and gradient-free Stein operators have been developed for this situation, as discussed in \Cref{sec: grad free operators}.
Finally we discuss the case of a discrete state space $\xset$, for which several different Stein operators have been developed, as described in \Cref{sec: discrete operators}.

\section{Langevin Operator}
\label{sec: Lang Stein}

Stein's density method operator \cref{eq: density method} can be applied to univariate distributions $P$, but we would like to develop analogous operators for the multivariate distributions commonly encountered in Bayesian statistics and probabilistic inference, where $P$ is supported on $\xset = \reals^d$.
To achieve this, we will first consider a beautiful idea due to \citet{barbour1988stein,barbour1990stein,gotze1991rate}, known today as the \emph{generator method}.
Barbour and G\"{o}tze noticed that if one can identify a (time-homogeneous) Markov process for which $P$ is invariant (\Cref{def: invariant}) then, under mild conditions, the infinitesimal generator (\Cref{def: generator}) of that Markov process is a Stein operator for $P$. %
This is intuitively sensible; if we initialize a Markov process at a state sampled from $P$, then the distribution of later states will also be $P$ provided that the process is $P$-invariant.
The generator of a Markov process tells us the rate of change in the expected value of a test function, which is zero for a process initialized at an invariant distribution.
\citet{gorham2015measuring} instantiated this idea for a specific Markov process -- the overdamped Langevin diffusion introduced in \Cref{ex: overdamped Langevin}.

\begin{assumption}[Regularity of $P$]
\label{ass: reg langevin}
    Let $P$ be a probability distribution with a positive and differentiable density $p$ on $\mathcal{X} = \reals^d$.
\end{assumption}

Under \Cref{ass: reg langevin}, the infinitesimal generator of the overdamped Langevin diffusion is the differential operator $u \mapsto (\Delta u) + (\nabla u) \cdot (\nabla \log p)$; see \Cref{ex: overdamped II}.
Replacing $\nabla u$ with a more general vector-valued function $g$, we arrive at what is now called the Langevin Stein operator:

\begin{definition}[Langevin Stein operator; \citealp{gorham2015measuring}]
\label{def: Lang SO}
In the setting of \Cref{ass: reg langevin}, the \emph{Langevin Stein operator} for $P$ satisfies
\begin{align}
(\TP g)(x) = (\nabla \cdot g)(x) + g(x) \cdot (\nabla \log p)(x) \label{eq: Lang SO}
\end{align}
for suitably regular vector-valued functions $g: \reals^d \to \reals^d$ and all $x \in \reals^d$.
\end{definition}

\noindent The Langevin Stein operator exactly recovers the density method operator \cref{eq: density method} in the case $d = 1$.
Let us now be more precise about the regularity that we require of the vector-valued function $g$ appearing in \Cref{def: Lang SO}.
Recall from \Cref{subsec: lebesgue spaces} that, for a vector-valued function $g$, we use the shorthand $g \in \Lp[1](P)$ to denote that all components of $g$ are in $\Lp[1](P)$.
For what follows we will require a version of the divergence theorem that applies to potentially unbounded domains.
To this end, we first present several useful versions of the divergence theorem and their proofs in full.
We do this to emphasize that continuity of the divergence -- which is typically assumed in most textbook treatments of the divergence theorem -- is \emph{not} required for the conclusion of the divergence theorem to hold:

\begin{theorem}[{Divergence theorem on $[-1,1]^d$}]\label{thm:divergence standard}
Suppose that the divergence $\Div{v}$ exists everywhere on $\xset \defeq [-1,1]^d$ for a vector-valued function $v : \xset \mapsto \reals^d$. 
If $\Div{v}\in\Lp[1](\xset)$, then
\begin{align*}
\int_{\xset} (\Div{v})(x) \dx 
    = 
\oint_{\partial \xset} v(x) \cdot \normal(x) \dx
\end{align*}
where $\normal(x)$ is the outward pointing unit normal at $x \in \partial \xset$.
\end{theorem}
\begin{proof}
For each dimension $j\in[d]$ and $x\in \xset$, let $x_{-j}$ represent the subvector of $x$ with the $j$-th coordinate removed. 
Since $\Div{v}\in\Lp[1](\xset)$, Fubini's theorem \citep[Thm.~8.8]{rudin1987real} implies that $x_{j}\mapsto\grad_j v_j(x) \in\Lp[1]([-1,1])$ for each $j\in[d]$ and almost every $x_{-j}\in [-1,1]^{d-1}$.
Since, in addition, $\Div{v}$ exists everywhere on $\xset$, the fundamental theorem of calculus \citep[Thm.~7.21]{rudin1987real} implies that
\begin{align*}
\evalat{v_j(x)}{x_j=1} - \evalat{v_j(x)}{x_j=-1} 
    = 
\int_{-1}^1 \grad_j v_j(x) \dx_j
\end{align*}
for each $j\in [d]$ and almost every $x_{-j}\in [-1,1]^{d-1}$.
Hence we find that
\begin{align*}
\oint_{\partial \xset} v(x) \cdot \normal(x) \dx
    &= 
\sum_{j=1}^d \int_{[-1,1]^{d-1}} \evalat{v_j(x)}{x_j=1} - \evalat{v_j(x)}{x_j=-1} \dx_{-j} \\
    &=
\sum_{j=1}^d \int_{[-1,1]^{d-1}} \int_{-1}^1 \grad_j v_j(x) \dx_j \dx_{-j} \\
    &=
\int_{\xset} (\Div{v})(x) \dx,
\end{align*}
where the final step again uses Fubini's theorem.
\end{proof}
\begin{theorem}[Divergence theorem on $\Rd$]\label{thm: divergence}
Suppose that the divergence $\Div{v}$ exists everywhere on $\Rd$ for a vector-valued function $v : \Rd \rightarrow \reals^d$. 
If $\Div{v}\in\Lp[1](\Rd)$ then
\begin{align*}
\int (\Div{v})(x) \dx 
    = 
\lim_{r\to\infty}\oint_{\partial \xset_r} v(x) \cdot \normal_r(x) \dx
\end{align*}
where, for each $r\in\reals$,  $\xset_r\defeq [-r,r]^d$ and $\normal_r(x)$ is the outward pointing unit normal at $x \in \partial \xset_r$.
If, in addition, $v\in\Lp[1](\Rd)$, then
\begin{align*}
\int (\Div{v})(x) \dx 
    = 
0.
\end{align*}
\end{theorem}
\begin{proof}
Since $\Div{v} \in \Lp[1](\Rd)$ exists everywhere on $\Rd$, the first result follows from 
Lebesgue's dominated convergence theorem \citep[Thm.~1.34]{rudin1987real} and the divergence theorem (\cref{thm:divergence standard}):
\begin{align*}
\int (\Div{v})(x) \dx 
    &= 
\int \lim_{r\to\infty}\indic{x\in\xset_r} (\Div{v})(x) \dx \\
    &=
\lim_{r\to\infty} \int_{\xset_r} (\Div{v})(x) \dx 
    = 
\lim_{r\to\infty}\oint_{\partial \xset_r} v(x) \cdot \normal_r(x) \dx.
\end{align*}

Now suppose $v\in\Lp[1](\Rd)$ as well and, for each $j\in[d]$ and $r\in\R$, define the functions 
\begin{align*}
f_j(r) 
    \defeq 
\int \evalat{|v_j(x)|}{x_j=r} \dx_{-j}
    \geq
\int_{\infnorm{x_{-j}}\leq |r|} \evalat{|v_j(x)|}{x_j=r} \dx_{-j}
\end{align*} 
where $x_{-j}$ is the subvector of $x$ obtained by dropping the $j$-th coordinate.
This definition provides the bound
\begin{align*}
\left|\int (\Div{v})(x) \dx\right|
    &= 
\lim_{r\to\infty}\left|\oint_{\partial \xset_r} v(x) \cdot \normal_r(x) \dx\right| \\
    &=
\liminf_{r\to\infty}\left|\oint_{\partial \xset_r} v(x) \cdot \normal_r(x) \dx\right|\\
    &=
\liminf_{r\to\infty}\left|\sum_{j=1}^d \int_{\infnorm{x_{-j}}\leq r} \evalat{v_j(x)}{x_j=r} -\evalat{v_j(x)}{x_j=-r} \dx_{-j}\right|\\
    &\leq
\liminf_{r\to\infty} \sum_{j=1}^d f_j(r) + f_j(-r).
\end{align*}
To conclude, we will show that the right-hand side of this inequality is zero using a proof by contradiction.

To this end, fix any $j\in[d]$ and $\eps>0$, and suppose that $\liminf_{r\to\infty} f_j(r) + f_j(-r) \geq \eps$.
Then, there exists an $r' > 0$ for which $\inf_{r\geq r'} f_j(r) + f_j(r) \geq \eps/2$ so that
$\int_0^\infty f_j(r) + f_j(-r) \dr \geq \int_{r'}^\infty  f_j(r) + f_j(-r) \dr \geq  \int_{r'}^\infty \eps/2 \dr = \infty$.
This is a contradiction as 
Fubini's theorem \citep[Thm.~8.8]{rudin1987real} and the assumed integrability of $v$ imply that \begin{align*}
\int_0^\infty f_j(r) + f_j(-r) \dr = \int |v_j(x)| \dx < \infty.
\end{align*}
Therefore, we must have $\sum_{j=1}^d\liminf_{r\to\infty} f_j(r) + f_j(-r) = 0$ and hence $\left|\int (\Div{v})(x) \dx\right|=0$, as advertised.
\end{proof}
\begin{theorem}[Divergence theorem on a convex set]\label{thm:convex_divergence}
Suppose that a vector-valued function $v : \xset \rightarrow \reals^d$ is differentiable on a convex set $\xset\subseteq\Rd$. 
If $\Div{v}\in\Lp[1](\xset)$ then
\begin{align*}
\int_{\xset} (\Div{v})(x) \dx 
    = 
\lim_{r\to\infty}\oint_{\partial \xset_r} v(x) \cdot \normal_r(x) \dx
\end{align*}
where, for each $r\in\reals$,  $\xset_r\defeq \xset \cap[-r,r]^d$ and $\normal_r(x)$ is the outward pointing unit normal at $x \in \partial \xset_r$.
\end{theorem}
\begin{proof}
Since $\xset$ is convex, each  $\xset_r$ is bounded and convex and therefore has finite perimeter \citep[Lem.~2.4]{buttazzo1995minimum}. 
Since, in addition, $v$ is differentiable on $\xset$, the result follows from Lebesgue's dominated convergence theorem \citep[Thm.~1.34]{rudin1987real} and the Gauss-Green theorem of \citet[Thm.~5.19 and Prop.~5.8]{pfeffer2012divergence}:
\begin{align*}
\int_{\xset} (\Div{v})(x) \dx 
    &= 
\int \lim_{r\to\infty}\indic{x\in\xset_r} (\Div{v})(x) \dx \\
    &=
\lim_{r\to\infty} \int_{\xset_r} (\Div{v})(x) \dx 
    = 
\lim_{r\to\infty}\oint_{\partial \xset_r} v(x) \cdot \normal_r(x) \dx.
\end{align*}

\end{proof}

\begin{proposition}[Domain of Langevin operator] \label{lem: Stein op int by parts}
In the setting of \Cref{ass: reg langevin}, suppose that $g : \reals^d \rightarrow \reals^d$ is differentiable.
The following claims hold true for $\TP$ as defined in \cref{eq: Lang SO}.
\begin{enumerate}
\item If $g\cdot \grad \log p, \grad \cdot g \in \mathcal{L}^1(P)$, then $\TP g \in \mathcal{L}^1(P)$.
\item If $g, \TP g \in \mathcal{L}^1(P)$, then $P(\TP g) = 0$. 
\end{enumerate}
\end{proposition}
\begin{proof}
The first claim follows immediately from \cref{eq: Lang SO}.
For the second claim, 
\begin{align*}
\int (\TP g)(x) \; \mathrm{d}P(x) 
 = \int \frac{1}{p(x)} (\nabla \cdot (p g))(x) \; \mathrm{d}P(x)
 = \int (\nabla \cdot (p g))(x) \dx .
\end{align*}
Under our assumptions, the vector field $v = pg$ satisfies $v \in \mathcal{L}^1(\mathbb{R}^d)$ and $\nabla \cdot v \in \mathcal{L}^1(\mathbb{R}^d)$, so that we can evoke the divergence theorem on $\mathbb{R}^d$ (\Cref{thm: divergence}) to establish the result.
\end{proof}

Analogous Stein operators have been developed for distributions supported on smooth manifolds \citep{le2020diffusion,barp2022riemann}, where appropriate generalizations of the Langevin diffusion can be constructed.
This requires a manifold generalization of the divergence theorem; for a technical discussion of the mathematics involved, see \citet{pigola2014global}.

For our later discussion of Stein variational gradient in \Cref{sec: stein variational gradient} it will be useful to additionally introduce an operator that acts on scalar-valued functions $h$ to produce vector-valued functions $\AP h$, in such a manner that each component of $\AP h$ has expectation equal to zero under $P$.

\begin{definition}[Vector-valued Langevin Stein operator; \citealp{liu2016kernelized}]
In the setting of \Cref{ass: reg langevin}, the \emph{vector-valued Langevin Stein operator} for $P$ satisfies
\begin{align}\label{eq: Vec Lang SO}
(\AP h)(x) = (\nabla h)(x) + h(x) \cdot (\nabla \log p)(x)
\end{align}
for suitably regular scalar-valued functions $h: \reals^d \to \reals$ and all $x \in \reals^d$.
\end{definition}

We can derive a suitable domain for the vector-valued Langevin operator in much the same way we did for its scalar counterpart.

\begin{proposition}[Domain of vector-valued Langevin operator] \label{lem: vector Stein op int by parts}
In the setting of \Cref{ass: reg langevin}, suppose that $h : \reals^d \rightarrow \reals$ is differentiable.
The following claims hold true for $\AP$ as defined in \cref{eq: Vec Lang SO}.
\begin{enumerate}
\item If $h \cdot \grad \log p, \grad h \in \Lp[1](P)$, then $\AP h \in \Lp[1](P)$.
\item If $h, \AP h \in \Lp[1](P)$, then $P(\AP h) = 0$. 
\end{enumerate}
\end{proposition}
\begin{proof}
The first claim follows immediately from the definition \cref{eq: Vec Lang SO}. 
We will deduce the second result by reducing to the case of the scalar Langevin operator $\TP$ of \cref{def: Lang SO}.
Fix any coordinate $j\in \{1 , \dots , d\}$, and define the vector-valued function $g(x) = e_j h(x)$ where $e_j$ is the $j$-th standard basis vector in $\reals^d$.  
Then $(\AP h)_j = \TP g$, and $g$ satisfies the preconditions of \Cref{lem: Stein op int by parts}, hence $P((\AP h)_j) = P(\TP g) = 0$.
\end{proof}

\section{Diffusion Operators}
\label{sec: diffusion operators}

In the previous section we saw that the generator of a particular Markov process can be used to construct a Stein operator for a target distribution $P$ provided that the Markov process is $P$-invariant.
However, there are an infinite number of Markov processes which meet this requirement!
Depending on the context for which we require a Stein discrepancy, there is no reason to suppose that the choice of the Langevin Stein operator is optimal.
Fortunately, \citet{Gorham2019} demonstrated that it is possible to extend the construction in \Cref{sec: Lang Stein} to a broad class of Markov processes known as \emph{\Ito diffusions}. 

\begin{definition}[\Ito diffusion]
\label{def:diffusion}
A (time-homogeneous) \emph{\Ito diffusion} 
with starting point $x\in\reals^d$ and 
locally Lipschitz, linear-growth\footnote{The locally Lipschitz and linear-growth assumptions on $b$ and $\sigma$ ensure that the \ac{sde} \cref{eqn:diffusion} has a unique \citep[][Chap.~IV, Thm.~3.1]{ikeda2014stochastic} and non-explosive \citep[][Thm.~5.2.1]{oksendal2013stochastic}  solution.} \emph{drift} $b : \reals^d \to \reals^d$ and 
\emph{diffusion} $\sigma : \reals^d \to \reals^{d\times m}$ \emph{coefficients }
is a stochastic process $(\process{t}{x})_{t\geq0}$ solving the \Ito \ac{sde}
\begin{align}\label{eqn:diffusion}
\mathrm{d}Z_{t,x} = b(Z_{t,x}) \; \mathrm{d}t + \sigma(Z_{t,x}) \; \mathrm{d}W_t \qtext{with} Z_{0,x} = x, 
\end{align}
where $(W_t)_{t \geq 0}$ denotes a standard $d$-dimensional Wiener process on $\reals^d$.
\end{definition}

To exploit \Ito diffusions in a similar manner to how we exploited the Langevin diffusion, we require constraints on the drift coefficient $b$ and the diffusion coefficient $\sigma$ to ensure that the process is $P$-invariant.
The following result, stated in \citet[][Theorem 2]{Gorham2019} but distilled from \citet[][Theorem 2]{ma2015complete} and \citet[][Section 4.6]{pavliotis2016stochastic}, completely characterizes the set of \Ito diffusions that leave $P$ invariant.
Recall that, for a matrix-valued function $M(x)$, the notation $\nabla \cdot M$ refers to the vector with $j$th component $\nabla \cdot (M^\top e_j)$ where $e_j$ is the $j$th basis vector of $\mathbb{R}^d$.

\begin{theorem}[Complete recipe for $P$-invariant diffusions]
\label{thm:invariance}
In the setting of \Cref{ass: reg langevin}, consider the \Ito diffusion \cref{eqn:diffusion} with drift coefficient $b \in C^1(\reals^d,\reals^d)$ and diffusion coefficient $\sigma \in C^1(\reals^d,\reals^{d \times m})$, and define the \emph{covariance coefficient} $a(x) \defeq \sigma(x)\sigma(x)^\top$.
The diffusion is $P$-invariant if and only if
\begin{align}\label{eqn:complete_drift_f}
  b(x) = \frac{1}{2} \frac{1}{p(x)} (\nabla \cdot (pa))(x) + f(x) 
\end{align}
for a \emph{non-reversible component} $f \in C^1(\reals^d,\reals^d)$ satisfying $\nabla \cdot (pf) = 0$.
In addition, if $b$ satisfies \cref{eqn:complete_drift_f} with $f\in\Lp[1](P)$, then\footnote{To avoid ambiguity, let us emphasize that in what follows $p(a+c)$ denotes the function $x \mapsto p(x)(a(x) + c(x))$ and not ``the application of $p$ to $a + c$''.} 
\begin{align}
\label{eqn:drift}
  b(x) = \frac{1}{2} \frac{1}{p(x)} (\nabla \cdot (p(a+c)))(x)
\end{align}
for a differentiable $P$-integrable skew-symmetric $d\times d$ matrix-valued function $c$ termed the \emph{stream coefficient} \citep{conca2007periodic,landim1998convection}.
In this case, the infinitesimal generator of the diffusion takes the form
\begin{align*}
u \mapsto \frac{1}{2p} \nabla \cdot ( p (a+c) (\nabla u) ) 
\end{align*}
for all suitably regular $u \in C^2(\reals^d, \reals)$.
\end{theorem}

These observations motivated \citet{Gorham2019} to propose the following generalization of the Langevin Stein operator:

\begin{definition}[Diffusion Stein operator; {\citealp[(8)]{Gorham2019}}]
\label{def: diffusion operator}
In the setting of \Cref{ass: reg langevin}, suppose $a(x) \in \reals^{d\times d}$ is symmetric positive-semidefinite, and $c(x)$ is skew-symmetric, with $x \mapsto p(x)(a(x)+c(x))$ differentiable over $x \in \mathbb{R}^d$.
Then the \emph{$(a,c)$-diffusion Stein operator} for $P$ satisfies
\begin{align}
(\TP g)(x) = \frac{1}{p(x)} (\nabla \cdot (p(a+c)g))(x)  \label{eqn:diffusion-operator}
\end{align}
for all suitably regular vector-valued functions $g: \reals^d \to \reals^d$ and all $x \in \reals^d$.
\end{definition}

The choice of covariance coefficient $a$ and stream coefficient $c$ appearing in \cref{eqn:diffusion-operator} can be application-dependent, conferring a range of different properties to Stein discrepancies based on the diffusion Stein operator.
Several important special cases are generated by 
\begin{enumerate}
\item the \emph{preconditioned Langevin diffusion}, for which $c \equiv 0$ and $a=\sigma\sigma^\top$ for a constant diffusion coefficient $\sigma\in\reals^{d\times m}$  \citep{stuart2004conditional}; 
\item the \emph{Riemannian Langevin diffusion}, for which $c \equiv 0$ and $a(x) = G^{-1}(x)$, where $G$ is a positive-definite metric tensor \citep{patterson2013stochastic,xifara2014langevin,ma2015complete};
\item the \emph{non-reversible preconditioned Langevin diffusion}, for which $c \not\equiv 0$ and $a=\sigma\sigma^\top$ for a constant  $\sigma\in\reals^{d\times m}$ \citep{ma2015complete,duncan2016variance,rey2014irreversible}; and
\item the \emph{underdamped Langevin diffusion} \citep{horowitz1987second}, which targets the augmented distribution $P \times \Gsn(0,I)$ on $\reals^{2d}$ with 
\begin{align*}
c\equiv 2
\begin{pmatrix}
0 & -I \\
I & 0 
\end{pmatrix}
\qtext{and}
a\equiv 2
\begin{pmatrix}
0 & 0 \\
0 & I
\end{pmatrix}.
\end{align*}
\end{enumerate}
Further discussion is deferred to \Cref{subsec: sep for lang ksd}.

The following result, a generalization of \citet[][Proposition 3]{Gorham2019}, clarifies the minimum regularity we will require for the vector-valued functions $g : \reals^d \rightarrow \reals^d$.
Recall that, for $g : \mathbb{R}^d \rightarrow \mathbb{R}^d$, the notation $\nabla g : \mathbb{R}^d \rightarrow \mathbb{R}^{d \times d}$ is understood as $[\nabla g]_{i,j} = \partial_j g_i$, while for matrices $A$ and $B$, the double dot product is $A:B = \mathrm{tr}(AB^\top)$.
In particular, for a matrix-valued function $M(x)$, the product rule writes as $\nabla \cdot (Mg) = (\nabla \cdot M) \cdot g + M : \nabla g$.

\begin{proposition}[Domain of diffusion operator] \label{classical_mean_zero}
In the setting of \Cref{ass: reg langevin}, let $\TP$ be the $(a,c)$-diffusion Stein operator from \cref{eqn:diffusion-operator}, let $b$ be defined as in \cref{eqn:drift}, 
and let $g : \reals^d \rightarrow \reals^d$ be differentiable.
The following claims hold true
\begin{enumerate}
\item If $b \cdot g, (a+c) : \grad g \in \Lp[1](P)$, 
then $\TP g \in \Lp[1](P)$.
\item If $(a+c) g , \TP g \in \Lp[1](P)$, then $P(\TP g) = 0$.
\end{enumerate}
\end{proposition}
\begin{proof}
The proof is analogous to the proof of \Cref{lem: vector Stein op int by parts}.
For the first claim, since $g$ is differentiable with $ b \cdot g, (a+c) : \grad g \in \Lp[1](P)$, it follows immediately that 
\begin{align*}
\TP g
	=
2 b \cdot g 
	+
(a+c) : \grad g
	\in
\Lp[1](P).
\end{align*}
For the second claim, 
\begin{align*}
\int (\TP g)(x) \; \mathrm{d}P(x) 
 = \int (\nabla \cdot (p (a + c) g))(x) \dx .
\end{align*}
Under our assumptions, the vector field $v = p(a+c)g$ satisfies $v \in \mathcal{L}^1(\mathbb{R}^d)$ and $\nabla \cdot v \in \mathcal{L}^1(\mathbb{R}^d)$, so that we can again evoke \Cref{thm: divergence} to establish the result.

\end{proof}

Finally, we remark that since only the composition $\mathcal{T}_P g$ of a Stein operator and a function $g$ appears in Stein's method, there is some flexibility to simultaneously modify the Stein operator $\mathcal{T}_P$ and the functions $g$ while leaving the composition $\mathcal{T}_P g$ unchanged.
In this way one can consider the diffusion Stein operator as equivalent to applying the Langevin Stein operator (\Cref{def: Lang SO}) to functions of the form $(a+c)g$; this perspective enables the results that we present for Stein discrepancy based on the Langevin Stein operator in \Cref{chap: Stein discrepancies} to be transferred to Stein discrepancy based on the diffusion Stein operator.

\section{Constrained Operators}
\label{sec: mirrored operators}

So far, we have only introduced the Langevin and diffusion Stein operators only for distributions supported on $\reals^d$, but oftentimes we wish to consider distributions supported on subsets $\xset \subset \reals^d$.
Unfortunately, taking a Stein operator $\mathcal{T}_P$ designed for distributions on $\mathbb{R}^d$ and restricting it to $\mathcal{X}$ does not work; the condition $P(\mathcal{T}_P g) = 0$ in \cref{def: stein op} will be violated as the boundary term 
\begin{align*}
    \oint_{\partial \mathcal{X}} p(x) g(x) \cdot \mathrm{n}(x) \; \mathrm{d}x
\end{align*}
appearing in the divergence theorem will not vanish in general.
Recall that $\partial\mathcal{X}$ denotes the boundary of the set $\mathcal{X}$, and that by definition $\partial \mathcal{X}$ is empty when $\mathcal{X} = \mathbb{R}^d$.
There is then a choice for how to proceed.

One option is to restrict attention to functions $g$ chosen such that the boundary term is zero by construction; this can be achieved by requiring the vanishing on the boundary condition \citep{gorham2015measuring}
\begin{align}\label{eq:vanishing-boundary-condition}
g(x) \cdot \mathrm{n}(x) = 0
\qtext{for all}
x \in \partial\mathcal{X}.
\end{align}
This construction is compatible with the classical Stein discrepancies defined in \Cref{sec: classical sd}, with the graph Stein discrepancies defined in \cref{sec: graph sd}, and with the kernel Stein discrepancies defined in \cref{sec: ksds}.

A second option is to employ a Stein operator adapted to $\mathcal{X}$.
The \emph{mirror Langevin diffusion} \citep{zhang2020wasserstein,chewi2020exponential}, 
\begin{align}
Z_{t,x} = \grad \psi^*(\eta_{t,x}),
\quad
\mathrm{d}\eta_{t,x} &= \grad \log p(Z_{t,x}) \mathrm{d}t  + \sqrt{2} \Hess\psi(Z_{t,x})^{1/2} \mathrm{d}B_t \notag\\
\label{eqn:mirror_diffusion}
\qtext{for}
\eta_{0,x} &= \grad \psi(x), 
\end{align}
provides an elegant example of a $P$-invariant Markov process that can be used to construct a Stein operator in situations where the domain $\xset$ is convex and closed. 
Here a \emph{mirror map} $\grad \psi$ is used to transform an input variable into a \emph{mirror variable} $\eta_{0,x}$, 
the mirror variables $\eta_{t,x}$ evolve according to the \ac{sde} \cref{eqn:mirror_diffusion}, 
and finally $\eta_{t,x}$ is mapped backed into the original constrained sample space $\xset$ via the inverse mirror map $\grad \psi^*$.

To be clear, we are not restricting attention to domains $\xset$ that are bounded.
A simple example of a convex and closed set $\xset$ that is not bounded is $\{x \in \reals^d : x_1,\dots,x_d \geq 0\}$.
For a closed convex set whose boundary $\partial \xset$ can be locally represented as $F(x) = 0$, its outward unit normal vector satisfies
$$
n(x) \in \left\{\pm \frac{\nabla F(x)}{\twonorm{\nabla F(x)}} \right\}.
$$

\citet{zhang2020wasserstein} highlighted that the mirror diffusion is an instance of the Riemannian Langevin diffusion with metric tensor $G = \nabla^2 \psi$, for a twice differentiable function $\psi : \xset \rightarrow \reals \cup \{\infty\}$, which we require to have 
\begin{enumerate}
    \item $(\nabla^2 \psi)^{-1}$ Lipschitz and differentiable on $\xset$, 
    \item $\twonorm{(\nabla \psi)(x_n)} \rightarrow \infty$ whenever $x_n \rightarrow x \in \partial \xset$, and
    \item for some $c > 0$, $\Hess \psi(x) \psdge c I \qtext{for all} x \in \xset.$ %
\end{enumerate}

\begin{assumption}[Regularity of $P$]
\label{ass: reg mirror}
    Let $P$ be a probability distribution on a convex and closed set $\xset \subset \reals^d$ with finite mean and positive differentiable density $p$ on $\xset$. %
\end{assumption}

Consideration of the infinitesimal generator of mirror Langevin diffusion led \citet{shi2022sampling} to propose the following Stein operator:

\begin{definition}[Mirrored Stein operator; \citealp{shi2022sampling}]
\label{def: mirrored stein operator}
    Consider a convex and closed set $\xset \subset \reals^d$ and let $\psi$ be a function with the properties just described.
    In the setting of \Cref{ass: reg mirror}, the \emph{mirrored Stein operator} for $P$ satisfies
    \begin{align} \label{eq:mirror-stein-op}
	(\mathcal{T}_{P,\psi} g)(x) & = g(x) \cdot (\nabla^2\psi)(x)^{-1} (\nabla \log p)(x) \\
    & \hspace{50pt} + \nabla\cdot ((\nabla^2\psi)(x)^{-1}g(x)), \nonumber
    \end{align}    
    for suitably regular vector-valued functions $g: \xset \to \reals^d$ and all $x \in \xset$.
\end{definition}

Since the mirrored Stein operator is a special case of a diffusion Stein operator, the domain of the mirrored Stein operator can be deduced following a similar argument to \Cref{classical_mean_zero}  \citep[see also][Prop.~1]{shi2022sampling}.
The only real point of distinction here is that care must be taken to ensure that $(\nabla^2\psi)^{-1}$ cancels any growth of $p$ on the boundary of $\xset$.

Recall that, for a matrix-valued function $M(x)$, the notation $\nabla \cdot M$ refers to the vector with $j$th component $\nabla \cdot (M^\top e_j)$ where $e_j$ is the $j$th basis vector of $\mathbb{R}^d$.

\begin{proposition}[Domain of mirrored operator] \label{prop:mirror-stein-identity}
    In the setting of \Cref{ass: reg mirror}, suppose that  $g : \mathcal{X} \to \reals^d$ is differentiable and $\psi : \xset \rightarrow \reals \cup \{\infty\}$ is sufficiently regular for the following quantities to exist.
    Then the following claims hold:
\begin{enumerate}
\item If $g \cdot (\nabla^2\psi)^{-1} (\nabla \log p)$, $\nabla\cdot ((\nabla^2\psi)^{-1}g)  \in \Lp[1](P)$ then $\TPpsi g \in \Lp[1](P)$.
\item %
If $\TPpsi g \in \Lp[1](P)$, $v\defeq p(\nabla^2\psi)^{-1}g$ is differentiable, and
\begin{align}
\lim_{n \rightarrow \infty} \oint_{\partial \mathcal{X}_n} v(x) \cdot \normal(x) \dx = 0  \label{eq: vanishing intergal mirror}
\end{align}
where $\mathcal{X}_n \defeq  \{x \in \mathcal{X} : \|x\|_\infty \leq n\}$, then $P(\TP g) = 0$.
\end{enumerate}
\end{proposition}
\begin{proof}
    The first claim follows immediately from \cref{eq:mirror-stein-op}.
    The second follows from the divergence theorem for convex sets (\cref{thm:convex_divergence}) and \cref{eq: vanishing intergal mirror} since  $\nabla \cdot v \in \mathcal{L}(\mathcal{X})$. 
\end{proof}

As an illustration, consider the probability simplex $\xset = \{x \in [0,1]^d : \sum_{i=1}^d x_i \leq 1\}$, where for example we may be interested in a Dirichlet target $P$.
Taking $\psi$ to be the negative entropy $\psi(x) = \sum_{i=1}^{d+1} x_i \log x_i$ with $x_{d+1} \defeq  1 - \sum_{i=1}^d x_i$, the mirror Langevin diffusion coincides with the \emph{Wright--Fisher diffusion} \citep{ethier1976class}.

\section{Gradient-Free Operators}
\label{sec: grad free operators}

The previous examples of Stein operators were based on the idea that the gradient $\nabla \log p$, and certain linear transformations of this gradient, integrate to zero under $P$.
However, in some applications the calculation of gradients is associated with a prohibitive computational cost, or gradients may even fail to exist.
In such scenarios we can adapt ideas from importance sampling to leverage instead gradients of a tractable approximating distribution $\Pi$.

\begin{assumption}[Regularity of $P$]
\label{ass: reg grad free}
    Let $P$ and $\Pi$ be probability distributions on $\reals^d$.
    Let $P$ admit a positive density $p$ and let $\Pi$ admit a positive and differentiable density $\pi$ on $\reals^d$.
\end{assumption}

The following Stein operator is termed \emph{gradient-free} because gradients with respect to the density function of the target distribution $P$ are not involved:

\begin{definition}[Gradient-free Stein operator; \citealp{han2018stein}]
\label{def: gf stein op}
    In the setting of \Cref{ass: reg grad free}, the \emph{gradient-free Langevin Stein operator} for $P$, based on $\Pi$, satisfies
    \begin{align}
        (\mathcal{T}_{P,\Pi} \, g)(x) = \frac{\pi(x)}{p(x)} \left[ (\nabla \cdot g)(x) + g(x) \cdot (\nabla \log \pi)(x) \right] \label{eq: gf op}
    \end{align}
    for suitably regular vector-valued functions $g : \reals^d \rightarrow \reals^d$ and all $x \in \reals^d$.
\end{definition}

\noindent If $P$ and $\Pi$ are equal then the gradient-free Stein operator coincides with the Langevin Stein operator $\TP$ in \cref{eqn:diffusion-operator}.
If $P$ and $\Pi$ are not equal then it is straightforward to see that $P(\mathcal{T}_{P,\Pi} \, g) = 0$ if and only if $\Pi(\mathcal{T}_{\Pi} g) = 0$.
Thus an analogous argument leads to the following result:

\begin{proposition}[Domain of gradient-free operator] \label{cor: domain GF}
In the setting of \Cref{ass: reg grad free}, suppose that $g : \reals^d \rightarrow \reals^d$ is differentiable.
The following claims hold true for $\mathcal{T}_{P,\Pi}$ as defined in \cref{eq: gf op}.
\begin{enumerate}
\item If $g\cdot \grad \log \pi, \grad \cdot g \in \mathcal{L}^1(\Pi)$, then $\mathcal{T}_{P,\Pi} g \in \mathcal{L}^1(P)$.
\item If $g, \mathcal{T}_{P,\Pi} g \in \mathcal{L}^1(P)$, then $P(\mathcal{T}_{P,\Pi} g) = 0$. 
\end{enumerate}
\end{proposition}
\begin{proof}
The first claim follows immediately from \cref{eq: gf op}.
For the second claim, 
\begin{align*}
\int (\mathcal{T}_{P,\Pi} g)(x) \; \mathrm{d}P(x) 
 = \int (\nabla \cdot (\pi g))(x) \dx .
\end{align*}
Under our assumptions, the vector field $v = \pi g$ satisfies $v \in \mathcal{L}^1(\mathbb{R}^d)$ and $\nabla \cdot v \in \mathcal{L}^1(\mathbb{R}^d)$, so that we can evoke \Cref{thm: divergence} to establish the result.
\end{proof}

The gradient-free Stein operator was first introduced in the context of variational sampling algorithms \citep{han2018stein}, which will be discussed in \Cref{subsec: batch}.
The operator can also be viewed as a diffusion Stein operator in the special case where $a + c = \pi / p$ in \cref{eqn:diffusion-operator}.

\section{Discrete Operators and Beyond}
\label{sec: discrete operators}

The final part of this Chapter is mainly devoted to the case where the target $P$ has \emph{discrete} support $\xset$, meaning that $\xset$ is a countable set, such as encountered in the analysis of count data, text data, network data, and so forth. 
However, we will also indicate how these ideas can be extended to more general state spaces, such as manifolds or infinite-dimensional state spaces, in \Cref{subsec: operators from Markov}.

In the discrete setting, $P$ can be represented by the values $p(x)$ of its probability mass function (\Cref{def: pmf}).

\begin{assumption}[Regularity of $P$]
\label{ass: discrete}
    Let $P$ be a probability distribution with positive probability mass function $p$ on a countable set $\xset$.
\end{assumption}

As observed in \citet{henderson1997variance,shi2022gradient}, the generator method of \citet{barbour1988stein,barbour1990stein,gotze1991rate} can also be applied to discrete state spaces. 
The generator $A$ of a (time-homogeneous) Markov process on a discrete state space $\xset$ (\Cref{def: generator}) can be represented as a transition rate matrix $Q$ with elements
\begin{align*}
    Q_{x,y} = (A f_y)(x), \qquad f_y(x) = \left\{ \begin{array}{ll} 1 & x = y \\ 0 & x \neq y \end{array} \right. ,
\end{align*}
describing the infinitesimal rate at which the process transitions from state $x$ to state $y$, with $x,y \in \xset$.
Several examples will now be presented.

\subsection{Zanella Operator}

The first construction we explore attempts to encode sparsity into the Stein operator, which can simplify its evaluation in the case of a large or uncountable set $\xset$.
To this end, endow $\xset$ with an undirected graph structure, using $\nbr \subset \xset$ to denote the neighboring vertices of $x \in \xset$ and requiring $x \in \nbr[y]$ if $y \in \nbr$.
The edge set of this graph will be denoted $\mathcal{E} = \{(x,y) : x \in \xset, y \in \nbr \} \subset \xset \times \xset$.
The \emph{Zanella process} \citep{zanella2019informed} has a transition matrix of the form
\begin{align} \label{eq:zanella-generator}
    Q_{x,y} = \left\{ \begin{array}{ll} \kappa\left(\frac{p(y)}{p(x)}\right) & y \in \nbr, y \neq x \\
    -\sum_{z \neq x} Q_{x,z} & y = x \end{array} \right. ,
\end{align}
with a \emph{balancing function} $\kappa$, meaning a continuous function $\kappa : (0,\infty) \rightarrow (0,\infty)$ with the \emph{balancing property} that $\kappa(t) = t\kappa(1/t)$ for all $t \in (0,\infty)$.
The role of the neighborhood structure is to impose sparsity into the process, while the balancing function ensures that \emph{detailed balance} is satisfied so that the process is $P$-invariant.
The associated Stein operator was termed the \emph{Zanella Stein operator} in \citet{hodgkinson2020reproducing,shi2022gradient}:

\begin{definition}[Zanella operator] \label{def: Zanella SO}
Fix a balancing function $\kappa : (0,\infty) \rightarrow (0,\infty)$.
In the setting of \Cref{ass: discrete}, the \emph{Zanella Stein operator} for $P$ satisfies
\begin{align*}
    (\TP g)(x) = \sum_{y \in \nbr, y\neq x} \kappa\left(\frac{p(y)}{p(x)}\right) (g(y) - g(x)) 
\end{align*}
for all sufficiently regular scalar-valued functions $g : \xset \rightarrow \reals$ and all $x \in \xset$.
\end{definition}

Compared to the Stein operators we have seen so far for continuous domains, the Zanella Stein operator depends on $p$ through ratios $p(x) / p(y)$ instead of through its gradient.
Since $p$ appears only as a ratio, the Zanella Stein operator remains compatible with the situation where $p$ is specified up to an intractable normalization constant, as discussed at the start of this Chapter. 
Special cases of the Zanella Stein operator include the \emph{minimum probability flow Stein operator} ($\kappa(t) = \sqrt{t}$; \citealp{barp2019minimum}) based on the minimum probability flow process of \citet{sohldickstein2011minimumprobabilityflowlearning} and the \emph{Barker Stein operator} ($\kappa(t) = \frac{t}{t+1}$; \citealp{hodgkinson2020reproducing}) based on the proposal of \citet{barker1965monte}.

\begin{proposition}[Domain of Zanella operator]
    In the setting of \Cref{ass: discrete}, let 
    $$
    x \mapsto \sum_{y \in \nbr} \kappa\left( \frac{p(y)}{p(x)} \right) g(y)
    $$ 
    be in $\Lp[1](P)$.
    Then $P(\TP g) = 0$.
\end{proposition}
\begin{proof}
The balancing property implies that $p(x) \kappa\left(\frac{p(y)}{p(x)}\right) = p(y) \kappa\left(\frac{p(x)}{p(y)}\right)$ for all $x,y \in \xset$, which are known as the \emph{detailed balance} equations.
Using these detailed balance equations and using the assumed integrability to rearrange the absolutely continuous series,
    \begin{align*}
        P(\TP g) & = \sum_{x \in \xset} p(x) (\TP g)(x) \\
        & = \sum_{x \in \xset} \sum_{y \in \nbr} p(x) \kappa\left(\frac{p(y)}{p(x)}\right) (g(y) - g(x)) \\
        & = \sum_{(x,y) \in \mathcal{E}} \left[ p(x) \kappa\left(\frac{p(y)}{p(x)}\right) g(y) - p(y) \kappa\left(\frac{p(x)}{p(y)}\right) g(x) \right] = 0 ,
    \end{align*}
    as required.
\end{proof}

\subsection{Birth-Death Stein Operators}

The intuition that we gained from the continuous case can also be brought to bear on the discrete case by endowing a discrete domain $\xset$ with an ordering.
For simplicity here we assume that 
\begin{align}
\xset \cong \{0,1,\dots,m_1-1\} \times \cdots \times \{0,1,\dots,m_d - 1\} , \label{eq: discrete X}
\end{align}
meaning that the set $\xset$ has cardinality $m_1 \cdots m_d$ and its elements can be mapped onto a Cartesian grid.
Then define the \emph{increment} and \emph{decrement} operators $\mathrm{inc}_i : \xset \rightarrow \xset$ and $\mathrm{dec}_i : \xset \rightarrow \xset$, which take an input $x \in \xset$ and replace its $i$th coordinate with, respectively, $x_i + 1$ modulo $m_i$ and $x_i - 1$ modulo $m_i$.
Define the \emph{birth rates} $b_{i,x} = p(\mathrm{inc}_i(x)) / p(x)$ and \emph{death rates} $d_{i,x} = 1$ noting that, as for the Zanella process, the dependence on the probability mass function $p$ occurs only through a ratio, mitigating the need to obtain an explicit normalizing constant.
The \emph{birth-death} process on $\xset$ \citep{karlin1957classification} is a continuous time Markov process defined by the transition rate matrix 
\begin{align*}
	Q_{x,y} = \frac{1}{d}\sum_{i=1}^d
	b_{i,x} \indic{y = \mathrm{inc}_i(x)} 
	+ d_{i,x} \indic{y = \mathrm{dec}_i(x)} 
	- (b_{i,x}+d_{i,x}) \indic{y = x},
\end{align*}
where increment events are termed "births" and decrement events are termed "deaths" in this context.
 The generator of the birth-death process gives rise the the \emph{birth-death Stein operator}, studied by authors including \citet{brown2001stein,holmes2004stein,eichelsbacher2008stein,hodgkinson2020reproducing,shi2022gradient}:

\begin{definition}[Birth-death operator]
In the setting of \Cref{ass: discrete} and \cref{eq: discrete X}, let the birth rates $b_{i,x}$ and death rates $d_{i,x}$ be as previously defined.
 The \emph{birth-death Stein operator} for $P$ satisfies
 \begin{align*}
     (\TP g)(x) = \frac{1}{d}\sum_{i=1}^d  b_{i,x}(g_i(\mathrm{inc}_i(x)) - g_i(x)) -  d_{i,x}(g_i(x) - g_i(\mathrm{dec}_i(x)))
 \end{align*}
 for all vector-valued functions $g : \xset \rightarrow \reals^d$ and all $x \in \xset$.
\end{definition}

All functions on $\xset$ belong to the domain of the birth-death Stein operator due to the finiteness of the state space, as explained in the following result:

\begin{proposition}[Domain of birth-death operator]
    In the setting of \Cref{ass: discrete} and \cref{eq: discrete X}, let $g : \xset \rightarrow \reals^d$.
    Then $P(\TP g) = 0$.
\end{proposition}
\begin{proof}
From the change of variable $x \mapsto \mathrm{dec}_i(x)$, we obtain for each $i \in \{1 , \dots , d\}$ that
\begin{align*}
    \sum_{x \in \xset} p(x) b_{i,x} (g_i(\mathrm{inc}_i(x)) - g_i(x)) & = \sum_{x \in \xset} p(\mathrm{inc}_i(x)) (g_i(\mathrm{inc}_i(x)) - g_i(x)) \\
    & = \sum_{x \in \xset} p(x) (g_i(x) - g_i(\mathrm{dec}_i(x)))
\end{align*}    
from which the result is established.
\end{proof}

The birth-death Stein operator is closely related to the \emph{difference Stein operator}, which for $h : \mathcal{X} \rightarrow \mathbb{R}$ is defined as
\begin{align*}
    [(\gset_P h)(x)]_i = h(\mathrm{dec}_i(x)) - b_{i,x}  h(x) ,
\end{align*}
studied in \citet{yang2018goodness}.
Indeed, if we let $h_i(x) = g_i(x) - g_i(\mathrm{inc}_i(x))$ then we have the relation $(\mathcal{T}_P g)(x) = \frac{1}{d} \sum_{i=1}^d [(\gset_P h_i)(x)]_i$, so that the birth-death Stein operator can be thought of as a coordinate-wise application of the difference Stein operator.
The extension to the case where $\xset$ is countably infinite is discussed in \citet{hodgkinson2020reproducing,matsubara2023generalized}.

\subsection{Operators from Discrete Time Markov Chains}
\label{subsec: operators from Markov}

The generator method for constructing Stein operators is convenient because, informally speaking, computing quantities in the infinitesimal $t \rightarrow 0$ limit does not require integrals to be computed.
Indeed, we recall from \Cref{def: invariant} that any $P$-invariant Markov transition kernel $K_t$ gives rise to a Stein operator
\begin{align}
    (\TP g)(x) = \int g(y) K_t(x, \mathrm{d}y) - g(x) , \label{eq: MS operator}
\end{align}
but this construction is not practical as the integral over the state $y$ of the Markov process at $t > 0$ units of time after initialization at $x$ cannot typically be computed.
However, in the case of a discrete time Markov process and discrete domains exhibiting an appropriate sparsity structure, such integrals may amount to finite sums that can be computed.
Indeed, the one step transition probabilities of a discrete time Markov process on a discrete space $\xset$ can be characterized by a \emph{transition matrix} $K$, whose entries $K_{x,y}$ represent the probability of moving to state $y$ if the process is initialized in state $x$, for $x,y \in \xset$.
Then \cref{eq: MS operator} becomes what we term in this book a \emph{Markov chain Stein operator}:

\begin{definition}[Markov chain Stein operator]
In the setting of \Cref{ass: discrete}, let $K$ be the transition matrix of a discrete time Markov process for which $P$ is invariant.
The associated \emph{Markov chain Stein operator} for $P$ satisfies
    \begin{align*}
    (\TP g)(x) = \left( \sum_{y \in \xset} g(y) K_{x,y} \right) - g(x) ,
\end{align*}    
for all sufficiently regular scalar-valued functions $g : \xset \rightarrow \reals$ and all $x \in \xset$.
\end{definition}

The Markov chain Stein operator can be efficiently computed when the number of non-zero entries in each row of the transition matrix is small.
A simple characterization of the domain of the Markov chain operator can be obtained under the assumption that the discrete time Markov process is reversible, meaning that the detailed balance equations are satisfied:

\begin{proposition}[Domain of Markov chain operator]
In the setting of \Cref{ass: discrete}, let the discrete time Markov process be $P$-invariant, meaning that $\sum_{x \in \mathcal{X}} p(x) K_{x,y} = p(y)$ for all $y \in \xset$, and let $g \in \Lp[1](P)$.
    Then $P(\TP g) = 0$.
\end{proposition}
\begin{proof}
    From $P$-invariance,
    \begin{align*}
        P(\TP g) & = \sum_{x \in \xset} p(x) \sum_{y \in \xset} g(y) K_{x,y} - \sum_{x \in \xset} p(x) g(x) \\
        & = \sum_{y \in \xset} g(y) p(y) - \sum_{x \in \xset} p(x) g(x) = 0 ,
    \end{align*}
    where rearrangement of the sum is justified by its absolute convergence under the integrability that we assumed.
\end{proof}

Consider for instance the \emph{random scan Gibbs sampler}, which is a discrete time Markov chain on a $d$-dimensional product space $\xset$ defined by the transition matrix
\begin{align*}
	K_{x,y} = \frac{1}{d}\sum_{i=1}^d p(y_i|x_{-i})\indic{y_{-i} = x_{-i}}, 
\end{align*}
where $p(y_i | x_{-i})$ denotes the probability that $X_i = y_i$ given that $X_j = x_j$ for all $j \neq i$ under the law $X \sim P$.
These dynamics induce the \emph{Gibbs Stein operator} 
\begin{align}
    (\TP g)(x) = \frac{1}{d} \sum_{i=1}^d \left( \sum_{y \in \xset : y_{-i}=x_{-i} } q(y_i|x_{-i}) g(y) \right) - g(x) \label{eq: Gibbs stein operator}
\end{align}    
as studied in \citet{bresler2019stein,reinert2019approximating,shi2022gradient}.

\subsection{General Stein Operators}
\label{subsec: operators in general}

The focus of our exposition has been for probability distributions $P$ supported on subsets $\mathcal{X} \subseteq \mathbb{R}^d$, but for certain applications a generalization beyond finite-dimensional Euclidean space is required.
Though we do not attempt to discuss these generalizations of Stein operators in detail, we do wish to briefly mention two directions in which the Langevin Stein operator can be generalized.

The case where $P$ is supported on a manifold has been considered by several authors; here one can proceed using generalizations of the $P$-invariant Langevin diffusion adapted to the manifold \citep{le2020diffusion,xu2021interpretable,barp2022riemann,qu2025theory}.
Such situations arise for example in spherical data analysis, or inference for a matrix-valued parameter subject to a nonlinear constraint (e.g. positive definiteness).

The case where random variables $X \sim P$ are infinite-dimensional objects is commonly encountered in functional data analysis, where $P$ is supported on an appropriate space of functions.
Simply taking $d \rightarrow \infty$ in any of the constructions that we have discussed does not work, as one cannot interchange limits in general.
Instead, \citet{wynne2025fourier} developed appropriate mathematical structure for an analogue of the Langevin Stein operator to be defined in the infinite-dimensional context.
The key observation in that work is that $P$ is absolutely continuous with respect to a Gaussian measure; the structure of the Gaussian reference measure is used to ensure integrability in the infinite-dimensional context.

\bigskip

This completes our discussion of Stein operators, and sets the scene for the next Chapter where we will construct Stein discrepancies by combining a Stein operator with a Stein set.
The decision of which Stein operator to use for a particular application depends in part on the computational requirements of each Stein operator in a given applied context, but also on the properties that each Stein operator confers to the associated Stein discrepancy, and we therefore postpone further discussion of the choice of Stein operator to \Cref{chap: Stein discrepancies}.

\chapter{Stein Discrepancies}
\label{chap: Stein discrepancies}

A \emph{Stein discrepancy} is a special type of statistical divergence $D : \mathcal{P} \times \mathcal{P} \rightarrow [0,\infty]$ (\Cref{def: divergence}), which takes the form of an \emph{integral probability pseudo-metric} \citep{muller1997integral}
\begin{align}
    D(P,Q) = \sup_{f \in \mathcal{F}} \; | P(f) - Q(f) | \label{eq: ipm}
\end{align}
where the test set $\mathcal{F}$ is designed to avoid explicit computation of integrals $P(f)$ with respect to the target distribution $P$.
This is achieved using a Stein operator $\mathcal{T}_P$ (\Cref{def: stein op}) to generate test functions $f = \mathcal{T}_P g$ that have zero mean under $P$, so that $P(f) = 0$ can be trivially computed.
Despite the relatively recent introduction of the concept of Stein discrepancy in \citet{gorham2015measuring}, there are now myriad applications where Stein discrepancies are used; we defer all discussion to Chapter \ref{sec:ksd_application}.
The aims of the present Chapter are to formally define the concept of a Stein discrepancy, to introduce specific examples of Stein discrepancies and to explain how they can be computed, and to discuss the properties of Stein discrepancies, focusing in particular on concepts called \emph{convergence detection} and \emph{convergence control}. 

First a rigorous definition of Stein discrepancy will be presented.
For a measurable space $\mathcal{X}$, let $\mathcal{P}_{\mathcal{X}}$ denote the set of probability measures on $\mathcal{X}$, and,
for an operator $\mathcal{T}$ on a set $\mathcal{G}$, define $\mathcal{T} \mathcal{G} \defeq  \{\mathcal{T} g : g \in \mathcal{G}\}$.
The following definition is based on \citet[Sec.~3]{gorham2015measuring} and \citet[Def.~2]{barp2022targeted} 
and uses the shorthand $\mathcal{L}^1_+(Q)$ for the set of measurable functions $f$ with $Q$-integrable positive part $(f)_+ \defeq  \max(f, 0)$. %

\begin{definition}[Stein discrepancy] \label{def:stein_discrepancy}
Let $P$ be a probability distribution on a measurable space $\mathcal{X}$.
The \emph{Stein discrepancy} with Stein set $\mathcal{G}$ and Stein operator $\mathcal{T}_P : \mathcal{G} \rightarrow \mathcal{L}^1(P)$  (\cref{def: stein op}) is the map 
$\mathcal{S}(\cdot,\mathcal{T}_P,\mathcal{G}): \mathcal{P}_{\mathcal{X}}  \to  [0,\infty]$ with, for each distribution $Q\in \mathcal{P}_{\mathcal{X}}$, 
\begin{align}
\mathcal{S}(Q,\mathcal{T}_P,\mathcal{G})  \defeq 
\sup_{f\,\in\, (\mathcal{T}_P\mathcal{G})\,\cap\,\mathcal{L}^1_+(Q)} | Q(f) |.
\label{eq: SD def}
\end{align}
\end{definition}

\noindent Comparing \eqref{eq: SD def} to \eqref{eq: ipm}, we see that a Stein discrepancy has a test set $\mathcal{F} = (\mathcal{T}_P \mathcal{G}) \cap \mathcal{L}^1_+(Q)$ that explicitly depends on both the target distribution $P$ and the candidate distribution $Q$ being assessed.
By design, the Stein operator and Stein set ensure that every test function $f\in\mathcal{T}_P \mathcal{G}$ is $P$-integrable with $P(f) = 0$. 
However, not every $f\in\mathcal{T}_P \mathcal{G}$ need be integrable under $Q$, so the additional requirement $f \in \mathcal{L}^1_+(Q)$ ensures that the integral $Q(f)$ is well defined.

The specification of the Stein operator $\mathcal{T}_P$ and Stein set $\gset$ appearing in \Cref{def:stein_discrepancy} impacts both the mathematical properties of the Stein discrepancy and the ease with which the Stein discrepancy can be computed.
The remainder of this Chapter is devoted to a discussion of this issue, and several specific examples of Stein discrepancies will be presented and discussed in detail.

\section{Fisher Divergence}
\label{sec: Fisher divergence}

Here we open our discussion with the Fisher divergence, a well-known statistical divergence which we will see is in fact a special instance of the Stein discrepancy formalism that we just introduced.

\begin{definition}[Fisher divergence] \label{def: fisher div}
    Let $P$ and $Q$ be (Borel) distributions on $\mathbb{R}^d$ admitting densities $p$ and $q$ such that $\nabla \log p , \nabla \log q \in \mathcal{L}^2(Q)$.
    Then the \emph{Fisher divergence} is defined as
    \begin{align*}
        \mathrm{FD}(Q||P) \defeq \int \| \nabla \log p - \nabla \log q \|_2^2 \; \mathrm{d}Q .
    \end{align*}
\end{definition}

\noindent This divergence can be computed without knowledge of the normalizing constant of $p$ and, furthermore, expectations with respect to $P$ are not required.
This makes it a popular choice for estimating parameters of intractable statistical models, where the methodology is called \emph{score matching}. 
Indeed, suppose that $P$ involves parameters $\theta$, denoted $P_\theta$, with corresponding density $p_\theta$. 
Score matching refers to selecting parameter values $\theta$ that minimize $\mathrm{FD}(Q||P_\theta)$.
Assuming that $\Delta \log p_\theta \in \mathcal{L}^1(Q)$, \citet{hyvarinen2005estimation} showed using integration by parts that
\begin{align}
    \mathrm{FD}(Q||P_\theta) = [\text{constant in $\theta$}] + Q( 2 \Delta \log p_\theta + \|\nabla \log p_\theta\|^2 ) . \label{eq: DF2} 
\end{align}
It is straightforward to obtain consistent Monte Carlo approximations of the integral appearing in \eqref{eq: DF2} if samples from $Q$ can be obtained.
This makes score matching a widely-applicable technique for statistical estimation that benefits from many of the desirable characteristics that we introduced as motivation in \Cref{chap: intro}.
Further, several generalizations of Fisher divergence are possible, for instance to discrete domains \citep{matsubara2023generalized}.

However, there are two main limitations of Fisher divergence which motivate the subsequent discussion of other Stein discrepancies in this book.
First, a practical drawback of the Fisher divergence relative to most of the Stein discrepancies considered in this book, is that second order derivatives of $p$ (or $p_\theta$) are required to use \eqref{eq: DF2}.
For complex data-generating distributions $P$. such as those based on solving systems of physical governing equations, it might be possible to obtain first order derivatives at additional computational expense, but obtaining second order derivatives is almost always prohibitively difficult.
A second potential drawback, depending on the applied context, is that Fisher divergence is stronger than most other statistical divergences, including Kullback--Leibler (\Cref{def: KL divergence}) and the kernel Stein discrepancies that we will meet in \Cref{sec: ksds}; for further discussion of this point see respectively \citet{ley2013stein} and \citet[][Thm~5.1]{Liu2016}.
The most obvious drawback with using a strong notion of divergence is that it may be infinite when comparing two distributions of genuine interest, limiting its usefulness only to a small subset of probability distributions whose regularity is sufficient to be quantitatively compared.
Nevertheless, we open our discussion of Stein discrepancies by demonstrating that the Fisher divergence can be recovered as an instance of the general Stein discrepancy framework, albeit with a $Q$-dependent Stein set:

\begin{proposition}[Fisher divergence as Stein discrepancy] \label{prop: fisher as stein}
    Let $P$ and $Q$ be (Borel) distributions on $\mathbb{R}^d$ admitting densities $p$ and $q$ such that $\nabla \log p , \nabla \log q \in \mathcal{L}^2(Q)$ and $\Delta \log p, \Delta \log q \in \mathcal{L}^1(Q)$.
    Then $\mathrm{FD}(Q||P) = \mathcal{S}(Q,\mathcal{T}_P,\mathcal{G}_Q)$ where the Stein discrepancy is based on the Langevin Stein operator $(\mathcal{T}_P g)(x) = (\nabla \cdot g)(x) + g(x) \cdot (\nabla \log p)(x)$ from \Cref{eq: Lang SO}, and the Stein set $\mathcal{G}_Q = \{g : \sum_{j=1}^d |g_j|_{\mathcal{L}^2(Q)}^2 \leq 1 , \; \mathcal{T}_Pg, \mathcal{T}_Q g \in \mathcal{L}^1(Q)\}$.
\end{proposition}
\begin{proof}
For the Langevin Stein operator and $g \in \mathcal{G}$,
\begin{align*}
    Q(\mathcal{T}_P g) 
    & = Q(\mathcal{T}_P g) - Q(\mathcal{T}_Q g) \\
    & = Q(\mathcal{T}_P g - \mathcal{T}_Q g)
    = Q( g \cdot (\nabla \log p - \nabla \log q) ) ,
\end{align*}
where the first equality follows since $g$ is in the domain of $\mathcal{T}_Q$ from \Cref{lem: Stein op int by parts}.
Then, from the duality structure of $\mathcal{L}^2(P)$,
\begin{align*}
    \mathrm{FD}(Q||P) & = \sup \left\{ \sum_{i=1}^d \int g_i (\nabla \log p - \nabla \log q)_i \; \mathrm{d}Q \; : \; \sum_{i=1}^d |g_i|_{\mathcal{L}^2(Q)}^2 \leq 1 \right\} 
\end{align*}
where the supremum is attained at $g^\star \propto \nabla \log p - \nabla \log q$, since our assumptions ensure that $g^\star \in \mathcal{L}^2(Q)$.
Our assumptions further imply that $g^\star \in \mathcal{G}$ and thus
\begin{align*}    
    \mathrm{FD}(Q||P) & = \sup \left\{ \sum_{i=1}^d \int g_i (\nabla \log p - \nabla \log q)_i \; \mathrm{d}Q \; : \; g \in \mathcal{G} \right\} \\
    & = \sup_{f \in \mathcal{T}_P(\mathcal{G})} Q(f) 
\end{align*}
which is identical to the Stein discrepancy \eqref{eq: SD def} since by construction $\mathcal{T}_P (\mathcal{G}) \subset \mathcal{L}^1(Q)$, and the modulus in \eqref{eq: SD def} can be dropped since $\mathcal{T}_P (\mathcal{G})$ is a symmetric set.
\end{proof}

The asymmetric nature of the Stein discrepancy, where $P$ is fixed and $Q$ is varied, suggests using a slightly finer notion than that of a statistical divergence, called \emph{separation}.
It is stated below for any \emph{discrepancy} with respect to $P$, meaning a map $D(\cdot || P) : \mathcal{P} \rightarrow [0,\infty]$ for which $D(Q || P)$ measures the dissimilarity of $Q$ as an approximation to $P$.

\begin{definition}[Separation]
\label{def: separating}
    Let $P \in \mathcal{P}$ and $\mathcal{Q} \subset \mathcal{P}$.
    A discrepancy $D(\cdot || P)$ is said to \emph{separate $P$ from $\mathcal{Q}$} if $D(Q||P) = 0$ implies that $P = Q$ for all $Q \in \mathcal{Q}$.
\end{definition}

\noindent Note that any valid statistical divergence (c.f. \Cref{def: divergence}) separates $P$ from $\mathcal{P}$, for all choices of $P \in \mathcal{P}$.
The Fisher divergence, despite the name, is not really a statistical divergence since it is not well-defined when the density $p$ does not exist.

The following result illustrates separation in the context of the Fisher divergence; the statement is not as general as possible -- we will present more general results later in this Chapter -- but we choose to present it here because it admits a simple and instructive proof:

\begin{proposition}[A separation result for Fisher divergence]
    Let $P$ and be a (Borel) distribution on $\mathbb{R}^d$ admitting a density $p$ with $\nabla \log p \in \mathcal{L}^2(P)$.
    Then the Fisher divergence separates $P$ from the set $\mathcal{Q}$, consisting of (Borel) distributions $Q$ on $\mathbb{R}^d$ that admit a density $q$ for which $\nabla \log p, \nabla \log q \in \mathcal{L}^2(Q)$.
\end{proposition}
\begin{proof}
    If $\mathrm{FD}(Q||P) = 0$ then, since $q > 0$ and
    \begin{align*}
        \mathrm{FD}(Q||P) = \int q(x) \left\| \nabla \log \frac{p}{q} \right\|_2^2 \; \mathrm{d}x ,
    \end{align*}
    it follows that $\nabla \log (p/q) = 0$ almost everywhere on $\mathbb{R}^d$.
    This implies $p/q$ is almost everywhere constant, and since both $p$ and $q$ must integrate to one, we deduce that $P$ and $Q$ are equal.
\end{proof}

\section{Classical Stein Discrepancies}
\label{sec: classical sd}

Stein's method is a theoretical device for which there is a natural incentive to consider Stein sets $\mathcal{G}$ that are small, and thus easier to handle theoretically, compared for example to the Stein set used to construct Fisher divergence in \Cref{sec: Fisher divergence}.
For a target distribution $P$ supported on $\R^d$,  
\emph{classical Stein discrepancies} \citep{gorham2015measuring,Gorham2019} 
employ a Stein set of bounded functions with bounded Lipschitz derivatives.
To make this definition explicit some notation is required.
First introduce the following notation that will be used throughout the remainder of this book.
For a norm $\|\cdot\|$ on $\mathbb{R}^d$, recall that $\|\cdot\|^*$ denotes the dual norm from \Cref{chap: intro}.

\begin{definition}[Classical Stein set; \citealp{gorham2015measuring}]
\label{def:classical_set}
Given a norm $\norm{\cdot}$ on $\R^d$, we define the \emph{classical Stein set} $\steinset$ as the set of functions $g : \mathbb{R}^d \rightarrow \mathbb{R}^d$ for which 
\begin{align}\notag
\sup_{x \neq y} \quad \max \left( \dualnorm{g(x)} , \dualnorm{\nabla g(x)}, \frac{\dualnorm{\nabla g(x) - \nabla g(y)}}{\norm{x-y}} \right) \leq 1 .
\end{align}
\end{definition}

\noindent The construction can be extended to domains $\mathcal{X} \subseteq \mathbb{R}^d$ as described in \Cref{sec: mirrored operators} and \citet{gorham2015measuring}, but our focus here is on the case $\mathcal{X} = \mathbb{R}^d$.
These classical Stein sets are commonly paired with the diffusion Stein operators of \Cref{sec: diffusion operators}.

\begin{definition}[Classical diffusion Stein discrepancy; \citealp{Gorham2019}]
\label{def:classical_diffusion_sd}
In the setting of \Cref{ass: reg langevin}, a \emph{classical diffusion Stein discrepancy} $\mathcal{S}(\cdot,\mathcal{T}_P,\steinset)$ combines a diffusion Stein operator $\mathcal{T}_P$ for $P$ (\Cref{def: diffusion operator}) and the classical Stein set $\steinset$ (\Cref{def:classical_set}).
\end{definition}

As anticipated at the start of this Chapter, the usefulness of a Stein discrepancy depends on three factors: \emph{convergence detection}, \emph{convergence control}, and \emph{computability}. %
The first two factors, convergence detection and convergence control, refer to the set of sequences $(Q_n)_{n \in \mathbb{N}}$ for which $\mathcal{S}(Q_n,\mathcal{T}_P,\mathcal{G}) \rightarrow 0$ in the large $n$ limit.
This set automatically includes the constant sequence with $Q_n = P$ whenever the Stein discrepancy is well-defined, but we would typically hope this property holds for sequences $(Q_n)_{n \in \mathbb{N}}$ that converge to $P$ in a sense that is reasonably standard. 
From a mathematical perspective, convergence detection and convergence control refer, respectively, to upper-bounding and lower-bounding the Stein discrepancy in terms of another statistical divergence, typically one whose mathematical properties are well-understood.
The first result of this kind that we present, due to \citet{Gorham2019}, shows that classical diffusion Stein discrepancies \emph{detect} convergence in $s$-Wasserstein metric, provided that the diffusion coefficients are Lipschitz. 
Recall that the $s$-Wasserstein metric $\mathrm{W}_{\|\cdot\|}^s$ was introduced in \Cref{def: Wasserstein}.
Given a function $g : \mathcal{X} \rightarrow \mathbb{R}^d$, let $M_0(g) \defeq  \sup_{x \in \mathcal{X}} \|g(x)\|_{\mathrm{op}}$ and 
$$
M_k(g) \defeq \sup_{x \neq y} \frac{ \|(\nabla^{k-1} g)(x) - (\nabla^{k-1} g)(y)\|_{\mathrm{op}}}{\|x-y\|_2} 
$$
for $k \in \{1,2\}$.
The upper bounds presented in \Cref{prop:discrepancy-upper-bound} and \Cref{thm:constant-lower-bound} feature constants that are expressed in terms of e.g. $M_1(\cdot)$; it is implicit that these quantities are assumed to be finite, otherwise the bound becomes trivial.

\begin{proposition}[Convergence detection; Prop.~8 of \citealp{Gorham2019}] \label{prop:discrepancy-upper-bound}
In the setting of \Cref{ass: reg langevin}, consider a classical diffusion Stein discrepancy $\mathcal{S}(\cdot,\mathcal{T}_P,\steinset)$ where $\|\cdot\| \geq \| \cdot \|_2$.
Let $a$ denote the covariance coefficient, $b$ denote the drift, and $c$ denote the stream coefficient.
If $a, b, c \in \mathcal{L}^1(P)$ and $p(a+c) \in C^1(\mathbb{R}^d,\mathbb{R}^{d \times d})$, then, with $m \defeq  a + c$,
\begin{align*}
	\mathcal{S}(Q,\mathcal{T}_P,\steinset)
		& \leq \mathrm{W}_{\|\cdot\|}^s (Q,P) (2M_1(b)+M_1(m)) \\
		& \quad + \mathrm{W}_{\|\cdot\|}^s(Q,P)^t\, 2^{1-t}\, \E_{Z \sim P} \left[ {(2\norm{b(\PVAR)} + \norm{m(\PVAR)})^{\frac{s}{s-t}}} \right]^{\frac{s-t}{s}}
\end{align*}
for any $s\geq 1$ and $t \in (0,1]$.
Moreover, for $\mu_0 \defeq \mathbb{E}_{Z \sim P} [ e^{2\norm{b(\PVAR)}+\norm{m(\PVAR)}} ]$,
\begin{align*}
\mathcal{S}(Q,\mathcal{T}_P,\steinset) &\le
\mathrm{W}_{\|\cdot\|}^1(Q,P) (2M_1(b)+M_1(m)) \\
& \qquad + \min( \mathrm{W}_{\|\cdot\|}^1(Q,P), 2) \log\left( \frac{e\mu_0}{\min( \mathrm{W}_{\|\cdot\|}^1(Q,P), 2)} \right).
\end{align*}
\end{proposition}

\noindent Thus if a sequence $(Q_n)_{n \in \mathbb{N}}$ converges to $P$ in the sense of $s$-Wasserstein, meaning that $\mathrm{W}_{\|\cdot\|}^s(Q_n,P) \rightarrow 0$ as $n \rightarrow \infty$, it follows that this convergence is \emph{detected} by the classical diffusion Stein discrepancy, in the sense that $\mathcal{S}(Q_n,\mathcal{T}_P,\steinset) \rightarrow 0$ in the same limit.

The second factor determining the usefulness of a Stein discrepancy is \emph{convergence control}, which refers to the construction of a Stein discrepancy lower bound.
The following theorem of \citet{Gorham2019} shows that classical diffusion Stein discrepancies based on fast mixing diffusions \emph{control} convergence in $1$-Wasserstein metric.
To state this result, we adopt the same notation introduced in \Cref{subsec: Markov processes} and consider a collection of Markov processes $(X_t^x)_{t \geq 0}$, each corresponding to the solution of the \Ito diffusion on which the diffusion Stein operator $\mathcal{T}_P$ is based (\Cref{def:diffusion}), but with initial distributions $P_0 = \delta_x$ for $x \in \mathbb{R}^d$.
Let $P_t^x$ denote the probability measure associated with $X_t^x$.
In addition, if $r : [0,\infty) \rightarrow \mathbb{R}$ is a non-increasing integrable function such that $\mathrm{W}_{\|\cdot\|}^1(P_t^x , P_t^y) \leq r(t) \mathrm{W}_{\|\cdot\|}^1(\delta_x,\delta_y)$ holds for all $x,y \in \mathbb{R}^d$, then we say that $r$ is a \emph{Wasserstein decay rate} associated with $\mathcal{T}_P$.
Further, for a matrix-valued function $g$ on $\mathbb{R}^d$, let
$$
F_k(g) \defeq \sup_{x \in \mathbb{R}^d, \|v_1\|_2 = 1 , \dots , \|v_k\|_2 = 1} \| (\nabla^k g)(x)[v_1 , \dots , v_k] \|_F
$$
where $\|\cdot\|_F$ is the Frobenius norm, and let
$$
M_1^*(g) \defeq \sup_{x \neq y} \frac{ \| g(x) - g(y)\|_{\mathrm{op}}^*}{\|x-y\|_2} .
$$

\begin{theorem}[Convergence control; Thms.~6 \& 7 of \citealp{Gorham2019}]
\label{thm:constant-lower-bound}
In the setting of \Cref{ass: reg langevin}, consider a classical diffusion Stein discrepancy $\mathcal{S}(\cdot,\mathcal{T}_P,\steinset)$ where $\|\cdot\| \geq \| \cdot \|_2$.
Let $a$ denote the covariance coefficient, $b$ denote the drift, $c$ denote the stream coefficient, and $\sigma$ denote the diffusion coefficient.
Further assume that the \Ito diffusion corresponding to $\mathcal{T}_P$ has Wasserstein decay rate $r$, and $\sr \defeq \int_0^\infty r(t)\dt$.
\begin{itemize}[leftmargin=*]
    \item If $m \defeq a + c$ is constant, then 
\begin{align}\label{eqn:stein-discrepancy-constant}
\mathrm{W}_{\|\cdot\|_2}^1(Q,P) & \leq\  3\sr \max\Big( \mathcal{S}(Q,\mathcal{T}_P,\steinset) , \\
& \hspace{40pt} \textstyle \sqrt[3]{\mathcal{S}(Q,\mathcal{T}_P,\steinset) \sqrt{2}\,\Earg{\twonorm{G}}^{2}(2M_1(b)+\frac{1}{\sr})^2} \Big), \notag
\end{align}
where 
$G \in \reals^d$ is a standard normal vector.

\item If $\sigma$ is Lipschitz and $b$ and $\sigma$ have locally Lipschitz second derivatives, then 
\begin{align*}%
\mathrm{W}_{\|\cdot\|_2}^1(Q,P) 
	\le \beta_1 \max\left(\mathcal{S}(Q,\mathcal{T}_P,\steinset), \sqrt{\mathcal{S}(Q,\mathcal{T}_P,\steinset) \Earg{\twonorm{G}} }\right), \notag
\end{align*}
for $\beta_1$ a constant depending on 
$M_0(\sigma^{-1}), M_1(\sigma),  F_2(\sigma), M_1(b), M_2(b), s_r,$
and $M_1^*(m)$.

\item If $\sigma$ is Lipschitz and $\grad^3b$ and $\grad^3\sigma$ are locally Lipschitz, then
\begin{align*}
\nonumber \mathrm{W}_{\|\cdot\|_2}^1(Q,P)\leq\, & \beta_2\, \mathcal{S}(Q,\mathcal{T}_P,\steinset)
\sqrt{d}\max\left( 1, \log\left(\frac{1}{\mathcal{S}(Q,\mathcal{T}_P,\steinset)} \right) \right) %
\end{align*}
for a constant $\beta_2 > 0$ depending on $\beta_1, M_{3}(\sigma), M_{3}(b),$ and $r$.
\end{itemize}
\end{theorem}

\noindent Thus if a sequence $(Q_n)_{n \in \mathbb{N}}$ converges to $P$ according to the classical diffusion Stein discrepancy, in the sense that $\mathcal{S}(Q_n,\mathcal{T}_P,\steinset) \rightarrow 0$ as $n \rightarrow \infty$, it follows that the sequences converges in the $1$-Wasserstein metric, in the sense that $\mathrm{W}_{\|\cdot\|_2}^1 (Q_n,P) \rightarrow 0$ in the same limit. See also \citet[Thm.~2.1 and Lem.~2.2]{MackeyGo16} and \citet[Thm.~2]{gorham2015measuring} for analogous convergence control results under alternative preconditions.

The final factor determining the usefulness of a Stein discrepancy is the ease with which it can be computed.
If the candidate distribution $Q$ is supported on a finite set of size $n$ in dimension $d = 1$, then \citet[Thm.~9]{gorham2015measuring} and \citet{Gorham2019} showed that the classical diffusion Stein discrepancy can be computed exactly as the solution of a convex quadratically-constrained quadratic program with linear objective, $O(n)$ variables, and $O(n)$ constraints.  
However, an efficient procedure for computing the classical diffusion Stein discrepancy is unknown for larger $d$.
This limitation stems from the origin of the classical Stein set $\steinset$ as a theoretical rather than a computational tool.
Fortunately this computational challenge can be side-stepped if we consider alternatives for the Stein set, as we will discuss next.

\section{Graph Stein Discrepancies}
\label{sec: graph sd}
To enable generally computable Stein discrepancies, \citet{gorham2015measuring} introduced \emph{graph Stein sets} which impose boundedness and smoothness constraints only at pairs of points belonging to the support of the approximating distribution $Q$.

\begin{definition}[Graph Stein set; \citealp{gorham2015measuring}]\label{def:graph_set}
For a probability distribution $Q$ on $\mathbb{R}^d$ with discrete support and a graph $G = (V,E)$ with $V = \supp{Q}$, we define the \emph{graph Stein set},
\begin{align*}
\mathcal{G}_{\|\cdot\|,Q,G} \defeq \bigg\{ g : & 
	\max\left(\norm{g(v)}^*, \norm{\grad g(v)}^*,
   	\textfrac{\norm{g(x) - g(y)}^*}{\norm{x - y}},
    \textfrac{\norm{\grad g(x) - \grad g(y)}^*}{\norm{x - y}} \right) \le 1,  \\
    & \textstyle\frac{\norm{g(x) - g(y) - {\grad g(x)}{(x - y)}}^*}{\frac{1}{2}\norm{x - y}^2} \leq 1,
    \textstyle\frac{\norm{g(x) - g(y) -{\grad g(y)}{(x -
    y)}}^*}{\frac{1}{2}\norm{x - y}^2} \leq 1, \\
    & \; \forall (x,y)\in E, v \in V \bigg\}.
\end{align*}
\end{definition}

\noindent As such, the graph Stein set $\mathcal{G}_{\|\cdot\|,Q,G}$ is a superset of the classical Stein set $\steinset$ (\Cref{def:classical_set}).
Note that, unlike the classical Stein set, the graph Stein set is $Q$-dependent and also requires the edge set $E$ of the graph $G$ to be specified.
It turns out that this dependence can be theoretically justified provided that an appropriate graph $G$ is used to construct the graph Stein set, as we will see in \Cref{prop:spanner-equivalence} of the sequel.
\citet{gorham2015measuring,Gorham2019} paired these graph Stein sets with diffusion Stein operators (\Cref{def: diffusion operator}) to obtain what is called a \emph{graph diffusion Stein discrepancy}:

\begin{definition}[Graph diffusion Stein discrepancy; \citealp{Gorham2019}]\label{def:graph_diffusion_sd}
In the setting of \Cref{ass: reg langevin}, a \emph{graph diffusion Stein discrepancy} $Q \mapsto \mathcal{S}(Q,\mathcal{T}_P,\mathcal{G}_{\|\cdot\|,Q,G})$ combines a diffusion Stein operator $\mathcal{T}_P$ for $P$ (\Cref{def: diffusion operator}) and the graph Stein set $\mathcal{G}_{\|\cdot\|,Q,G}$ (\cref{def:graph_set}).
\end{definition}

Remarkably, when $Q$ has finite support $\{x_1, \cdots, x_n\}$ and the norm $\norm{\cdot}$ is $\norm{\cdot}_1$, so that the dual norm $\norm{\cdot}^*$ is $\norm{\cdot}_\infty$, a graph diffusion Stein discrepancy can be computed by solving $d$ independent linear programs in parallel using efficient off-the-shelf solvers, as captured in the following result:

\begin{proposition}[Computing graph diffusion Stein discrepancy; Sec.~4.2 of \citealp{Gorham2019}]
\label{prop: compute graph stein}
Suppose $Q = \sum_{i=1}^n q(x_i) \delta_{x_i}$ for $x_i\in\Rd$ and $q$ a probability mass function on $V = \supp{Q}$. 
For any edge set $E$, the graph diffusion Stein discrepancy $\diffstein{Q}{\gsteinset{1}{G}}$ with graph $G=(V,E)$ equals
\begin{align}
\sum_{j=1}^d 
\sup_{\psi_j\in \reals^{n}, \Psi_j\in\reals^{d\times n}}\; 
\bigg\{ &
\sum_{i=1}^{n} q(x_i) \bigg( 2b_j(x_i) \psi_{ji} + \sum_{k=1}^d m_{jk}(x_i)\Psi_{jki} \bigg) \label{eqn:l1-program} \\ 
& \; \text{s.t. }
\infnorm{\psi_j} \leq 1, \;
\infnorm{\Psi_j} \leq 1, \nonumber \\
& \; \text{and for all } i\neq l, (x_i,x_l) \in E, \nonumber \\
& \; \max \bigg(
  \textstyle\frac{|\psi_{ji} - \psi_{jl}|}{\onenorm{x_i - x_l}},
  \textstyle\frac{\infnorm{\Psi_j(e_i - e_k)}}{\onenorm{x_i - x_l}},
  \textstyle\frac{|\psi_{ji} - \psi_{jl} - \inner{\Psi_j e_i}{x_i - x_l}|}{\frac{1}{2}\onenorm{x_i - x_l}^2}, \nonumber \\
  & \qquad \qquad 
  \textstyle\frac{|\psi_{ji} - \psi_{jl} - \inner{\Psi_j e_i}{x_l - x_i}|}{\frac{1}{2}\onenorm{x_i - x_l}^2}
\bigg) \le 1 \bigg\} , \nonumber
\end{align}
where $\psi_{ji}$ and $\Psi_{jki}$ represent the values $g_j(x_i)$ and $\grad_k g_j(x_i)$ respectively from \Cref{def:graph_set}.
\end{proposition}

\noindent The number of constraints in each linear program is determined by the number of edges $|E|$ in the graph, and roughly speaking a larger number of constraints renders the numerical solution of the linear program more difficult.
On the other hand, if the edge set $E$ is too small then the graph Stein set may be too large to detect convergence. 
Consideration of the convergence detection and control properties of the graph diffusion Stein discrepancy can provide insight into how the edge set $E$ should be selected.

To analyze the effect of the choice of edge set $E$ on convergence detection and control, some additional terminology is required.
A graph $G = (V,E)$ with vertex set $V$ embedded in $\mathbb{R}^d$ is termed a \emph{$t$-spanner} \citep{Chew86,peleg1989graph} if, when each edge $(x,y)\in E$ is assigned a weight $\norm{x-y}$ equal to its length in $\mathbb{R}^d$, then for all $x, y \in V$ with $x \neq y$ there exists a path between $x$ and $y$ with total path weight no greater than $t\norm{x-y}$.
To gain intuition, in the case $d=1$ a $t$-spanner can be obtained by first sorting the vertices into ascending order, at a cost $O(n \log n)$, and then joining adjacent vertices, requiring $n-1$ edges in total.
Provided that the graph $G$ is a $t$-spanner, the graph diffusion Stein discrepancy inherits the convergence detection and control properties of the corresponding classical diffusion Stein discrepancy, and in fact the two are equivalent:
\begin{proposition}[Equivalence of classical and graph diffusion Stein discrepancies; Prop.~13 of \citealp{Gorham2019}] \label{prop:spanner-equivalence}
If $G = (\supp{Q}, E)$ is a $t$-spanner for $t \geq 1$, then
\begin{align*}
\mathcal{S}(Q,\mathcal{T}_P,\mathcal{G}_{\|\cdot\|}) \leq \mathcal{S}(Q,\mathcal{T}_P,\mathcal{G}_{\|\cdot\|,Q,G}) \leq \kappa_d t^2\, \mathcal{S}(Q,\mathcal{T}_P,\mathcal{G}_{\|\cdot\|})
\end{align*}
where $\kappa_d$ is independent of $(Q,P,\mathcal{T}_P,G)$ and depends only on $d$ and $\norm{\cdot}$.
\end{proposition}

\noindent Thus from the perspective of retaining the desirable convergence detection and control properties of the classical diffusion Stein discrepancy, we ought to pick $G$ to be a $t$-spanner for some $t \geq 1$.
What is the associated computational cost?
If we consider the norm $\norm{\cdot} = \norm{\cdot}_p$ on $\mathbb{R}^d$, one can construct a $2$-spanner with $O(\kappa_d' n)$ edges in $O(\kappa_d' n\log(n))$ expected time, where $\kappa_d'$ is a $d$- and $p$-dependent constant \citep{har2005fast}.
Once the $t$-spanner is constructed, the graph diffusion Stein discrepancy can be computed by solving a finite-dimensional convex optimization problem with a linear objective, $O(n)$ variables, and $O(\kappa_d' n)$ convex constraints, using \Cref{prop: compute graph stein}.
Finally, we note that graph Stein discrepancies can also be used to enforce boundary constraints on constrained domains $\xset\subset\reals^d$ as described in \citet[Sec.~4.4]{gorham2015measuring}.

\section{Kernel Stein Discrepancies}
\label{sec: ksds}

Although graph Stein discrepancies overcome the algorithmic complexity associated with classical Stein discrepancies, their recourse to linear programming renders their practical implementation non-trivial.
This section discusses an alternative class of Stein discrepancies, called \emph{kernel Stein discrepancies}, for which computation is more straightforward but analysis of convergence detection and control is more involved.
First we present a definition in which the Stein operator and the kernel can be general, before exploring the consequences of this definition for specific Stein operators and kernels in detail.

\begin{definition}[Kernel Stein discrepancy] \label{def:kernel_stein_discrepancy}
Let $P$ be a probability distribution on a measurable space $\mathcal{X}$ and let $\mathcal{T}_P : \mathcal{D} \rightarrow \mathcal{L}^1(P)$ be a Stein operator for $P$ whose domain $\mathcal{D}$ contains functions of the form $g : \mathcal{X} \rightarrow \mathbb{R}^d$.
Let $K : \mathcal{X} \times \mathcal{X} \rightarrow \mathbb{R}^{d \times d}$ be a (matrix-valued, if $d > 1$) reproducing kernel for which $\mathcal{H}_K \subset \mathcal{D}$.
A \emph{kernel Stein discrepancy} 
$$
\mathrm{KSD}_P(Q) \defeq \mathcal{S}(Q,\mathcal{T}_P,\mathcal{G}_K)
$$ 
combines a Stein operator $\mathcal{T}_P$ for $P$ and a Stein set $\mathcal{G}_K \defeq  \{g : \|g\|_{\mathcal{H}_K} \leq 1\}$, the unit ball in $\mathcal{H}_K$.
\end{definition}

Perhaps the most important property of \acp{ksd} is that they can be computed in closed form for distributions $Q = \sum_{i=1}^n w_i \delta_{x_i}$ with finite support.
Indeed, for such $Q$ we automatically have $\mathcal{T}_P g \in \mathcal{L}^1(Q)$ and thus from \Cref{def:stein_discrepancy} the \ac{ksd} can be expressed as
\begin{align}
    \mathrm{KSD}_P(Q) = \sup_{f \in \mathcal{T}_P (\mathcal{G}_K) } |Q(f)| & = \sup_{\|g\|_{\mathcal{H}_K} \leq 1 } Q(\mathcal{T}_P g)   = \sup_{\|g\|_{\mathcal{H}_K} \leq 1 } \sum_{i=1}^n w_i (\mathcal{T}_P g)(x_i) \label{eq: ksd derive}
\end{align}
where we have removed the modulus using the fact that $\mathcal{T}_P(\mathcal{G}_K)$ is a symmetric set (which follows from the fact that $\mathcal{G}_K$ is a symmetric set).
The subsequent steps require some regularity to hold; here we sketch the argument, before stating precise results in \Cref{subsec: lksd in detail}. 
Consider the linear functional $L : \mathcal{H}_K \to \reals$ with 
\begin{align*}
    L(g) \defeq \sum_{i=1}^n w_i (\mathcal{T}_P g)(x_i)
    \qtext{for each}
    g\in\mathcal{H}_K.
\end{align*}
Whenever $L$ is continuous, we can write $L(g) = \langle g , \ell \rangle_{\mathcal{H}_K}$ for some Riesz representer $\ell \in \mathcal{H}_K$ by \Cref{thm: Riesz}.
Using this Riesz representation, the supremum in \eqref{eq: ksd derive} is seen to be achieved when $g = \ell / \|\ell\|_{\mathcal{H}_K}$, and the value of the supremum itself is $\|\ell\|_{\mathcal{H}_K}$.
To explicitly determine $\ell$, we can fix $x \in \mathcal{X}$ and $u \in \mathbb{R}^d$ and use the reproducing property from $\mathcal{H}_K$ (\Cref{thm: existence vvrkhs}) together with the notation $K_x = K(\cdot,x)$, to see that
\begin{align*}
    \langle \ell(x) , u \rangle = \langle \ell , K_x u \rangle_{\mathcal{H}_K} = L(K_x u) = \sum_{i=1}^n w_i (\mathcal{T}_P K_x u)(x_i) .
\end{align*}
Since this holds for all $x \in \mathcal{X}$ and $u \in \mathbb{R}^d$, we have that
\begin{align*}
    \ell(\cdot) = \sum_{i=1}^n w_i (\mathcal{T}_P^{(1)} K_{\cdot})(x_i)
\end{align*}
where $\mathcal{T}_P^{(1)}$ indicates that the Stein operator acts on the rows on $K_{\cdot}$, so that $(\mathcal{T}_P^{(1)} K_{\cdot})(x_i) \in \mathbb{R}^d$.
Assuming that $\mathcal{T}_P$ is a first order differential operator, we can use the \emph{differential reproducing property} \citep[Lem.~4]{barp2022targeted} to obtain
\begin{align*}
    \|\ell\|_{\mathcal{H}_K}^2 & = \left\langle \sum_{i=1}^n w_i (\mathcal{T}_P K_{\cdot})(x_i) , \sum_{j=1}^n w_j (\mathcal{T}_P K_{\cdot})(x_j) \right\rangle_{\mathcal{H}_K} \\
    & = \sum_{i=1}^n \sum_{j=1}^n w_i w_j \underbrace{ \mathcal{T}_P^{(1)} \mathcal{T}_P^{(2)} K(x_i,x_j) }_{(\star)} ,
\end{align*}
where $\mathcal{T}_P^{(2)}$ acts on the columns of $K$.
The expression $(\star)$ itself defines a kernel 
\begin{align}\label{eq:general-stein-kernel}
    k_p(\cdot,\cdot) \defeq  \mathcal{T}_P^{(1)} \mathcal{T}_P^{(2)} K(\cdot,\cdot) : \mathbb{R}^d \times \mathbb{R}^d \rightarrow \mathbb{R},
\end{align}
and we call this the \emph{Stein kernel} associated to the Stein operator $\mathcal{T}_P$ and the \emph{base kernel} $K$. 
Putting this all together, we have shown that
\begin{align*}
    \mathrm{KSD}_P(Q) & = \sqrt{ \sum_{i=1}^n \sum_{j=1}^n w_i w_j k_p(x_i,x_j) } ,
\end{align*}
which is an \emph{explicit} formula, in contrast to the situation for graph Stein discrepancies where linear programming tools were required.
The availability of an explicit formula for the Stein discrepancy underpins many of the algorithms discussed in \Cref{sec:ksd_application}.
The next section makes the above argument rigorous.

\subsection{Langevin KSD in Detail}
\label{subsec: lksd in detail}

This section sets the scene for describing some precise results for the Langevin Stein operator $\mathcal{T}_P$ (\Cref{def: Lang SO}).
As such, we will be in the setting of \Cref{ass: reg langevin}, where $P$ admits a positive and differentiable density $p$ on $\mathcal{X} = \mathbb{R}^d$, for which this Stein operator is well-defined.
The first result we present makes a rigorous claim about when \acp{ksd} can be explicitly computed.

\begin{theorem}[Langevin KSD as MMD; Thm.~1 of \citealp{barp2022targeted}]\label{ksdDef}
    In the setting of \Cref{ass: reg langevin}, let $K : \mathcal{X} \times \mathcal{X} \rightarrow \mathbb{R}^{d \times d}$ be a matrix-valued kernel for which the elements of $\langevin(\H_K) \defeq  \{\langevin f : f \in \mathcal{H}_K\}$ are well-defined.
    Then $\langevin (\H_K)$ can be endowed with the structure of an \ac{rkhs}, whose reproducing kernel is
    \begin{align}\label{eq:stein-kernel}
        \ks(x,y) \defn
    \sum_{i=1}^d \sum_{j=1}^d \frac1{\p( x)\p(y)} \partial_{y_j} \partial_{x_i} \left( \p( x)  K_{i,j}( x, y) \p( y) \right).
    \end{align}
   Moreover, if $\mathcal{T}_P(\mathcal{H}_K) \subset \mathcal{L}^1(P)$ and $P(\mathcal{T}_P (\mathcal{H}_K)) = 0$, meaning that $\mathcal{T}_P g \in \mathcal{L}^1(P)$ and $P(\mathcal{T}_Pg) = \{0\}$ for all $g \in \mathcal{H}_K$, then the \ac{ksd} coincides with \ac{mmd} based on the \emph{Stein kernel} $k_p$; i.e. $\mathrm{KSD}_P(Q) = \mathrm{MMD}_{\ks}(P, Q)$ for all $Q \in \mathcal{P}_{\mathcal{X}}$.
\end{theorem}

The computability of \ac{mmd} for distributions $Q = \sum_{i=1}^n w_i \delta_{x_i}$ with finite support was discussed in \Cref{subsec: mmd}, and following a similar argument we arrive at
\begin{align*}
    \mathrm{KSD}_P(Q) & = \sqrt{ \sum_{i=1}^n \sum_{j=1}^n w_i w_j k_p(x_i,x_j) } ,
\end{align*}
demonstrating a computational role of the Stein kernel.
This formalizes the sketch argument of \Cref{sec: ksds} in the case of the Langevin \ac{ksd}.

Our attention now turns to the conditions $\mathcal{T}_P(\mathcal{H}_K) \subset \mathcal{L}^1(P)$ and $P(\mathcal{T}_P (\mathcal{H}_K)) = 0$ appearing in \Cref{ksdDef}.
For convenience we let $\mathcal{P}_{K,0}$ denote the set of distributions $Q$ on $\mathcal{X}$ for which $\mathcal{T}_Q(\mathcal{H}_K) \subset \mathcal{L}^1(Q)$ and $Q(\mathcal{T}_Q (\mathcal{H}_K)) = 0$, so that the conditions of \Cref{ksdDef} can be compactly expressed as $P \in \mathcal{P}_{K,0}$.
From \Cref{cor: embeddability in general}, we have that $\mathcal{T}_P(\mathcal{H}_K) \subset \mathcal{L}^1(P)$ if and only if $\mathcal{H}_{k_p} \subset \mathcal{L}^1(P)$ if and only if $P \in \mathcal{P}_{\mathcal{H}_{k_p}}$, and from \Cref{prop: sufficient for embed} each of these is implied if 
\begin{align}
    \int \sqrt{k_p(x,x)} \mathrm{d}P(x) < \infty . \label{eq: sufficient embed stein}
\end{align}
Using the explicit for of the Stein kernel in \eqref{eq:stein-kernel}, one can deduce that \eqref{eq: sufficient embed stein} holds whenever the elements of $\mathcal{H}_K$ are each bounded with bounded partial derivatives, and $\int \|\nabla \log p\| \mathrm{d}P < \infty$.
Finally, \Cref{lem: Stein op int by parts} provides regularity conditions under which $P(\mathcal{T}_P(\mathcal{H}_K)) = \{0\}$.
These arguments are collected together in \Cref{thm:embeddability_conditions}.

\begin{proposition}[Stein embeddability conditions; Prop.~3 and Rem.~5 of \citealp{barp2022targeted}]\label{thm:embeddability_conditions}
In the setting of \Cref{ass: reg langevin}, let $K : \mathcal{X} \times \mathcal{X} \rightarrow \mathbb{R}^{d \times d}$ be a matrix-valued kernel for which $\langevin(\H_K)$ is well-defined.
Then:
\begin{enumerate}[label=\textup{(\alph*)}]
\item $\langevin (\H_{\Kb}) \subset \mathcal{L}^1(P) \iff P\in \mathcal{P}_{\mathcal{H}_{k_p}}$.
\item If $\int \sqrt{k_p(x,x)} \mathrm{d}P(x) < \infty$, then $P \in \mathcal{P}_{\mathcal{H}_{k_p}}$.
\item If $\int \|\nabla \log p\| \mathrm{d}P < \infty$ and all $v$ in $\H_K$ are bounded with bounded partial derivatives, then $\int \sqrt{k_p(x,x)} \mathrm{d}P(x) < \infty$.
\item If $P\in \mathcal{P}_{\mathcal{H}_{k_p}}$, then
$\iint \ks(x,y) \dd P(x)\dd P(y) = 0 \Leftrightarrow P( \H_{\ks})=\{0\} \Leftrightarrow P \in \mathcal{P}_{K,0}$.
\item If $P\in \mathcal{P}_{\mathcal{H}_{k_p}}$ and $\H_K \subset \mathcal{L}^1(P) \cap C^1(\mathbb{R}^d)$, then $P \in \mathcal{P}_{K,0}$.
\item If $x \mapsto \|\Kb(x,x)\|$ is bounded, then all $v$ in $\H_K$ are bounded.
\item If $(x,y) \mapsto \|\partial_{x^i}\partial_{y^i} K(x,y)\|$ exists and is bounded,
then  all $v$ in $\H_K$ have bounded $x^i$-partial derivatives.
\item If $K \in C^{(1,1)}_{b}(\R^d)$, then $\H_K \subset C^{1}_{b}(\R^d)$.
\end{enumerate}
\end{proposition}

\noindent In particular, the condition $P \in \mathcal{P}_{K,0}$ required for \Cref{ksdDef} is met when $\int \|\nabla \log p\| \mathrm{d}P < \infty$ and $K \in C_b^{(1,1)}(\mathbb{R}^d)$.

Finally, we end this section by remarking on a more explicit formula for the Stein kernel in a special case of interest:

\begin{remark}[Scalar kernel Langevin KSD]
    In the case where $K(x,y) = k(x,y) I$ for a scalar kernel $k : \mathcal{X} \times \mathcal{X} \rightarrow \mathbb{R}$, 
    \begin{align*} 
        k_p(x,y) & = \nabla_x \cdot \nabla_y k(x,y) + (\nabla \log p)(x) \cdot \nabla_y k(x,y) \\
        & \quad + (\nabla \log p)(y) \cdot \nabla_x k(x,y) + (\nabla \log p)(x) \cdot (\nabla \log p)(y) k(x,y) ,
    \end{align*}
    where e.g. $\nabla_x k(x,y)$ denotes the gradient with respect to the $x$ argument of $k(x,y)$.
\end{remark}

\subsection{Separating Distributions with Langevin KSD}
\label{subsec: sep for lang ksd}

The focus of this section is on sufficient conditions for the Langevin \ac{ksd} to separate $P$ from other distributions $\mathcal{Q}$; recall from \Cref{def: separating} that this means $\mathrm{KSD}_P(Q) = 0$ if and only if $Q$ and $P$ are equal, for all $Q \in \mathcal{Q}$.
Three main separation result will be presented.
The first result, which we call \emph{score-based separation} (\Cref{coroKSDwc}), requires a weak notion of characteristicness of the kernel $K$, but requires also some smoothness of $K$ and in limits the set of distributions $Q$ that can be separated from $P$.
The second result, which we call \emph{$\mathcal{L}^2$-based separation} (\Cref{thm: L2 sep}), does not require smoothness of $K$ and applies to certain distributions $Q$ that are excluded by the first result, but requires that $Q$ admits a density on $\mathbb{R}^d$.
The third result is perhaps most useful, requiring only a suitably smooth, translation-invariant kernel, and guaranteeing separation of $P$ from \emph{all} alternatives $Q \in \mathcal{P}$ (\Cref{Stein kernel control tight convergence via bounded P separation}).

\begin{definition}[$\DL{1}(\R^d)$, embedding, and a $\DL{1}(\R^d)$-characteristic kernel; Defs.~6 \& 8 of \citealp{barp2022targeted}]
    Let $\DL{1}(\R^d)$ denote the set of continuous linear functionals on $C_0(\R^d)$.
    We say that $\DL{1}(\R^d)$ \emph{embeds into} $\mathcal{H}_K$ if, for each $D \in \DL{1}(\R^d)$ there exists $\phi_D \in \mathcal{H}_K$ such that $D(h) = \langle \phi_D , h \rangle_{\mathcal{H}_K}$ for all $h \in \mathcal{H}_K$.
    Further, if the embedding $D \mapsto \phi_D$ is injective then we say that $K$ is a $\DL{1}(\R^d)$-characteristic kernel.
\end{definition}

\noindent This extends the notion of a characteristic kernel from \Cref{def: char kernel} to more general linear operators than just integration with respect to probability distributions.
\citet[][Thm. 12, Tab. 1, and Cor. 38]{simon2018kernel} showed that any $C_0^1(\mathbb{R}^d)$-universal scalar-valued kernel $k$ and any $C^{(1,1)}(\mathbb{R}^d)$ translation-invariant $k$ with fully supported spectral
measure is $\DL{1}(\R^d)$-characteristic in dimension $d = 1$. 
These results cover all of the translation-invariant
base kernels commonly used with \acp{ksd} including Gaussian, inverse multiquadric and Mat\'{e}rn.
A set of tools for constructing $\DL{1}(\R^d)$-characteristic matrix-valued kernels $K$ is provided next:

\begin{proposition} [$\DL{1}(\R^d)$-characteristic conditions; Prop.~4 of \citealp{barp2022targeted}]\label{thm:composition-kernel-properties}
Suppose a matrix-valued kernel $K$ with $\H_K \subset C_b^1(\R^d)$ is $\DL{1}(\R^d)$-characteristic.  Then the following claims hold true.
\begin{enumerate}[label=\textup{(\alph*)}]
\item  If $\tilt\in C_b^1$ is strictly positive, then $\tilt( x) K( x,  y)
    \tilt( y)$ is $\DL{1}(\R^d)$-characteristic. 
\item  If $b : \R^d\to\R^d$ is a Lipschitz $C^1(\R^d)$-diffeomorphism, then the composition kernel $K(b( x), b( y))$  
is $\DL{1}(\R^d)$-characteristic.
\item If $\k_j$ is $\DL{1}$-characteristic for $j\in \{1,\dots,d\}$, then  $\mathrm{diag}(\k_1,\ldots, \k_d)$
is $\DL{1}(\R^d)$-characteristic.
\end{enumerate}
\end{proposition}

Now we are ready to present the first main separation result; score-based separation.
From \Cref{cor: embeddability in general}, we have that $P \in \mathcal{P}_{\mathcal{H}_{k_p}}$ if and only if $\mathcal{H}_{k_p} \subset \mathcal{L}^1(P)$; here we naturally extend the notation so that, given $f : \mathcal{X} \rightarrow \mathbb{R}^d$, we let $\mathcal{P}_f$ be the set of measures $Q$ on $\mathcal{X}$ for which $x \mapsto \|f(x)\|$ is an element of $\mathcal{L}^1(Q)$.
Further, for a density $p$ we use $s_p \defeq  \nabla \log p$ as a shorthand for the Stein score.

\begin{theorem}[Score-Based Separation with Langevin KSD; Thm.~3 of \citealp{barp2022targeted}]
\label{coroKSDwc}
Suppose a matrix-valued kernel $K$ with $\H_K \subset C_b^1(\R^d)$ is $\DL{1}(\R^d)$-characteristic.
    If  $P \in \embedtozero$,
    then the Langevin \ac{ksd} separates $P$ from $\Pset_{s_p}$.
\end{theorem}

\noindent That is, if $K$ satisfies the conditions of  \Cref{coroKSDwc} and $\int \|\nabla \log p\| \; \mathrm{d}Q < \infty$, then the Langevin \ac{ksd} satisfies $\mathrm{KSD}_P(Q) = 0$ if and only if $P$ and $Q$ are equal.

Next we turn attention to the second main separation result; $\mathcal{L}^2$-based separation.
For this the following, somewhat restrictive assumption on the kernel $K$ is required:

\begin{definition}[$\mathcal{L}^2(\R^d)$-Integrally Strictly Positive Definite; Def.~3.1 of \citealp{liu2016kernelized}]
We say that a matrix-valued kernel $K$ is \emph{$\mathcal{L}^2(\R^d)$-integrally strictly positive definite} (ISPD) if
$\H_\Kb \subset \mathcal{L}^2(\R^d)$ and
\begin{align*}
\iint g(x)^\top K(x,y) g(y) \; \dd x \dd y > 0
\end{align*}
for all $g \in \mathcal{L}^2(\mathbb{R}^d)$ with $|g|_{\mathcal{L}^2(\mathbb{R}^d)} = (\int g(x)^2 \, \mathrm{d}x)^{1/2} > 0$.
\end{definition}

For continuous translation-invariant kernels, we can use Bochner's theorem (\Cref{thm: bochner}) to deduce conditions under which $K$ is $\mathcal{L}^2(\mathbb{R}^d)$-ISPD:

\begin{theorem}[$\mathcal{L}^2(\mathbb{R}^d)$-ISPD conditions; Thm.~4 of \citealp{barp2022targeted}]
\label{thm:L2 characteristic kernels}
The following claims hold true for a matrix-valued kernel $\Kb : \mathbb{R}^d \times \mathbb{R}^d \rightarrow \mathbb{R}^{d \times d}$.
\begin{enumerate}[label=\textup{(\alph*)}]
\item Suppose $(\k_j)_{j=1}^d$ are continuous translation-invariant kernels on $\mathbb{R}^d$ with $\H_{\k_j} \subset \mathcal{L}^2(\mathbb{R})$.
If the  spectral measure of each $\k_j$ is fully supported, then $\Kb=\mathrm{diag}(\k_j)$ is $\mathcal{L}^2(\R^d)$-ISPD.
\item If $K$ is $\mathcal{L}^2(\R^d)$-ISPD and $A: \R^d \to \R^{d\times d}$ is bounded measurable with $A(x)$ invertible for each $x$, then the tilted kernel $A(x)K(x,y)A(y)^\top$ is also  $\mathcal{L}^2(\R^d)$-ISPD.
\item  If $\H_\Kb$ is separable,  $\sup_x | \Kb_x u |_{\mathcal{L}^1(\mathbb{R}^d)}<\infty$, and $\Kb_x u \in \mathcal{L}^2(\R^d)$ for each $x$ and $u\in \R^d$, then $\H_K \subset \mathcal{L}^2(\R^d)$.
\item Suppose $\Kb_x u\in\mathcal{L}^1(\R^d)$ for some $u \in \R^d$.
If $\Kb$ is translation-invariant or, more generally, if $\Kb_x u$ is bounded, then $\Kb_x u\in\mathcal{L}^2(\R^d)$. 
\end{enumerate}
\end{theorem}

Now we are ready to present the second main separation result.
Let $\partial \mathcal{H}_K$ denote the set of functions $g : \mathbb{R}^d \rightarrow \mathbb{R}^d$ of the form $g_i(x) = \partial_{x_i} h_i(x)$ where $h \in \mathcal{H}_K$.

\begin{theorem}[$\mathcal{L}^2$-Based Separation with Langevin KSD; Thm.~5 of \citealp{barp2022targeted}]
\label{thm: L2 sep}
    Suppose $P \in \mathcal{P}_{K,0}$ for a matrix-valued kernel $K : \mathbb{R}^d \times \mathbb{R}^d \rightarrow \mathbb{R}^{d \times d}$.
    Then
    \begin{enumerate}
        \item If $K$ is $\mathcal{L}^2(\mathbb{R}^d)$-ISPD, then the Langevin \ac{ksd} separates $P$ from $\{Q \in \mathcal{P}_{\mathcal{H}_{k_p}} \cap \mathcal{P}_{K,0} : (s_p - s_q) q \in \mathcal{L}^2(\mathbb{R}^d) \}$.
        \item If $Q \in \mathcal{P}_{\mathcal{H}_{k_p}}$, $(s_p - s_q) q \in \mathcal{L}^2(\mathbb{R}^d)$ and $\mathcal{H}_K \subset \mathcal{L}^2(\mathbb{R}^d) \cap \mathcal{L}^\infty(\mathbb{R}^d)$, then $Q \in \mathcal{P}_{K,0}$.
        \item If $\mathcal{H}_K \subset \mathcal{L}^2(\mathbb{R}^d)$, $\partial \mathcal{H}_K \subset \mathcal{L}^\infty(\mathbb{R}^d)$, and $s_q q \in \mathcal{L}^2(\mathbb{R}^d)$, then $Q \in \mathcal{P}_{\mathcal{H}_{k_p}}$.
    \end{enumerate}
\end{theorem}

Although the $\mathcal{L}^2$-ISPD requirement is somewhat restrictive (for instance, it precludes slowly-decaying inverse multi-quadric kernels; c.f. \Cref{ex: IMQ kernel}), it does enable separation for certain $Q$ that are not covered by score-based separation (for instance, \Cref{thm: L2 sep} applies to Cauchy $Q$ and Gaussian $P$, while \Cref{coroKSDwc} does not, since $Q \notin \mathcal{P}_{s_p}$).

One might hope that under appropriate assumptions the Langevin \ac{ksd} could separate $P$ from \emph{all} alternatives $Q \in \mathcal{P}$.
The main obstacle to establishing such general separation is the unboundedness of the Langevin Stein kernel $k_p$; for sufficiently heavy-tailed $Q$, the kernel mean embedding $\int k_p(\cdot,x) \, \mathrm{d}Q(x)$ will fail to exist, and $Q \notin \mathcal{H}_{k_p}$.
A na\"{i}ve solution is to modify the Stein kernel so that it is \emph{bounded}, meaning that $Q \in \mathcal{P}_{k_p}$ for all $Q \in \mathcal{P}$, but it is then not clear whether separation properties still hold.
A positive answer is provided by \Cref{Stein kernel control tight convergence via bounded P separation}, and is our third main separation result.

The idea behind \Cref{Stein kernel control tight convergence via bounded P separation} is to show that $\mathcal{H}_{k_p}$ contains a sub-\ac{rkhs} of \emph{bounded} functions that are rich enough to separate $P$ from alternatives $Q \in \mathcal{P}$.
As a shorthand, we say that a Stein discrepancy or Stein kernel is \emph{$P$-separating} if it separates $P$ from $\mathcal{P}$.

\begin{definition}[$P$-separating and Bounded $P$-separating; Def.~3 of \citealp{barp2022targeted}]
A set of functions $\F$  is \emph{bounded $P$-separating}  if $\mathcal{L}^\infty \cap \F$ is \emph{$P$-separating},
i.e., if $Q \in \Pset$ and $Q(h) = P(h)$ for all $h \in \mathcal{L}^\infty \cap \F$ then $Q =P$.
\end{definition}

The first step in the argument is to obtain conditions on the base kernel $k$ under which the Stein kernel $k_p$ in \eqref{eq:stein-kernel} will be bounded.

\begin{theorem}[Bounded Langevin Stein kernels; Thm.~7 of \citealp{barp2022targeted}]\label{thm: tilted controls tightness}
  Suppose a matrix-valued kernel $K$ with $\H_{K} \subset C_b^1(\R^d)$ is $\DL{1}(\R^d)$-characteristic. 
  If $\|s_p( x) \| \leq \growth(x)$  for $\growth\in C^1(\mathbb{R}^d)$ with  $ \frac{1}{\theta} \in C_b^1(\mathbb{R}^d)$, then the Stein kernel $k_p$ in \eqref{eq:stein-kernel} induced by the tilted base kernel 
  \begin{align}
  \frac{K(x,y)}{\growth(x)\growth(y)} \label{eq: tilted kernel}
  \end{align} 
  is bounded and $P$-separating.
\end{theorem}

\noindent One can equivalently view the Stein discrepancy associated to the tilted base kernel in \eqref{eq: tilted kernel} as a diffusion Stein operator (c.f. \Cref{sec: diffusion operators}) with covariance coefficient $a(x) = \theta(x)^{-1} I_{d \times d}$ applied to the untilted base kernel $K$.

Our next step is to show that \ac{rkhs}es based on standard translation-invariant kernels contain sub-\ac{rkhs}es with bounded kernels of the form \eqref{eq: tilted kernel}:

\begin{proposition}[Translation-invariant kernels have rapidly decreasing sub-\ac{rkhs}es; Thm.~8 of \citealp{barp2022targeted}]
\label{thm:Schwarz-tilted sub-RKHS}
Suppose a kernel $\k$ with $\H_{\k} \subset C^1(\mathbb{R}^d)$ is translation-invariant with a spectral density bounded away from zero on compact sets.
Then  there exist a translation-invariant,  $\DL{1}$-characteristic kernel $\k_s\in C^{(1,1)}(\mathbb{R}^d)$  and, for each $ c>0$, a positive-definite
function $f$ with 
$\frac{1}{f} \in C^1(\mathbb{R}^d)$ and
\begin{align*}
\max(|f(x)|, \norm{\partial f(x)}) = O(e^{-c\sum_{i=1}^d \sqrt{|x_i|}})
\end{align*}
such that
$\H_{\k_f}\subset \H_{\k}$ for $\k_f( x, y) \defn f( x)\k_s( x, y)f( y)$.
\end{proposition}

Combining \Cref{thm: tilted controls tightness,thm:Schwarz-tilted sub-RKHS} leads to \Cref{Stein kernel control tight convergence via bounded P separation}, our third main separation result.
A function $f$ is said to have \emph{at most root exponential growth} if 
\begin{align}
f(x) = O \left( \exp \left( c\sum_{i=1}^d \sqrt{|x_i|} \right) \right) \label{eq: root exponential growth}
\end{align} 
for some $c \in \mathbb{R}$.
We also introduce the shorthand $\mathcal{P}_{k,0}$ for $\mathcal{P}_{K,0}$ when $K = k I$.

\begin{theorem}[General $P$-separation with Langevin KSDs; Thm.~9 of \citealp{barp2022targeted}]
\label{Stein kernel control tight convergence via bounded P separation}
Suppose a kernel $\k$ with $\H_{\k} \subset C^1(\mathbb{R}^d)$ is translation-invariant with a spectral density bounded away from zero on compact sets.
Define the tilted kernel $k_f( x, y) \defeq f(x)\kb(x,y)f(y)$ for each
strictly positive $f \in C^1(\mathbb{R}^d)$.
\begin{enumerate}[label=\textup{(\alph*)}]
 \item If $P\in\embedtozero[\kb]$  and
$\| s_p \|$ has at most root exponential growth, then the Stein kernel induced by the base kernel $k$ is bounded $P$-separating. 
    \item If $P\in \mathcal{P}_{k_f,0}$  and  $f$, $\partial f$, and $f \| s_p \|$ have at most root exponential growth, then the Stein kernel induced by $k_f$ is bounded $P$-separating.
\end{enumerate}
\end{theorem}

Separation is the main requirement of Stein discrepancies as used in goodness-of-fit testing (c.f. \cref{sec: GoF}), but for other applications the stronger properties of convergence detection and control are required.
Our attention therefore turns next to convergence detection (\Cref{subsec: conv det lksd}) and control (\Cref{subsec: conv contrl lksd}).

\subsection{Detecting Convergence with Langevin KSD}
\label{subsec: conv det lksd}

In \Cref{sec: classical sd} we saw that classical Stein discrepancies detect convergence in the 1-Wasserstein metric (\Cref{prop:discrepancy-upper-bound}).
The same convergence detection holds also for Langevin \acp{ksd}, as stated in the following result.
In the special case of a kernel $K(x,y) = k(x,y) I$ we write $\mathcal{G}_k$ as a shorthand for $\mathcal{G}_K$.

\begin{proposition}[Langevin KSD detects convergence; Prop.~9 of \citealp{gorham2017measuring}]
\label{prop:ksd-upper-bound}
Assume that $\grad \log p$ is Lipschitz with $\nabla \log p \in \mathcal{L}^2(P)$.
Then, for $k \in C_b^{(2,2)}(\mathbb{R}^d)$, the Langevin \ac{ksd} satisfies
$\mathrm{KSD}_P(Q_n) \to 0$ whenever $\mathrm{W}_{\|\cdot\|_2}^1(Q_n,P)\to 0$.
\end{proposition}

Compared to classical Stein discrepancies (\Cref{prop:discrepancy-upper-bound} for the Langevin Stein operator; i.e. $a = I$, $c = 0$, $b = \frac{1}{2} \nabla \log p$), no boundedness assumption on the gradient of $\nabla \log p$ is required in \Cref{prop:ksd-upper-bound}.
However, \Cref{prop:ksd-upper-bound} requires $\nabla \log p \in \mathcal{L}^2(P)$, while \Cref{prop:discrepancy-upper-bound} requires only $\nabla \log p \in \mathcal{L}^{1 + \epsilon}(P)$ for some $\epsilon > 0$.

\subsection{Controlling Convergence with Langevin KSD}
\label{subsec: conv contrl lksd}

Next we consider the converse question; whether Langevin \acp{ksd} offer control of more standard statistical divergences.
Our principal focus is on controlling the Dudley metric (which we recall metrizes weak convergence of distributions; c.f. \Cref{def: dudley metric}), but the possibility to control Wasserstein metrics (similarly to classical Stein discrepancies; c.f. \Cref{thm:constant-lower-bound}) will also be discussed. 

As a useful shorthand, we say that the Langevin \ac{ksd} \emph{controls weak $P$-convergence} if $\mathrm{KSD}_P(Q_n) \rightarrow 0$ implies $\mathrm{BL}_{\| \cdot\|}(P,Q_n) \rightarrow 0$.
The concept of \emph{tightness} is often a necessary criterion for proving the weak convergence of a sequence of probability measures:

\begin{definition}[Tightness] \label{def: tight}
    A sequence of probability distributions $(Q_n)_{n \in \mathbb{N}} \subset \mathcal{P}(\mathbb{R}^d)$ is called \emph{tight} if, for each $\epsilon > 0$, there exists a compact set $S \subset \mathbb{R}^d$ such that $Q_n(S^c) \leq \epsilon$ for all $n \in \mathbb{N}$.
\end{definition}

\noindent Intuitively, a tight sequence $(Q_n)_{n \in \mathbb{N}}$ cannot have probability mass `escaping to infinity'.

\begin{theorem}[Controlling tight convergence with Langevin KSD; Thms.~7 \& 9 of \citealp{barp2022targeted}] \label{thm: ctrl tight con}
    Under the conditions of \cref{thm: tilted controls tightness} or \ref{Stein kernel control tight convergence via bounded P separation}, if $(Q_n)_{n \in \mathbb{N}}$ is tight, then $\mathrm{KSD}_P(Q_n) \rightarrow 0$ implies $\mathrm{BL}_{\| \cdot\|}(P,Q_n) \rightarrow 0$.
\end{theorem}

The question thus becomes one of understanding when convergence of the Langevin \ac{ksd} implies tightness of $(Q_n)_{n \in \mathbb{N}}$.
Recall from \eqref{ksdDef} that $\mathrm{KSD}_P(Q_n)$ can be expressed as $\mathrm{MMD}_{k_p}(P,\cdot)$ using the Stein kernel $k_p$ in \eqref{eq:stein-kernel}.
Further, from \eqref{eq: MMD as WCE} the \ac{mmd} has the representation
\begin{align*}
    \mathrm{MMD}_{k_p}(P,Q) = \sup_{\|f\|_{\mathcal{H}_{k_p}} \leq 1} P(f) - Q(f) 
\end{align*}
so that the \ac{rkhs} $\mathcal{H}_{k_p}$ determines the topology that is induced by the Langevin \ac{ksd}.
A richer set of test functions corresponds to a more stringent statistical divergence, and if we seek to deduce tightness of $(Q_n)_{n \in \mathbb{N}}$ from convergence of the Langevin \ac{ksd} then the following will be required of $\mathcal{H}_{k_p}$:

\begin{definition}[$P$-dominating indicators; Def.~4 of \citealp{barp2022targeted}]\label{def:indic-approx}
A set of  functions $\F \subset \mathcal{L}^1(P)$ is said to \emph{$P$-dominate indicators} if, for each $\epsilon > 0$, there exists a compact set $S \subset \mathbb{R}^d$ and a function $h \in \F$ that satisfy
\begin{talign}
h(x) - P(h) \geq \mathbb{I}_{S^c}(x) - \eps \label{eq:indic-approx}
\end{talign}
for all $x \in \mathbb{R}^d$.
\end{definition}

\noindent \Cref{def:indic-approx} ensures that a sequence $(Q_n)_{n \in \mathbb{N}}$ can only approximate $P$ well in the sense of integrating the elements of $\mathcal{F}$ if it places uniformly little mass outside of a compact set $S$.
Combining \Cref{thm: ctrl tight con} with \Cref{def:indic-approx} provides sufficient conditions for weak $P$-convergence control:

\begin{corollary}[Controlling weak $P$-convergence with KSDs; Cor.~3 of \citealp{barp2022targeted}]\label{ksd-tightness}
Under the conditions of \cref{coroKSDwc}, \cref{thm: tilted controls tightness} or \ref{Stein kernel control tight convergence via bounded P separation}, if $\mathcal{H}_{k_p}$ also $P$-dominates indicators, then the Langevin \ac{ksd} controls weak $P$-convergence.
\end{corollary}

All that remains is to relate the property of $P$-dominating indicators to the choice of the kernel $K$ in the Langevin \ac{ksd}.
As a useful stepping stone, we recall the definition of a \emph{coercive} function:

\begin{definition}[Coercive function; \citealp{hodgkinson2020reproducing}]
A function $h :\mathbb{R}^d \to \R$ is \emph{coercive} if, for any $M >0$, there exists a compact set $S \subset \mathbb{R}^d$ such that
$\inf_{x \in S^c} h(x) > M$.
\end{definition}

\noindent Note that any continuous coercive function is also bounded below, as continuous functions are bounded on any compact set.

\begin{lemma}[Coercive functions dominate indicators; Lem.~1 of \citealp{barp2022targeted}] \label{coercive-tightness}
If there exists an element $h\in \mathcal{H}_{k_p}$ that is coercive and bounded below and $P \in \mathcal{P}_{\mathcal{H}_{k_p}}$, then $\mathcal{H}_{k_p}$ $P$-dominates indicators.
\end{lemma}

Recall that the elements of $\mathcal{H}_{k_p}$ have the form $\mathcal{T}_P h = \nabla \cdot h + s_p \cdot h$ for $h \in \mathcal{H}_K$.
For most typical choices of $K$ the first term, $\nabla \cdot h$, will be bounded.
On the other hand, for $K$ sufficiently heavy-tailed and appropriately increasing $s_p$, the second term $s_p \cdot h$ can be coercive and lower-bounded.
It turns out that the inverse multi-quadric kernel (\Cref{ex: IMQ kernel}) fulfills this requirement:

\begin{theorem}[IMQ KSDs control weak $P$-convergence; Thm.~11 of \citealp{barp2022targeted}] \label{imq-tightness}
Consider a target measure $P\in\Pset$ with score $s_p \in C(\R^d) \cap \mathcal{L}^1(P)$ and suppose that, for some dissipativity rate $u > 1/2$ and $c, c_1,c_2 > 0$, $P$ satisfies the \emph{generalized dissipativity} condition 
\begin{align}
    \label{eq:generalized-dissipativity}
-\inner{s_p( x)}{ x} - c \onenorm{s_p( x)}\geq c_1\twonorm{ x}^{2u} - c_2 \qtext{for all}  x\in\R^d.
\end{align}
If $\kb$ is the inverse multi-quadric kernel (\Cref{ex: IMQ kernel}) with exponent $\beta \in (0, 2u-1)$, then $\Hks$ $P$-dominates indicators.
If, in addition, $\norm{s_p}$ has at most  root exponential growth in the sense of \eqref{eq: root exponential growth}, then the Langevin \ac{ksd} controls weak $P$-convergence.
\end{theorem}

\Cref{imq-tightness} established weak convergence control using a Stein kernel $k_p$ that can be unbounded.
Oftentimes in the analysis and application of kernel methods it is helpful to work with kernels that are bounded, and it turns out that one can also arrange for a Stein kernel to be bounded.
Moreover, the bounded Stein kernels we construct in \Cref{tilted-tightness} exactly metrize weak convergence to $P$, meaning that $\mathrm{KSD}_P(Q_n) \rightarrow 0$ if and only if $\mathrm{BL}_{\|\cdot\|}(P,Q_n) \rightarrow 0$.
The price that we pay for this stronger result is that the base kernel $K$ is no longer translation-invariant, meaning some reasonable choice of an origin is required:

\begin{theorem}[Metrizing weak $P$-convergence with bounded Stein kernels; Thm.~12 of \citealp{barp2022targeted}] \label{tilted-tightness}
Consider a target measure $P\in\Pset$ with score $s_p$
that, for some dissipativity rate $u > 1/2$ and $c, c_1,c_2 > 0$, satisfies the generalized dissipativity condition \cref{eq:generalized-dissipativity}.  
Define the Stein kernel with diagonal base kernel $K$ such that $K_{i,i}( x, y) = a(\twonorm{ x}) ( x_i  y_i +\k( x,  y)) a(\twonorm{ y}) $, i.e.,
\[
\ks( x, y) =
	\sum_{1 \leq i \leq d} \frac{\dx \dy(p( x) a(\twonorm{ x}) ( x_i  y_i +\k( x,  y)) a(\twonorm{ y}) p( y) )}{p( x)p( y)},
\]
for $\k$ characteristic to $\DL{1}$ with $\HK \subset C_0^1(\mathbb{R}^d)$
and $a(\twonorm{ x}) \defn (\sigma^2 + \twonorm{ x}^2)^{-\gamma}$ a tilting function with $\sigma > 0$ and $\gamma \leq u$.
The following statements hold true:
\begin{enumerate}[label=\textup{(\alph*)}]
    \item If $P\in\embedtozero$, then $\H_{\ks}$ $P$-dominates indicators.
    \item If $P\in\embedtozero$, $\gamma \geq 0$, and $\twonorm{s_p(x)} \leq (\sigma^2 + \twonorm{ x}^2)^{\gamma}$, then $\ks$ is bounded $P$-separating and controls weak $P$-convergence.
    \item If $\twonorm{s_p( x)}\cdot \twonorm{x} \leq (\sigma^2 + \twonorm{ x}^2)^{\gamma}$  and $s_p\in C(\mathbb{R}^d)$, then $\H_{\ks} \subset C(\mathbb{R}^d)$  and $\ks$ metrizes weak $P$-convergence; i.e. $\mathrm{KSD}_P(Q_n) \rightarrow 0$ if and only if $\mathrm{BL}_{\|\cdot\|}(P,Q_n) \rightarrow 0$.
\end{enumerate}
\end{theorem}

This discussion focused on controlling weak convergence with Langevin \acp{ksd}, but stronger notions of convergence are also available, and one might want to understand when these can be controlled using a \ac{ksd}.
The contribution of \citet{kanagawa2022controlling} was to establish sufficient conditions on $P$ and on $K$ under which the Langevin \ac{ksd} provides control over convergence in the sense of the Wasserstein metrics (c.f. \Cref{def: Wasserstein}); i.e. $\mathrm{KSD}_P(Q_n) \rightarrow 0$ implies $\mathrm{W}_{\|\cdot\|}^s(P,Q_n) \rightarrow 0$.
Convergence in $s$-Wasserstein is equivalent to weak convergence plus convergence in moments up to order $s$ when $\mathcal{X} = \mathbb{R}^d$, and it is thus clear that a stronger conclusion is obtained.

\subsection{Constrained Kernel Stein Discrepancies}\label{subsec: constrained ksdsl}

Consider now a constrained domain  $\xset\subset \reals^d$ with boundary  $\partial\xset$ and $n(x)$ the outward-facing unit normal to $\partial\xset$ at $x \in \partial\xset$. 
Selecting a kernel that vanishes on the boundary, i.e., $K(\cdot,x) n(x) = 0$ for all $x\in\partial\xset$, also ensures that each element $g$ of the kernel Stein set $\gset_K$ vanishes on the boundary in the sense of \cref{eq:vanishing-boundary-condition}.
This provides a particularly convenient way to ensure that 
\begin{align*}
    \oint_{\partial \mathcal{X}} p(x) g(x) \cdot \mathrm{n}(x) \; \mathrm{d}x
\end{align*}
and hence that $P(\mathcal{T}_P g) = 0$ via the divergence theorem \cref{thm: divergence}.

\subsection{Kernel Stein Discrepancies in General}\label{subsec: ksds in general}

Kernel Stein sets are computationally convenient, circumventing the need to numerically evaluate the supremum in \Cref{eq: SD def}. 
The previous sections focused on the combination of the Langevin Stein operator and a kernel Stein set, but several other choices of Stein operators are available (c.f. \Cref{chap: Stein operators}) and can also be combined with a kernel Stein set. 
For example, \emph{diffusion KSDs},  combining the diffusion Stein operators (\Cref{def: diffusion operator}) with kernel Stein sets, were introduced in \citet{barp2019minimum}, where sufficient conditions for separation were presented, and further analyzed in \citet{kanagawa2022controlling}, where conditions for Wasserstein convergence control were presented. Squared diffusion KSDs have also been applied and analyzed in the setting of causal model learning under the name \emph{kernel deviation from stationary} \citep{lorch2024causal,bleile2026efficient}.
The combination of mirrored Stein operators (\Cref{def: mirrored stein operator}) and kernel Stein sets was studied in \citet{shi2022sampling}, where sufficient conditions for weak convergence control were presented.
The gradient-free Stein operator (\Cref{def: gf stein op}) was paired with a kernel Stein set in \citet{fisher2023gradient}, with conditions for convergence detection and control established.
The combination of discrete Stein operators (\Cref{sec: discrete operators}) and kernel Stein sets was considered in \citet{yang2018goodness}, where sufficient conditions for separation were presented.
In addition, extension of \acp{ksd} to more complicated domains, such as high-dimensional Euclidean spaces \citep{gongsliced}, sequence spaces \citep{baum2023kernel}, Riemannian manifolds \citep{xu2020stein,barp2022riemann}, and infinite-dimensional spaces of functions \citep{wynne2022kernel} have been studied.

\section{Random Feature Stein Discrepancies}
\label{sec: rfsd}

One potential drawback of graph and kernel Stein discrepancies is that both have computational costs that grow at least quadratically with the sample size \citep{gorham2017measuring}.
To address this limitation, \citet{huggins2018random} developed a class of Stein discrepancies that can be cheaply and accurately approximated with random sampling.
The idea is to consider a so-called \emph{feature} Stein set:

\begin{definition}[Feature Stein set; \citealp{huggins2018random}]\label{def:feature_set}
Consider a \emph{feature function} $\feat : \reals^d \times \reals^d \to \complex$ 
which,  for some $r \in [1,\infty)$ and all $x, z\in\reals^d$, satisfies $\feat(x,\cdot) \in \mathcal{L}^r(\mathbb{R}^d)$  and $\feat(\cdot, z) \in C^1$.
We define the associated \emph{feature Stein set} as
\begin{align*}
\gset_{\feat,r} 
	\defined \Bigg\{ g : \reals^{d} \to \reals \Bigg| &  g_i(x) = \int \feat(x,z) \overline{f_i(z)} \dz 
		\qtext{  with  } \\
		& \qquad \sum_{i=1}^d  \norm{f_i}_{\mathcal{L}^s(\mathbb{R}^d)}^2 \leq 1
		\text{ for }
		s = \frac{r}{r-1}
		\Bigg\}.
\end{align*}
\end{definition}

\noindent Naturally a Stein discrepancy associated with a feature Stein set is called a \emph{feature} Stein discrepancy:

\begin{definition}[Feature Stein discrepancy; \citealp{huggins2018random}]
\label{def: feature SD}
Suppose $\xset$ is a convex subset of $\Rd$ and 
$P$ has density $p > 0$ on $\xset$.
If $\gset_{\feat,r}$ is a feature Stein set (\cref{def:feature_set})
and 
$\mathcal{T}_P$ is the Langevin Stein operator (\Cref{def: Lang SO}) for $P$, then 
we call 
$$
\FSD_{\feat,r}(\cdot,\targetdist) \defined \mathcal{S}(\cdot,\mathcal{T}_P,\gset_{\feat,r})
$$ 
a \emph{feature Stein discrepancy}.
\end{definition}

For the presentation in this section it is convenient to decompose the Langevin Stein operator as a sum
\begin{align}
    (\mathcal{T}_P g)(x) = \sum_{i=1}^d \partial_{x_i} g(x) + g_i(x) \partial_{x_i} (\log p)(x) 
    = \sum_{i=1}^d (\mathcal{T}_i g)(x) \label{eq: Lang op as sum}
\end{align}
where $\mathcal{T}_i$ is used to indicate the $i$th coordinate of $\mathcal{T}_P$, with the dependence on (fixed) $P$ left implicit.
For a bivariate function $\Phi(\cdot,\cdot)$ we let $\mathcal{T}_i \Phi$ denote the action of $\mathcal{T}_i$ on the first argument of $\Phi$.

\begin{proposition}[\FSD explicit form; \citealp{huggins2018random}]
In the setting of \Cref{def: feature SD},
\begin{align}
\FSD_{\feat,r}^2(Q,\targetdist) = \sum_{i=1}^d \staticnorm{Q(\opsub{i}{\feat})}_{\mathcal{L}^r(\mathbb{R}^d)}^2
\label{eqn:fsd}
\end{align}
for all $Q \in \mathcal{P}(\mathcal{X})$.
\end{proposition}
\begin{proof}
Using the notation in \eqref{eq: Lang op as sum}, and the definition of $\gset_{\feat,r}$ in \Cref{def:feature_set},
\begin{align*}
\FSD_{\feat,r}^2(Q,\targetdist) 
	&\defined \sup_{g\in\gset_{\feat,r}} |Q(\mathcal{T}_P g)|^2
	= \sup_{g\in\gset_{\feat,r}} \left|\sum_{i=1}^d Q(\opsub{i}{g_i})\right|^2  \\ 
	& \hspace{-30pt} = \sup_{f : v_i = \norm{f_i}_{\mathcal{L}^s(\mathbb{R}^d)}, \twonorm{v} \leq 1} \left|\sum_{i=1}^d \int Q(\opsub{i}{\Phi})(z) \overline{f_i(z)} \dz\right|^2 \\
	& \hspace{-50pt}  = \sup_{v : \twonorm{v} \leq 1} \left|\sum_{i=1}^d \norm{Q(\opsub{i}{\Phi})}_{\mathcal{L}^r(\mathbb{R}^d)} v_i\right|^2
	= \sum_{i=1}^d \staticnorm{Q(\opsub{i}{\feat})}_{\mathcal{L}^r(\mathbb{R}^d)}^2 , 
\end{align*}
as required.
\end{proof}

At this point no reduction in computational cost has been achieved; for instance when $r = 2$, the \FSD is simply the Langevin \ac{ksd} with base kernel
$K(x,y) = k(x,y) I_{d \times d}$ with $k(x,y) = \int \Phi(x,z)\overline{\Phi(y,z)} \dz $, and as such it is associated with a quadratic computational cost.
Indeed,
\begin{align*}
\mathrm{KSD}_P^2(Q) 
	& = \sum_{i=1}^d (Q \times Q)((\opsub{i} \otimes \opsub{i})\basekernel) \\
	& = \sum_{i=1}^d \staticnorm{Q(\opsub{i}{\feat})}_{\mathcal{L}^2(\mathbb{R}^d)}^2
	= \FSD_{\Phi,2}(Q,P)^2.
\end{align*}
To reduce the computational cost some form of approximation is needed.
Approximating the norm term $\staticnorm{Q(\opsub{i}{\feat})}_{\mathcal{L}^r(\mathbb{R}^d)}$ using a importance sampling leads to a randomized approximation of the feature Stein discrepancy:

\begin{definition}[Random feature Stein discrepancy; \citealp{huggins2018random}]\label{def:rfsd}
For a target feature Stein discrepancy $\FSD_{\feat,r}$ and an importance sampling distribution with (Lebesgue) density $\nu$ on $\mathbb{R}^d$, we define the \emph{random feature Stein discrepancy} 
\begin{align*}
\RFSD_{\feat,r,\isdist,M}^2(Q,\targetdist)
	&\defined \sum_{i=1}^d \left(\frac{1}{m}\sum_{j=1}^M \frac{|Q(\opsub{i}{\feat})(Z_j)|^r}{\isdist(Z_j)} \right)^{2/r} 
\end{align*}
for $Z_1, \dots, Z_M \distiid \isdist$ and $Q \in \mathcal{P}(\mathcal{X})$.
\end{definition}

\noindent For $Q = \sum_{i=1}^n w_i \delta_{x_i}$ with finite support, the \RFSD computation reduces to computing $mnd$ scaled \emph{random features}, $(\opsub{i}{\feat})(x_i, Z_j)/\isdist(Z_j)^{1/r}$, and this evaluation can be carried out in parallel.

\subsection{Special Cases of $\RFSD$}

One can straightforwardly use the $\RFSD$ framework of \cref{def:rfsd} to construct low-cost approximations to standard KSDs. 
For example when the \ac{ksd} base kernel $k(x,y) = \Psi(x-y)$ for some $\Psi \in \mathcal{L}^2(\mathbb{R}^d)$, one can choose $r=2$, $\nu\propto \FTop(\Psi)$, and $\feat(x,z) = e^{-i\inner{z}{x}}\FTop(\Psi)(z)^{1/2}$ to design a random Fourier feature~\citep[RFF]{rahimi2007random} approximation $\RFSD_{\feat,2,\isdist,M}$ to $\KSD_{\basekernel}$ \citep{huggins2018random}. 
While eminently practical, this RFF-KSD does suffer from a known limitation: for an uncountable number of distributions $Q\neq P$, this $\RFSD_{\feat,2,\isdist,M}(Q,P) = 0$ with high probability \citep[Prop. 1]{chwialkowski2015fast}.

To avoid this undesirable characteristic, \citet{jitkrittum2017linear} developed an alternative $\RFSD$ called the random finite set Stein discrepancy (FSSD-rand). 
In the notation of \cref{def:rfsd}, FSSD-rand is obtained by selecting $r=2$, any feature count $M$ and importance sampling distribution $\isdist$, and $\feat(x,z) = f(x,z) \isdist(z)^{1/2}$ where $f$ is a real analytic and $C_0$-universal \citep[Def.~4.1]{Carmeli2010} reproducing kernel. 
The real analyticity in particular ensures that for each $Q\neq P$, $\P(\RFSD_{\feat,2,\isdist,M}(Q,P) > 0) = 1$ \citep[Thm.~1]{jitkrittum2017linear}.
In \cref{sec:choosing-feat,sec:selecting-isdist}, we will see that $\RFSD$ features of a different form additionally give rise to strong convergence-determining properties and that selecting $r \neq 2$ can yield provably accurate approximations with a substantially smaller sampling budget $M$.

\subsection{Selecting a Feature Function $\feat$} \label{sec:choosing-feat}

The main concern when selecting a feature function $\feat$ is that we do not `lose' information that would be useful in comparing an approximation $Q_n$ to the target $P$.
The content of \Cref{prop:fsd-lower-bound} is to establish conditions under which $\FSD$ controls the Langevin KSD, whose properties are now well-understood (c.f. \Cref{subsec: lksd in detail}).
Let $\FTop(f)$ denote the generalized Fourier transform of $f : \mathbb{R}^d \rightarrow \mathbb{C}$.
The following follows from the generalized H\"{o}lder inequality and the Babenko--Beckner inequality:

\begin{proposition}[Controlling Langevin KSD with \FSD; Prop.~3.1 of \citealp{huggins2018random}]\label{prop:fsd-lower-bound}
If ${r \in [1,2]}$, $\rho\in \mathcal{L}^t(\mathbb{R}^d)$ for $t = {r}{/(2-r)}$, and $K(x,x') = k(x,x') I_{d \times d}$ with
\begin{align*}
    k(x, y) = \int \FTop(\feat(x,\cdot))(\omega) \overline{\FTop(\feat(y,\cdot))(\omega)}\rho(\omega)\domega ,
\end{align*}
then $\mathrm{KSD}_P^2(Q_n)  \leq  \textstyle\staticnorm{\rho}_{\mathcal{L}^t(\mathbb{R}^d)}\FSD^{2}_{\feat,r}(Q_n,P)$.
\end{proposition}

That is, for a kernel $k$ of the above form, convergence of the \FSD implies convergence of the Langevin KSD.
Once we have selected a feature function, we can now set about approximating it using so-called \emph{random features} \citep{rahimi2007random}.
Before presenting results for convergence detection and control using random features, we first describe assumptions on the base kernel.

\begin{assumption} \label{asm:basekernel-form}
The base kernel has the form $k(x, y) = A_n(x) \Psi(x - y) A_n(y)$ 
for $\Psi \in C^2(\mathbb{R}^d)$, $A \in C^1(\mathbb{R}^d)$, and $A_n(x) \defined A(x - m_n)$,
where $A >0$, $\grad \log A$ is bounded and Lipschitz, and $m_n \defined \int x \; \mathrm{d}Q_n(x)$. 
\end{assumption}

\noindent Notice that \Cref{asm:basekernel-form} allows the base kernel $k$ to be dependent (via the sequence index $n$) on the distribution $Q_n$ whose approximation quality is being assessed. %

\begin{assumption} \label{asm:feature-form}
\cref{asm:basekernel-form} holds and $\Phi(x, z) = A_n(x) F(x - z)$, where
$F \in C^{1}(\mathbb{R}^d)$ is positive, and there exist a norm $\norm{\cdot}$ and constants $s, C_1 > 0$ such that 
\[
|\partial_{x_d} \log \statfeat(x)| \le C_1 (1 + \norm{x}^{s}),
~~ 
\lim_{\norm{x} \to \infty}(1 + \norm{x}^{s})\statfeat(x) = 0, 
\]
and
\[
\statfeat(x - z) \le C_1 \frac{\statfeat(z)}{\statfeat(x)}.
\]
In addition, there exist a constant $C_2 \in (0, 1]$ and continuous, non-increasing function $\statfeatscalar$ such 
that $C_2 \,\statfeatscalar(\norm{x}) \le \statfeat(x) \le \statfeatscalar(\norm{x})$. 
\end{assumption}

The following result establishes that convergence in a \emph{tilted} 1-Wasserstein distance (c.f. \Cref{subsec: wass metric}) is detected by an appropriate $\RFSD$:

\begin{theorem}[Detecting Convergence with $\RFSD$; Prop.~3.3 of \citealp{huggins2018random}] \label{prop:KSD-upper-bound-for-FSD-and-RFSD}
Suppose \cref{asm:feature-form} holds with $F\in \mathcal{L}^r(\mathbb{R}^d)$, 
$1/A$ bounded, $x \mapsto x/A(x)$ Lipschitz,  and $\Esub{P}[{A}(Z) \twonorm{Z}^2] <\infty$.
Then 
$$
\mathrm{W}_{\|\cdot\|_2}^{1,A_n}(Q_n,P) \rightarrow 0 \implies \left\{ \begin{array}{l} \FSD_{\feat,r}(Q_n, P) \to 0 \\ \RFSD_{\feat,r,\isdist_n,M_n}(Q_n , P) \convP 0 \end{array} \right. 
$$
for any choices of $r\in[1,2]$, $\nu_n$, and $M_n \geq 1$.
\end{theorem}

\subsection{Selecting an Importance Sampling Distribution $\isdist$}
\label{sec:selecting-isdist}

The choice of importance sampling distribution $\nu$ is motivated by having $\RFSD$ close to its reference $\FSD$ even when the importance sample size $M$ is small. 
The strategy pursued in \citet{huggins2018random} was to choose $\nu$ such that the second moment of each random feature $|Q(\mathcal{T}_i \Phi)(Z)|^r / \nu(Z)$ is bounded by a
power of its mean:

\begin{definition}[$(C,\gamma)$ second moments] \label{defn:second-moment-property}
Fix a target distribution $\targetdist$.
For $Z \dist \isdist$, $i \in \{1,\dots,d\}$, and $n \ge 1$, let $Y_{n,i} \defined w_i(Z, Q_n) \defined  |(Q_n \mathcal{T}_i \Phi)(Z)|^r/\nu(Z)$. %
If for some $C > 0$ and $\gamma \in [0,2]$ we have
$
\EE[Y_{n,i}^2] \le C \EE[Y_{n,i}]^{2-\gamma}
$
for all $i \in \{1 , \dots , d\}$ and $n \ge 1$,
then we say \emph{$(\feat, r, \isdist)$ yields $(C, \gamma)$ second moments for $P$ and $Q_n$}. 
\end{definition}
\begin{theorem}[Controlling Langevin KSD with $\RFSD$; Prop.~3.6 of \citealp{huggins2018random}] \label{prop:estimated-ksd-lb-guarantee}
Suppose $(\feat,r,\isdist)$ yields $(C,\gamma)$ second moments for $P$ and $Q_n$.
If $M \ge 2C\EE[Y_{n,i}]^{-\gamma}\log(d/\delta)/\eps^2$ for all $i \in \{1 , \dots , d\}$, then,
with probability at least $1-\delta$, 
\[
\RFSD_{\feat,r,\isdist,M}(Q_n, P)
	\ge (1-\eps)^{1/r}\FSD_{\feat,r}(Q_n , \targetdist).
\]
Under the further assumptions of \cref{prop:fsd-lower-bound}, if the reference $\mathrm{KSD}_P(Q_n) \geq c_1 n^{-1/2}$ for some $c_1$, then there exists $c_2 > 0$ such that a sample size 
$$
M \geq \frac{c_2 C n^{\gamma r/2} \norm{\rho}_{\mathcal{L}^t(\mathbb{R}^d)}^{\gamma r/2}\log(d/\delta)}{\eps^2}
$$ 
suffices to have, with probability at least $1-\delta$,  
\[
\staticnorm{\rho}_{\mathcal{L}^t(\mathbb{R}^d)}^{1/2}\RFSD_{\feat,r,\isdist,M}(Q_n , \targetdist)
	\ge (1-\eps)^{1/r} \mathrm{KSD}_P(Q_n).
\]
\end{theorem}

\noindent There are several settings where $\mathrm{KSD}_P(Q_n) \gtrsim n^{-1/2}$ holds; most notably when the states $\{x_i\}_{i=1}^n$ are independent and identically distributed.
This makes the $\RFSD$ suitable for goodness-of-fit testing (c.f. \Cref{sec: GoF}), for example, since in this context the performance of the Stein discrepancy test statistic under the IID null is of interest.
A smaller $r$ leads to significant improvements in the sample complexity $M = \Omega(n^{\gamma r / 2})$. 
For example, if the weight function is bounded (so that $\gamma = 1$), it suffices to choose $r=1$ and $M = \Omega(n^{1/2})$. In what follows, we will demonstrate how to select $\nu$ with $\gamma$ 
arbitrarily close to 0.

To make use of \Cref{prop:estimated-ksd-lb-guarantee} it remains to establish when $(\feat,r,\isdist)$ yields $(C,\gamma)$ second moments for $P$ and $Q_n$.

\begin{assumption}[Distant Dissipativity; \citealp{Eberle2015,Gorham2019}] \label{ass: distant diss}
    The distribution $P$ has $\nabla \log p$ Lipschitz and there exist $\kappa > 0$ and $r \geq 0$ for which the $(\kappa,r)$-\emph{distant dissipativity} condition 
\begin{align*}
    \inner{\grad \log p(x)-\grad \log p(y)}{x-y} \leq - \kappa \twonorm{x-y}^2 + r, \qquad \forall x,y\in\reals^d
\end{align*}
is satisfied.
\end{assumption}

The following provides simple conditions and a choice for $\nu$ under which $(C,1)$ second moments are guaranteed:

\begin{proposition}[Prop.~3.7 of \citealp{huggins2018random}]  \label{prop:c-1-second-moment}
Assume that \cref{asm:basekernel-form,asm:feature-form} hold with $s=0$, that \Cref{ass: distant diss} holds, and there exists a constant $\mcC' > 0$ such that for all $n \ge 1$,
$Q_n([1 + \norm{\cdot}] A_{n}) \le \mcC'$.
If $\isdist(z) \propto Q_n([1 + \norm{\cdot}] \feat(\cdot, z))$,
then for any $r \ge 1$, $(\feat,r,\isdist)$ yields $(C,1)$ second moments for $P$ and $Q_n$. 
\end{proposition}

In order to obtain $(C,\gamma)$ moments for $\gamma < 1$, we will choose $\nu$ such that $w_i(z,Q_n)$ decays sufficiently quickly as $\|z\| \rightarrow \infty$.
For this, two integrability conditions involving the Fourier transforms of $\Psi$ and $F$ are required:

\begin{assumption} \label{asm:FT-Psi-decay-plus-Lt}
\cref{asm:feature-form,asm:basekernel-form} hold, $\omega_{1}^{2} \FTop(\statkernel)^{1/2}(\omega) \in \mathcal{L}^1(\mathbb{R}^d)$, and
for $t = r/(2 - r)$, $\FTop(\statkernel)/\FTop(\statfeat)^{2} \in \mathcal{L}^t(\mathbb{R}^d)$. 
\end{assumption}

\noindent The $\mathcal{L}^1$ condition is weak, while the $\mathcal{L}^t$ condition ensures that \Cref{prop:fsd-lower-bound} applies to our chosen $\FSD$.
\Cref{thm:RFSD-c-alpha-second-moment} shows that one can improve the importance sample growth rate $\gamma$ of an $\RFSD$ by increasing the smoothness $\lambda$ of $F$ and decreasing the over-dispersion parameter $\xi$ of $\nu$.

\begin{theorem}[Thm.~3.8 of \citealp{huggins2018random}] \label{thm:RFSD-c-alpha-second-moment}
Let \cref{asm:feature-form,asm:basekernel-form,ass: distant diss,asm:FT-Psi-decay-plus-Lt} hold, and suppose there exists $\mcC > 0$ such that,
\[
Q_n\left([1 + \norm{\cdot} + \norm{\cdot - m_n}^{s}] A_{n}/\statfeat(\cdot - m_n)\right) 
\le \mcC \qtext{for all} n \ge 1. \label{eq:strong-Q-moment-bound}
\]
Then there is a constant $b \in [0, 1)$ such that the following holds.
For any $\xi \in (0, 1 - b)$, $c > 0$, and  $\alpha > 2(1 - \overline{\lambda})$, 
if $\isdist(z) \ge c\,\statkernel(z - m_n)^{\xi r}$, then
there exists a constant $C_{\alpha} > 0$ such that
$(\feat,r,\isdist)$ yields $(C_{\alpha}, \gamma_{\alpha})$ second moments 
for $P$ and $Q_n$, where $\gamma_{\alpha} \defined \alpha + (2 - \alpha)\xi/(2-b-\xi)$. 
\end{theorem}

\section{Stochastic Stein Discrepancies}

The final class of Stein discrepancies that we consider were developed to deal with the scenario where the target distribution $P$ has a density that is the product
\begin{align*}
    p(x) = \prod_{l=1}^L p_l(x) 
\end{align*}
of a large number of factors $p_l(x)$, which need not individually integrate to 1. 
This structure arises in the \emph{tall data} context, where $x$ represents the parameter of a statistical model conditional on which data are treated as independent.
The principal challenge with using a Langevin Stein discrepancy in this setting is that computation of the gradient $\nabla \log p$ requires summing over $L$ separate terms, which for discrete $Q_n = \frac{1}{n} \sum_{i=1}^n \delta_{x_i}$ entails a $O(nL)$ cost, this can be a non-trivial computational requirement.
However, this is the same issue that is encountered in empirical risk minimization in machine learning, where powerful computational solutions have been proposed based on sub-sampling.
\emph{Stochastic Stein discrepancies} aim to also exploit sub-sampling to provide a meaningful discrepancy at reduced computational cost.
The main idea is to consider a decomposable Stein operator of the form $\mathcal{T}_P = \sum_{l=1}^L \mathcal{T}_l$ where $\mathcal{T}_l$ is a Stein operator targeting $p_l$, albeit $p_l$ need not be a normalized probability distribution in general.
Then, for a subset $\sset\subseteq [L]$, we denote the subset operator
$\opsub{\sset} \defeq \sum_{l \in \sset} \opsub{l}$ whose computational cost is linear in the size $|\sigma|$ of the subset.

\begin{definition}[Stochastic Stein discrepancy; \citealp{gorham2020stochastic}]
Consider a decomposable Stein operator of the form $\mathcal{T}_P = \sum_{l=1}^L \mathcal{T}_l$ and $Q_n = \frac{1}{n} \sum_{i=1}^n \delta_{x_i}$.
Fix a batch size $m$ and, for each $i \in [n]$, independently select a uniformly random subset $\sset_i$ of size $m$ from $[L]$. 
Then for any Stein set $\gset$ in the domain of each of the $\mathcal{T}_i$, we define the \emph{stochastic Stein discrepancy} as the random quantity
\begin{align}\label{ssd_subset}
    \mathcal{SS}(Q_n , \mathcal{T}_P , \mathcal{G}) 
        &\defeq \sup_{g\in\gset}\left |\frac{1}{n}\sum_{i=1}^n
    \frac{L}{m}\opsubarg{\sset_i}{g}{x_i}\right |.
\end{align}
\end{definition}

Although the \ac{ssd} coincides with the standard Stein discrepancy $\mathcal{S}(Q_n , \mathcal{T}_P , \mathcal{G})$ when the batch size $m$ is equal to $L$, we do not view \ac{ssd} as an attempt to approximate Stein discrepancy.
Rather, \ac{ssd} is a discrepancy \emph{in its own right} and is accompanied by (probabilistic) guarantees of convergence detection (\Cref{subsec: detect SSD}) and convergence control (\Cref{subsec: SSD control}).

\subsection{Detecting Convergence with SSDs}
\label{subsec: detect SSD}

The following allows for an evolving sequence of Stein sets $\mathcal{G}_n$ to accommodate the graph Stein sets of \citet{gorham2015measuring,Gorham2019}.
Let $\equisubs{L}{m} \defeq\{\sset\subseteq [L] : |\sset| = m\}$  represent all size $m$ subsets of $[L]$.

\begin{theorem}[SSDs detect convergence; Thm.~2 of \citealp{gorham2020stochastic}]
\label{ssds-detect-convergence}
Let $\mathcal{X}$ be a convex subset of $\mathbb{R}^d$.
Let $\mathcal{G} = \cup_{n \geq 1} \mathcal{G}_n$.
Consider a decomposable Stein operator of the form $\mathcal{T}_P = \sum_{l=1}^L \mathcal{T}_l$ where $P(\mathcal{T}_P \mathcal{G}) = \{0\}$.
Suppose that for some $a, c > 0$ and each $\sset\in\equisubs{L}{m}$, $\mathcal{T}_\sigma \mathcal{G} \subset C(\xset)$,
\begin{align*}
& \sup_{g \in \mathcal{G}} |\opsubarg{\sset}{g}{x}| \leq c(1+\twonorm{x}^a) \\
& \sup_{g \in \mathcal{G}} \sup_{x,y \in K} \frac{|\opsubarg{\sset}{g}{x} - \opsubarg{\sset}{g}{y}|}{\twonorm{x-y}} < \infty
\end{align*}
for each compact set $K$.
If $\mathrm{W}_{\|\cdot\|_2}^a(Q_n, P) \to 0$, then $\mathcal{SS}(Q_n , \mathcal{T}_P , \mathcal{G}_n) \toas 0$. 
\end{theorem}

\noindent \Cref{ssds-detect-convergence} shows that SSD detects Wasserstein convergence with probability 1 if the operators $\mathcal{T}_\sigma$ generate continuous functions that grow no more quickly than a polynomial and have locally bounded derivatives.

\subsection{Controlling Convergence with SSDs}
\label{subsec: SSD control}

First, to each Stein discrepancy $\mathcal{S}(Q_n,\mathcal{T}_P,\mathcal{G}_n)$ a \emph{bounded Stein discrepancy} based on the modified Stein set
\begin{align}
    \mathcal{G}_{b,n} \defeq \left\{g \in \mathcal{G}_n : \|\mathcal{T}_\sigma g\|_\infty \leq 1 \forall \sigma \in {\textstyle \equisubs{L}{m} } \right\} \label{eq: bounded SS}
\end{align}
in which each Stein function is constrained to be bounded under each subset operator $\mathcal{T}_\sigma$.

\begin{theorem}[Bounded Stein discrepancies control tight convergence; Thm.~3 of \citealp{gorham2020stochastic}]
\label{bounded-sds-detect-tight-non-convergence}
Let $\mathcal{X} = \mathbb{R}^d$.
Consider the Langevin Stein operator $\langevin{}$ (c.f. \Cref{def: Lang SO}) satisfying \Cref{ass: distant diss}.
Suppose $\sup_{x\in\xset}{\twonorm{\grad\log p_{\sset}(x)}}{/(1 + \twonorm{x})} < \infty$ for each $\sset\in\equisubs{L}{m}$, fix a sequence of probability measures $(Q_n)_{n=1}^\infty$, 
and 
consider the bounded Stein set 
$\gset_{b,n}$ in \eqref{eq: bounded SS}
for any of the following sets $\gset_n$:
\begin{enumerate}[label=(\textbf{A.\arabic*})]
    \item\label{classical-set} $\gset_n = \steinset$, the classical Stein set with arbitrary vector norm $\norm{\cdot}$ (\Cref{def:classical_set}).
    \item\label{graph-set}
    $\gset_n = \gsteinset{}{G}$, the graph Stein set with arbitrary vector norm $\norm{\cdot}$ and a finite graph $G = (V,E)$ with vertices $V \subset \xset$ (\Cref{def:graph_set}).  
    \item \label{kernel-set} $\gset_n = \mathcal{G}_K$, the kernel Stein set with $K(x,y) = \Phi(x-y) I_{d \times d}$ for $\Phi \in C^2(\mathbb{R}^d , \mathbb{R})$ with non-vanishing Fourier transform (\Cref{def:kernel_stein_discrepancy}).
\end{enumerate}
If $Q_n \not\Rightarrow P$, then either $\langstein{Q_n}{\gset_{b,n}}\not\to 0$ or $(Q_n)_{n=1}^\infty$ is not tight (c.f. \Cref{def: tight}).
\end{theorem}

\begin{theorem}[SSDs control bounded Stein discrepancy convergence; Thm.~4 of \citealp{gorham2020stochastic}]
\label{ssds-detect-bounded-sd-non-convergence}
If $\opstein{Q_n}{\gset_{b,n}} \not\to 0$, then, with probability $1$,
$\ssdn{Q_n} \not\to 0$. 
\end{theorem}

\begin{proposition}[Coercive SSDs enforce tightness; Prop.~5 of \citealp{gorham2020stochastic}]
\label{ssd-tightness}
If $(Q_n)_{n=1}^\infty$ is not tight and $\frac{L}{m}\opsub{\sset} g$ is coercive and bounded below for some $g\in
\bigcap_{n= 1}^\infty \gset_n$ and all $\sset\in\equisubs{L}{m}$, then surely $\ssdn{Q_n} \not\to 0$. %
\end{proposition}

Taken together, these results imply that \acp{ssd} equipped with the Langevin operator and any of the Stein sets in \Cref{bounded-sds-detect-tight-non-convergence} control convergence with probability $1$ under standard dissipativity and growth conditions on the subsampled operator:

\begin{theorem}[SSDs control convergence; Thm.~6 of \citealp{gorham2020stochastic}]
\label{coercive-ssds-detect-non-convergence}
Under the notation of \cref{bounded-sds-detect-tight-non-convergence}, suppose 
$\grad \log p$ is Lipschitz, $\sup_{x\in \mathbb{R}^d } \twonorm{\grad\log p_{\sset}(x)} / (1 + \twonorm{x}) < \infty$ for all $\sset\in\equisubs{L}{m}$, and, for some $\kappa>0$ and $r\geq 0$, a $(\kappa,r)$-distant dissipativity condition (c.f. \Cref{ass: distant diss}) holds uniformly for all $p_\sigma$ with $\sset\in\equisubs{L}{m}$.
Consider the radial function $\Phi(x) \defeq (1 + \twonorm{\Gamma x}^2)^{\beta}$ for $\beta \in (-1, 0)$ and any positive definite matrix $\Gamma$. 
Let $\mathcal{G}_n$ be any of the Stein sets in \Cref{bounded-sds-detect-tight-non-convergence}.
If $Q_n \not\Rightarrow P$, then, with probability $1$,
$\langssdn{Q_n} \not\to 0$.
\end{theorem}

\chapter{Stein Dynamics} 
\label{chap: Stein dynamics}

The previous Chapter discussed Stein discrepancy as a statistical divergence with attractive computational properties. 
In this Chapter we consider inference and learning via minimization of Kullback--Leilber divergence, drawing an important connection to Stein's method via mass transport.

\section{Dynamic Mass Transport}

Assume we are interested in transforming one distribution $Q$ into another distribution $P$ on the same domain $\mathcal{X} \subseteq \mathbb{R}^d$.
There are two different approaches to achieve this transformation.

One approach involves simulating from $Q$ and assigning each realization a weight based on the density ratio
$r = \frac{\dd P}{\dd Q}$,
which is related to techniques such as importance sampling, thinning, and birth-death processes.
However, the performance of this reweighting method tends to deteriorate when the variance of the density ratio
$\frac{\dd P}{\dd Q}$ between $Q$ and $P$ is high.

Another approach is transport-based, which applies a transformation to convert a random variable distributed according to $Q$ into a random variable distributed according to $P$.
This method allows gradient information of the density function of $P$ to be leveraged, and thus has the potential to avoid degradation issues associated with importance sampling and related methods.

Specifically, the transport-based method seeks to design a \emph{transport map}
$T\colon \mathcal{X}\to \mathcal{X}$
such that $T(Z)$ has distribution $P$ whenever $Z\sim Q$, which we denote as $P = T_\sharp Q$.
Assume $T$ is a differentiable bijection so that, by the change-of-variables formula for densities,
\begin{align}\label{equ:qpdet}
p(x) = q(T^{-1}(x))\,\bigl|\det(\nabla T^{-1}(x))\bigr|,
\end{align}
where $q$ and $p$ are the density functions of $Q$ and $P$, respectively.
Unfortunately, \eqref{equ:qpdet} depends on $T$ in a complicated nonlinear fashion, and practical methods based on \eqref{equ:qpdet} can be computationally challenging in high dimensions due to the presence of the Jacobian determinant.

This difficulty can be sidestepped by adopting a dynamic viewpoint, treating $T$ as the flow map of a continuous-time dynamical system.
Specifically, we assume that a random variable $Z_0$ drawn from $Q$ evolves continuously according to an  ordinary differential equation
\begin{align}\label{equ:odedt}
\frac{\df}{\df t} Z_t = v_t(Z_t), \qquad t \ge 0,
\end{align}
where $v \colon [0,\infty)\times \RR^d \to \RR^d$ is a continuously differentiable velocity field to be specified.
Starting from $Z_0 \sim Q$, our goal is to design the time-dependent velocity field $v_t$ so that the distribution of $Z_t$, denoted by $Q_t$, converges to the target distribution $P$ as $t \to \infty$.

Assume that the ODE \eqref{equ:odedt} admits a unique solution for every initialization in $\RR^d$,
and that the density $q_t$ of $Q_t$ exists and is continuously differentiable in $(t,x)$, with
$v_t$ continuously differentiable in $(t,x)$ as well.
Under these classical regularity conditions, the evolution of $Q_t$ is governed by the
\emph{continuity equation}
\begin{align}\label{equ:continueq}
\partial_t q_t(x)
=
- (\nabla \cdot (q_t v_t))(x).
\end{align}
Note that the continuity equation can be written in terms of the Langevin Stein operator $\mathcal{T}_{Q_t}$ associated with $q_t$ (cf. \Cref{def: Lang SO}), as
$\partial_t \log q_t(x) = - (\mathcal{T}_{Q_t} v_t)(x)$.

An alternative form of \eqref{equ:continueq} is \emph{Liouville’s formula}, which gives the total time derivative of $q_t(Z_t)$ when $Z_t$ evolves according to \eqref{equ:odedt}:
\begin{align}
\label{equ:liouville}
\frac{\df}{\df t} \log q_t(Z_t) = - (\nabla \!\cdot v_t)(Z_t).
\end{align}
This follows from the chain rule,
$\frac{\df}{\df t} q_t(Z_t) = \partial_t q_t(Z_t) + v_t(Z_t) \cdot \nabla q_t(Z_t)$,  
together with the continuity equation \eqref{equ:continueq}.

\section{Gradient Flow of the Kullback--Leibler Divergence}

One approach to designing the velocity field $v_t$ is to follow the gradient of a suitable divergence between the current distribution $Q_t$ and the target $P$.
The choice of divergence is crucial, as it must lead to a system that is computationally feasible to evaluate or approximate.
It turns out that the Kullback--Leibler divergence $\KL{Q_t}{P}$ is a particularly suitable candidate, thanks to the following result, which characterizes its rate of decrease under \eqref{equ:continueq} in connection with the Langevin Stein operator.

\begin{theorem}\label{thm: deriv KL}
Assume that the differential equation $\dot Z_t = v_t(Z_t)$ admits a unique solution on $[0,T]$ for every initialization in $\RR^d$.
Let $Q_t$ denote the law of $Z_t$, with density $q_t$.
Assume that $v_t$ and $\log q_t$ are continuously differentiable on $[0,T] \times \RR^d$.
Let $P$ be a probability measure with a positive and continuously differentiable density $p$, such that $\KL{Q_0}{P} < \infty$, and 
$\int_0^T Q_t(|(\mathcal{T}_P v_t)(Z_t)|)\df t < \infty,$ 
where $\mathcal{T}_P$ is the Langevin Stein operator for $P$.
Then
\begin{align}
\frac{\df}{\df t}\,
\KL{Q_t}{P}
=
-%
\,Q_t(\mathcal T_P v_t). 
\end{align} 
\end{theorem}

\begin{proof}  
Since $Z_t$ follows law $Q_t$, we have  
\[
\KL{Q_t}{P} = \E\left[\log q_t(Z_t) - \log p(Z_t)\right]. 
\]
Define $R_t = \log q_t(Z_t) - \log p(Z_t)$. Its time derivative is
\begin{align*}  
\dot R_t 
&= \frac{\df}{\df t} \log q_t(Z_t) - \frac{\df}{\df t} \log p(Z_t)  \\
&= - \nabla \cdot v_t(Z_t) - v_t(Z_t) \cdot \nabla \log  p(Z_t) \\
&= -\mathcal{T}_P v_t(Z_t),
\end{align*}
where we used Liouville’s formula \eqref{equ:liouville} for the first term,
and the chain rule for the second term,
\[
\frac{\df}{\df t} \log p(Z_t)
=
\nabla \log p(Z_t)^\top \dot Z_t
=
\nabla \log p(Z_t)^\top v_t(Z_t).
\]
Therefore,
\begin{align*}
\KL{Q_t}{P} - \KL{Q_0}{P} 
&= \E\left[R_t - R_0\right] \\
&= \E\left[\int_0^t \dot R_\tau\, \df \tau \right] \\
&= \int_0^t \E\left[\dot R_\tau\right]\, \df \tau \\
&= -\int_0^t \E\left[(\mathcal{T}_P v_\tau)(Z_\tau)\right]\, \df \tau,
\end{align*}
where we used Fubini--Tonelli’s theorem to exchange the order of integration, assuming  
\(\int_0^t \E [ \lvert (\mathcal{T}_P v_\tau)(Z_\tau)\rvert ]\, \df \tau < \infty\).
Differentiating the resulting identity and using continuity of  
\( t \mapsto \E \left[\mathcal T_P v_t(Z_t)\right] \) yields
$ 
\frac{\mathrm{d}}{\mathrm{d}t}\KL{Q_t}{P}
 = -\,\E\left[(\mathcal{T}_P v_t)(Z_t)\right],$
concluding the proof.
\end{proof}

\begin{remark}
We can interpret the result in discrete time by Taylor approximation.  
For a measure $Q$, 
let $Q' \defeq (\Id + \epsilon v)_\sharp Q$ be the law of $Z' \defeq Z + \epsilon v(Z)$ given $Z\sim Q$. Then we have 
\bb
\frac{1}{\epsilon}(\KL{(\Id + \epsilon v)_\sharp Q~}{~P}  
- \KL{Q}{P}) & 
\approx \frac{\dd}{\dd \epsilon} \KL{(\Id + \epsilon v)_\sharp Q~}{~P}  \bigg|_{\epsilon = 0}  \\ 
& = -  Q(\mathcal T_P v).
\ee 
Hence, the expectation of Stein operator $ Q(\mathcal T_P v)$ quantifies the rate at which Kullback--Leibler divergence $\KL{Q}{P}$ decreases when we apply an small displacement following vector field $v$ to $Q$.
\end{remark}

\section{Gradient Flow and Stein Discrepancy}  

Now that we have understood the relationship between the Kullback--Leibler divergence and the Langevin Stein operator, we seek a velocity field $v^*$ that \emph{efficiently} minimizes the Kullback--Leibler divergence, since this would enable computationally efficient measure transport. 
This can be formulated as the optimization problem
$$
v_t^* \in \argmax_{v_t \in \mathcal{G}_t} \left\{ - \ddt \KL{Q_t}{P}  \right\} , ~~~~~\forall t\geq0 ,
$$
where $\mathcal{G}_t$ is a candidate set of velocity fields. 
Using the derivative formula for the Kullback--Leibler divergence from \Cref{thm: deriv KL}, 
\begin{align}
v_t^* \in \argmax_{v_t \in \mathcal{G}_t} \; Q_t(\mathcal{T}_P v_t)  , ~~~~~\forall t\geq0 .  \label{eq: max of SD}
\end{align}
Thus, the maximum decreasing rate of Kullback--Leibler divergence is $Q_t(\mathcal{T}_P v_t^*)$ which, under appropriate regularity assumptions, coincides with the definition of Stein discrepancy $\mathcal{S}(Q_t , \mathcal{T}_P , \mathcal{G}_t)$ (c.f. \Cref{def:stein_discrepancy}). 
Hence, Stein discrepancy can be viewed as the maximal rate of decrease for the Kullback--Leibler divergence under mass transport, and the optimal vector field $v_t^*$ for mass transport is also the optimal element of the Stein set that realizes the Stein discrepancy. 

As we show in sequel, the dynamics with $v_t^*$ can be interpreted as a form of gradient flow of the Kullback--Leibler divergence, $\KL{Q_t}{P}$, if the candidate space $\mathcal{G}_t$ is defined as the tangent space of a geometric structure on the set of probability measures.  
Different choices of the candidate velocity set $\mathcal{G}_t$ allow us to derive different dynamics for mass transport. 
This includes Langevin diffusions, which leverage a candidate space $\mathcal{G}_t \subset \mathcal{L}^2(Q_t)$ so that  
$\mathcal{S}(Q_t , \mathcal{T}_P , \mathcal{G}_t)$ coincides with Fisher divergence (c.f. \Cref{prop: fisher as stein}); 
and Stein variational gradient descent, in which $\mathcal{G}_t$ is taken to be a ball of a \acf{rkhs}, and $\mathcal{S}(Q_t , \mathcal{T}_P , \mathcal{G}_t)$ coincides with a \acf{ksd} (c.f. \Cref{def:kernel_stein_discrepancy}). 

\begin{remark}
Assume $Q_t$ has reached the equilibrium, having converged to $P$. 
The Stein identity \eqref{eq: stein identity} is now equivalent to 
$$
P(\mathcal T_P v_t)  =-\frac{\dno }{\dno t} \KLL(Q_t || P)  
= 0, $$ 
which reflects the fact that $P$ is a stationary point of the Kullback--Leibler divergence $Q \mapsto \KLL( Q  || P)$. 
\end{remark}

\subsection*{Geometric Interpretation}

One may cast the optimization problem above in a geometric framework as a gradient flow of the Kullback--Leibler divergence functional over the space of probability measures, equipped with a suitable transport metric.
This viewpoint parallels the theory of gradient flows in spaces of distributions, developed for example in \citet{ambrosio2005gradient}, although here we only work at a formal and heuristic level.

Assume that the admissible velocity fields form the unit ball of a normed space,
\[
\mathcal{G}_t \defeq \{ v : \|v\|_t \le 1 \},
\]
where $\|\cdot\|_t$ is a (possibly time-dependent) norm to be specified later.
For two probability distributions $Q$ and $Q'$, we define a generalized transport cost
\begin{align}
\mathcal{W}(Q,Q')
\defeq \inf_{(q_t,v_t)}
\Bigg\{
& \Big( \int_0^1 \|v_t\|_t^2 \, \dd t \Big)^{1/2}
\;\; \text{s.t.} \label{eq: curly W} \\
& \partial_t q_t = - \nabla \cdot (v_t q_t), \qquad
Q_0 = Q, \quad Q_1 = Q'
\Bigg\}. \nonumber
\end{align}
This is a Benamou--Brenier formulation, generalizing the $2$-Wasserstein metric by replacing the Euclidean kinetic energy with the norm $\|\cdot\|_t$.

We now consider the gradient flow of $\KL{Q}{P}$ under the metric $\mathcal{W}$.
Formally, the steepest descent direction at $Q$ can be characterized by: %
\begin{equation}
\argmin_{Q'}
\left\{
\KL{Q'}{P}
\;:\;
\mathcal{W}(Q',Q) \le \epsilon
\right\},
\label{eq: proximal_KL}
\end{equation}
or by its penalized version
\[
\argmin_{Q'}
\left\{
\KL{Q'}{P} + \frac{1}{2\epsilon^2} \mathcal{W}^2(Q',Q)
\right\},
\]
which is standard in the theory of metric gradient flows.
Plugging the definition of $\mathcal W$, the optimization reduces to 
\[
\min_{\{v_t\}}
\left\{
\KL{Q^{v}}{P} + \frac{1}{2\epsilon^2} \int_0^1 \norm{v_t}_t^2 \df t 
\right\},
\]

To obtain a tractable first-order characterization, consider an infinitesimal perturbation of the form
\[
Q' = (\Id + \epsilon v)_\sharp Q,
\]
where $v$ is a smooth vector field and $\epsilon > 0$ is small.
Such a perturbation corresponds to a path solving the continuity equation with constant velocity $v_t \equiv v$, and therefore satisfies
\[
\mathcal{W}(Q',Q) \le \epsilon \|v\|_t .
\]
Restricting \eqref{eq: proximal_KL} to these perturbations yields the approximation
\begin{equation}
\argmin_{v}
\left\{
\KL{(\Id + \epsilon v)_\sharp Q}{P} - \KL{Q}{P}
\;:\;
\|v\|_t \le 1
\right\}.
\label{eq: first_order_KL}
\end{equation}
Next, using a first-order Taylor expansion of the Kullback--Leibler divergence along the pushforward map, we obtain
\[
\KL{(\Id + \epsilon v)_\sharp Q}{P}
=
\KL{Q}{P}
-
\epsilon \, Q(\mathcal{T}_P v)
+ o(\epsilon),
\]
where $\mathcal{T}_P$ denotes the Langevin Stein operator associated with $P$.
Substituting this expansion into \eqref{eq: first_order_KL} and neglecting higher-order terms, we arrive at
\[
\argmax_{v \in \mathcal{G}_t} \; Q(\mathcal{T}_P v),
\]
which recovers \eqref{eq: max of SD} as the steepest descent direction of the KL divergence under the transport metric $\mathcal{W}$.

From this geometric perspective, the Stein variational update can be viewed as a metric-dependent gradient flow of the Kullback--Leibler divergence.
Different choices of the norm $\|\cdot\|_t$, and hence of the admissible set $\mathcal{G}_t$, induce different geometries on the space of probability measures and lead to different dynamics.

\section{Langevin Diffusion as a Gradient Flow}

Assuming that $Q_t$ admits a sufficiently regular density $q_t$, by Stein's identity
\begin{align*}
Q_t (\mathcal T_Pv_t) 
& = \int \nabla \cdot v_t + v_t \cdot \nabla \log p \; \dno  Q_t  
 =  \langle \nabla \log p - \nabla \log q_t, ~v_t  \rangle_{L^2(Q_t)}
\end{align*} 
where $L^2(Q_t)$ denote the Lebesgue space (\Cref{ex: Leb is Hilb}).
Let $\mathcal{G}_t = \{f \colon \norm{f}_{L^2(Q_t)}\leq 1\}$ be the unit ball of $L^2(Q_t)$ equipped with norm $\norm{f}_{L^2(Q_t)} = Q_t(\norm{f}^2_2)^{1/2}$. 
The optimal  $ v_t \in L^2(Q_t)$ should solve
$$
\argmax_{ v_t } \; \langle \nabla \log p - \nabla \log q_t, ~v_t  \rangle_{L^2(Q_t)} ~~~~\text{s.t.}~~~ \norm{ v_t }_{L^2(Q_t)} \leq 1, 
$$
whose solution is $Q_t$-a.e. equal to
$$
 v_t ^* = \frac{ \nabla \log p- \nabla  \log q_t }{ \| \nabla \log p- \nabla  \log q_t \|_{L^2(Q_t)} } .
$$
Plugging $v_t^*$ into the continuity equation \eqref{equ:continueq} yields 
\begin{align}
\dot q_t = - \nabla \cdot ((\nabla \log p -  \nabla \log q_t) q_t) 
= - \nabla \cdot ((\nabla \log p) q_t) 
+ \Delta q_t,  \label{eq: Fokker Planck}
\end{align}
where $\Delta q_t \defeq \nabla \cdot \nabla q_t$ it the Laplacian operator. 
\Cref{eq: Fokker Planck} coincides with the Fokker--Planck equation of over-damped Langevin diffusion (c.f. \Cref{ex: overdamped Langevin})
$$
\dno  X_t = \nabla \log p(X_t) \dno t + \sqrt{2}\dno W_t.
$$
This shows that overdamped Langevin dynamics is the gradient flow of Kullback--Leibler divergence under the metric \eqref{eq: curly W}. 
Related, the rate of decrease of the Kullback--Leibler divergence along this flow is 
\begin{align*}
\frac{\dd}{\dd t} \KL{Q_t}{P}  & = - Q_t (\mathcal T_Pv_t^*) \\ 
& = - \norm{\nabla \log p -  \nabla \log q_t}_{L^2(Q_t)}^2
= - \mathrm{FD}(Q_t || P) ,
\end{align*} 
with the final expression being the Fisher divergence between the current distribution $Q_t$ and the target $P$ (\Cref{def: fisher div}).
This indicates that the decrease in Kullback-Leibler is more rapid when $Q_t$ is far from $P$ in the sense of $\mathrm{FD}(Q_t || P)$.

\section{Stein Variational Gradient}   
\label{sec: stein variational gradient}

The problem with the velocity in Langevin diffusion is that it depends on the log density of the evolving density, i.e. $\log q_t$, which is not computationally available. 
Using Stein's method we can instead obtain a computationally tractable optimal velocity by taking the candidate space $\mathcal{G}_t$ to be a unit ball in an \ac{rkhs}. 

Let $\mathcal{G}_t = \{\vecf \in \H_K \colon ||\vecf ||_{\H_K} \leq 1\} $ be the unit ball in a vector-valued \ac{rkhs} with kernel of the form $K(x,y) = k(x,y) I$ for some scalar-valued kernel $k : \mathbb{R}^d \times \mathbb{R}^d \rightarrow \mathbb{R}$ (c.f. \Cref{subsec: vvrkhs}).  
For $v_t \in \mathcal{G}_t$, 
by the linearity of the Langevin Stein operator,
\bb 
Q_t (\mathcal T_P v_t)  
= \langle v_t, v_t^* \rangle_{\H_K}
\ee
where 
\begin{align}
\label{equ:phipq}
\vecf_t^*(\cdot) 
  & = \int  k(x, \cdot) \nabla_x \log p(x)  + \nabla_{x} k(x, \cdot) \; \dd Q_t(x).
\end{align}
Hence, the optimization problem in \eqref{eq: max of SD} reduces to 
\bb 
\max_{v_t \in \mathcal{G}_t} \; \langle v_t, v_t^* \rangle_{\H_K}
\ee 
whose optimal solution is  
$$
v_t = \frac{ \vecf_t^*} { \|\vecf_t^*\|_{\H_K} }.
$$
Thus we have shown that the corresponding flow $(Q_t)_{t \geq 0}$ from \eqref{equ:continueq} monotonically decreases the Kullback--Leibler divergence, and the rate of decrease equals the \acf{ksd} (c.f. \Cref{def:kernel_stein_discrepancy}): 
\begin{align}\label{equ:dkdtD}
\frac{\dno}{\dt}\KLL(Q_{t}  ~||~ P) = - \KSD_P(Q_{t}). 
\end{align} 
The velocity field $v_t$ characterizes the so-called \emph{Stein variational gradient}, which forms the basis for several powerful numerical algorithms as detailed in \Cref{subsec: batch}.

\begin{remark}
This result suggests a path integration formula of for the Kullback--Leibler divergence; $\KLL(Q_0 || P) = \int_{0}^\infty \KSD_P(Q_t)  \dt,$ 
which can be useful for estimating the Kullback--Leibler divergence or the normalization constant \citep{han2017stein}.   
\end{remark}

\chapter{Applications} 
\label{sec:ksd_application} 

This Chapter provides a succinct overview of recent and emerging applications of Stein discrepancies across a range of important practical tasks arising in probabilistic inference and learning.

\section{Overview} 

The unique feature of the Stein discrepancies  introduced in \Cref{chap: Stein discrepancies} is that they enable the discrepancy between a distribution $P$ and an empirical measure 
\begin{align}
Q_n = \sum_{i=1}^n w_i \delta_{x_i}, \label{eq: generic empirical measure}
\end{align}
to be explicitly computed without tractable access to a probability mass or density function for $P$, when the intractability arises from difficulty in computing the normalization constant.
Here $P$ and $Q_n$ are probability distributions on a common measurable space $\mathcal{X}$, and $Q_n$ is supported on a finite set of points $x_1 , \dots , x_n \in \mathcal{X}$ weighted by $w_1 , \dots , w_n \in \mathbb{R}$ with $\sum_i w_i = 1$ and $w_i\geq 0$. 
Recall that a Stein discrepancy is defined by a Stein operator $\mathcal{T}_P$ and a Stein set $\mathcal{G}$.
For all of the applications that we are about to discuss there is flexibility in the choice of both $\mathcal{T}_P$ and $\mathcal{G}$, which enables an opportunity to tailor the Stein discrepancy to the task at hand.
However, to simplify the presentation in this Chapter we will leave these dependencies implicit and adopt the shorthand
\begin{align*}
    \mathcal{S}_P(Q_n) \defeq \mathcal{S}(Q_n,\mathcal{T}_P,\mathcal{G}) .
\end{align*}
Stein discrepancies and their related Stein dynamics from \Cref{chap: Stein dynamics} provide versatile tools that lend themselves naturally to a range of important statistical applications, broadly falling into four categories that will now be described.

\paragraph{Measuring Approximation Quality} (\Cref{sec:measure})   
Given samples $\{x_i\}_{i=1}^n$ and a distribution $P$ with density $p(x) \propto \exp(-V(x))$ specified up to a normalization constant, the magnitude of $\mathcal{S}_P(Q_n)$ can be used to quantify the dissimilarity between the empirical distribution $Q_n$ of the (uniformly-weighted) samples $\{x_i\}_{i=1}^n$ and $P$.
As such, Stein discrepancies have emerged as a popular and flexible approach to goodness-of-fit testing, where one wishes to test the null hypothesis that $x_i \stackrel{\mathrm{iid}}{\sim} P$.
Under the null hypothesis, $\mathcal{S}_P(Q_n) \rightarrow 0$ as $n \rightarrow \infty$, so one rejects the null if $\mathcal{S}_P(Q_n) \geq \tau$, where $\tau$ is an appropriately chosen threshold.
This method greatly extends traditional goodness-of-fit tests, such as chi-square tests, which only work for simple and low-dimensional distributions.
In parallel, Stein discrepancies have also become a popular tool for measuring the quality of approximations produced by algorithms that aim to sample from a posterior distribution $P$, enabling both the tuning of sampling methods and the comparison of competing methodologies in the Bayesian statistical context.

\paragraph{Algorithms for Particle-Based Approximation} (\Cref{sec:improving})  
Given a probability distribution $P$ specified in terms of a probability mass or density function $p(x) \propto \exp(-V(x))$ up to an intractable normalization constant, one can cast the problem of numerically approximating $P$ as the optimization problem of finding states $x_1,\dots,x_n \in \mathcal{X}$ and corresponding weights $w_1,\dots,w_n \in \mathbb{R}$ such that the associated empirical distribution \eqref{eq: generic empirical measure} minimizes a Stein discrepancy $\mathcal{S}_P(Q_n)$.
Fixing uniform weights $w_i = \frac{1}{n}$,  optimization of the particles $\{x_i\}_{i=1}^n$ can be performed using gradient descent or any other suitable numerical optimization method, yielding new families of particle-based algorithms.
On the other hand, the Stein dynamics discussed in \Cref{chap: Stein dynamics} can be used to perform a version of gradient descent on the Kullback--Leibler divergence, giving rise to an algorithm known as \ac{svgd}.
Fixing the particles $\{x_i\}_{i=1}^n$ and instead optimizing the weights $\{w_i\}_{i=1}^n$ yields algorithms that are similar in spirit to importance sampling, and can correct for bias if the $\{x_i\}_{i=1}^n$ arose in such a manner other than being sampled from $P$.  
Moreover, imposing sparsity on the weights $\{w_i\}$ yields \emph{thinning} algorithms that select a small subset of $\{x_i\}_{i=1}^n$ to achieve accurate approximation of $P$. 
Finally, given a specific function $f : \mathcal{X} \rightarrow \mathbb{R}$ whose expectation with respect to $P$ is of interest, one can deduce optimal values for the weights $w_i$ appearing in a cubature approximation $\sum_{i=1}^n w_i f(x_i)$, yielding an effective variance reduction technique in situations where the $x_i$ are randomly sampled.

\paragraph{Training Generative Models} 
(\Cref{sec: param est})
Given a dataset $\{x_i\}_{i=1}^n$, one can seek a suitable generative model from a collection $\{P_\theta\}_{\theta \in \Theta}$ by searching for appropriate values for the parameter $\theta$ across an index set $\Theta$.
For tractable generative models, standard statistical techniques such as maximum likelihood estimation can be used.
However, there are many important examples of generative models that are not tractable; for example, $P_\theta$ may be specified via a density function $p_\theta(x) \propto \exp(-V_\theta(x))$ up to an intractable normalization constant.
Stein discrepancy and Stein dynamics provide a variety of useful alternative parameter estimation methods in this context.

\paragraph{Gradient Estimation}
(\Cref{sec: gradient est})
A common technical challenge encountered in machine learning is that gradients $\mathbb{E}_{X \sim P_\theta}[f(X)]$ with respect to the parameters $\theta$ of a distribution $P_\theta$ are unavailable analytically and must be approximated. 
This challenge arises in many learning problems, including training latent variable models for variational inference and reinforcement learning with policy gradients. 
Monte Carlo methods can often be used, but their associated errors can be substantial, negatively affecting performance in the downstream machine learning task. 
Stein's method has given rise to several variance reduction strategies that can be employed in this context.

\section{Measuring Approximation Quality}
\label{sec:measure} 

The first application that we consider in detail is the problem of measuring the quality of an empirical approximation to an unnormalized distributional target.
To be precise, suppose that $P$ is a distribution of interest, defined on a probability space $(\mathcal{X},\mathcal{S}_{\mathcal{X}},\lambda_{\mathcal{X}})$, for which a p.d.f. $p(x)$ with respect to $\lambda_{\mathcal{X}}$ is available up to an intractable normalization constant:  i.e.
\begin{align*}
    p(x) \propto \frac{\tilde{p}(x)}{Z}
\end{align*}
where $\tilde{p}(x)$ is explicitly available but the normalization constant $Z = \int \tilde{p}(x) \; \mathrm{d}\lambda_{\mathcal{X}}(x)$ is an intractable integral.
In what follows we consider two distinct scenarios (quantifying the performance of sampling methods in \Cref{subsec: app to Bayes}, and goodness-of-fit testing in \Cref{sec: GoF}), where in each case the task is to determine the quality of an empirical approximation $Q_n = \frac{1}{n} \sum_{i=1}^n \delta_{x_i}$ to $P$, where $x_1,\dots,x_n \in \mathcal{X}$ are the states on which $Q_n$ is supported.

\subsection{Quantifying the Performance of Sampling Methods}
\label{subsec: app to Bayes}

The problem of quantifying the performance of sampling methods commonly arises in Bayesian statistics, where $P$ is a posterior distribution whose density $p(x) \propto \pi(x) \mathcal{L}(x)$ is available in unnormalized form as the product of the prior $\pi(x)$ and the likelihood $\mathcal{L}(x)$ (c.f. \Cref{sec: Stein as a tool}).
Here $\tilde{p}(x)$ can be identified with the product $\pi(x) \mathcal{L}(x)$ and the normalization constant
$$
Z = \int \pi(x) \mathcal{L}(x) \; \mathrm{d}\lambda_{\mathcal{X}}(x)
$$
is recognized as the marginal likelihood.
The task of measuring sample quality is encountered in the form of monitoring the convergence of an extensible\footnote{A sampling algorithm is \emph{extensible} if it can in principle produce an infinite sequence $(x_n)_{n \in \mathbb{N}}$ of states, such that in practice one can truncate this sequence and terminate the algorithm to obtain an approximation $Q_n$ to $P$ once some appropriate stopping criterion is met. } sampling algorithm, when comparing the approximations produced by different (exact and/or approximate) sampling algorithms, and in tuning the hyper-parameters of a particular sampling algorithm.
In each case, Stein discrepancies have been employed as quantitative criteria through which performance can be explicitly measured.

\begin{example}[Tuning the unadjusted Langevin algorithm]
The \ac{ula} for approximate sampling from a target $P$ is an Euler--Maruyama discretization $(X_t)_{t \in \mathbb{N}}$ of a $P$-invariant overdamped Langevin diffusion (c.f. \Cref{ex: overdamped Langevin}):
$$
X_{t+1} = X_t + \frac{\epsilon}{2} \nabla \log p(X_t) + \sqrt{\epsilon} Z_t
$$
where $Z_t$ is a standard normal random variable generated independently from $X_1,\dots,X_t$.
\begin{figure}[t!]
\includegraphics[width=\textwidth]{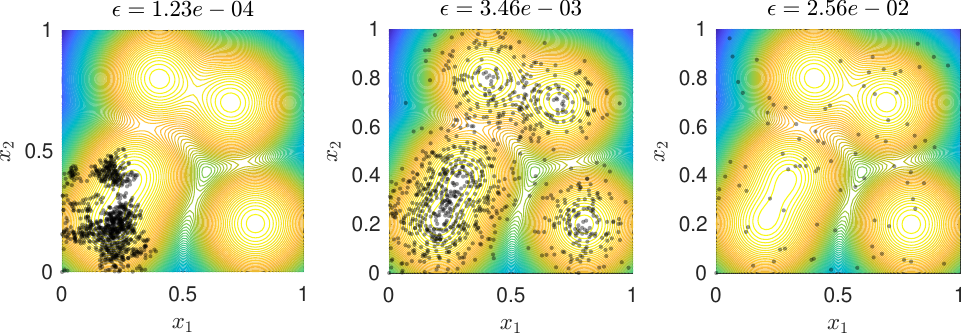}
\caption{Samples generated using the unadjusted Langevin algorithm (ULA).
Here the target $P$ is a Gaussian mixture model with density $p$ and $\epsilon$ denotes the step size parameter of ULA.
Contours of $\log p$ are depicted.}
\label{fig: Example_6_1_Fig_1}
\end{figure}
Here $\epsilon > 0$ is a \emph{step size} parameter that must be selected.
Suppose that we have a computational budget which permits computing $(X_t)_{t=1}^n$ with $n = 1,000$; we wish to select $\epsilon$ for which the empirical distribution $Q_n = \frac{1}{n} \sum_{t=1}^n \delta_{X_t}$ is likely to provide the best approximation to $P$.
To be precise, we suppose that we are interested in approximating the mean and variance of each component of the distribution $P$ in \Cref{fig: Example_6_1_Fig_1} using empirical averages from \ac{ula}.
Taking $\epsilon$ too small prevents the stochastic process from effectively exploring the high-probability regions of $P$ (\Cref{fig: Example_6_1_Fig_1}; left panel).
On the other hand, taking $\epsilon$ too big introduces bias into the distribution of the $X_t$ due to the first order discretization of the \ac{sde}, and can ultimately cause explosive behavior in \ac{ula} (\Cref{fig: Example_6_1_Fig_1}; right panel).
Unfortunately, an appropriate value of $\epsilon$ (\Cref{fig: Example_6_1_Fig_1}; middle panel) will be unknown in general due to the intractable nature of the target.
A solution is provided by Stein discrepancy, which enables the quality of the approximation $Q_n$ to be explicitly measured, i.e. $\mathcal{S}_P(Q_n)$.
Indeed, the mean square error associated with the empirical estimates of the means and variances of $P$ (which cannot be computed in general) is minimized around $\epsilon \approx 10^{-2}$ (\Cref{fig: Example_6_1_Fig_2}; left panel), while the \acf{ksd} (which can be computed) is also minimized around $\epsilon \approx 10^{-2}$ (\Cref{fig: Example_6_1_Fig_2}; right panel).
\begin{figure}[t!]
\includegraphics[width=\textwidth]{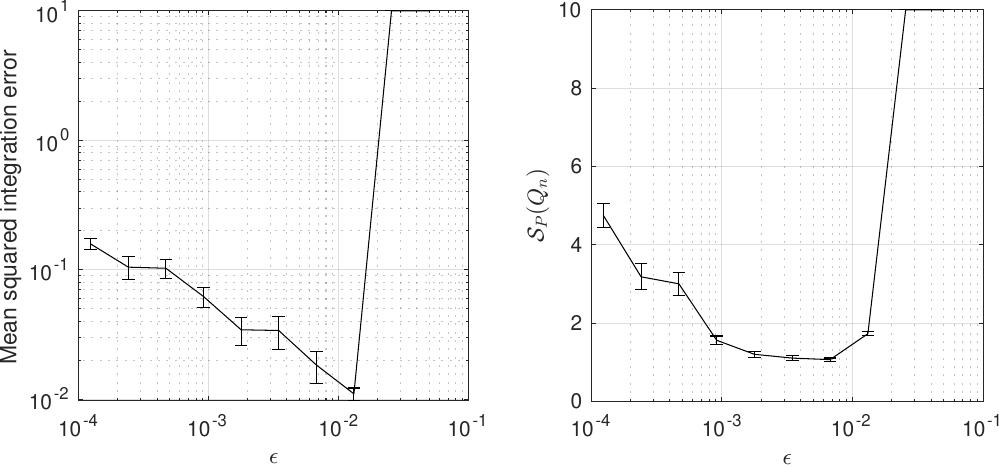}
\caption{Tuning the unadjusted Langevin algorithm (ULA).
The left panel presents the mean square error associated with the empirical estimates for the means and variances of $P$ (which cannot be computed in general), while the right panel presents the \acf{ksd} (which can be computed).
Here we used the inverse multi-quadric kernel (\Cref{ex: IMQ kernel}) with exponent $\beta = 1$ and bandwidth $\ell = 0.01$.
The experiment was repeated 10 times and means and standard errors are reported.}
\label{fig: Example_6_1_Fig_2}
\end{figure}
\end{example}

A related setting concerns so-called \emph{intractable} probabilistic models, popular in statistical physics, which are specified in terms of relative probabilities $p(x) / p(y)$ to circumvent a challenging normalization constant.
Since here exact simulation from $P$ is often infeasible, a range of sampling algorithms have been developed \citep[see e.g.][]{newman1999monte} and their assessment can proceed using Stein discrepancy as just described.

\subsection{Goodness-of-Fit Testing}
\label{sec: GoF}

The problem of measuring sample quality is also encountered in goodness-of-fit testing for statistical modes specified up to an intractable normalization constant.
That is, given a dataset $\{x_i\}_{i=1}^n$ and a probability distribution $P$ with p.d.f. of the form $p(x) \propto \exp(-V(x))$, where the potential $V(x)$ is explicitly provided, we seek to test the null hypothesis that $x_i \stackrel{\mathrm{iid}}{\sim} P$.
Any Stein discrepancy that separates distributions, meaning that $\mathcal{S}_P(Q) = 0$ if and only if $Q = P$, provides a computable statistic $\mathcal{S}_P(Q_n)$ on which to conduct such a test.
If the null hypothesis holds then we might expect $\mathcal{S}_P(Q_n) \rightarrow 0$ as $n \rightarrow \infty$ and we would therefore want to reject the null hypothesis if $\mathcal{S}_P(Q_n)$ is not small.
In practice, the sampling distribution of the Stein discrepancy $\mathcal{S}_P(Q_n)$ under the null hypothesis is usually unknown, but can be approximated using a bootstrap method.
The main considerations in selecting a Stein discrepancy here are that it can be rapidly computed (to facilitate the bootstrap resampling step) and that, in the case of the null hypothesis being incorrect, it has sufficient power to distinguish between $P$ and the true mechanism that gives rise to the dataset.

For ease of computability, we employ a \acf{ksd} (\Cref{{sec: ksds}}) with Stein kernel $k_p : \mathcal{X} \times \mathcal{X} \rightarrow \mathbb{R}$.
Assuming that $\int k_p(x,x) \mathrm{d}P(x) < \infty$, then under the null hypothesis
\begin{align*}
\mathbb{E}\left[ \mathcal{S}_P(Q_n)^2 \right] & = \mathbb{E}\left[ \frac{1}{n^2} \sum_{i=1}^n \sum_{j=1}^n k_p(x_i,x_j) \right]  \\
& = \mathbb{E}\left[ \frac{1}{n^2} \sum_{i=1}^n k_p(x_i,x_i) + \frac{1}{n^2} \sum_{i = 1}^n \sum_{j \neq i} k_p(x_i,x_j) \right] \\
& = \frac{1}{n} \int k_p(x,x) \mathrm{d}P(x) \rightarrow 0,
\end{align*}
where the last line follows since $x_i \sim P$ are independent under the null and $\int k_p(\cdot,x) \; \mathrm{d}P(x) = 0$ since $k_p$ is a Stein kernel.
Thus we may construct a goodness-of-fit test of size $\alpha$ using the statistic $\mathcal{S}_P(Q_n)$ and rejecting the null hypothesis when $\mathcal{S}_P(Q_n) > \tau$, where the threshold $\tau$ is selected such that $\mathbb{P}(\mathcal{S}_P(Q_n) \leq \tau) = 1 - \alpha$.
Since one typically does not have the ability to simulate from $P$ (due to the intractable normalizing constant), numerically calculating the $1 - \alpha$ quantile of the sampling distribution of $\mathcal{S}_P(Q_n)$ under the null hypothesis is not straightforward.
An approach that has become popular in the literature \citep{liu2016kernelized,chwialkowski2016kernel} is to use the \emph{wild bootstrap} \citep{shao2010dependent,fromont2012kernels,leucht2013dependent}, which introduces additional independent Rademacher random variables $\{\epsilon_i\}_{i=1}^n$ (i.e. each $\epsilon_i$ is uniform on $\{-1,1\}$) and notes that the distribution of $n \hat{D}_P(Q_n)^2$ approaches that of $n\mathcal{S}_P(Q_n)^2$ under the null, where
\begin{align}
    \hat{D}_P(Q_n)^2 = \frac{1}{n^2} \sum_{i=1}^n \sum_{j=1}^n \epsilon_i \epsilon_j k_p(x_i,x_j)  \label{eq: wild}
\end{align}
and we consider the $n \rightarrow \infty$ limit.
Generating multiple instantiations of the Rademacher random variables and re-computing \eqref{eq: wild}, an empirical approximation to the sampling distribution of $\mathcal{S}_P(Q)$ under the null is obtained, from which an appropriate threshold $\tau$ can be extracted.

\begin{example}
Consider testing the goodness-of-fit of a Gauss--Bernoulli restricted Boltzmann machine to a dataset $\{x_i\}_{i=1}^n \subset \mathbb{R}^d$.
This model has latent variables $h \in \{-1,1\}^{d'}$ and joint density of the form
\begin{align*}
    p(x,h) = \frac{1}{Z} \exp\left( \frac{1}{2} x^\top B h + b^\top x + c^\top h + \frac{1}{2} x^\top x \right)
\end{align*}
where $Z$ is the appropriate normalizing constant.
The marginal distribution $P$ for the observable $x$ has density
\begin{align*}
    p(x) = \sum_{h \in \{-1,1\}^{d'}} p(x,h) ,
\end{align*}
but this can be intractable due to the dependence on $Z$.
On the other hand, the score function admits a closed form
\begin{align*}
    \nabla \log p(x) = b - x + \frac{1}{2} B \tanh\left( \frac{1}{2} B^\top x + c \right) ,
\end{align*}
so we can exploit Stein discrepancy to construct a goodness-of-fit test.

For illustration we follow the setting of \citet{liu2016kernelized}, who took $d = 50$, $d'=10$, $n=100$, sampled entries of $b$ and $c$ from a standard Gaussian, and sampled the entries of $B$ uniformly from $\{-1,1\}$.
Our experiments used the inverse multi quadric kernel $k$ (\Cref{ex: IMQ kernel}) with length scale $\ell > 0$ to construct the Stein kernel $k_P$.
Data were generated either from the true model, or from a perturbation of the true model where Gaussian noise of variance $\sigma^2$ was added to the entries of $B$.
\Cref{fig: gof1} shows the null distribution of the \ac{ksd} test statistic as approximated using the wild bootstrap, together with the actual value of the \ac{ksd} test statistic, and the rejection threshold $\tau$ corresponding to $\alpha = 0.05$.
In the case $\sigma = 0$, where the model is correct, the actual \ac{ksd} value falls into the central region of the null and the test does not reject, while in the case $\sigma = 0.1$ the actual \ac{ksd} falls far into the tail of the null and the null hypothesis is rejected.

\begin{figure}[t!]
    \includegraphics[width=\textwidth]{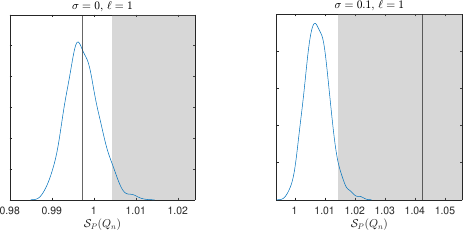}
    \caption{Goodness-of-fit testing with \ac{ksd}.
    Here we plot the null distribution of the \ac{ksd} test statistic as approximated using the wild bootstrap (blue), the actual value of the \ac{ksd} test statistic (black), and the values for which the null would be rejected (shaded).
    The amount of model misspecification is denoted with $\sigma$, so that when $\sigma = 0$ the model is correct.
    The inverse multi-quadric kernel with length scale $\ell$ was used.}
    \label{fig: gof1}
\end{figure}

A natural question is how to select a Stein kernel $k_p$ so that the power of the test (i.e. the probability of detecting departures from the null) is maximized.
For our simple setting, \Cref{fig: gof2} plots the test power for three different kernel length scale $\ell$, in each case as a function of the amount of model misspecification $\sigma$.
Interestingly, it seems that the power of this test is almost independent of $\ell$; all choices perform comparably well.

\begin{figure}[t!]
    \centering
    \includegraphics[width=0.5\textwidth]{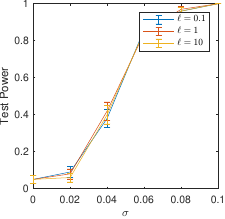}
    \caption{Goodness-of-fit testing with \ac{ksd}.
    Here we plot the test power (i.e. the probability of rejecting the null) as a function of both the amount of model misspecification $\sigma$ and the kernel length scale $\ell$.
    Means and standard errors over 100 experiments are reported.}
    \label{fig: gof2}
\end{figure}

\end{example}

In general, selecting an appropriate Stein kernel might involve techniques such as data-splitting and a tractable approximation to the power of the test, suitable for being optimized over a training split of the dataset.
As a starting point, the reader is referred to \citet{jitkrittum2017linear}, where an explicit approximation to the test power is presented.
Several authors have considered additional strategies to increase test power, such as aggregating several tests \citep{schrab2022ksd}, carefully selecting the Stein operator \citep{liu2023using}, and spectral regularisation \citep{hagrass2024minimax}.
At the same time, the computational cost of testing is being reduced \citep{jitkrittum2017linear,huggins2018random}, and additional functionalities such as robust testing \citep{liu2024robustness}, online testing \citep{martinez2024sequential}, and relative goodness-of-fit testing for  latent variable models  \citep{kanagawa2023kernel} are being developed.

\section{Algorithms for Particle-Based Approximation} \label{sec:improving} 

In the previous section we assumed the role of a passive observer, measuring the quality of a sample that has already been provided.
Now our attention turns to the question of how to actively \emph{construct} discrete approximations to a distributional target $P$.
Again, we are motivated by applications such as those arising in Bayesian statistics, where one has access to $P$ only through a density function that is unnormalized.
For demonstration purposes we will repeatedly consider the two-dimensional \emph{Rosenbrock} target
\begin{align}
    p(x,y) \propto \exp\left(-x^2 - 3(y-x^2)^2 \right) .  \label{eq: Rosenbrock}
\end{align}
First, in \Cref{subsec: sequential,subsec: batch} we present a selection of algorithms that aim to select states $\{x_i\}_{i=1}^n$ such that the associated (uniformly-weighted) empirical measure is an accurate approximation to the target.
In \Cref{subsec: stein importance sampling} we consider the task of assigning weights to a given sample, to produce a weighted empirical measure that may represent a more faithful approximation to the target.
This is complimented in \Cref{subsec: sparse approximation}, where we consider sparse approximation schemes that can facilitate both compression and improvement of sample quality at a reduced computational cost.
The situation is then specialized to that of approximating a single expected quantity of interest in \Cref{subsec: control variates}.

\subsection{Sequential Algorithms}
\label{subsec: sequential}

First we consider sequential, or \emph{extensible} algorithms, meaning that in principle an infinite sequence $(x_i)_{i \in \mathbb{N}}$ is produced.
In practice such algorithms are terminated after a finite number $n$ of iterations, and the first $n$ particles $\{x_i\}_{i=1}^n$ are taken to form an empirical approximation $Q_n = \frac{1}{n} \sum_{i=1}^n \delta_{x_i}$ to the distributional target $P$.
The main advantage of extensible algorithms is that they may be run for as long as required until a sufficiently accurate approximation to $P$ is obtained.

\paragraph{Stein Points}

A natural algorithm to consider is sequential greedy minimization of Stein discrepancy, which selects the $n$th particle $x_n$ in a manner that depends on the previously selected particles $x_1, \dots , x_{n-1}$ according to
\begin{align}
x_n \in \argmin_{x \in \mathcal{X}} \; \mathcal{S}_P\left( \frac{1}{n} \delta_x + \frac{1}{n} \sum_{i=1}^{n-1} \delta_{x_i} \right)   \label{eq: Stein point update}
\end{align}
where as usual $P$ is the distributional target, defined on $\mathcal{X}$.
The particles $(x_i)_{i \in \mathbb{N}}$ selected in this manner are termed \emph{Stein Points}.
The optimization problem in \eqref{eq: Stein point update} can in practice rarely be analytically solved, and numerical methods are required.
\citet{chen2018stein} analyzed the accuracy of numerical optimization that would be required at each iteration to ensure convergence $\mathcal{S}_P(Q_n) \rightarrow 0$ in the case of the kernel Stein discrepancy.
For Stein discrepancies that control convergence, we can then conclude that $Q_n$ converges to $P$.
Unfortunately the use of generic numerical optimization techniques introduces a curse of dimension in $d = \mathrm{dim}(\mathcal{X})$.
To alleviate the curse of dimension, \citet{chen2019stein} proposed \emph{Stein Point MCMC} which instead solves 
\begin{align}
x_n \in \argmin_{x \in \{y_1^{(n)} , \dots , y_{m_n}^{(n)}\}} \; \mathcal{S}_P\left( \frac{1}{n} \delta_x + \frac{1}{n} \sum_{i=1}^{n-1} \delta_{x_i} \right)   \label{eq: Stein point MCMC}
\end{align}
where $(y_i^{(n)})_{i = 1}^{m_n}$ is a $P$-invariant Markov chain of length $m_n$,  initialized at $y_1^{(n)} \in \mathcal{X}$.
See the left panel of \Cref{fig: stein points}.
Provided that the mixing of the Markov chain is sufficiently rapid and $m_n \rightarrow \infty$ as $n \rightarrow \infty$, the consistency of \eqref{eq: Stein point MCMC} was established under quite general conditions.
For example, it is possible to let the initial state $y_1^{(n)}$ depend on the history of selected points $x_1, \dots , x_{n-1}$, and a performant choice is the so-called \emph{most influential} point $x_i$ for which the Stein discrepancy with $x_i$ removed, i.e.
\begin{align*}
    \mathcal{S}_P\left( \frac{1}{n-2} \sum_{j \in \{1,\dots,n-1\} \setminus \{i\}} \delta_{x_j} \right) ,
\end{align*}
is maximized.
The intuition for this choice is that $x_i$ is critical to the approximation quality of $Q_{n-1}$ and therefore adding another state $x_n$ in a neighborhood of $x_i$ may be beneficial.

\begin{figure}[t!]
    \centering
    \includegraphics[width=\textwidth]{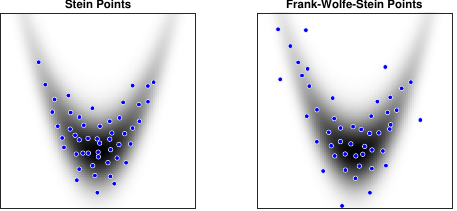}
    \caption{Stein points (left) and Frank--Wolfe--Stein points (right) were used to approximate the Rosenbrock target \eqref{eq: Rosenbrock} (black shaded). }
    \label{fig: stein points}
\end{figure}

\paragraph{Frank--Wolfe--Stein Points}

Consider now the Stein Points algorithm in the specific case of the \ac{ksd}, so that the update \eqref{eq: Stein point update} can be explicitly written in terms of the Stein reproducing kernel $k_p$ (e.g. as in \eqref{eq:stein-kernel} for the Langevin KSD) as
\begin{align}
    x_n \in \argmin_{x \in \mathcal{X}} \; \frac{k_p(x,x)}{2} + \sum_{i=1}^n k_p(x_i,x)  \label{eq: explicit stein points}
\end{align}
whose evaluation cost is seen to be $O(n)$.
The first term in \eqref{eq: explicit stein points} can be interpreted as a regulariser, and indeed one can also consider solving an unregularised version of this problem, to obtain a different but related sequence of points
\begin{align}
    \tilde{x}_n \in \argmin_{x \in \mathcal{X}} \; \sum_{i=1}^n k_p(\tilde{x}_i,x)  \label{eq: FW stein points} .
\end{align}
\citet{chen2018stein} noticed that sequence $(\tilde{x}_i)_{i=1}^n$ produced by solving the unregularised problem in \eqref{eq: FW stein points} is exactly the sequence one would obtain by applying the \emph{Frank--Wolfe} optimization method to Stein discrepancy minimization, as has been studied in the literature on kernel methods \citep[e.g.][]{chen2010super,bach2012equivalence} where the name \emph{herding} is also used.\footnote{Another related technique, \emph{MMD-FW}~\citep{futami2019bayesian}, uses Frank--Wolfe and the Langevin Stein operator to minimize an associated kernel MMD rather than the Stein discrepancy directly.}
See the right panel of \Cref{fig: stein points}.
This perspective suggests a variety of techniques from the Frank--Wolfe literature that could be used to improve approximation accuracy, such as additionally allowing the removal of sub-optimal states that were previously selected, or the use of weighted approximation (we will discuss weighted approximation in \Cref{subsec: stein importance sampling}); see \citet{bomze2024frank} for background. 

Perhaps the main consideration in choosing a Stein discrepancy in this context is that the convergence control properties of the Stein discrepancy should reflect the application in which the output of a sampling algorithm is to be used. 
For example, if samples are to be used to approximate posterior expectations, then a Stein discrepancy that controls the convergence of such expectations should be used (c.f. \Cref{subsec: conv contrl lksd}).]

\subsection{Batch Algorithms}
\label{subsec: batch}

An extensible algorithm cannot produce highly accurate approximations for all the values of $n$ at which it can be terminated \citep{owen2016constraint}.
This can be seen intuitively through \emph{symmetry-breaking}; suppose one wished to approximate a Gaussian $P = N(0,1)$ using $n$ particles.
For $n = 1$, it is natural to select $x_1 = 0$ at the central point, but then we are immediately forced to make a difficult choice $x_2$; either we also select $x_2 = 0$, in which case we are not accurately capturing the spread of $P$, or we are forced to break the symmetry of our approximation (even though $P$ is symmetric).
Batch algorithms seek to avoid this predicament by specifying the number $n$ of particles at the outset.
This allows for improved approximation for a comparable number $n$ of particles compared to extensible algorithms, but means that if further accuracy is desired then (at least in principle) the algorithm must be re-run from scratch after $n$ is increased.
Two illustrative batch algorithms will now be discussed:

\paragraph{Stein Discrepancy Descent}

One natural construction of a batch algorithm starts with the so-called \emph{Wasserstein gradient flow} of the Stein discrepancy $\mathcal{S}_P(Q)$, which is in effect a form of continuous-time gradient descent with respect to the $Q$ argument, producing a measure-valued trajectory $(Q^{(t)})_{t \geq 0}$ where each $Q^{(t)}$ is a probability distribution on $\mathcal{X}$.
Intuitively, at least, the gradient flow $Q^{(t)}$ should converge to $P$ under appropriate regularity conditions as $t \rightarrow \infty$.
For numerical purposes, the measure $Q^{(t)}$ is approximated by a (uniformly-weighted) empirical distribution $Q_n^{(t)} = \frac{1}{n} \sum_{i=1}^n \delta_{x_i^{(t)}}$ whose support points $x_i^{(t)}$ are now $t$-dependent.
This approach was considered in \citet{korba2021kernel}, who derived the following numerical scheme based on gradient descent of the (squared) \ac{ksd}:
\begin{align*}
    x_i^{(t+1)} & = x_i^{(t)} -  \frac{\epsilon}{n^2} \sum_{j=1}^n \nabla_2 k_p(x_j^{(t)} , x_i^{(t)})
\end{align*}
where $\epsilon > 0$ is a \emph{learning rate} to be specified.
It is also possible to consider the use of higher-order numerical optimization techniques, such as L-BFGS; we refer the reader to \citet{korba2021kernel}.
Stein discrepancy descent can produce compact representations of the target $P$.
However, this comes at the expense of requiring second-order derivatives of the log-density of $P$.
An alternative strategy, which requires only first-order derivatives, is considered next.

\paragraph{Stein Variational Gradient Descent}

The application of gradient descent to Stein discrepancy led to second-order derivatives of the log-density of $P$ because the Stein discrepancy itself involved first-order derivatives of $P$.
One solution is to consider gradient descent on a different objective; in particular it is natural to consider the Kullback--Leibler divergence (\Cref{def: KL divergence}).

Let $\mathcal{X} = \mathbb{R}^d$ and consider a map $T : \mathbb{R}^d \rightarrow \mathbb{R}^d$ of the form $T(x) = x + \epsilon g(x)$ where $\epsilon \ll 1$ and $g : \mathbb{R}^d \rightarrow \mathbb{R}^d$.
From \Cref{thm: deriv KL} in \Cref{chap: Stein dynamics} we know that
\begin{align}
    - \left. \frac{\mathrm{d}}{\mathrm{d}\epsilon} \KL{T_\# Q}{P} \right|_{\epsilon = 0} = Q(\mathcal T_P g ) \label{equ:klstein00}
\end{align}
where $\mathcal{T}_P$ is the Langevin Stein operator (\Cref{def: Lang SO}), and we therefore assume that $g$ is regular enough to belong to the domain of this Stein operator (cf. \Cref{lem: Stein op int by parts}).
Given a set $\mathcal{G}$ of candidates for the function $g$, a direction of steepest descent corresponds to selecting 
\begin{align}
    g^* \in \argmax_{g \in \mathcal{G}} Q(\mathcal{T}_P g) . \label{eq: opt search dir}
\end{align}
A computationally attractive choice of set $\mathcal{G}$ is the unit ball of a vector-valued \ac{rkhs} $\mathcal{H}(K)$, for which the objective in \eqref{eq: opt search dir} is recognized as a kernel Stein discrepancy (\Cref{def:kernel_stein_discrepancy}).
Indeed, in this case
\begin{align}
    g_{P,Q}^*(y) = \int \mathcal{T}_P^{(1)} K(x,y) \; \mathrm{d}Q(x) \label{equ:ffs}
\end{align}
can be explicitly calculated.
These calculations suggest a practical algorithm, where $Q$ is replaced by a discrete distribution $Q_n^{(t)} = \frac{1}{n} \sum_{i=1}^n \delta_{x_i^{(t)}}$ and the locations of the support points $x_i^{(t)}$ are updated in a time $t$-dependent manner following the direction of steepest descent:
\begin{align}
    x_i^{(t+1)} & = x_i^{(t)} + \epsilon g_{P,Q_n^{(t)}}^*(x_i^{(t)})  \label{eq: SVGD update rule}
\end{align}
For the special case where the matrix-valued kernel $K$ takes the form $K(x,y) = k(x,y) I$, the explicit form of the direction of steepest descent is
\begin{align}
    g_{P,Q_n^{(t)}}^*(x_i^{(t)}) & = \frac{1}{n}\sum_{j=1}^n  (\nabla \log p)(x_j^{(t)}) k( x_j^{(t)}, x_i^{(t)}) + \nabla_1 k(x_j^{(t)}, x_i^{(t)})  \label{eq: steep descent}
\end{align}
and this algorithm is known as \acf{svgd} \citep{Liu2016}.
(The case of a general matrix-valued kernel is discussed in \citet{zhuo2018message,wang2018stein}.)

For any reasonable kernel, such as the Gaussian kernel with fixed length scale $\ell$, the sequence of discrete distributions $Q_n^{(t)}$ converges to an approximation of $P$, as illustrated in \Cref{fig: SVGD} (left). 
In fact, with the right choice of initialization, step size, and step count $t_n$, the SVGD approximation $Q_n^{(t_n)}$ is known to converge to $P$ in KSD \citep{gorham2020stochastic,shi2023finite} at a $O(1/\sqrt{n})$ rate \citep{balasubramanian2024improved}.
In practice, the kernel length scale $\ell$ is usually made $t$-dependent, and a default setting is to take $\ell$ equal to the median pairwise distances between the $\{x_i^{(t)}\}_{i=1}^n$; see \Cref{fig: SVGD} (right). 

\begin{figure}[t!]
    \centering
    \includegraphics[width=\textwidth]{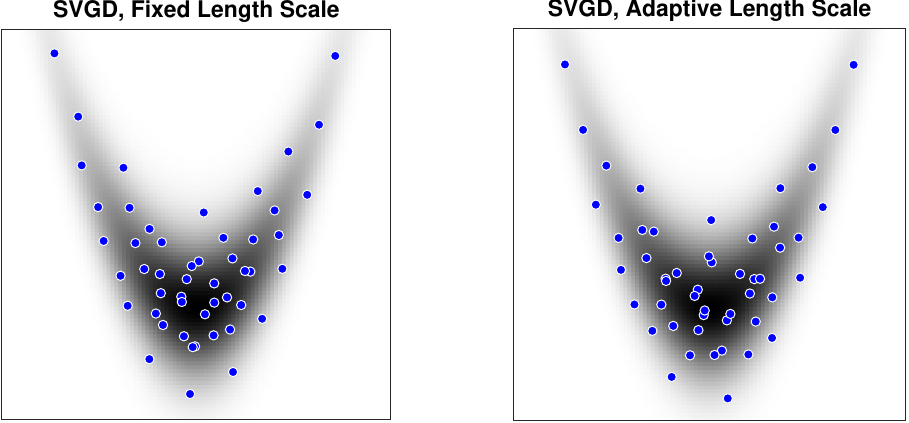}
    \caption{Stein variational gradient descent with fixed kernel length scale $\ell$ (left) and adaptively selected kernel length scale $\ell$ (right) were used to approximate the Rosenbrock target \eqref{eq: Rosenbrock} (black shaded).}
    \label{fig: SVGD}
\end{figure}

The two terms in 
\eqref{eq: steep descent} plays intuitive roles: 
The term with the gradient $(\nabla\log p)(x)$ 
pushes the particles towards 
the high probability regions of $P$,
while the term with $\nabla_1 k(x, x_i)$ 
serves as a repulsive force to enforce  
 a degree of diversity between the particles. 
If there is no repulsive force, 
or when there is only a single particle 
(and the kernel satisfies $\nabla_1 k(x,x') = 0$ for $x = x'$),  
the solution would collapse to the local optima of the density for $P$, reducing to the \ac{map} problem. 
Therefore,  by using different particle sizes, \ac{svgd} provides an interpolation between \ac{map} to a full particle-based approximation,  enabling an efficient trade-off between accuracy and computation cost.

\begin{remark}[Well-definedness of \ac{svgd} for empirical measures] 
\ac{svgd} can be viewed as iteratively updating the empirical particle measure $Q_n^{(t)}$ in order to minimize $\KL{Q_n^{(t)}}{P}$. A crucial but subtle point here is that $\KL{Q_n^{(t)}}{P}$ is technically infinite (or undefined) because the empirical measure $Q_n^{(t)}$ is not absolutely continuous with respect to $P$; more specifically, $$ \KL{Q}{P} = \E_{X\sim Q} [ \log q(X) - \log p(X)]$$
depends on the density $q$ of measure $Q$, and is not properly defined when $Q$ is an empirical measure.

A key observation is that \emph{differences} of KL divergences under an invertible transport
remain finite and are computable from expectations under $Q$, even when $Q$ is empirical.
Let $T:\R^d\to\R^d$ be a $C^1$ diffeomorphism. Then
\begin{align}
& \KL{T_\# Q}{P} - \KL{Q}{P}\notag\\ 
&~~~~~= \KL{Q}{T^{-1}_\# P} - \KL{Q}{P} \label{eq:pushforward_identity}\\
&~~~~~=  -\,\E_{X\sim Q}\!\left[\log p(T(X)) + \log\bigl|\det \nabla T(X)\bigr| - \log p(X)\right], \label{equ:diffkl} 
\end{align}
where $T^{-1}_\#P$ denotes the pushforward of $P$ by $T^{-1}$, whose density is
$
x \mapsto p(T(x))\,\bigl|\det \nabla T(x)\bigr|. 
$ 
We used in \eqref{eq:pushforward_identity} the invariance of KL divergence
under simultaneous invertible transformations, namely  
$\KL{T_\# Q}{T_\# P} = \KL{Q}{P}.$ 

Note that the right-hand side of \eqref{equ:diffkl} depends on $Q$ only through the expectation $\E_{X\sim Q}[\cdot]$,
so it remains well-defined for empirical $Q$.

Moreover, when $T(x)=x+\epsilon g(x)$ with small $\epsilon$, the first-order expansion of
\eqref{equ:diffkl} recovers the \ac{svgd} objective decrease formula (e.g. \eqref{equ:klstein00}),
whose right-hand side also depends on $Q$ only through $\E_{X\sim Q}$.
This is why one can derive the \ac{svgd} particle update by replacing $Q$ with the empirical
particle measure.
\end{remark}

\begin{remark}[Comparison to Stein Discrepancy Descent]
    Stein discrepancy descent and \ac{svgd} are closely related.
    To draw an analogy, if we are interested in minimizing a convex function $f(x)$ over a Euclidean space $x \in \mathbb{R}^d$, then we may attempt to minimize either $f(x)$ or $\|\nabla f(x)\|$; minimizing either objective leads to the same result.
    Accordingly, Stein discrepancy descent aims to minimize Stein discrepancy, which is akin to minimizing the gradient of the Kullback--Leibler divergence (\Cref{thm: deriv KL}), while in \ac{svgd} it is the Kullback--Leibler divergence itself which is minimized.
\end{remark}

Since its introduction in \citet{Liu2016}, \ac{svgd} has been extended and improved in various ways. 
A non-exclusive list of examples include: amortized \ac{svgd} \citep{feng2017learning,wang2016learning,liu2017learning} that learns neural samplers instead of particle approximation; gradient-free \ac{svgd} \citep{han2018stein} which requires no gradient information of the target distribution $P$; 
graphical \ac{svgd} \citep{wang2018stein, zhuo2018message} and matrix-kernel \ac{svgd} \citep{wang2019stein} which incorporate structured information in kernel to improve the performance in high dimensions; 
\ac{svgd} on Riemannian manifolds \citep{liu2017riemannian} and Newton variants of \ac{svgd} \citep{detommaso2018stein, chen2019projected};   
stochastic variants that speed up by sub-sampling the particles \citep{li2020stochastic, gorham2020stochastic}; 
quantile \ac{svgd} that minimize a quantile loss for increased robustness \citep{gong2019quantile}; 
nonlinear \ac{svgd} that minimizes more general nonlinear loss functions beyond KL divergence \citep{wang2019nonlinear}; 
Stein variational importance sampling \citep{han2017stein}; and a 
general particle optimization framework \citep{chen2018unified}.

\subsection{Stein Importance Sampling}
\label{subsec: stein importance sampling}

Up to this point we have considered approximation using particles that are uniformly weighted.
An appealing feature of Stein discrepancies is that they enable a tractable solution to the \emph{optimal weighted approximation} problem where, given \emph{fixed} states $x_1,\dots,x_n \in \mathcal{X}$ (which may have been obtained by any of the aforementioned algorithms in this Section), we seek a weighted measure
\begin{align*}
Q_n^\star = \sum_{i=1}^n w_i^\star \delta(x_i), \qquad w^\star \in \argmin_{w \geq 0, \; 1^\top w = 1} \mathcal{S}_P \left( \sum_{i=1}^n w_i \delta(x_i)  \right) 
\end{align*}
for which the Stein discrepancy between $Q_n^\star$ and $P$ is minimized.
Indeed, consider the \ac{ksd} from \Cref{sec: ksds} with Stein kernel $k_p : \mathcal{X} \times \mathcal{X} \rightarrow \mathbb{R}$.
Since 
\begin{align*}
    \mathcal{S}_P\left( \sum_{i=1}^n w_i \delta(x_i)  \right)^2 = \sum_{i=1}^n \sum_{j=1}^n w_i w_j k_p(x_i,x_j) ,
\end{align*}
the optimal weights $w^\star = (w_1^\star,\dots,w_n^\star)^\top$ are then the solution to the linearly-constrained quadratic program
\begin{align}
\argmin_{w \in \mathbb{R}^d} \; w^\top K_p w \qquad \text{s.t.} \qquad w \geq 0 , \; 1^\top w = 1  \label{eq: constrained program}
\end{align} 
where $[K_p]_{i,j} = k_p(x_i,x_j)$.
The constraints appearing in \eqref{eq: constrained program} ensure that $Q_n^\star$ is a probability distribution on $\mathcal{X}$.
This program does not admit a closed-form solution, but can be numerically solved.

This approach was first introduced in \citet{liu2017black} where it was called \emph{black-box importance sampling}, since in contrast to traditional importance sampling (which samples $x_i \sim \pi$ and assigns $\pi$-dependent weights $\tilde{w}_i \propto p(x_i) / \pi(x_i)$), knowledge of how the states $x_i$ were generated is not required in \eqref{eq: constrained program}.
It has seen success in Bayesian statistics, where $P$ is a posterior distribution that is sampled using a (possibly biased) \ac{mcmc} method; here Stein importance sampling can mitigate systematic biases in an input sample, for example, due to burn-in, tempering, approximate \ac{mcmc}, or other off-target sampling. 
Indeed, under appropriate conditions on the input sample, the Stein importance sampling approximation $Q_n^\star$ is consistent in the sense that  $\mathcal{S}_P(Q_n^\star) \rightarrow 0$ almost surely \citep[Thm.~3]{riabiz2020optimal} with $\mathcal{S}_P(Q_n^\star) = O(1/\sqrt{n})$ in probability \citep[Thm.~1]{hodgkinson2020reproducing,li2024debiased}. In fact, under more stringent conditions, Stein importance sampling converges to $P$ more quickly than an \iid sample from $P$:  $\E[\mathcal{S}_P(Q_n^\star)^2] = o(1/n)$ \citep[Thm.~2]{li2024debiased}.

The algorithm is illustrated in \Cref{fig: SIS} (left), where we first sample states $x_i$ independently from the standard bivariate normal distribution and then assign a weight to each state in order to approximate the Rosenbrock target in \eqref{eq: Rosenbrock}.
Compared to standard importance sampling, the weights assigned in Stein importance sampling are dependent; sampling $x_i$ and $x_j$ too close together due to chance is mitigated in Stein importance sampling, but not in standard importance sampling.
The right hand panel of \Cref{fig: SIS} illustrates the behavior of Stein importance sampling with the positivity constraint removed; it is interesting to see that negative weights are active in enabling improved approximation in the sense of \ac{ksd}.
This observation will become relevant in \Cref{subsec: control variates}.

\begin{figure}[t!]
    \centering
    \includegraphics[width=\textwidth]{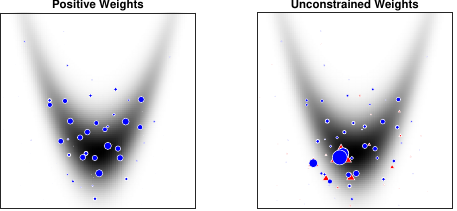}
    \caption{Stein importance sampling with positive weights (left) and unconstrained weights (right) were used to approximate the Rosenbrock target \eqref{eq: Rosenbrock} (black shaded). The size of each marker is proportional to the weight that the associated state is assigned; blue circles indicate a positive weight and red triangles indicate a negative weight.}
    \label{fig: SIS}
\end{figure}

Encouraged by the performance of Stein importance sampling in \Cref{fig: SIS}, we can then ask which states $\{x_1,\dots,x_n\}$ minimize the approximation error $\mathcal{S}_P(Q_n^\star)$.
This is a challenging and open mathematical problem.
Even if we suppose that the states are drawn independently from a distribution $\pi(x)$, then an optimal choice of $\pi$ is likely to be $n$-dependent.
As a heuristic, based on analysis of self-normalised importance sampling, \citet{wang2023stein} proposed to take $\pi(x) \propto p(x) \sqrt{k_p(x,x)}$.
For the Langevin \ac{ksd} with base kernel $k(x,y) = \phi(x-y)$, we have that $k_p(x,x) = - \Delta\phi(0) + \phi(0) \|\nabla \log p(x)\|^2$ which typically increases in the $\|x\| \rightarrow \infty$ tail.
This choice of $\pi$ therefore has the potential to introduce over-dispersion relative to $P$ and to place greater emphasis on the areas of $P$ with the highest magnitude of gradient.

\subsection{Sparse Approximation}
\label{subsec: sparse approximation}

Despite the elegance of Stein importance sampling, the computational burden of solving \eqref{eq: constrained program} can be considerable when the number $n$ of samples is large; roughly speaking the cost is $O(n^3)$.
Unfortunately $n$ is often large for algorithms based on laws of large numbers, such as \ac{mcmc}.
One can sub-sample a smaller number of states $m \ll n$ at regular intervals from the \ac{mcmc} output (this is usually called \emph{thinning}), but then the approximation quality of the sub-samples may be poor in general.
It is therefore natural to seek instead an \emph{optimal} subset of size $m$ from the \ac{mcmc} output.

This issue is closely related to that of optimal weighted approximation in \Cref{subsec: stein importance sampling}, in that we can seek weight vectors for which at most $m \ll n$ entries are non-zero.
This observation has motivated a range of sparse approximation techniques, which aim to iteratively construct an approximation of the form $Q_{n,m} = \frac{1}{m} \sum_{i=1}^m  \delta(y_i)$, where each $y_i$ is an element from a list of candidate states $\{x_1,\dots,x_n\}$, which could be \ac{mcmc} output.
The canonical example is the greedy algorithm which, at iteration $j$, selects a state
\begin{align}
y_j \in \argmin_{y \in \{x_1,\dots,x_n\} } \mathcal{S}_P\left( \frac{1}{j} \delta(y) + \frac{1}{j} \sum_{i=1}^{j-1} \delta(y_i) \right)  \label{eq: greedy alg}
\end{align}
for which the Stein discrepancy is minimized.
In the context of \ac{ksd}, the greedy algorithm \eqref{eq: greedy alg} is called \emph{Stein Thinning} and has computational cost $O(m n)$ if implemented with appropriate sufficient statistics \citep[App.~D.1]{li2024debiased}. 
Furthermore, under appropriate assumptions, the sparse approximation converges to the optimally weighted approximation $Q_n^\star$ as $m \rightarrow \infty$ with $n$ fixed, since
\begin{align}\label{stein_thin_guarantee}
\smash{\mathcal{S}_P(Q_{n,m})^2 \leq \mathcal{S}_P(Q_n^\star)}^2 + \frac{\max_{i}k_p(x_i,x_i)(1+\log(m))}{m}.
\end{align}
Hence, Stein Thinning performs both bias correction and compression, yielding a sparse output nearly as accurate as the best simplex-weighted input. 
See \citet{riabiz2020optimal} for full details and \citet{teymur2021optimal} where non-myopic and mini-batch extensions of the greedy algorithm are also considered.

Sparse approximation is particularly valuable in the context of predictive modeling, where one needs to integrate out posterior uncertainty regarding model parameters for the purpose of probabilistic prediction.
For sophisticated predictive models, such as those based on detailed mechanistic descriptions of physical processes, there can be a considerable computational cost associated with forward simulation from a model.
Then the ability to produce a sparse approximation of the posterior distribution $P$ translates into ability to control the overall simulation cost.

One downside of the guarantee \eqref{stein_thin_guarantee} is that it requires $m=\Omega(n)$ sample points to match the usual $O(1/\sqrt{n})$ convergence rate of $\mathcal{S}_P(Q_n^\star)$. Fortunately, a number of extensions have been developed to recover the standard $O(1/\sqrt{n})$ convergence rate with substantially fewer sample points. For example, the \emph{Stein Kernel Thinning} algorithm of \citet{li2024debiased} can deliver $\mathcal{S}_P(Q_{n,m}) = O(1/\sqrt{n})$ error with $m = \sqrt{n} \polylog(n)$ points, while, by using a weighted approximation, the \emph{Stein Recombination} and \emph{Stein Cholesky} of \citet{li2024debiased} can each match this error with only $m=\polylog(n)$ points.

\subsection{Integral Approximation}
\label{subsec: control variates}

Often one is interested in approximating a posterior $P$ only as a stepping-stone to calculating posterior expectations of interest.
Indeed, given a (possibly weighted) empirical approximation $Q_n = \sum_{i=1}^n w_i \delta_{x_i}$ to $P$, one can construct a corresponding approximation
\begin{align}
    \int f \; \mathrm{d}P \approx \sum_{i=1}^n w_i f(x_i) \label{eq: cubature}
\end{align}
to posterior expectations of interest.
However, in this two-step approach the weights $w_i$ and states $x_i$ have been selected in a manner that is agnostic to the function $f : \mathcal{X} \rightarrow \mathbb{R}$.
It is therefore natural to ask whether one can do better by allowing either the weights, states, or both to be $f$-dependent.

\subsubsection*{Integration via Solution of the Stein Equation}

One approach to improvement, outlined in the early work of \citet{stein1986approximate}, is to choose a Stein operator $\mathcal{T}_P$ and Stein set $\mathcal{G}$ and to attempt to solve the \emph{Stein equation}
\begin{align}
    f(x) & = c + (\mathcal{T}_P g)(x) \label{eq: Stein equation}
\end{align}
for a constant $c \in \mathbb{R}$, an element $g\in\mathcal{G}$, and $P$-almost-all $x \in \mathcal{X}$.
Indeed, from the definition of a Stein operator we would then have $c = \int f \; \mathrm{d}P$ being precisely the integral of interest.
Through the application of numerical methods to \eqref{eq: Stein equation}, one can arrive at $f$-dependent weights $w_i$ and states $x_i$ for approximation as in \eqref{eq: cubature}, as will now be demonstrated.

For the purposes of illustration, take $\mathcal{G}$ to be the unit ball in a vector-valued \ac{rkhs} $\mathcal{H}(K)$, and consider the minimal norm interpolant
\begin{align}
    \argmin_{c' \in \mathbb{R} , \; g' \in \mathcal{G}}  \|g'\|_{\mathcal{H}(K)} \quad \text{such that} \quad f(x_i) = c' + (\mathcal{T}_P g')(x_i), \quad \forall i \label{eq: min norm interp}
\end{align}
at fixed \emph{collocation nodes} $\{x_i\}_{i=1}^n$.
Assuming distinct collocation nodes and $K$ being positive definite, the value $c_n(f)$ of $c'$ that solves \eqref{eq: min norm interp} can be analytically calculated:
\begin{align}
    c_n(f) = \frac{1_n^\top K_p^{-1} f_n}{1_n^\top K_p^{-1} 1_n} \label{eq: CF estimator}
\end{align}
where $[K_p]_{i,j} = k_p(x_i,x_j)$ and $k_p$ is the Stein kernel \eqref{eq:general-stein-kernel}, $1_n$ denotes a column vector of ones of length $n$, and $[f_n]_i = f(x_i)$.

The \emph{control functionals} estimator \eqref{eq: CF estimator} corresponds to a weighted approximation of the form \eqref{eq: cubature} where the weights are $w \propto K_p^{-1} 1_n$, normalized such that $1_n^\top w = 1$. 
Importantly, this non-uniform weighting enables higher-fidelity estimation of the integral $\int f \; \mathrm{d}P$. For example, for suitable $f$, $\mathcal{T}_P$, and $\{x_i\}_{i=1}^n$ drawn independently from $P$, the estimator \eqref{eq: CF estimator} enjoys
\begin{align}
\E\left[\left(\int f \; \mathrm{d}P - c_n(f)\right)^2\right] = O(n^{-7/6})
\end{align} 
mean squared error \citep[Thm.~2]{oates2017control}, a strict improvement over the $\Theta(n^{-1})$ mean squared error of standard Monte Carlo integration. 
The weights $w$ can also be understood as a relaxed solution to the Stein importance sampling problem in \eqref{eq: constrained program} with the positivity constraint removed; cf. the right panel of \Cref{fig: SIS}.

The weights we have just derived are $f$-independent and the states were simply fixed; how can we take the specific function of interest $f$ into account?
The answer comes through $f$-dependent selection of the kernel $K$.
Numerous practical methods for kernel choice are available, such as cross-validation, but in an attempt to give insight we consider here the error bound
\begin{align*}
    \left| \int f \; \mathrm{d}P - c_n(f) \right| & \leq \underbrace{ \left( \inf_{\substack{f = c + \mathcal{T}_P g \\ c \in \mathbb{R}, \; g \in \mathcal{G}}} \|g\|_{\mathcal{H}(K)} \right) }_{\text{(i)}} \underbrace{ \mathcal{S}_P\left( \sum_{i=1}^n w_i \delta_{x_i} \right) }_{\text{(ii)}}
\end{align*}
that decomposes integration error into the product of (i) a term that depends on the true function $f$ and the kernel $K$, but not on the states $\{x_i\}_{i=1}^n$, and (ii) the Stein discrepancy associated to the empirical distribution $\sum_{i=1}^n w_i \delta_{x_i}$ as an approximation to $P$, which is $f$-independent.
Compared to Stein importance sampling, which considers only term (ii) for a fixed kernel, we can instead aim to pick $K$ to simultaneously balance the product of terms (i) and (ii) and thus minimize the overall error bound.
The $f$-dependence of term (i) then leads to $f$-dependent selection of the kernel $K$, and this can substantially improve the suitability of the weights that are used in approximating the integral of interest.

\begin{remark}[Stein Control Variates]
\label{rem: cvs}
    \citet{stein2004use} adopt a related perspective on solving the Stein equation that takes the form of a variance reduction technique for Monte Carlo.
    Indeed, the standard Monte Carlo estimator
    $$
        \frac{1}{m} \sum_{i=1}^m f(X_i), \qquad X_1,\dots,X_m \stackrel{\text{i.i.d.}}{\sim} P
    $$
    can be replaced by the alternative estimator
    \begin{align}
        \frac{1}{m} \sum_{i=1}^m f(X_i) - (\mathcal{T}_P g)(X_i) \label{eq: stein cv}
    \end{align}
    with any element $g$ from the Stein set $\mathcal{G}$.
    An optimal \citep[or \emph{zero variance}, as termed by][]{assaraf2003zero,mira2013zero} choice of $g$ would correspond to a solution of the Stein equation \eqref{eq: Stein equation}, while in practice one can seek $g$ to minimize an estimate for the variance of \eqref{eq: stein cv}.
    For applications involving Bayesian posteriors it is unusual to have access to exact samples from $P$, and therefore direct solution of the Stein equation might be preferable to trying to reduce the variance of an estimator based on \ac{mcmc}.
    On the other hand, Stein control variates have been successfully used for gradient estimation, as we will see in \Cref{subsec: discrete gradient}.
\end{remark}

Several works have contributed to the development of numerical methods for solution of the Stein equation, including 
\citet{oates2017control,oates2019convergence,barp2022riemann,si2020scalable,belomestny2020variance,south2022semi,south2022postprocessing,south2023regularized,sun2023vector,leluc2023speeding,belomestny2024theoretical}.

\subsubsection*{Collocation Nodes via \ac{svgd}} 
\label{sec:svgdmoment}

Given a (possibly $f$-dependent) kernel $K$, \ac{svgd} (\Cref{subsec: batch}) provides an attractive method for selecting collocation nodes $\{x_i\}_{i=1}^n$ for integral approximation; we now discuss the properties of these nodes assuming that they are a fixed point of \ac{svgd}. 
To simplify discussion we suppose that $K(x,y) = k(x,y) I$.
By the update rule \eqref{eq: SVGD update rule}, at a fixed point of \ac{svgd}
\begin{align}\label{equ:avgk}
\frac{1}{n}\sum_{j=1}^n 
\mathcal T_P^{(1)} k(\vx_j,  x_i) 
= 0, \qquad \forall i = 1, \ldots, n. 
\end{align}
On the other hand, by Stein's identity, $\E_{X \sim P} [\mathcal{T}_P^{(1)} k(X, x_i)] = 0$.
This implies that the particles \emph{exactly} estimate the expectation of functions of form $f(\vx) = \mathcal{T}_P^{(1)} k(\vx,\vx_i)$ for each $i = 1,\dots,n$. 
By the linearity of expectation, the same holds for all functions in the linear span of this set.
That is,
$$
\frac{1}{n} \sum_{i=1}^nf(\vx_i) 
=  P(f)
$$
holds for all $f \in \mathcal F^*$, the \emph{Stein matching set}
$$ 
\mathcal F^* \defeq  
\left \{ \sum_{i = 1}^n  a_i^\top  \mathcal{T}_P^{(1)} k(\cdot,x_i) + b \colon \forall    a_i\in \R^d,   ~ b \in \R \right \} .
$$
Note that exactness holds when the collocation nodes $\{x_i\}_{i=1}^n$ are a fixed point set of \ac{svgd}, but there need not be a unique fixed point set.
Extending this, the expectation of functions that are close to $\F^*$ can be estimated better than the ones far away from $\F^*$. 
Specifically, let $\F^*_{\vv\xi}$ be the $\vv\xi$ neighborhood of $\F^*$, that is, 
$\F_\vv\xi^* = \{f \colon \inf_{f'\in \F} \|f -f'\|_\infty \leq \vv\xi  \},$ 
then it is easily shown that 
$$
\left | \frac{1}{n}\sum_{i=1}^n f(\vx_i^* ) -   P(f) \right |  \leq 2 \vv\xi, ~~~~~~~\forall f \in \mathcal F^*_\vv\xi, 
$$
see \citet{liu2018stein}.

An interesting special case is when $P$ is multivariate Gaussian and $k$ is a polynomial kernel of degree $p$ (\Cref{ex: poly space kernel}), in which case the numerical approximation produced by \ac{svgd} exactly matches the first $p$ moments of $P$.
For example, if we use a linear kernel  $k(\vx,x') = 1 + x^\top x'$ and the 
number of independent particles is no smaller than $d$ (more precisely, when the rank of matrix $[x_i;1]_{i=1}^n$ is no smaller than $d+1$), then the matching set $\mathcal F^*$ contains all the linear and quadratic functions, and hence any fixed point of \ac{svgd} exactly estimates both the mean and the covariance matrix of $P$. 
Because many practical distributions are close to Gaussian, thanks to the central limit theorem and Bernstein–von Mises theorem, this observation can be leveraged to design better algorithms to perform particularly well for Gaussian like distributions. %

\section{Training Generative Models}
\label{sec: param est}

This section focuses on the challenges associated with training generative models.
An important class of \emph{energy-based} generative models are associated with an intractable normalizing constant, and for this class Stein discrepancy and Stein dynamics provide useful alternatives to maximum likelihood estimation, in the form of \emph{Stein contrastive divergence} (\Cref{sec:steincd}) and \emph{minimum Stein discrepancy estimation} (\Cref{subsec: MSDE}).
For generating ultra-realistic data, adversarial approaches are often used; here Stein dynamics offers a complementary approach to training \emph{generative adversarial networks} (\Cref{sec:steingan}).
Bayesian approaches can often be computationally challenging outside of simple conjugate settings; here Stein discrepancy can provide a useful objective for \emph{variational Bayesian} methods (\Cref{subsec: var with stein}), while in a related way Stein dynamics can enable more expressive modeling choices in the setting of \emph{variational autoencoders} (\Cref{subsec: VAEs}).

\subsection{Stein Contrastive Divergence}
\label{sec:steincd}

A generic approach to constructing flexible probabilistic models on $\mathbb{R}^d$ is to take a flexible function approximator, such as a deep neural network, and to apply a positivity-enforcing transform to it.
This leads to a so-called \emph{energy-based} model $P_\theta$ with density of the form
\begin{align}
    p_\theta(x) =  \frac{1}{Z_\theta} \exp( f_\theta(x) ) , ~~~  Z_\theta = \int \exp(f_\theta(x))\df x, 
    \label{eq: energy based model} 
\end{align}
where $f_\theta : \mathbb{R}^d \rightarrow \mathbb{R}$ is a flexible (negative) \emph{energy} function with parameters $\theta \in \Theta$.
Given a dataset $\{x_i\}_{i = 1}^n \subset \mathbb{R}^d$, an important task is to select an appropriate value for the parameter $\theta$ so that the energy-based model is capable of generating samples that are statistically similar to those in the dataset (the precise meaning of which will be application-dependent).
Unfortunately, the absence of a normalized density for $P_\theta$ provides a substantial barrier to the use of classical statistical procedures such as maximum likelihood, since the implicit normalization constant in \eqref{eq: energy based model} will be $\theta$-dependent in general.
Indeed, if we inspect the gradient of the log-likelihood
\begin{align*}
    \nabla_\theta \sum_{i=1}^n \log p_\theta(x_i) & = \left[ \sum_{i=1}^n f_\theta(x_i) \right] - \nabla_\theta \log Z_\theta 
\end{align*}
then the first term is computable but the second is
\begin{align*}
    \nabla_\theta \log Z_\theta & = \frac{1}{Z_\theta} \nabla_\theta \int \exp(f_\theta(x)) \mathrm{d}x \\
    & = \frac{1}{Z_\theta} \int \exp(f_\theta(x)) \nabla_\theta f_\theta(x) \mathrm{d}x 
    = \mathbb{E}_{X \sim P_\theta}[\nabla_\theta f_\theta(X)]
\end{align*}
which requires computation of an intractable expectation with respect to $P_\theta$.
Letting $P_n = \frac{1}{n} \sum_{i=1}^n \delta_{x_i}$ denote the empirical distribution of the dataset, one step of gradient ascent can be written as
$$
\vtheta \gets \vtheta + \epsilon \{ \E_{X \sim P_n} [\nabla_\vtheta f_\theta (X)]  -  \E_{X \sim P_\theta} [\nabla_\vtheta f_\theta (X) ] \}
$$
for some learning rate $\epsilon > 0$.
Intuitively, this update rule iteratively increases $f_\theta(\vx)$  on the observed data from $P_n$ (or the \emph{positive sample}), while decreasing $f_\theta(\vx)$ on the data drawn from the hypothesized model $P_\theta$ (\emph{a.k.a. negative sample}).  
When the algorithm converges, the expectations of $\nabla_\vtheta f_\theta (x)$ under the empirical distribution $P_n$ and the model $P_\theta$ should be equal.  

To proceed, practical algorithms need to approximate the expectation with respect to $P_\theta$, which could be done, for example, by using
\ac{mcmc} \citep[e.g.,][]{geyer1991markov}, contrastive divergence \citep{hinton2002training}, or variational inference \citep[e.g.,][]{wainwright2008graphical}.  
In particular, \ac{cd} stands out as a simple approach that approximates the model expectation term by running $m$ steps of a $P_\theta$-invariant Markov transition kernel starting from the data distribution $P_n$: 
\begin{align}\label{equ:cdgrad}
\vtheta \gets \vtheta + \epsilon \{ \underbrace{ \E_{X \sim P_n} [\nabla_\vtheta f_\theta (X)]  -  \E_{X \sim K_m^\theta(P_n , \cdot)} [\nabla_\vtheta f_\theta (X) ] }_{(\star)} \} ,
\end{align}
where $K_\theta$ is a $P_\theta$-invariant Markov transition kernel and $K_m^\theta(P_n,\cdot)$ denotes the distribution obtained by applying $m$ iterations of the Markov chain initialized at a random sample from $P_n$.
For an ergodic Markov chain, $K_m^\theta(P_n , \cdot)$ converges to $P_\theta$ as $m \rightarrow \infty$, and hence \ac{cd} recovers gradient ascent on the log-likelihood. 
In practice, however, it is advised to use a small $m$, to save computation. 
In this case, the increment $(\star)$ can be interpreted (up to a sign) as $\E_{X \sim P_n} [\mathcal{T}_{P_\theta} g(X)]$, where
\begin{align*}
    (\mathcal{T}_{P_\theta} g)(x) = \int g(y) K_m^\theta(x,\mathrm{d}y) - g(x)
\end{align*}
is a Stein operator constructed from the discrete time Markov chain (cf. \Cref{subsec: operators from Markov}) and $g(x) = \nabla_\vtheta f_\theta (x)$.

Instead of using a Markov chain, we may use the \ac{svgd} update as the perturbation in \ac{cd}. 
That is, we can perturb the observed data $\{x_i\}_{i=1}^n$ with a deterministic transform $ x \mapsto x + \epsilon g_\theta(x)$ as in \ac{svgd}, 
where the velocity field
\begin{align}
    g_\theta(x) = g_{P_\theta,P_n}^*(x) = \frac{1}{n} \sum_{i=1}^n \mathcal{T}_{P_\theta}^{(1)} K(x_i,x) \label{eq: CD optimal}
\end{align}
is associated with steepest descent, as in \eqref{equ:ffs}.
Note that the Stein operator in \Cref{eq: CD optimal} is the Langevin Stein operator used in \ac{svgd}, which provides a deterministic counterpart to the Markov transition in contrastive divergence and naturally connects to Stein discrepancy, analogously to how contrastive divergence connects to Fisher divergence.
The gradient update of $\theta$ is then
\begin{align}\label{equ:steincd}
\vtheta \gets \vtheta + \epsilon \{ \E_{X \sim P_n} [\nabla_\vtheta f_\theta (X)  -  \nabla_\vtheta f_\theta (X + \epsilon g_\theta(X)) ] \}, 
\end{align}
where $g_\theta(x)$ denotes the optimal velocity field from \ac{svgd}. 
\emph{Stein contrastive divergence}  \citep{liu2017learning} 
uses a symmetric variant of the perturbation  
\begin{align}\label{equ:steincd_adv}
\theta \gets \theta + \epsilon \{ \E_{X \sim P_n}[\nabla_\theta f_\theta(X - \epsilon g_\theta^*(X)) - \nabla_\theta f_\theta(X + \epsilon g_\theta^*(X))] \} ,
\end{align}
where the contrast is taken with respect to the perturbation $x \pm \epsilon g_\theta^*(x)$ from two opposite directions; the rationale for this is explained in \Cref{subsec: MSDE}.

\subsection{Minimum Stein Discrepancy Estimators}
\label{subsec: MSDE}

Another approach to parameter estimation in  energy-based models \eqref{eq: energy based model} is to 
minimize the Stein discrepancy between the empirical distribution of the data, $P_n$, and the parametric model, $P_\theta$. 
Recall from \Cref{sec: Fisher divergence} that the (classical) score matching approach of \citet{hyvarinen2005estimation} selects $\theta$ to minimize the Fisher divergence
\begin{align*}
    \mathrm{FD}(P_n||P_\theta) = [\text{constant in $\theta$}] + \frac{1}{n} \sum_{i=1}^n  2 (\Delta_x f_\theta)(x_i) + \|(\nabla_x f_\theta)(x_i)\|^2  ,
\end{align*}
which is obtained by plugging in the empirical distribution $P_n$ in place of $Q$ in \eqref{eq: DF2}.
The gradient descent of $\theta$ is then $\vtheta \gets \vtheta - \epsilon  g_{\mathrm{FD}},$ where 
\begin{align}
\label{eq: grad update FD} 
g_{\mathrm{FD}} \defeq  \E_{X \sim P_n} [\nabla_\vtheta (\Delta_x f_\theta) (X) + (\nabla_x f_\theta)(X) \cdot \nabla_\vtheta (\nabla_x f_\theta)(X) ]. 
\end{align}
As it turns out, this update rule is the limit of \ac{cd} where the perturbation is defined by a single step of Langevin dynamics, and where the step size $\epsilon$ approaches zero: 
\begin{align} \label{equ:gdcdfishereq}
\E_{X \sim P_n} [\nabla_\theta f_\theta(\hat X) - \nabla_\theta f_\theta(X)] 
= \epsilon g_{\mathrm{FD}}  +o(\epsilon),  
\end{align}
where $\hat X = X + {\epsilon} \nabla_x f_\theta(X) + \sqrt{2\epsilon} Z$ is the Langevin update and $Z$ is standard normal. 
This can be seen by a second-order Taylor expansion. Let 
\(
\delta = \epsilon \nabla_xf_\theta(X) + \sqrt{2\epsilon}\, Z
\) so that $\hat X = X + \delta$. 
Then
\begin{align*}
\nabla_\theta f_\theta(X+\delta)
&= \nabla_\theta f_\theta(X)
  + \nabla_x \nabla_\theta f_\theta(X)\,\delta
  + \tfrac12\, \delta^\top \nabla_x^2 \nabla_\theta f_\theta(X)\,\delta
  + o(\|\delta\|^2).
\end{align*}
Taking expectation over \(Z\) and using \(\E[Z]=0\) and \(\E[ZZ^\top]=I\), we obtain
\begin{align*}
& \hspace{-20pt} \E_Z[\nabla_\theta f_\theta(\hat X)]- \nabla_\theta f_\theta(X) \\
& \hspace{20pt} = \epsilon (\nabla_x f_\theta(X)\cdot \nabla_\theta\nabla_x f_\theta(X)
+ \,\nabla_\theta \Delta_x f_\theta(X))
+ o(\epsilon).
\end{align*}
Substituting this into \eqref{equ:gdcdfishereq}  and taking expectation over
\(X\sim P_n\) yields the claimed expansion.

Fisher divergence is just one instance of a Stein discrepancy; more generally a \emph{minimum Stein discrepancy estimator} is defined as
\begin{align*}
    \min_\theta \; \mathcal{S}_{P_\theta}(P_n) ,
\end{align*}
for \emph{some} Stein discrepancy $\mathcal{S}_{P_\theta}$  \citep{liu2017learning,liu2019estimation,barp2019minimum}.
Consider for instance the Langevin \ac{ksd} (\Cref{subsec: lksd in detail}), which has been proposed as a convenient way to construct estimators for energy-based models that are robust to outliers in the dataset through appropriate choices of the kernel \citep{barp2019minimum}. 
Gradient descent on the (squared) \ac{ksd} amounts to working with the gradient
\begin{align}
\nabla_{\theta} \mathrm{KSD}_{P_\theta}^2(P_n) 
=  2 \E_{X\sim P_n}[\nabla_\theta (\nabla_xf_\theta)(X) g_\theta(X)], \label{eq: grad KSD2}
\end{align} 
where $g_\theta$ was defined in \eqref{eq: CD optimal}.

Compared with the classical score matching update \eqref{eq: grad update FD}, which involves calculating a third-order derivative $\nabla_\theta\Delta_x f$, gradient descent on the (squared) \ac{ksd} involves only second-order derivatives, and is therefore easier to implement.

Further, performing a first-order Taylor expansion of \eqref{eq: grad KSD2} yields the Stein \ac{cd} update in \eqref{equ:steincd}:
\begin{align*}
& \hspace{-30pt} \mathbb{E}_{X \sim P_n}[ \nabla_\theta f_\theta(\hat X_{\mathrm{SVGD}}) - \nabla_\theta f_\theta(X) ] \\
& \hspace{50pt} =  \epsilon \mathbb{E}_{X \sim P_n}[  \nabla_\theta (\nabla_xf_\theta)(X) g_\theta(X) ] + o(\epsilon), 
\end{align*}
where $\hat X_{\mathrm{SVGD}} = X + \epsilon g_\theta(X)$. 
Therefore, Stein \ac{cd} provides a finite difference approximation of gradient descent on the \ac{ksd}, requiring only first order derivatives to be computed. 
This finite difference interpretation motivates the symmetric perturbation used in \eqref{equ:steincd_adv}, since the centered finite difference approximation provides closer alignment with the gradient of the \ac{ksd}. 

The asymptotic statistical properties of minimum Stein discrepancy estimators can be understood through the framework of $M$-estimators \citep{barp2019minimum,matsubara2022robust,oates2022minimum}.
The specific case of minimum Stein discrepancy estimators based on finite Stein sets was considered in \citet{ebner2025stein}. 
Minimum KSD estimators with diffusion Stein operators (see \cref{subsec: ksds in general}) have also been deployed and analyzed in the setting of causal model learning \citep{lorch2024causal,bleile2026efficient}.

\subsection{Stein GAN}
\label{sec:steingan}

Contrastive divergence has been widely used for learning energy-based models $P_\theta$ \eqref{eq: energy based model} , and can often train models with good test likelihood. 
However, models trained by \ac{cd} often struggle to generate truly realistic images.
This is because \ac{cd} learns the model parameters based on a local perturbation in the neighborhood of the observed data, 
and does not explicitly train the model to create images from scratch. 
This difficult was sidestepped by \acp{gan} \citep{goodfellow2014generative}, 
which explicitly train a \emph{generator} (a deep neural network that takes random noise and outputs
images) to match the observed data with the help of a \emph{discriminator} that acts adversarially, to distinguish the generated data from data in the training set. 
Motivated by \acp{gan}, we can modify the \ac{cd} idea to explicitly incorporate a generator into the training process. 

The idea is based on \emph{amortizing} the sampling process of $P_\theta$ with a generator, and using the simulated samples as the negative samples to update $\theta$. 
To be specific, let $G_\beta$ be the distribution obtained by passing random noise $\xi$ (e.g. standard Gaussian noise) through a neural network $N_\beta(\cdot)$ with parameters $\beta$.
We will first seek to adjust the parameters $\beta$ adaptively to make $G_\beta$ as close as possible to the model $P_\theta$, and then we will update $\theta$ using \ac{cd} with the generator $G_\beta$ in place of the model $P_\theta$, i.e.
\begin{align}\label{equ:thetaG}
\theta \gets \theta + \epsilon \{ \E_{X \sim P_n}[\nabla_\theta f_\theta(X)] - \E_{X \sim G_\beta}[\nabla_\theta f_\theta(X)] \} .
\end{align}
The key question here is how to  update $\beta$ so that the distribution $G_\beta$ closely approximates the model $P_\theta$. 
Given samples $\{\tilde{x}_i\}_{i=1}^m$ from the current generator $G_\beta$, i.e. generated as $\tilde{x}_i = N_\beta(\xi_i)$, applying the \ac{svgd} update leads to new samples $\{\tilde{x}_i'\}_{i=1}^m$, where $\tilde{x}_i' = \tilde{x}_i + \epsilon g_\theta(\tilde{x}_i)$, that should represent a more accurate approximation to $P_\theta$. 
The idea of amortized \ac{svgd}  \citep{feng2017learning} is to update $\beta$ by taking a gradient descent step on the least-squares objective
\begin{align*}
\sum_{i=1}^m \left\| N_\beta(\xi_i) - \tilde{x}_i' \right\|^2 ,
\end{align*}
so that the generator $G_\beta$ is encouraged to approximate the new samples $\{\tilde{x}_i\}_{i=1}^m$, which represent a better representation of $P_\theta$.
An explicit calculation leads to
\begin{align}
\beta \gets \beta  + \epsilon \sum_{i=1}^m (\nabla_\beta N_\beta)(\xi_i) g_\theta(N_\beta(\xi_i)) .  \label{eq: LS fit NN}
\end{align}
This is known as \emph{SteinGAN} \citep{liu2017learning}. 
Formally, this can be
viewed as approximately solving the following minimax objective function based on Kullback--Leibler divergence (\Cref{def: KL divergence}):   
\begin{align*}
\min_\theta \max_{\beta} \big\{ \KL{P_n}{P_\theta} - \KL{G_\beta}{P_\theta} \big \} 
\end{align*}
Here the energy model $P_\theta$, serving as a discriminator,  attempts to get closer to the observed data $P_n$, and keep away from the ``fake'' data distribution $G_\beta$, both in terms of  Kullback--Leibler divergence, while the generator $G_\beta$ attempts to get closer to the energy model $P_\theta$ using amortized \ac{svgd}. 

\emph{SteinGAN} can be viewed as a Kullback–Leibler divergence variant of the \ac{gan}-style adversarial game \citep{goodfellow2014generative}, and draws explicit connections to maximum likelihood training of energy-based models, where the intractable sampling (negative phase) is amortized via a neural generator trained to approximate \ac{svgd} dynamics. 
An example of the realistic images generated by SteinGAN is contained in \Cref{fig: SteinGAN}.

\begin{figure}[t!]
\centering
    \includegraphics[width=0.5\textwidth]{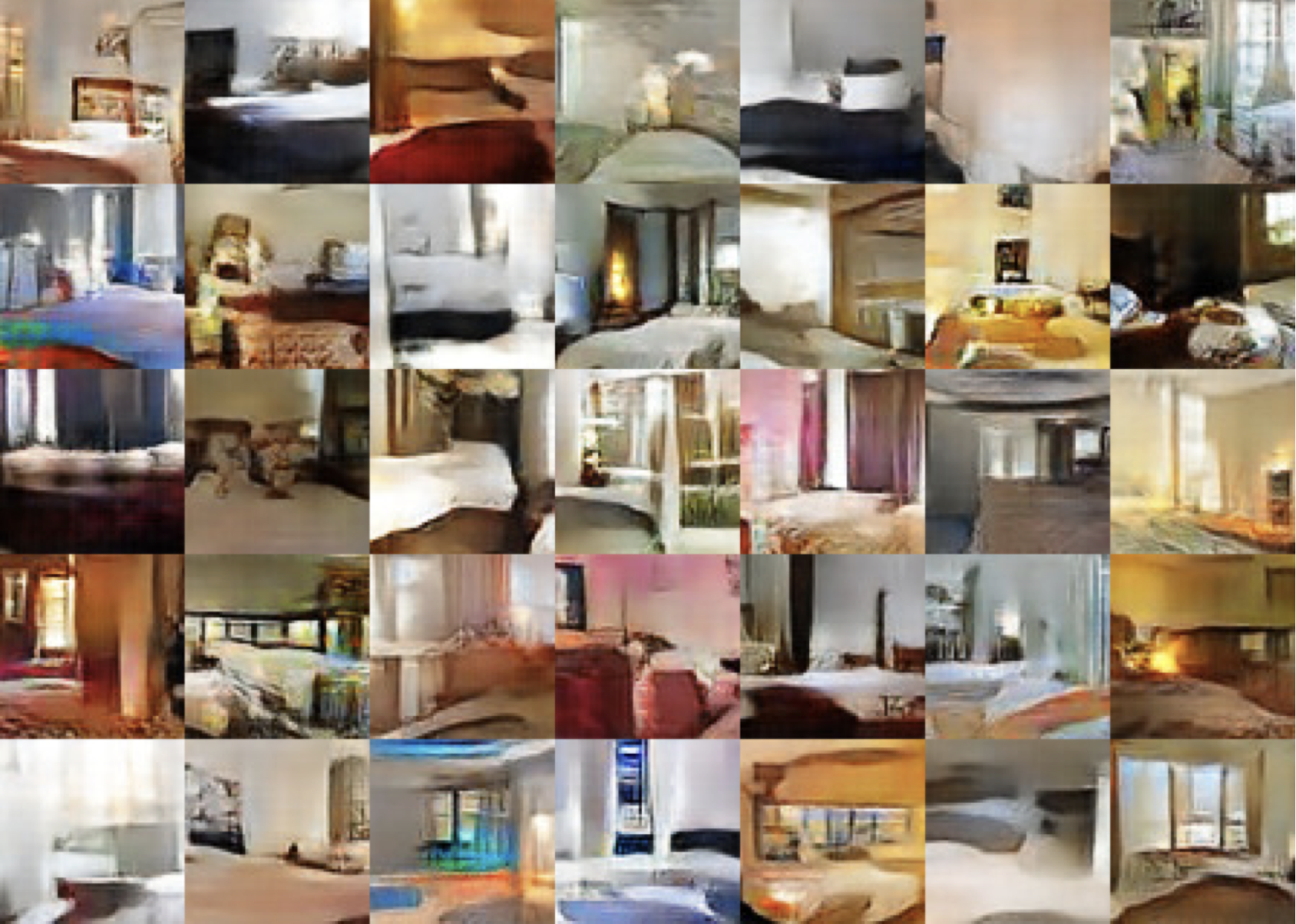}
    \caption{Images generated by SteinGAN trained on the LSUN dataset \citep{yu15lsun}, consisting of nearly 3M images of bedrooms. 
    Reproduced with permission from \citet{liu2017learning}.}
    \label{fig: SteinGAN}
\end{figure}

\subsection{Variational Methods for Posterior Approximation}
\label{subsec: var with stein}

A popular class of numerical methods for Bayesian analysis come under the umbrella of \emph{variational Bayes}; the idea, in brief, is that the posterior distribution is approximated using a generative model.
This can offer several advantages, most notably making it straightforward to employ the posterior as the prior for subsequent analyses, since its density is explicit.
The main technical issue is how one goes about approximating the posterior with a generative model.

Classical approaches to variational inference attempt to approximate the target posterior $P$ by selecting from a tractable family of distributions $\mathcal{Q}$ one for which the Kullback–Leibler divergence (\Cref{def: KL divergence})
\begin{align}
\argmin_{Q \in \mathcal Q} \; \KL{Q}{P},  \label{eq: var bayes}
\end{align}
is minimized.
Letting $p(x,y)$ denote the joint density of the unobserved $x$ and the observed data $y$, and similarly letting $p(x|y)$ denote the conditional and $p(y)$ denote the marginal, and $q(x)$ the density of $Q$, the Kullback--Leibler divergence is
\begin{align*}
\KL{Q}{P} & = \int q(x) \log \frac{q(x)}{p(x|y)} \; \mathrm{d}x \\
& = \int q(x) \log \frac{q(x)}{p(x,y)} \; \mathrm{d}x + \log p(y)
\end{align*}
so that minimization of the Kullback--Leibler divergence is equivalent to minimization of the \ac{elbo} 
\begin{align*}
    \mathcal{L}_Q(y) = \int q(x) \log \frac{p(x,y)}{q(x)} \; \mathrm{d}x .
\end{align*}
For certain combinations of prior, likelihood, and variational family $\mathcal{Q}$, the \ac{elbo} can be analytically computed and thus the optimization problem \eqref{eq: var bayes} can be approached as a numerical optimization task.
However, this restriction to ``conjugate'' combinations limits the extent to which the variational approximation can accurately reflect the posterior in general.
Beyond the conjugate setting, Monte Carlo methods can be used to approximate the \ac{elbo} \citep[and its gradient; see e.g.][]{ruiz2016generalized}, but these require explicit access to the density of $Q$ (and its gradient), inspiring a line of research into \emph{normalizing flows} \citep{rezende2015variational,kingma2016improved,dinh2016density}.

An alternative solution, which does not require access to the density of $Q$, is provided by Stein discrepancy.
That is, we can formulate
\begin{align}
    \argmin_{Q \in \mathcal Q} \; \mathcal{S}_P(Q) \label{eq: Stein VI}
\end{align}
and select a Stein discrepancy so that $\mathcal{S}_P(Q)$ can either be exactly computed or consistently approximated (e.g., from samples from $Q$).
This unlocks the possibility of solving \eqref{eq: Stein VI} as a numerical optimization task.
In particular, this removes the restriction of an explicit density for $Q$, and enables flexible approximations such as $Q_\theta = T_\#^\theta Q_0$ for a general neural network $T^\theta$, parametrized by $\theta$, and some reference distribution $Q_0$.
The latter construction was studied in combination with the Langevin \ac{ksd} in \citet{Fisher2021}, who found that in regular cases the approximations produced by minimizing \ac{ksd} were no less accurate than those produced by minimizing the Kullback--Leibler divergence; see \Cref{fig: KSDVI}.

\begin{figure}[t!]
    \centering
    \includegraphics[width=\textwidth]{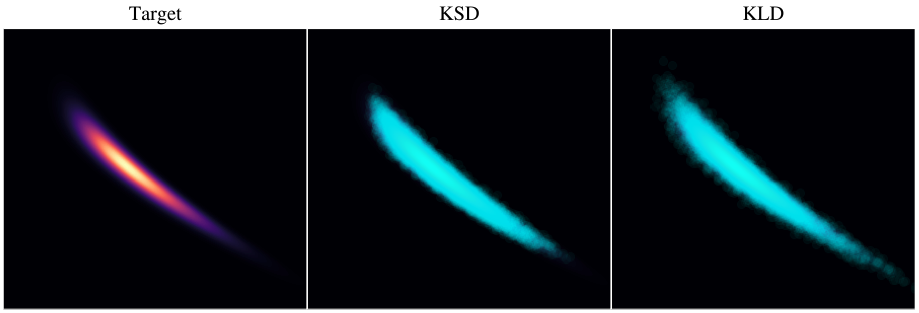}
    \caption{A comparison of variational approximations to a target $P$ (left) produced by minimizing \ac{ksd} (centre) and Kullback--Leibler divergence (right).
    The variational approximation was based on a block neural autoregressive flow and the target $P$ arose from a Bayesian analysis of a biochemical oxygen demand model.
    Reproduced with permission from \citet{Fisher2021}.
    }
    \label{fig: KSDVI}
\end{figure}

A related approach, called \emph{operator variational inference} in \citet{ranganath2016operator}, considers the same minimax problem in \eqref{eq: Stein VI}, i.e.
$$
\min_{\theta} \max_{g \in \mathcal{G}} \; \mathbb{E}_{X \sim Q_\theta}[ \mathcal{T}_P g(X) ]
$$
where $\mathcal{T}_P$ and $\mathcal{G}$ are the Stein operator and Stein set associated to the Stein discrepancy $\mathcal{S}_P$.
However, the authors instead proposed to parametrize $g$ using a neural network and then to alternate between taking a stochastic gradient step to update $\theta$, and a stochastic gradient step to update the parameters of this neural network $g$.
This enables greater flexibility in the Stein set $\mathcal{G}$, in principle enabling the differences between $Q_\theta$ and $P$ to be better detected, but comes at a cost of tuning a more involved numerical optimization method.

\subsection{Learning Variational Autoencoders}
\label{subsec: VAEs}

An elegant approach to constructing expressive generative models begins with a parametric generative model on an extended state space; let $p_\theta(x,z)$ be a parametrized probability density where $x \in \mathbb{R}^d$ are to be observed and $z \in \mathbb{R}^{d'}$ are latent.
Even if $p_\theta(x,z)$ is a relatively simple model, the marginal distribution with density
\begin{align}
    p_\theta(x) = \int p_\theta(x,z) \; \mathrm{d}z \label{eq: VAE marginal}
\end{align}
can capture rather complex dependencies among the variables that are observed.
The aim is thus to learn $\theta$, for example using maximum likelihood based on a dataset $\{x_i\}_{i=1}^n$.
However, this can be challenging because \eqref{eq: VAE marginal} does not possess a closed form in general.

To proceed, explicit approximations are required.
One solution is to train an \emph{encoder} network $q_\phi(z|x)$ to approximate the conditional $p_\theta(z|x)$.
This is the main idea behind \acp{vae}, which maximize an \ac{elbo}
\begin{align}
    \mathcal{L}_{\theta,\phi}(x) = \int q_\phi(z|x) \log \frac{p_\theta(x,z)}{q_\phi(z|x)} \; \mathrm{d}z \label{eq: VAE elbo}
\end{align}
to jointly train both the model and encoder parameters $\theta$ and $\phi$. 
Indeed, we can rewrite \eqref{eq: VAE elbo} as
\begin{align*}
    \mathcal{L}_{\theta,\phi}(x) = \log p_\theta(x) - \mathrm{KL}(q_\phi(\cdot | x) || p_\theta(\cdot | x) )
\end{align*}
which makes clear we are promoting large values of the likelihood $p_\theta(x)$ while also requiring the encoder $q_\phi(z|x)$ to be a good approximation to the conditional $p_\theta(z|x)$.
However, working with \eqref{eq: VAE elbo} usually requires a tractable density for $q_\phi(z | x)$, the same issue we encountered for variational Bayes in \Cref{subsec: var with stein}.
As a result, the expressiveness of the encoder network, and hence the potential of the \ac{vae} to uncover hidden structure, can be limited.

A potential solution is the Stein \ac{vae} \citep{pu2017stein,feng2017learning}.
The idea of the Stein \ac{vae} is that, instead of specifying an encoder network, we can use \ac{svgd} to represent the conditional distributions $p_\theta(z|x_i)$ implicitly via a set of particles $\{z_{i,j}\}_{j=1}^m$.
Compared with the standard VAE, which assumes $q{\phi}(\cdot\mid x)$ to be a Gaussian conditioned on $x$, the nonparametric particle approximation can capture richer posterior structures. 

The approach of \citet{feng2017learning} then updates the model parameters $\theta$ according to
\begin{align*}
    \theta \gets \theta + \epsilon \sum_{i=1}^n \sum_{j=1}^m\nabla_\theta \log p_\theta(x_i,z_{i,j}) ,
\end{align*}
which can be interpreted as gradient ascent on the joint likelihood where the pairs $(x_i,z_{i,j})$ are viewed as a pseudo-dataset.
By the Fisher identity
\begin{align*}
\nabla_\theta \log p_\theta(x)
= \mathbb{E}_{z \sim p_\theta(z \mid x)}
\big[ \nabla_\theta \log p_\theta(x,z) \big].
\end{align*}
Hence, if the particles \(\{z_{i,j}\}_{j=1}^m\) were exact samples from the posterior
\(p_\theta(z \mid x_i)\), the update above would recover (up to Monte Carlo error) the exact gradient of the marginal log-likelihood \(\sum_i \log p_\theta(x_i)\).
To speed up computation, a \emph{recognition} network $N_\beta(\cdot)$ can be trained to mimic the \ac{svgd} particle dynamics, i.e.
$$
z_{ij} = N_\beta(\xi_j, x_i), \qquad \forall i = 1,\dots,n,~~~ j = 1,\dots, m, 
$$
where $\xi_i$ are sampled from a noise distribution (e.g. standard Gaussian) and the parameters $\beta$ are trained based on least squares, implemented via gradient descent as in \eqref{eq: LS fit NN}.
This \emph{amortization} avoids the need to run \ac{svgd} separately for each entry $x_i$ in the dataset.

\section{Gradient Estimation}
\label{sec: gradient est}

The computational task of numerically computing a gradient 
\begin{align}\label{eq:gradient_estimation_objective}
\nabla_\phi \mathbb{E}_{X \sim Q_\phi}[f(X)]    
\end{align}
with respect to the parameters $\phi$ of a distribution $Q_\phi$ is encountered in many learning problems, including variational inference \citep[cf. \Cref{subsec: var with stein} and][]{paisley2012variational} and training variational autoencoders \citep[cf. \Cref{subsec: VAEs} and][]{Kingma2014}, where one seeks to perform gradient ascent on an \ac{elbo}, and in reinforcement learning when one seeks to improve a policy by following a policy gradient \citep[cf. \Cref{subsec: CVs for RL} and][]{williams1992simple}. 
The difficulty arises, at a fundamental level, when the expectation \cref{eq:gradient_estimation_objective} cannot be exactly computed and when evaluation of $f$ incurs a substantial computational cost.
As such, many authors have proposed numerical methods to approximate \eqref{eq:gradient_estimation_objective} using a small number $m$ of samples (often $m \in \{1,2,3\}$); we will see how Stein operators can play a useful role in this context, focusing on the case where $Q_\phi$ is discrete (\Cref{subsec: discrete gradient}) and the case of a \emph{policy gradient} as encountered in applications of reinforcement learning (\Cref{subsec: CVs for RL}).

\subsection{Gradient Estimation with Discrete Stein Operators}
\label{subsec: discrete gradient}

This section focuses on the case where $Q_\phi$ is a discrete distribution on a set $\mathcal{X}$, where the size of $\mathcal{X}$ renders exact calculation of \eqref{eq:gradient_estimation_objective} impractical.
A popular strategy in this instance is to rewrite the gradient as
\begin{align}
\grad_{\phi}\mathbb{E}_{X \sim Q_\phi}[f(X)]
    & = \nabla_\phi \sum_{x \in \mathcal{X}} f(x) q_\phi(x) \nonumber \\ 
    & = \sum_{x \in \mathcal{X}} f(x) \frac{\nabla_\phi q_\phi(x)}{q_\phi(x)} q_\phi(x) \nonumber \\ 
    & = \mathbb{E}_{X \sim Q_\phi}[f(X) \nabla_\phi \log q_\phi(X)] \label{eq:gradient_estimation_gradient}
\end{align}
where $q_\phi$ is the probability mass function for $Q_\phi$.
A direct Monte Carlo approximation to \eqref{eq:gradient_estimation_gradient}, 
\begin{align*}
    \frac{1}{m} \sum_{i=1}^m f(X_i) \nabla_\phi \log q_\phi(X_i) , \qquad X_1,\dots,X_m \stackrel{\text{i.i.d.}}{\sim} Q_\phi
\end{align*}
is an unbiased estimator of \eqref{eq:gradient_estimation_objective} but can have a high variance unless an extremely large number $m$ of Monte Carlo samples is used.
A simple trick to reduce the variance is called REINFORCE \citep{glynn1990likelihood,williams1992simple},
\begin{align}
    \frac{1}{m} \sum_{i=1}^m [f(X_i) - B_i] \nabla_\phi \log q_\phi(X_i)  \label{eq: reinforce}
\end{align}
which preserves unbiasedness as long as 
\begin{align}\label{eq:mean-zero-product}
\E[B_i\grad_{\phi} \log q_\phi(X_i)] = 0,
\end{align}
that is, as long as $B_i\grad_{\phi} \log q_\phi(X_i)$ is a \emph{control variate}, a random variable with known (in this case, zero) mean. 
Fortunately, the control variate condition \cref{eq:mean-zero-product} is straightforward to engineer. For example, $B_i^{\textup{ind}}\grad_{\phi} \log q_\phi(X_i)$ has zero mean whenever $B_i^{\textup{ind}}$ and $X_i$ are independent as
\begin{align*}
    \mathbb{E}[B_i^{\textup{ind}} \nabla_\phi \log q_\phi(X_i)] 
    & = \mathbb{E}[B_i^{\textup{ind}}] \sum_{x \in \mathcal{X}} \frac{\nabla_\phi q_\phi(X)}{q_\phi(X)} q_\phi(X) \\
    & = \mathbb{E}[B_i^{\textup{ind}}] \sum_{x \in \mathcal{X}} \nabla_\phi q_\phi(X) \\
    & = \mathbb{E}[B_i^{\textup{ind}}] \nabla_\phi \sum_{x \in \mathcal{X}} q_\phi(X) 
    = \mathbb{E}[B_i^{\textup{ind}}] \nabla_\phi 1 = 0 .
\end{align*}
Similarly, $B_i^{\textup{zero}}\grad_{\phi} \log q_\phi(X_i)$ is zero mean whenever  $\E[B_i^{\textup{zero}}\mid X_i] = 0$, as, by the tower property, 
\begin{align*}
    \mathbb{E}[B_i^{\textup{zero}} \nabla_\phi \log q_\phi(X_i)] 
    & =     \mathbb{E}[\E[B_i^{\textup{zero}}\mid X_i] \nabla_\phi \log q_\phi(X_i)] = 0.
\end{align*}
More generally, $B_i = B_i^{\textup{ind}}+B_i^{\textup{zero}}$ also satisfies the covariate condition \cref{eq:mean-zero-product}, and we will use the full flexibility of this decomposition in what follows.

This original REINFORCE algorithm set each $B_1,\dots,B_m$ equal to a constant, the value of which was chosen to minimize the variance of \eqref{eq: reinforce}.
For $m > 1$ a popular variant is obtained by replacing $B_i$ by the leave-one-out average of function values, resulting in the \textbf{R}EINFORCE \textbf{L}eave-\textbf{O}ne-\textbf{O}ut (RLOO)  estimator~\citep{salimans2014using,Kool2019Buy4R,richter2020},
\begin{align} \label{eq:rloo}
	\frac{1}{m}\sum_{i=1}^m \left(f(X_i) - \frac{1}{m-1}\sum_{j\neq i} f(X_j)\right) \nabla_\phi \log q_\phi (X_i).
\end{align}
The RLOO estimate is again, by construction, unbiased for the target gradient.
However, despite improving on the former estimators, the variance of RLOO can still be stubbornly high unless a large number $m$ of samples are used. 

Stein operators can serve as a useful tool to further reduce the variance of RLOO without sacrificing unbiasedness \citep{shi2022gradient}.
The idea is that we would like to design a better choice for $B_i$, ideally one for which $B_i \approx f(X_i)$, so that the variance of \eqref{eq: reinforce} in turn is small.
To this end, let $\mathcal{T}_{Q_\phi}$ be a Stein operator for $Q_\phi$; several possible discrete Stein operators were discussed in \Cref{sec: discrete operators}.
Then since $\E[(\mathcal{T}_{Q_\phi} g_j)(X_j)\mid X_i] = 0$ for $j\neq i$ and any $g_j$ in the domain of $\mathcal{T}_{Q_\phi}$, we can consider taking $B_i$ of the form
\begin{align}
\!\!\! B_i 
    =
B_i^{\textup{ind}} 
    +
B_i^{\textup{zero}}   
    =
\frac{1}{m-1}\sum_{j\neq i} f(X_j) + \frac{1}{m-1}\sum_{j\neq i} (\mathcal{T}_{Q_\phi} g_j)(X_j). \label{eq: rodeo b}
\end{align}
Since our goal is to approximate $f(X_i)$, 
a suitable choice for each $g_i$ could in principle be obtained by solving the Stein equation \eqref{eq: Stein equation}
\begin{align*}
    f(x) = \mathbb{E}_{X \sim Q_\phi}[f(X)] + (\mathcal{T}_{Q_\phi} g)(x) , \qquad \forall x \in \mathcal{X},
\end{align*}
as discussed in \Cref{subsec: control variates}, but for the settings that we have in mind the size of $\mathcal{X}$ would make direct solution of the Stein equation impractical.
As a practical way forward, \citet{shi2022gradient} recommended using the Gibbs Stein operator \eqref{eq: Gibbs stein operator} and to select $g_i$ of the form 
\begin{align}\label{eq:discrete_cv_form}
g_j(x) & = \frac{1}{m - 1}\sum_{k\neq j} \! N_{\beta}( f(X_k), \nabla f(X_k)^\top (x - X_k))
\end{align} 
where $N_{\beta}$ is a neural network with parameters $\beta$ that are to be optimized. 
Note that, even though each function $g_j$ depends on $X_i$, we still have $\E[(\mathcal{T}_{Q_\phi} g_j)(X_j)\mid X_i] = 0$, and hence the gradient estimator remains unbiased. 
As such, $\beta$ can be optimized via gradient descent on the sample variance of the gradient estimator, similarly to \citet{grathwohl2017backpropagation}.
The use of the Gibbs Stein estimator was motivated by numerical stability in settings where the components of $X$ take a small number of values (e.g. $\mathcal{X} = \{0,1\}^d$), while for settings where the components of $X$ take many different values the Barker Stein operator (a special case of the Zanella Stein operator; \Cref{def: Zanella SO}) was recommended in \citet{shi2022gradient}.
Further, in certain applications (such as variational autoencoders; \Cref{subsec: VAEs}) the evaluations of $\nabla f$ come at no additional cost \citep{titsias2022double}, meaning that \eqref{eq:discrete_cv_form} requires no more evaluations of $f$ than RLOO.

\begin{figure}[t!]
	\centering
    \includegraphics[width=\textwidth,clip,trim = 20.5cm 0cm 0cm 0cm]{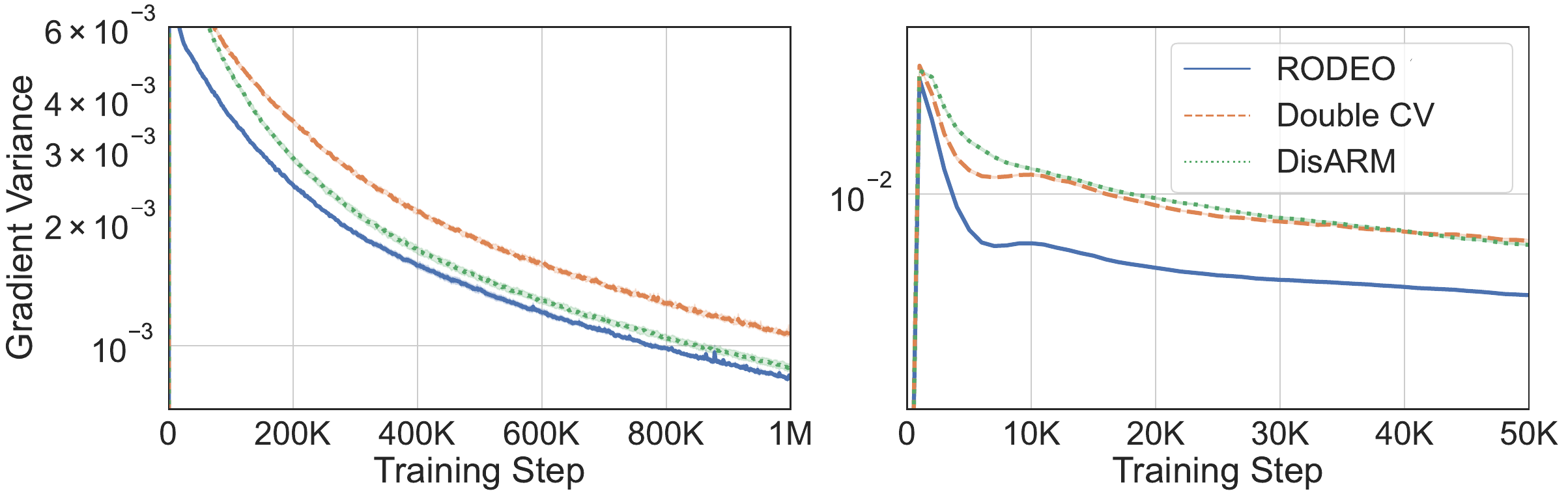}
	\caption{Training variational autoencoders on the Fashion MNIST dataset; a binary latent $z$ of dimension $d'=200$ was employed and $m=2$ samples were used to evaluate each gradient.
        Reproduced with permission from \citet{shi2022gradient}.
	} 
	\label{fig: RODEO}
\end{figure}

On top of the local Stein control variates (cf. \Cref{rem: cvs}) in \cref{eq: rodeo b}, \citet{shi2022gradient} introduce a second set of global Stein control variates with additional learned test functions $\tilde{g}_k$ of the form \cref{eq:discrete_cv_form} to further reduce variance:
\begin{align}
    \frac{1}{m}\sum_{k=1}^m\big[&(f(x_k) - \frac{1}{m-1}\sum_{j\neq k} (f(x_j) + (\mathcal{T}_{Q_\phi}g_j)(x_j))) \nabla_\eta \log q_\eta (x_k)  \\&+ (\mathcal{T}_{Q_\phi}\tilde{g}_k)(x_k)\big].
\end{align}
This overall approach, where the Stein equation is exploited as a \emph{double variance reduction} tool, was termed RODEO (\textbf{R}L\textbf{O}O with \textbf{D}iscrete St\textbf{E}in \textbf{O}perators) in \citet{shi2022gradient}.
The authors applied RODEO to train binary latent variational autoencoders \citep[cf. \Cref{subsec: VAEs} and][]{Kingma2014}. 
The dimension of the latent variable $z$ was $d' = 200$, and the Fashion MNIST dataset \citep{xiao2017fashion} was considered.
In \Cref{fig: RODEO}, we see that RODEO significantly reduces gradient variance along the training path relative to other gradient estimation schemes such as \emph{DisARM} \citep{dong2020disarm} and \emph{Double CV} \citep{titsias2022double}.
This in turn tends to lead to improved performance of the \ac{vae}; for further results see \citet{shi2022gradient}.

\subsection{Control Variates for Policy Gradient}
\label{subsec: CVs for RL}

The final application that we present concerns \emph{reinforcement learning}, whose starting point is a \emph{Markov decision process} consisting of a set of \emph{states} $\mathcal{S}$, a set of \emph{actions} $\mathcal{A}$, and a Markov process $(S_t)_{t =0}^\infty \subset \mathcal{S}$ such that $S_{t+1}$ depends only on $S_t$ and an action $A_t \in \mathcal{A}$.
In addition, we have a sequence of random variables called \emph{rewards} $R(S_t,A_t)$ which depend only on the current state $S_t$ and action $A_t$.
Reinforcement learning seeks a (stochastic) \emph{policy} $\pi : \mathcal{S} \rightarrow \mathcal{P}(\mathcal{A})$, meaning that to decide which action to take at time $t$ we sample $A_t \sim \pi(S_t)$, for which the expected discounted cumulative reward
$$
J(\pi) = \mathbb{E}\left[ \left. \sum_{t=0}^\infty \gamma^t R(S_t,A_t) \right| S_0 = s_0 \right]
$$
is maximal.
Here the initial state $s_0$ is considered to be fixed, and $\gamma \in (0,1)$ is a \emph{discount factor}, prioritizing the early part of the reward sequence, which is user-specified.

To simplify presentation, let $\mathcal{S}$ be a countable set, let $\mathcal{A} = \mathbb{R}^d$, and let $P_\pi$ be the \emph{discounted state visitation} distribution with mass function
$$
p_\pi(s) = (1 - \gamma) \sum_{t=0}^\infty \gamma^t \mathbb{P}_\pi(S_t = s | S_0 = s_0) ,
$$
where $\mathbb{P}_\pi$ indicates that actions $A_t$ are chosen according to the policy $\pi(S_t)$, so that we can use the shorthand
$$
J(\pi) = \mathbb{E}_\pi[R(S,A)] = \mathbb{E}_{S \sim P_\pi} [ \mathbb{E}_{A|S \sim \pi(S)}[ R(S,A) | S ] ] .
$$
In the sequel we slightly overload notation so that $\pi(\cdot | s)$ denotes the density function for the distribution $\pi(s)$, for each $s \in \mathcal{S}$.

In \emph{policy gradient} methods, the policy $\pi$ is parametrized, which we denote as $\pi_\theta$, and the parameters $\theta$ are updated by gradient ascent on $\theta \mapsto J(\pi_\theta)$.
From the \emph{policy gradient theorem} \citep{williams1992simple}, the gradient of this objective is
$$
\nabla_\theta J(\pi_\theta) = \mathbb{E}_{\pi_\theta} \left[ Q^{\pi_\theta}(S_t, A_t) (\nabla_\theta \log \pi_\theta)(A_t|S_t)  \right],
$$
where the \emph{action-value function} $Q^\pi(s, a)$ denotes the expected return if the initial state is $s \in \mathcal{S}$ and the initial action is $a \in \mathcal{A}$, and for all subsequent steps the policy $\pi$ is used.
Monte Carlo can then be used to approximate this gradient, assuming that the action-value function is available or has been estimated.
A key challenge is that Monte Carlo estimation of this gradient often suffers from high variance.

To address this, techniques such as REINFORCE (cf. \Cref{subsec: discrete gradient}) are often used to formulate an equivalent expectation, for which the variance of the associated Monte Carlo estimator is reduced.
A direct application leads to
\begin{align*}
    \nabla_\theta J(\theta) = \mathbb{E}_{\pi_\theta} \left[ \left(Q^{\pi_\theta}(S, A) - B(S)\right) (\nabla_\theta \log \pi_\theta)(A|S) \right] ,
\end{align*}
where $B(S)$ depends on the state $S$ but not the action $A$.
This ensures the value of the gradient is unchanged, via a similar argument to \eqref{eq:rloo}.
Common choices of $B$ are a constant, or the \emph{value function} $B(s) = \max_{a \in \mathcal{A}} Q(s,a)$ \citep{sutton1998reinforcement}.
However, because these traditional baselines are based on functions only of the state $S$, their potential to substantially reduce variance is limited.
Stein’s identity unlocks the possibility of making $B$ both state- and action-dependent, while not altering the value of the gradient.

For any element $g_s$ in the domain of the Langevin Stein operator for $\pi_\theta(s)$ (\Cref{def: Lang SO}),
\begin{align*}
0 & = \mathbb{E}_{A \sim \pi_\theta(s)}[(\mathcal{T}_{\pi_\theta(s)} g_s)(A)] \\
& = \mathbb{E}_{A \sim \pi_\theta(s)} \left[ g_s(A) \cdot (\nabla \log \pi_\theta)(A|s) + (\nabla g_s)(A) \right] = 0 .
\end{align*}
However, here the gradients are taken with respect to the action $A$, while policy gradients involve gradients with respect to parameters $\theta$. 
A reparameterization trick can be used to bridge this gap: namely, if we can express actions sampled from $\pi_\theta(s)$ as $A = f_\theta(s, \xi)$, where $\xi$ is random noise independent of $\theta$ and $s$, then \citet[Theorem 3.1 of][]{liu2018action} showed that
$$
\mathbb{E}_{A \sim \pi_\theta(s)} \left[ g_s(A) \cdot (\nabla_\theta \log \pi_\theta)(A|s)  \right]
= \mathbb{E} \left[ \nabla_\theta f_\theta(s, \xi) \cdot (\nabla g_s)(f_\theta(s, \xi)) \right].
$$
where the final expectation is with respect to the innovation noise $\xi$.
This leads to a Stein operator-based alternative to REINFORCE,
$$
\nabla_\theta J(\pi_\theta) = \mathbb{E}_{\pi_\theta} \left[ \begin{array}{l} (Q^{\pi_\theta}(S,A) - B(S, A)) (\nabla_\theta \log \pi_\theta)(A|S) \\
\hspace{80pt} + \nabla_\theta f_\theta(S, \xi) \cdot \nabla_A B(S, A) \end{array} \right],
$$
where $B(s,a) = g_s(a)$ and each $g_s$ is in the domain of the Langevin Stein operator for $\pi(s)$, for each $s \in \mathcal{S}$.
Thus we have derived state- and action-dependent functions that can be optimized to reduce the variance of the associated Monte Carlo estimator, while ensuring resulting estimator remains unbiased.
An empirical investigation in \citet{liu2018action} found that substantial variance reductions, and thus improved sample efficiency, can be achieved in a spectrum of reinforcement learning applications using $B(S,A)$ compared with simpler state-dependent $B(S)$.

\chapter{Conclusion}

This monograph has collected together rigorous definitions and results that underpin recent and emerging methodological applications of Stein's method.
In doing so, we hope we have provided a convenient and singular reference for practitioners in probabilistic inference and learning.
Since this topic represents an active area of research, we can at best hope to provide a snapshot of the state of knowledge at the time of writing, and our perspectives and understanding of the methodological aspects of Stein's method will surely be further developed.
Indeed, even in this respect we did not aim for a truly comprehensive treatment, limiting scope to only the most canonical Stein discrepancies and the most straightforward methodological applications of Stein's method.
There now is a rich and growing literature, some of which we highlighted in passing, which we hope the reader will be inspired to explore in detail.

\bibliographystyle{abbrvnat}
\bibliography{refs}

\begin{thebibliography}{185}
\providecommand{\natexlab}[1]{#1}
\providecommand{\url}[1]{\texttt{#1}}
\expandafter\ifx\csname urlstyle\endcsname\relax
  \providecommand{\doi}[1]{doi: #1}\else
  \providecommand{\doi}{doi: \begingroup \urlstyle{rm}\Url}\fi

\bibitem[Ambrosio et~al.(2005)Ambrosio, Gigli, and Savar{\'e}]{ambrosio2005gradient}
L.~Ambrosio, N.~Gigli, and G.~Savar{\'e}.
\newblock \emph{Gradient Flows: {I}n Metric Spaces and in the Space of Probability Measures}.
\newblock Springer, 2005.

\bibitem[Anastasiou et~al.(2023)Anastasiou, Barp, Briol, Ebner, Gaunt, Ghaderinezhad, Gorham, Gretton, Ley, Liu, Mackey, Oates, Reinert, and Swan]{anastasiou2021stein}
A.~Anastasiou, A.~Barp, F.-X. Briol, B.~Ebner, R.~E. Gaunt, F.~Ghaderinezhad, J.~Gorham, A.~Gretton, C.~Ley, Q.~Liu, L.~Mackey, C.~J. Oates, G.~Reinert, and Y.~Swan.
\newblock {S}tein's method meets statistics: {A} review of some recent developments.
\newblock \emph{Statistical Science}, 38\penalty0 (1):\penalty0 120--139, 2023.

\bibitem[Assaraf and Caffarel(2003)]{assaraf2003zero}
R.~Assaraf and M.~Caffarel.
\newblock Zero-variance zero-bias principle for observables in quantum monte carlo: Application to forces.
\newblock \emph{The Journal of Chemical Physics}, 119\penalty0 (20):\penalty0 10536--10552, 2003.

\bibitem[Bach et~al.(2012)Bach, Lacoste-Julien, and Obozinski]{bach2012equivalence}
F.~Bach, S.~Lacoste-Julien, and G.~Obozinski.
\newblock On the equivalence between herding and conditional gradient algorithms.
\newblock In \emph{Proceedings of the 29th International Conference on Machine Learning}, 2012.

\bibitem[Banerjee et~al.(2025)Banerjee, Balasubramanian, and Ghosal]{balasubramanian2024improved}
S.~Banerjee, K.~Balasubramanian, and P.~Ghosal.
\newblock Improved finite-particle convergence rates for {S}tein variational gradient descent.
\newblock In \emph{The 13th International Conference on Learning Representations}, 2025.

\bibitem[Barbour(1988)]{barbour1988stein}
A.~D. Barbour.
\newblock {S}tein's method and {P}oisson process convergence.
\newblock \emph{Journal of Applied Probability}, 25\penalty0 (A):\penalty0 175--184, 1988.

\bibitem[Barbour(1990)]{barbour1990stein}
A.~D. Barbour.
\newblock {S}tein's method for diffusion approximations.
\newblock \emph{Probability Theory and Related Fields}, 84\penalty0 (3):\penalty0 297--322, 1990.

\bibitem[Barbour and Chen(2005)]{barbour2005introduction}
A.~D. Barbour and L.~H.~Y. Chen.
\newblock \emph{An introduction to {Stein's} method}.
\newblock World Scientific, 2005.

\bibitem[Barker(1965)]{barker1965monte}
A.~A. Barker.
\newblock Monte {C}arlo calculations of the radial distribution functions for a proton-electron plasma.
\newblock \emph{Australian Journal of Physics}, 18\penalty0 (2):\penalty0 119--134, 1965.

\bibitem[Barp et~al.(2019)Barp, Briol, Duncan, Girolami, and Mackey]{barp2019minimum}
A.~Barp, F.-X. Briol, A.~Duncan, M.~Girolami, and L.~Mackey.
\newblock Minimum {S}tein discrepancy estimators.
\newblock In \emph{Proceedings of the 33rd International Conference on Neural Information Processing Systems}, 2019.

\bibitem[Barp et~al.(2022)Barp, Oates, Porcu, and Girolami]{barp2022riemann}
A.~Barp, C.~J. Oates, E.~Porcu, and M.~Girolami.
\newblock A riemann--stein kernel method.
\newblock \emph{Bernoulli}, 28\penalty0 (4):\penalty0 2181--2208, 2022.

\bibitem[Barp et~al.(2024)Barp, Simon-Gabriel, Girolami, and Mackey]{barp2022targeted}
A.~Barp, C.-J. Simon-Gabriel, M.~Girolami, and L.~Mackey.
\newblock Targeted separation and convergence with kernel discrepancies.
\newblock \emph{Journal of Machine Learning Research}, 25\penalty0 (378):\penalty0 1--50, 2024.

\bibitem[Baum et~al.(2023)Baum, Kanagawa, and Gretton]{baum2023kernel}
J.~Baum, H.~Kanagawa, and A.~Gretton.
\newblock A kernel {S}tein test of goodness of fit for sequential models.
\newblock In \emph{Proceedings of the 40th International Conference on Machine Learning}, 2023.

\bibitem[Belomestny et~al.(2020)Belomestny, Iosipoi, Moulines, Naumov, and Samsonov]{belomestny2020variance}
D.~Belomestny, L.~Iosipoi, E.~Moulines, A.~Naumov, and S.~Samsonov.
\newblock Variance reduction for markov chains with application to {MCMC}.
\newblock \emph{Statistics and Computing}, 30:\penalty0 973--997, 2020.

\bibitem[Belomestny et~al.(2024)Belomestny, Goldman, Naumov, and Samsonov]{belomestny2024theoretical}
D.~Belomestny, A.~Goldman, A.~Naumov, and S.~Samsonov.
\newblock Theoretical guarantees for neural control variates in {MCMC}.
\newblock \emph{Mathematics and Computers in Simulation}, 220:\penalty0 382--405, 2024.

\bibitem[Blei et~al.(2017)Blei, Kucukelbir, and McAuliffe]{blei2017variational}
D.~M. Blei, A.~Kucukelbir, and J.~D. McAuliffe.
\newblock Variational inference: A review for statisticians.
\newblock \emph{Journal of the American Statistical Association}, 112\penalty0 (518):\penalty0 859--877, 2017.

\bibitem[Bleile et~al.(2026)Bleile, Lumpp, and Drton]{bleile2026efficient}
F.~Bleile, S.~Lumpp, and M.~Drton.
\newblock Efficient learning of stationary diffusions with stein-type discrepancies, 2026.

\bibitem[Bomze et~al.(2024)Bomze, Rinaldi, and Zeffiro]{bomze2024frank}
I.~M. Bomze, F.~Rinaldi, and D.~Zeffiro.
\newblock Frank--wolfe and friends: {A} journey into projection-free first-order optimization methods.
\newblock \emph{Annals of Operations Research}, 343\penalty0 (2):\penalty0 607--638, 2024.

\bibitem[Bresler and Nagaraj(2019)]{bresler2019stein}
G.~Bresler and D.~Nagaraj.
\newblock Stein’s method for stationary distributions of {M}arkov chains and application to {I}sing models.
\newblock \emph{The Annals of Applied Probability}, 29\penalty0 (5):\penalty0 3230--3265, 2019.

\bibitem[Briol et~al.(2025)Briol, Karvonen, Gessner, and Mahsereci]{briol2025dictionary}
F.-X. Briol, T.~Karvonen, A.~Gessner, and M.~Mahsereci.
\newblock A dictionary of closed-form kernel mean embeddings.
\newblock In \emph{Proceedings of the 1st International Conference on Probabilistic Numerics}, 2025.

\bibitem[Brooks et~al.(2011)Brooks, Gelman, Jones, and Meng]{brooks2011handbook}
S.~Brooks, A.~Gelman, G.~Jones, and X.-L. Meng.
\newblock \emph{Handbook of {M}arkov Chain {M}onte {C}arlo}.
\newblock CRC Press, 2011.

\bibitem[Brown and Xia(2001)]{brown2001stein}
T.~C. Brown and A.~Xia.
\newblock Stein's method and birth-death processes.
\newblock \emph{The Annals of Probability}, 29\penalty0 (3):\penalty0 1373--1403, 2001.

\bibitem[Buttazzo et~al.(1995)Buttazzo, Ferone, and Kawohl]{buttazzo1995minimum}
G.~Buttazzo, V.~Ferone, and B.~Kawohl.
\newblock Minimum problems over sets of concave functions and related questions.
\newblock \emph{Mathematische Nachrichten}, 173\penalty0 (1):\penalty0 71--89, 1995.

\bibitem[Carmeli et~al.(2006)Carmeli, De~Vito, and Toigo]{carmeli2006vector}
C.~Carmeli, E.~De~Vito, and A.~Toigo.
\newblock Vector valued reproducing kernel {H}ilbert spaces of integrable functions and {M}ercer theorem.
\newblock \emph{Analysis and Applications}, 4\penalty0 (04):\penalty0 377--408, 2006.

\bibitem[Carmeli et~al.(2010)Carmeli, De~Vito, Toigo, and Umanit{\'a}]{Carmeli2010}
C.~Carmeli, E.~De~Vito, A.~Toigo, and V.~Umanit{\'a}.
\newblock Vector valued reproducing kernel hilbert spaces and universality.
\newblock \emph{Analysis and Applications}, 8\penalty0 (01):\penalty0 19--61, 2010.

\bibitem[Chang(1996)]{chang1996strictly}
K.-F. Chang.
\newblock Strictly positive definite functions.
\newblock \emph{Journal of Approximation Theory}, 87\penalty0 (2):\penalty0 148--158, 1996.

\bibitem[Chen et~al.(2018{\natexlab{a}})Chen, Zhang, Wang, Li, and Chen]{chen2018unified}
C.~Chen, R.~Zhang, W.~Wang, B.~Li, and L.~Chen.
\newblock A unified particle-optimization framework for scalable {Bayesian} sampling.
\newblock In \emph{Proceedings of the 34th Conference on Uncertainty on Artificial Intelligence}, 2018{\natexlab{a}}.

\bibitem[Chen et~al.(2010{\natexlab{a}})Chen, Goldstein, and Shao]{chen2010normal}
L.~H. Chen, L.~Goldstein, and Q.-M. Shao.
\newblock \emph{Normal Approximation by Stein's Method}.
\newblock Springer Science \& Business Media, 2010{\natexlab{a}}.

\bibitem[Chen et~al.(2019{\natexlab{a}})Chen, Wu, Chen, O'Leary-Roseberry, and Ghattas]{chen2019projected}
P.~Chen, K.~Wu, J.~Chen, T.~O'Leary-Roseberry, and O.~Ghattas.
\newblock Projected {S}tein variational {N}ewton: {A} fast and scalable {B}ayesian inference method in high dimensions.
\newblock In \emph{Proceedings of the 33rd Conference on Neural Information Processing Systems}, 2019{\natexlab{a}}.

\bibitem[Chen et~al.(2018{\natexlab{b}})Chen, Mackey, Gorham, Briol, and Oates]{chen2018stein}
W.~Y. Chen, L.~Mackey, J.~Gorham, F.-X. Briol, and C.~J. Oates.
\newblock {S}tein points.
\newblock In \emph{Proceedings of the 35th International Conference on Machine Learning}, 2018{\natexlab{b}}.

\bibitem[Chen et~al.(2019{\natexlab{b}})Chen, Barp, Briol, Gorham, Girolami, Mackey, and Oates]{chen2019stein}
W.~Y. Chen, A.~Barp, F.-X. Briol, J.~Gorham, M.~Girolami, L.~Mackey, and C.~J. Oates.
\newblock Stein point {M}arkov chain {M}onte {C}arlo.
\newblock In \emph{Proceedings of the 36th International Conference on Machine Learning}, 2019{\natexlab{b}}.

\bibitem[Chen et~al.(2010{\natexlab{b}})Chen, Welling, and Smola]{chen2010super}
Y.~Chen, M.~Welling, and A.~Smola.
\newblock Super-samples from kernel herding.
\newblock In \emph{Proceedings of the 26th Conference on Uncertainty in Artificial Intelligence}, 2010{\natexlab{b}}.

\bibitem[Chew(1986)]{Chew86}
P.~Chew.
\newblock There is a {P}lanar {G}raph {A}lmost {A}s {G}ood {A}s the {C}omplete {G}raph.
\newblock In \emph{Proceedings of the 2nd Annual Symposium on Computational Geometry}, 1986.

\bibitem[Chewi et~al.(2020)Chewi, Le~Gouic, Lu, Maunu, Rigollet, and Stromme]{chewi2020exponential}
S.~Chewi, T.~Le~Gouic, C.~Lu, T.~Maunu, P.~Rigollet, and A.~Stromme.
\newblock Exponential ergodicity of mirror-{L}angevin diffusions.
\newblock In \emph{Proceedings of the 34th Conference on Neural Information Processing Systems}, 2020.

\bibitem[Chwialkowski et~al.(2016)Chwialkowski, Strathmann, and Gretton]{chwialkowski2016kernel}
K.~Chwialkowski, H.~Strathmann, and A.~Gretton.
\newblock A kernel test of goodness of fit.
\newblock In \emph{Proceedings of the 33rd International Conference on Machine Learning}, 2016.

\bibitem[Chwialkowski et~al.(2015)Chwialkowski, Ramdas, Sejdinovic, and Gretton]{chwialkowski2015fast}
K.~P. Chwialkowski, A.~Ramdas, D.~Sejdinovic, and A.~Gretton.
\newblock Fast two-sample testing with analytic representations of probability measures.
\newblock In \emph{Proceedings of the 29th Conference on Neural Information Processing Systems}, 2015.

\bibitem[Clement and Desch(2008)]{clement2008elementary}
P.~Clement and W.~Desch.
\newblock An elementary proof of the triangle inequality for the {W}asserstein metric.
\newblock \emph{Proceedings of the American Mathematical Society}, 136\penalty0 (1):\penalty0 333--339, 2008.

\bibitem[Conca and Vanninathan(2007)]{conca2007periodic}
C.~Conca and M.~Vanninathan.
\newblock Periodic homogenization problems in incompressible fluid equations.
\newblock \emph{Handbook of Mathematical Fluid Dynamics}, 4:\penalty0 649--698, 2007.

\bibitem[Detommaso et~al.(2018)Detommaso, Cui, Marzouk, Scheichl, and Spantini]{detommaso2018stein}
G.~Detommaso, T.~Cui, Y.~Marzouk, R.~Scheichl, and A.~Spantini.
\newblock A {Stein} variational {Newton} method.
\newblock In \emph{Proceedings of the 22nd Conference on Neural Information Processing Systems}, 2018.

\bibitem[Dinh et~al.(2017)Dinh, Sohl-Dickstein, and Bengio]{dinh2016density}
L.~Dinh, J.~Sohl-Dickstein, and S.~Bengio.
\newblock Density estimation using {Real NVP}.
\newblock In \emph{Proceedings of the 5th International Conference on Learning Representations}, 2017.

\bibitem[Dong et~al.(2020)Dong, Mnih, and Tucker]{dong2020disarm}
Z.~Dong, A.~Mnih, and G.~Tucker.
\newblock Disarm: An antithetic gradient estimator for binary latent variables.
\newblock In \emph{Proceedings of the 34th Conference on Neural Information Processing Systems}, 2020.

\bibitem[Dudley(2018)]{dudley2018real}
R.~M. Dudley.
\newblock \emph{Real Analysis and Probability}.
\newblock CRC Press, 2018.

\bibitem[Duncan et~al.(2016)Duncan, Lelievre, and Pavliotis]{duncan2016variance}
A.~B. Duncan, T.~Lelievre, and G.~Pavliotis.
\newblock Variance reduction using nonreversible {L}angevin samplers.
\newblock \emph{Journal of Statistical Physics}, 163\penalty0 (3):\penalty0 457--491, 2016.

\bibitem[Dunford(1937)]{dunford1937integration}
N.~Dunford.
\newblock Integration of vector-valued functions.
\newblock \emph{Bulletin of the American Mathematical Society}, page~43, 1937.

\bibitem[Eberle(2015)]{Eberle2015}
A.~Eberle.
\newblock Reflection couplings and contraction rates for diffusions.
\newblock \emph{Probability Theory and Related Fields}, 166:\penalty0 851--886, 2015.

\bibitem[Ebner et~al.(2025)Ebner, Fischer, Gaunt, Picker, and Swan]{ebner2025stein}
B.~Ebner, A.~Fischer, R.~E. Gaunt, B.~Picker, and Y.~Swan.
\newblock Stein's method of moments.
\newblock \emph{Scandinavian Journal of Statistics}, 52\penalty0 (4):\penalty0 1594--1624, 2025.

\bibitem[Eichelsbacher and Reinert(2008)]{eichelsbacher2008stein}
P.~Eichelsbacher and G.~Reinert.
\newblock Stein’s method for discrete gibbs measures.
\newblock \emph{The Annals of Applied Probability}, 18\penalty0 (4):\penalty0 1588--1618, 2008.

\bibitem[Ethier(1976)]{ethier1976class}
S.~N. Ethier.
\newblock A class of degenerate diffusion processes occurring in population genetics.
\newblock \emph{Communications on Pure and Applied Mathematics}, 29\penalty0 (5):\penalty0 483--493, 1976.

\bibitem[Feng et~al.(2017)Feng, Wang, and Liu]{feng2017learning}
Y.~Feng, D.~Wang, and Q.~Liu.
\newblock Learning to draw samples with amortized {S}tein variational gradient descent.
\newblock In \emph{Proceedings of the 33rd Conference on Uncertainty in Artificial Intelligence}, 2017.

\bibitem[Fisher et~al.(2021)Fisher, Nolan, Graham, Prangle, and Oates]{Fisher2021}
M.~Fisher, T.~Nolan, M.~Graham, D.~Prangle, and C.~J. Oates.
\newblock Measure transport with kernel {S}tein discrepancy.
\newblock In \emph{Proceedings of the 24th International Conference on Artificial Intelligence and Statistics}, pages 1054--1062, 2021.

\bibitem[Fisher and Oates(2023)]{fisher2023gradient}
M.~A. Fisher and C.~J. Oates.
\newblock Gradient-free kernel {S}tein discrepancy.
\newblock In \emph{Proceedings of the 37th Conference on Neural Information Processing Systems}, 2023.

\bibitem[Fromont et~al.(2012)Fromont, Laurent, Lerasle, and Reynaud-Bouret]{fromont2012kernels}
M.~Fromont, B.~Laurent, M.~Lerasle, and P.~Reynaud-Bouret.
\newblock Kernels based tests with non-asymptotic bootstrap approaches for two-sample problems.
\newblock In \emph{Proceedings of the 25th Conference on Learning Theory}, 2012.

\bibitem[Futami et~al.(2019)Futami, Cui, Sato, and Sugiyama]{futami2019bayesian}
F.~Futami, Z.~Cui, I.~Sato, and M.~Sugiyama.
\newblock Bayesian posterior approximation via greedy particle optimization.
\newblock In \emph{Proceedings of the AAAI Conference on Artificial Intelligence}, volume~33, pages 3606--3613, 2019.

\bibitem[Gallegos-Herrada et~al.(2024)Gallegos-Herrada, Ledvinka, and Rosenthal]{gallegos2024equivalences}
M.~A. Gallegos-Herrada, D.~Ledvinka, and J.~S. Rosenthal.
\newblock Equivalences of geometric ergodicity of markov chains.
\newblock \emph{Journal of Theoretical Probability}, 37\penalty0 (2):\penalty0 1230--1256, 2024.

\bibitem[Gelman et~al.(2014)Gelman, Carlin, Stern, Dunson, Vehtari, and Rubin]{GelmanCaStDuVeRu2014}
A.~Gelman, J.~Carlin, H.~Stern, D.~Dunson, A.~Vehtari, and D.~Rubin.
\newblock \emph{Bayesian data analysis}.
\newblock Texts in Statistical Science Series. CRC Press, Boca Raton, FL, third edition, 2014.
\newblock ISBN 978-1-4398-4095-5.

\bibitem[Geyer(1991)]{geyer1991markov}
C.~J. Geyer.
\newblock Markov chain {M}onte {C}arlo maximum likelihood.
\newblock In \emph{Computer Science and Statistics: Proceedings of the 23rd Symposium of the Interface of Computer Science and Statistics}, 1991.

\bibitem[Glynn(1990)]{glynn1990likelihood}
P.~W. Glynn.
\newblock Likelihood ratio gradient estimation for stochastic systems.
\newblock \emph{Communications of the ACM}, 33\penalty0 (10):\penalty0 75--84, 1990.

\bibitem[Gong et~al.(2019)Gong, Peng, and Liu]{gong2019quantile}
C.~Gong, J.~Peng, and Q.~Liu.
\newblock Quantile {S}tein variational gradient descent for parallel {B}ayesian optimization.
\newblock In \emph{Proceedings of the 36th International Conference on Machine Learning}, 2019.

\bibitem[Gong et~al.(2021)Gong, Li, and Hern{\'a}ndez-Lobato]{gongsliced}
W.~Gong, Y.~Li, and J.~M. Hern{\'a}ndez-Lobato.
\newblock Sliced kernelized {S}tein discrepancy.
\newblock In \emph{Proceedings of the 9th Conference on Learning Representations}, 2021.

\bibitem[Goodfellow et~al.(2014)Goodfellow, Pouget-Abadie, Mirza, Xu, Warde-Farley, Ozair, Courville, and Bengio]{goodfellow2014generative}
I.~Goodfellow, J.~Pouget-Abadie, M.~Mirza, B.~Xu, D.~Warde-Farley, S.~Ozair, A.~Courville, and Y.~Bengio.
\newblock Generative adversarial nets.
\newblock In \emph{Proceedings of the 28th Conference on Neural Information Processing Systems}, 2014.

\bibitem[Gorham and Mackey(2015)]{gorham2015measuring}
J.~Gorham and L.~Mackey.
\newblock Measuring sample quality with {S}tein's method.
\newblock In \emph{Proceedings of the 29th Conference on Neural Information Processing Systems}, pages 226--234, 2015.

\bibitem[Gorham and Mackey(2017)]{gorham2017measuring}
J.~Gorham and L.~Mackey.
\newblock Measuring sample quality with kernels.
\newblock In \emph{Proceedings of the 34th International Conference on Machine Learning}, 2017.

\bibitem[Gorham et~al.(2019)Gorham, Duncan, Vollmer, and Mackey]{Gorham2019}
J.~Gorham, A.~B. Duncan, S.~J. Vollmer, and L.~Mackey.
\newblock Measuring sample quality with diffusions.
\newblock \emph{The Annals of Applied Probability}, 29\penalty0 (5):\penalty0 2884--2928, 2019.

\bibitem[Gorham et~al.(2020)Gorham, Raj, and Mackey]{gorham2020stochastic}
J.~Gorham, A.~Raj, and L.~Mackey.
\newblock Stochastic {S}tein discrepancies.
\newblock In \emph{Proceedings of the 37th Conference on Neural Information Processing Systems}, 2020.

\bibitem[Gotze(1991)]{gotze1991rate}
F.~Gotze.
\newblock On the rate of convergence in the multivariate {CLT}.
\newblock \emph{The Annals of Probability}, 19\penalty0 (2):\penalty0 724--739, 1991.

\bibitem[Grathwohl et~al.(2018)Grathwohl, Choi, Wu, Roeder, and Duvenaud]{grathwohl2017backpropagation}
W.~Grathwohl, D.~Choi, Y.~Wu, G.~Roeder, and D.~Duvenaud.
\newblock Backpropagation through the void: {O}ptimizing control variates for black-box gradient estimation.
\newblock In \emph{Proceedings of the 6th International Conference on Learning Representations}, 2018.

\bibitem[Hagrass et~al.(2026)Hagrass, Sriperumbudur, and Balasubramanian]{hagrass2024minimax}
O.~Hagrass, B.~Sriperumbudur, and K.~Balasubramanian.
\newblock Minimax optimal goodness-of-fit testing with kernel {S}tein discrepancy.
\newblock \emph{Bernoulli}, 32\penalty0 (1):\penalty0 299--324, 2026.

\bibitem[Han and Liu(2017)]{han2017stein}
J.~Han and Q.~Liu.
\newblock Stein variational adaptive importance sampling.
\newblock In \emph{Proceedings of the 33rd Conference on Uncertainty on Artificial Intelligence}, 2017.

\bibitem[Han and Liu(2018)]{han2018stein}
J.~Han and Q.~Liu.
\newblock Stein variational gradient descent without gradient.
\newblock In \emph{Proceedings of the 37th International Conference on Machine Learning}, 2018.

\bibitem[Har-Peled and Mendel(2005)]{har2005fast}
S.~Har-Peled and M.~Mendel.
\newblock Fast construction of nets in low dimensional metrics, and their applications.
\newblock In \emph{Proceedings of the 21st Annual Symposium on Computational Geometry}, 2005.

\bibitem[Henderson(1997)]{henderson1997variance}
S.~G. Henderson.
\newblock \emph{Variance reduction via an approximating Markov process}.
\newblock PhD thesis, Stanford University, 1997.

\bibitem[Hinton(2002)]{hinton2002training}
G.~E. Hinton.
\newblock Training products of experts by minimizing contrastive divergence.
\newblock \emph{Neural Computation}, 14\penalty0 (8):\penalty0 1771--1800, 2002.

\bibitem[Hodgkinson et~al.(2020)Hodgkinson, Salomone, and Roosta]{hodgkinson2020reproducing}
L.~Hodgkinson, R.~Salomone, and F.~Roosta.
\newblock The reproducing stein kernel approach for post-hoc corrected sampling, 2020.

\bibitem[Holmes(2004)]{holmes2004stein}
S.~Holmes.
\newblock Stein's method for birth and death chains.
\newblock In \emph{Stein's Method: {E}xpository Lectures and Applications}, pages 45--67. Institute of Mathematical Statistics, 2004.

\bibitem[Horowitz(1987)]{horowitz1987second}
A.~Horowitz.
\newblock The second order {L}angevin equation and numerical simulations.
\newblock \emph{Nuclear Physics B}, 280:\penalty0 510--522, 1987.

\bibitem[Huggins and Mackey(2018)]{huggins2018random}
J.~Huggins and L.~Mackey.
\newblock Random feature {S}tein discrepancies.
\newblock In \emph{Proceedings of the 32nd Conference on Neural Information Processing Systems}, 2018.

\bibitem[Hyv{\"a}rinen(2005)]{hyvarinen2005estimation}
A.~Hyv{\"a}rinen.
\newblock Estimation of non-normalized statistical models by score matching.
\newblock \emph{Journal of Machine Learning Research}, 6\penalty0 (4):\penalty0 695--709, 2005.

\bibitem[Ikeda and Watanabe(2014)]{ikeda2014stochastic}
N.~Ikeda and S.~Watanabe.
\newblock \emph{Stochastic Differential Equations and Diffusion Processes}.
\newblock Elsevier, 2014.

\bibitem[Jitkrittum et~al.(2017)Jitkrittum, Xu, Szab{\'o}, Fukumizu, and Gretton]{jitkrittum2017linear}
W.~Jitkrittum, W.~Xu, Z.~Szab{\'o}, K.~Fukumizu, and A.~Gretton.
\newblock A linear-time kernel goodness-of-fit test.
\newblock In \emph{Proceedings of the 31st Conference on Neural Information Processing Systems}, 2017.

\bibitem[Kanagawa et~al.(2023)Kanagawa, Jitkrittum, Mackey, Fukumizu, and Gretton]{kanagawa2023kernel}
H.~Kanagawa, W.~Jitkrittum, L.~Mackey, K.~Fukumizu, and A.~Gretton.
\newblock A kernel stein test for comparing latent variable models.
\newblock \emph{Journal of the Royal Statistical Society Series B: Statistical Methodology}, 85\penalty0 (3):\penalty0 986--1011, 2023.

\bibitem[Kanagawa et~al.(2025)Kanagawa, Barp, Gretton, and Mackey]{kanagawa2022controlling}
H.~Kanagawa, A.~Barp, A.~Gretton, and L.~Mackey.
\newblock Controlling moments with kernel {S}tein discrepancies.
\newblock \emph{The Annals of Applied Probability}, 35\penalty0 (6):\penalty0 3818--3843, 2025.

\bibitem[Karlin and McGregor(1957)]{karlin1957classification}
S.~Karlin and J.~McGregor.
\newblock The classification of birth and death processes.
\newblock \emph{Transactions of the American Mathematical Society}, 86\penalty0 (2):\penalty0 366--400, 1957.

\bibitem[Kingma and Welling(2013)]{Kingma2014}
D.~P. Kingma and M.~Welling.
\newblock Auto-encoding variational {Bayes}.
\newblock In \emph{Proceedings of the 1st International Conference on Learning Representations}, 2013.

\bibitem[Kingma et~al.(2016)Kingma, Salimans, Jozefowicz, Chen, Sutskever, and Welling]{kingma2016improved}
D.~P. Kingma, T.~Salimans, R.~Jozefowicz, X.~Chen, I.~Sutskever, and M.~Welling.
\newblock Improved variational inference with inverse autoregressive flow.
\newblock In \emph{Proceedings of the 30th Conference on Neural Information Processing Systems}, 2016.

\bibitem[Kool et~al.(2019)Kool, Hoof, and Welling]{Kool2019Buy4R}
W.~Kool, H.~V. Hoof, and M.~Welling.
\newblock Buy 4 {REINFORCE} samples, get a baseline for free!
\newblock In \emph{Proceedings of the Workshop DeepRLStructPred@ICLR}, 2019.

\bibitem[Korba et~al.(2021)Korba, Aubin-Frankowski, Majewski, and Ablin]{korba2021kernel}
A.~Korba, P.-C. Aubin-Frankowski, S.~Majewski, and P.~Ablin.
\newblock Kernel stein discrepancy descent.
\newblock In \emph{Proceedings of the 40th International Conference on Machine Learning}, 2021.

\bibitem[Landim et~al.(1998)Landim, Olla, and Yau]{landim1998convection}
C.~Landim, S.~Olla, and H.~Yau.
\newblock Convection--diffusion equation with space--time ergodic random flow.
\newblock \emph{Probability Theory and Related Fields}, 112\penalty0 (2):\penalty0 203--220, 1998.

\bibitem[Le et~al.(2024)Le, Lewis, Bharath, and Fallaize]{le2020diffusion}
H.~Le, A.~Lewis, K.~Bharath, and C.~Fallaize.
\newblock A diffusion approach to {S}tein’s method on {R}iemannian manifolds.
\newblock \emph{Bernoulli}, 30\penalty0 (2):\penalty0 1079--1104, 2024.

\bibitem[Leluc et~al.(2025)Leluc, Portier, Segers, and Zhuman]{leluc2023speeding}
R.~Leluc, F.~Portier, J.~Segers, and A.~Zhuman.
\newblock Speeding up {M}onte {C}arlo integration: {C}ontrol neighbors for optimal convergence.
\newblock \emph{Bernoulli}, 31\penalty0 (2):\penalty0 1160--1180, 2025.

\bibitem[Leucht and Neumann(2013)]{leucht2013dependent}
A.~Leucht and M.~H. Neumann.
\newblock Dependent wild bootstrap for degenerate u-and v-statistics.
\newblock \emph{Journal of Multivariate Analysis}, 117:\penalty0 257--280, 2013.

\bibitem[Ley and Swan(2013)]{ley2013stein}
C.~Ley and Y.~Swan.
\newblock Stein's density approach and information inequalities.
\newblock \emph{Electronic Communications in Probability}, 18\penalty0 (7):\penalty0 1--14, 2013.

\bibitem[Ley et~al.(2017)Ley, Reinert, and Swan]{LeyReSw2017}
C.~Ley, G.~Reinert, and Y.~Swan.
\newblock Stein's method for comparison of univariate distributions.
\newblock \emph{Probability Surveys}, 14:\penalty0 1--52, 2017.

\bibitem[Li et~al.(2020)Li, Li, Liu, Liu, and Lu]{li2020stochastic}
L.~Li, Y.~Li, J.-G. Liu, Z.~Liu, and J.~Lu.
\newblock A stochastic version of {S}tein variational gradient descent for efficient sampling.
\newblock \emph{Communications in Applied Mathematics and Computational Science}, 15\penalty0 (1):\penalty0 37--63, 2020.

\bibitem[Li et~al.(2024)Li, Dwivedi, and Mackey]{li2024debiased}
L.~Li, R.~Dwivedi, and L.~Mackey.
\newblock Debiased distribution compression.
\newblock In R.~Salakhutdinov, Z.~Kolter, K.~Heller, A.~Weller, N.~Oliver, J.~Scarlett, and F.~Berkenkamp, editors, \emph{Proceedings of the 41st International Conference on Machine Learning}, volume 235 of \emph{Proceedings of Machine Learning Research}, pages 27675--27731. PMLR, 21--27 Jul 2024.
\newblock URL \url{https://proceedings.mlr.press/v235/li24r.html}.

\bibitem[Liu and Zhu(2018)]{liu2017riemannian}
C.~Liu and J.~Zhu.
\newblock Riemannian {Stein} variational gradient descent for {Bayesian} inference.
\newblock In \emph{Proceedings of the 32nd AAAI Conference on Artificial Intelligence}, 2018.

\bibitem[Liu et~al.(2018)Liu, Feng, Mao, Zhou, Peng, and Liu]{liu2018action}
H.~Liu, Y.~Feng, Y.~Mao, D.~Zhou, J.~Peng, and Q.~Liu.
\newblock Action-dependent control variates for policy optimization via {S}tein’s identity.
\newblock In \emph{Proceedings of the 6th International Conference on Learning Representations}, 2018.

\bibitem[Liu et~al.(2019)Liu, Mehta, Tao, and Carin]{liu2019estimation}
J.~Liu, N.~Mehta, C.~Tao, and L.~Carin.
\newblock Estimation and sampling of unnormalized statistical models with {S}tein score matching.
\newblock In \emph{ICML 2019 Workshop on Stein’s Method for Machine Learning and Statistics}, 2019.

\bibitem[Liu and Lee(2017)]{liu2017black}
Q.~Liu and J.~Lee.
\newblock Black-box importance sampling.
\newblock In \emph{Proceedings of the 20th International Conference on Artificial Intelligence and Statistics}, 2017.

\bibitem[Liu and Wang(2016)]{Liu2016}
Q.~Liu and D.~Wang.
\newblock Stein variational gradient descent: {A} general purpose {B}ayesian inference algorithm.
\newblock In \emph{Proceedings of the 30th Annual Conference on Neural Information Processing Systems}, 2016.

\bibitem[Liu and Wang(2017)]{liu2017learning}
Q.~Liu and D.~Wang.
\newblock Learning deep energy models: {C}ontrastive divergence vs. amortized {MLE}, 2017.

\bibitem[Liu and Wang(2018)]{liu2018stein}
Q.~Liu and D.~Wang.
\newblock Stein variational gradient descent as moment matching.
\newblock In \emph{Proceedings of the 32nd Conference on Neural Information Processing Systems}, 2018.

\bibitem[Liu et~al.(2016)Liu, Lee, and Jordan]{liu2016kernelized}
Q.~Liu, J.~Lee, and M.~Jordan.
\newblock A kernelized stein discrepancy for goodness-of-fit tests.
\newblock In \emph{Proceedings of the 35th International Conference on Machine Learning}, 2016.

\bibitem[Liu and Briol(2025)]{liu2024robustness}
X.~Liu and F.-X. Briol.
\newblock On the robustness of kernel goodness-of-fit tests.
\newblock \emph{Journal of Machine Learning Research}, 26\penalty0 (262), 2025.

\bibitem[Liu et~al.(2023)Liu, Duncan, and Gandy]{liu2023using}
X.~Liu, A.~Duncan, and A.~Gandy.
\newblock Using perturbation to improve goodness-of-fit tests based on kernelized {S}tein discrepancy.
\newblock In \emph{Proceedings of the 40th International Conference on Machine Learning}, 2023.

\bibitem[Lorch et~al.(2024)Lorch, Krause, and Sch{\"o}lkopf]{lorch2024causal}
L.~Lorch, A.~Krause, and B.~Sch{\"o}lkopf.
\newblock Causal modeling with stationary diffusions.
\newblock In \emph{Proceedings of the 27th International Conference on Artificial Intelligence and Statistics}, 2024.

\bibitem[Ma et~al.(2015)Ma, Chen, and Fox]{ma2015complete}
Y.~Ma, T.~Chen, and E.~Fox.
\newblock A complete recipe for stochastic gradient {MCMC}.
\newblock In \emph{Proceedings of the 29th Conference on Neural Information Processing Systems}, 2015.

\bibitem[Mackey and Gorham(2016)]{MackeyGo16}
L.~Mackey and J.~Gorham.
\newblock Multivariate {S}tein factors for a class of strongly log-concave distributions.
\newblock \emph{Electronic Communications in Probability}, 21:\penalty0 1--14, 2016.

\bibitem[Martinez-Taboada and Ramdas(2025)]{martinez2024sequential}
D.~Martinez-Taboada and A.~Ramdas.
\newblock Sequential kernelized stein discrepancy.
\newblock In \emph{Proceedings of the 28th International Conference on Artificial Intelligence and Statistics}, 2025.

\bibitem[Matsubara et~al.(2022)Matsubara, Knoblauch, Briol, and Oates]{matsubara2022robust}
T.~Matsubara, J.~Knoblauch, F.-X. Briol, and C.~J. Oates.
\newblock Robust generalised {B}ayesian inference for intractable likelihoods.
\newblock \emph{Journal of the Royal Statistical Society: Series B}, 84\penalty0 (3):\penalty0 997--1022, 2022.

\bibitem[Matsubara et~al.(2024)Matsubara, Knoblauch, Briol, and Oates]{matsubara2023generalized}
T.~Matsubara, J.~Knoblauch, F.-X. Briol, and C.~J. Oates.
\newblock Generalized bayesian inference for discrete intractable likelihood.
\newblock \emph{Journal of the American Statistical Association}, 119\penalty0 (547):\penalty0 2345--2355, 2024.

\bibitem[Mira et~al.(2013)Mira, Solgi, and Imparato]{mira2013zero}
A.~Mira, R.~Solgi, and D.~Imparato.
\newblock Zero variance {M}arkov chain {M}onte {C}arlo for {B}ayesian estimators.
\newblock \emph{Statistics and Computing}, 23\penalty0 (5):\penalty0 653--662, 2013.

\bibitem[Muandet et~al.(2017)Muandet, Fukumizu, Sriperumbudur, and Sch{\"o}lkopf]{muandet2016kernel}
K.~Muandet, K.~Fukumizu, B.~Sriperumbudur, and B.~Sch{\"o}lkopf.
\newblock Kernel mean embedding of distributions: {A} review and beyond.
\newblock \emph{Foundations and Trends{\textregistered} in Machine Learning}, 10\penalty0 (1-2):\penalty0 1--141, 2017.

\bibitem[M{\"u}ller(1997)]{muller1997integral}
A.~M{\"u}ller.
\newblock Integral probability metrics and their generating classes of functions.
\newblock \emph{Advances in Applied Probability}, 29\penalty0 (2):\penalty0 429--443, 1997.

\bibitem[Neal(1993)]{neal1993probabilistic}
R.~M. Neal.
\newblock Probabilistic inference using markov chain monte carlo methods.
\newblock 1993.

\bibitem[Newman and Barkema(1999)]{newman1999monte}
M.~E. Newman and G.~T. Barkema.
\newblock \emph{Monte Carlo methods in Statistical Physics}.
\newblock Clarendon Press, 1999.

\bibitem[Oates(2022)]{oates2022minimum}
C.~J. Oates.
\newblock Minimum kernel discrepancy estimators.
\newblock In \emph{Proceedings of the 15th International Conference on Monte Carlo and Quasi-Monte Carlo Methods in Scientific Computing}, 2022.

\bibitem[Oates et~al.(2017)Oates, Girolami, and Chopin]{oates2017control}
C.~J. Oates, M.~Girolami, and N.~Chopin.
\newblock Control functionals for {M}onte {C}arlo integration.
\newblock \emph{Journal of the Royal Statistical Society, Series B}, 79:\penalty0 695--718, 2017.

\bibitem[Oates et~al.(2019)Oates, Cockayne, Briol, and Girolami]{oates2019convergence}
C.~J. Oates, J.~Cockayne, F.-X. Briol, and M.~Girolami.
\newblock Convergence rates for a class of estimators based on {S}tein’s method.
\newblock \emph{Bernoulli}, 25\penalty0 (2):\penalty0 1141--1159, 2019.

\bibitem[Oksendal(2013)]{oksendal2013stochastic}
B.~Oksendal.
\newblock \emph{Stochastic Differential Equations: {A}n Introduction with Applications}.
\newblock Springer Science \& Business Media, 2013.

\bibitem[Owen(2016)]{owen2016constraint}
A.~B. Owen.
\newblock A constraint on extensible quadrature rules.
\newblock \emph{Numerische Mathematik}, 132:\penalty0 511--518, 2016.

\bibitem[Paisley et~al.(2012)Paisley, Blei, and Jordan]{paisley2012variational}
J.~Paisley, D.~M. Blei, and M.~I. Jordan.
\newblock Variational {B}ayesian inference with stochastic search.
\newblock In \emph{Proceedings of the 31st International Conference on Machine Learning}, 2012.

\bibitem[Patterson and Teh(2013)]{patterson2013stochastic}
S.~Patterson and Y.~Teh.
\newblock Stochastic gradient {R}iemannian {L}angevin dynamics on the probability simplex.
\newblock In \emph{Proceedings of the 27th Conference on Neural Information Processing Systems}, 2013.

\bibitem[Pavliotis(2016)]{pavliotis2016stochastic}
G.~A. Pavliotis.
\newblock \emph{Stochastic Processes and Applications}.
\newblock Springer, 2016.

\bibitem[Peleg and Sch{\"a}ffer(1989)]{peleg1989graph}
D.~Peleg and A.~A. Sch{\"a}ffer.
\newblock Graph spanners.
\newblock \emph{Journal of Graph Theory}, 13\penalty0 (1):\penalty0 99--116, 1989.

\bibitem[Pfeffer(2012)]{pfeffer2012divergence}
W.~F. Pfeffer.
\newblock \emph{The Divergence Theorem and Sets of Finite Perimeter}.
\newblock CRC Press Boca Raton, 2012.

\bibitem[Pigola and Setti(2014)]{pigola2014global}
S.~Pigola and A.~G. Setti.
\newblock Global divergence theorems in nonlinear pdes and geometry.
\newblock \emph{Ensaios Matem{\'a}ticos}, 26\penalty0 (1-77):\penalty0 2, 2014.

\bibitem[Pu et~al.(2017)Pu, Gan, Henao, Li, Han, and Carin]{pu2017stein}
Y.~Pu, Z.~Gan, R.~Henao, C.~Li, S.~Han, and L.~Carin.
\newblock Stein variational autoencoder.
\newblock In \emph{Proceedings of the 31st Conference on Neural Information Processing Systems}, 2017.

\bibitem[Qu and Vemuri(2025)]{qu2025theory}
X.~Qu and B.~C. Vemuri.
\newblock Theory and applications of kernel stein discrepancy on riemannian manifolds, 2025.

\bibitem[Rahimi and Recht(2007)]{rahimi2007random}
A.~Rahimi and B.~Recht.
\newblock Random features for large-scale kernel machines.
\newblock In \emph{Proceedings of the 21st Conference on Neural Information Processing Systems}, 2007.

\bibitem[Ranganath et~al.(2016)Ranganath, Tran, Altosaar, and Blei]{ranganath2016operator}
R.~Ranganath, D.~Tran, J.~Altosaar, and D.~Blei.
\newblock Operator variational inference.
\newblock In \emph{Proceedings of the 30th Conference on Neural Information Processing Systems}, 2016.

\bibitem[Reinert and Ross(2019)]{reinert2019approximating}
G.~Reinert and N.~Ross.
\newblock Approximating stationary distributions of fast mixing glauber dynamics, with applications to exponential random graphs.
\newblock \emph{Annals of Applied Probability}, 29\penalty0 (5):\penalty0 3201--3229, 2019.

\bibitem[Rey-Bellet and Spiliopoulos(2015)]{rey2014irreversible}
L.~Rey-Bellet and K.~Spiliopoulos.
\newblock Irreversible {L}angevin samplers and variance reduction: {A} large deviations approach.
\newblock \emph{Nonlinearity}, 28\penalty0 (7):\penalty0 2081--2103, 2015.

\bibitem[Rezende and Mohamed(2015)]{rezende2015variational}
D.~Rezende and S.~Mohamed.
\newblock Variational inference with normalizing flows.
\newblock In \emph{Proceedings of the 34th International Conference on Machine Learning}, 2015.

\bibitem[Riabiz et~al.(2022)Riabiz, Chen, Cockayne, Swietach, Niederer, Mackey, and Oates]{riabiz2020optimal}
M.~Riabiz, W.~Chen, J.~Cockayne, P.~Swietach, S.~A. Niederer, L.~Mackey, and C.~J. Oates.
\newblock Optimal thinning of {MCMC} output.
\newblock \emph{Journal of the Royal Statistical Society, Series B}, 84\penalty0 (4):\penalty0 1059--1081, 2022.

\bibitem[Richter et~al.(2020)Richter, Boustati, N\"{u}sken, Ruiz, and Akyildiz]{richter2020}
L.~Richter, A.~Boustati, N.~N\"{u}sken, F.~Ruiz, and O.~D. Akyildiz.
\newblock {VarGrad}: {A} low-variance gradient estimator for variational inference.
\newblock In \emph{Proceedings of the 34th Conference on Neural Information Processing Systems}, 2020.

\bibitem[Ross(2011)]{Ross2011}
N.~Ross.
\newblock Fundamentals of {S}tein's method.
\newblock \emph{Probability Surveys}, 8:\penalty0 210--293, 2011.

\bibitem[Rudin(1987)]{rudin1987real}
W.~Rudin.
\newblock \emph{Real and Complex Analysis}.
\newblock McGraw-Hill, Inc., USA, 3 edition, 1987.

\bibitem[Ruiz et~al.(2016)Ruiz, Titsias, and Blei]{ruiz2016generalized}
F.~J. Ruiz, M.~K. Titsias, and D.~M. Blei.
\newblock The generalized reparameterization gradient.
\newblock In \emph{Proceedings of the 30th Conference on Neural Information Processing Systems}, 2016.

\bibitem[Salimans and Knowles(2014)]{salimans2014using}
T.~Salimans and D.~A. Knowles.
\newblock On using control variates with stochastic approximation for variational {B}ayes and its connection to stochastic linear regression, 2014.

\bibitem[Schrab et~al.(2022)Schrab, Guedj, and Gretton]{schrab2022ksd}
A.~Schrab, B.~Guedj, and A.~Gretton.
\newblock Ksd aggregated goodness-of-fit test.
\newblock In \emph{Proceedings of the 36th Conference on Neural Information Processing Systems}, 2022.

\bibitem[Schwabik and Ye(2005)]{schwabik2005topics}
S.~Schwabik and G.~Ye.
\newblock \emph{Topics in Banach Space Integration}.
\newblock World Scientific, 2005.

\bibitem[Shao(2010)]{shao2010dependent}
X.~Shao.
\newblock The dependent wild bootstrap.
\newblock \emph{Journal of the American Statistical Association}, 105\penalty0 (489):\penalty0 218--235, 2010.

\bibitem[Shi and Mackey(2023)]{shi2023finite}
J.~Shi and L.~Mackey.
\newblock A finite-particle convergence rate for stein variational gradient descent.
\newblock \emph{Advances in Neural Information Processing Systems}, 36:\penalty0 26831--26844, 2023.

\bibitem[Shi et~al.(2022{\natexlab{a}})Shi, Liu, and Mackey]{shi2022sampling}
J.~Shi, C.~Liu, and L.~Mackey.
\newblock Sampling with mirrored stein operators.
\newblock In \emph{Proceedings of the 10th International Conference on Learning Representations}, 2022{\natexlab{a}}.

\bibitem[Shi et~al.(2022{\natexlab{b}})Shi, Zhou, Hwang, Titsias, and Mackey]{shi2022gradient}
J.~Shi, Y.~Zhou, J.~Hwang, M.~Titsias, and L.~Mackey.
\newblock Gradient estimation with discrete {S}tein operators.
\newblock In \emph{Proceedings of the 36th Conference on Neural Information Processing Systems}, 2022{\natexlab{b}}.

\bibitem[Si et~al.(2020)Si, Oates, Duncan, Carin, and Briol]{si2020scalable}
S.~Si, C.~J. Oates, A.~B. Duncan, L.~Carin, and F.-X. Briol.
\newblock Scalable control variates for {Monte Carlo} methods via stochastic optimization.
\newblock In \emph{Proceedings of the 14th International Conference on Monte Carlo and Quasi-Monte Carlo Methods in Scientific Computing}, 2020.

\bibitem[Simon-Gabriel and Sch{\"o}lkopf(2018)]{simon2018kernel}
C.-J. Simon-Gabriel and B.~Sch{\"o}lkopf.
\newblock Kernel distribution embeddings: Universal kernels, characteristic kernels and kernel metrics on distributions.
\newblock \emph{Journal of Machine Learning Research}, 19\penalty0 (44):\penalty0 1--29, 2018.

\bibitem[Simon-Gabriel et~al.(2023)Simon-Gabriel, Barp, Sch{\"o}lkopf, and Mackey]{simon2020metrizing}
C.-J. Simon-Gabriel, A.~Barp, B.~Sch{\"o}lkopf, and L.~Mackey.
\newblock Metrizing weak convergence with maximum mean discrepancies.
\newblock \emph{Journal of Machine Learning Research}, 24\penalty0 (184):\penalty0 1--20, 2023.

\bibitem[Sohl-Dickstein et~al.(2011)Sohl-Dickstein, Battaglino, and DeWeese]{sohldickstein2011minimumprobabilityflowlearning}
J.~Sohl-Dickstein, P.~Battaglino, and M.~R. DeWeese.
\newblock Minimum probability flow learning, 2011.

\bibitem[South et~al.(2022{\natexlab{a}})South, Karvonen, Nemeth, Girolami, and Oates]{south2022semi}
L.~F. South, T.~Karvonen, C.~Nemeth, M.~Girolami, and C.~J. Oates.
\newblock Semi-exact control functionals from {Sard’s} method.
\newblock \emph{Biometrika}, 109\penalty0 (2):\penalty0 351--367, 2022{\natexlab{a}}.

\bibitem[South et~al.(2022{\natexlab{b}})South, Riabiz, Teymur, and Oates]{south2022postprocessing}
L.~F. South, M.~Riabiz, O.~Teymur, and C.~J. Oates.
\newblock Postprocessing of {MCMC}.
\newblock \emph{Annual Review of Statistics and Its Application}, 9\penalty0 (1):\penalty0 529--555, 2022{\natexlab{b}}.

\bibitem[South et~al.(2023)South, Oates, Mira, and Drovandi]{south2023regularized}
L.~F. South, C.~J. Oates, A.~Mira, and C.~Drovandi.
\newblock Regularized zero-variance control variates.
\newblock \emph{Bayesian Analysis}, 18\penalty0 (3):\penalty0 865--888, 2023.

\bibitem[Sriperumbudur et~al.(2011)Sriperumbudur, Fukumizu, and Lanckriet]{sriperumbudur2011universality}
B.~K. Sriperumbudur, K.~Fukumizu, and G.~R. Lanckriet.
\newblock Universality, characteristic kernels and rkhs embedding of measures.
\newblock \emph{Journal of Machine Learning Research}, 12:\penalty0 2389--2410, 2011.

\bibitem[Sriperumbudur et~al.(2012)Sriperumbudur, Fukumizu, Gretton, Sch{\"o}lkopf, and Lanckriet]{sriperumbudur2012empirical}
B.~K. Sriperumbudur, K.~Fukumizu, A.~Gretton, B.~Sch{\"o}lkopf, and G.~R. Lanckriet.
\newblock On the empirical estimation of integral probability metrics.
\newblock \emph{Electronic Journal of Statistics}, 6:\penalty0 1550--1599, 2012.

\bibitem[Stein(1972)]{stein1972bound}
C.~Stein.
\newblock A bound for the error in the normal approximation to the distribution of a sum of dependent random variables.
\newblock In \emph{Proceedings of the 6th Berkeley Symposium on Mathematical Statistics and Probability, Volume 2: Probability Theory}, pages 583--602. University of California Press, 1972.

\bibitem[Stein(1986)]{stein1986approximate}
C.~Stein.
\newblock \emph{Approximate Computation of Expectations}.
\newblock Lecture Notes-Monograph Series. Institute of Mathematical Statistics, 1986.

\bibitem[Stein et~al.(2004)Stein, Diaconis, Holmes, Reinert, et~al.]{stein2004use}
C.~Stein, P.~Diaconis, S.~Holmes, G.~Reinert, et~al.
\newblock Use of exchangeable pairs in the analysis of simulations.
\newblock In \emph{{Stein's} Method}, pages 1--25. Institute of Mathematical Statistics, 2004.

\bibitem[Steinwart and Christmann(2008)]{steinwart2008support}
I.~Steinwart and A.~Christmann.
\newblock \emph{Support Vector Machines}.
\newblock Springer Science \& Business Media, 2008.

\bibitem[Steinwart et~al.(2006)Steinwart, Hush, and Scovel]{steinwart2006explicit}
I.~Steinwart, D.~Hush, and C.~Scovel.
\newblock An explicit description of the reproducing kernel hilbert spaces of gaussian rbf kernels.
\newblock \emph{IEEE Transactions on Information Theory}, 52\penalty0 (10):\penalty0 4635--4643, 2006.

\bibitem[Stuart et~al.(2004)Stuart, Voss, Wilberg, et~al.]{stuart2004conditional}
A.~Stuart, J.~Voss, P.~Wilberg, et~al.
\newblock Conditional path sampling of {SDE}s and the {Langevin} {MCMC} method.
\newblock \emph{Communications in Mathematical Sciences}, 2\penalty0 (4):\penalty0 685--697, 2004.

\bibitem[Sun et~al.(2023)Sun, Barp, and Briol]{sun2023vector}
Z.~Sun, A.~Barp, and F.-X. Briol.
\newblock Vector-valued control variates.
\newblock In \emph{Proceedings of the 42nd International Conference on Machine Learning}, 2023.

\bibitem[Sutton and Barto(1998)]{sutton1998reinforcement}
R.~S. Sutton and A.~G. Barto.
\newblock \emph{Reinforcement Learning: {A}n Introduction}.
\newblock MIT Press, 1998.

\bibitem[Teymur et~al.(2021)Teymur, Gorham, Riabiz, and Oates]{teymur2021optimal}
O.~Teymur, J.~Gorham, M.~Riabiz, and C.~J. Oates.
\newblock Optimal quantisation of probability measures using maximum mean discrepancy.
\newblock In \emph{Proceedings of the 24th International Conference on Artificial Intelligence and Statistics}, 2021.

\bibitem[Tian et~al.(2016)Tian, Bi, and Taylor]{tian2016magic}
X.~Tian, N.~Bi, and J.~Taylor.
\newblock Magic: a general, powerful and tractable method for selective inference.
\newblock \emph{arXiv preprint arXiv:1607.02630}, 2016.

\bibitem[Titsias and Shi(2022)]{titsias2022double}
M.~Titsias and J.~Shi.
\newblock Double control variates for gradient estimation in discrete latent variable models.
\newblock In \emph{Proceedings of the 25th International Conference on Artificial Intelligence and Statistics}, 2022.

\bibitem[Wainwright et~al.(2008)Wainwright, Jordan, et~al.]{wainwright2008graphical}
M.~J. Wainwright, M.~I. Jordan, et~al.
\newblock Graphical models, exponential families, and variational inference.
\newblock \emph{Foundations and Trends{\textregistered} in Machine Learning}, 1\penalty0 (1--2):\penalty0 1--305, 2008.

\bibitem[Wang et~al.(2023)Wang, Chen, Kanagawa, and Oates]{wang2023stein}
C.~Wang, W.~Chen, H.~Kanagawa, and C.~J. Oates.
\newblock Stein $\pi$-importance sampling.
\newblock In \emph{Proceedings of the 37th Conference on Neural Information Processing Systems}, 2023.

\bibitem[Wang and Liu(2016)]{wang2016learning}
D.~Wang and Q.~Liu.
\newblock Learning to draw samples: With application to amortized mle for generative adversarial learning, 2016.

\bibitem[Wang and Liu(2019)]{wang2019nonlinear}
D.~Wang and Q.~Liu.
\newblock Nonlinear {S}tein variational gradient descent for learning diversified mixture models.
\newblock In \emph{Proceedings of the 36th International Conference on Machine Learning}, 2019.

\bibitem[Wang et~al.(2017)Wang, Zeng, and Liu]{wang2018stein}
D.~Wang, Z.~Zeng, and Q.~Liu.
\newblock Stein variational message passing for continuous graphical models.
\newblock In \emph{Proceedings of the 34th International Conference on Machine Learning}, 2017.

\bibitem[Wang et~al.(2019)Wang, Tang, Bajaj, and Liu]{wang2019stein}
D.~Wang, Z.~Tang, C.~Bajaj, and Q.~Liu.
\newblock Stein variational gradient descent with matrix-valued kernels.
\newblock In \emph{Proceedings of the 33rd Conference on Neural Information Processing Systems}, 2019.

\bibitem[Welling and Teh(2011)]{welling2011bayesian}
M.~Welling and Y.~W. Teh.
\newblock {{Bayesian}} learning via stochastic gradient {Langevin} dynamics.
\newblock In \emph{International Conference on Machine Learning}, 2011.

\bibitem[Wendland(2004)]{wendland2004scattered}
H.~Wendland.
\newblock \emph{Scattered Data Approximation}.
\newblock Cambridge University Press, 2004.

\bibitem[Williams(1992)]{williams1992simple}
R.~J. Williams.
\newblock Simple statistical gradient-following algorithms for connectionist reinforcement learning.
\newblock \emph{Machine Learning}, 8\penalty0 (3):\penalty0 229--256, 1992.

\bibitem[Wynne and Duncan(2022)]{wynne2022kernel}
G.~Wynne and A.~B. Duncan.
\newblock A kernel two-sample test for functional data.
\newblock \emph{Journal of Machine Learning Research}, 23\penalty0 (73):\penalty0 1--51, 2022.

\bibitem[Wynne et~al.(2025)Wynne, Kasprzak, and Duncan]{wynne2025fourier}
G.~Wynne, M.~J. Kasprzak, and A.~B. Duncan.
\newblock A fourier representation of kernel stein discrepancy with application to goodness-of-fit tests for measures on infinite dimensional hilbert spaces.
\newblock \emph{Bernoulli}, 31\penalty0 (2):\penalty0 868--893, 2025.

\bibitem[Xiao et~al.(2017)Xiao, Rasul, and Vollgraf]{xiao2017fashion}
H.~Xiao, K.~Rasul, and R.~Vollgraf.
\newblock Fashion-{MNIST}: {A} novel image dataset for benchmarking machine learning algorithms, 2017.

\bibitem[Xifara et~al.(2014)Xifara, Sherlock, Livingstone, Byrne, and Girolami]{xifara2014langevin}
T.~Xifara, C.~Sherlock, S.~Livingstone, S.~Byrne, and M.~Girolami.
\newblock Langevin diffusions and the metropolis-adjusted langevin algorithm.
\newblock \emph{Statistics \& Probability Letters}, 91:\penalty0 14--19, 2014.

\bibitem[Xu and Matsuda(2020)]{xu2020stein}
W.~Xu and T.~Matsuda.
\newblock A {S}tein goodness-of-fit test for directional distributions.
\newblock In \emph{Proceedings of the 23rd International Conference on Artificial Intelligence and Statistics}, 2020.

\bibitem[Xu and Matsuda(2021)]{xu2021interpretable}
W.~Xu and T.~Matsuda.
\newblock Interpretable stein goodness-of-fit tests on riemannian manifolds.
\newblock In \emph{Proceedings of the 38th International Conference on Machine Learning}, 2021.

\bibitem[Yang et~al.(2018)Yang, Liu, Rao, and Neville]{yang2018goodness}
J.~Yang, Q.~Liu, V.~A. Rao, and J.~Neville.
\newblock Goodness-of-fit testing for discrete distributions via {S}tein discrepancy.
\newblock In \emph{Proceeding of the 35th International Conference on Machine Learning}, 2018.

\bibitem[Yu et~al.(2015)Yu, Zhang, Song, Seff, and Xiao]{yu15lsun}
F.~Yu, Y.~Zhang, S.~Song, A.~Seff, and J.~Xiao.
\newblock Lsun: Construction of a large-scale image dataset using deep learning with humans in the loop, 2015.

\bibitem[Zanella(2019)]{zanella2019informed}
G.~Zanella.
\newblock Informed proposals for local mcmc in discrete spaces.
\newblock \emph{Journal of the American Statistical Association}, 115\penalty0 (530):\penalty0 852--865, 2019.

\bibitem[Zhang et~al.(2020)Zhang, Peyr{\'e}, Fadili, and Pereyra]{zhang2020wasserstein}
K.~S. Zhang, G.~Peyr{\'e}, J.~Fadili, and M.~Pereyra.
\newblock Wasserstein control of mirror {L}angevin {M}onte {C}arlo.
\newblock In \emph{Proceedings of the 33rd Annual Conference on Learning Theory}, 2020.

\bibitem[Zhuo et~al.(2018)Zhuo, Liu, Shi, Zhu, Chen, and Zhang]{zhuo2018message}
J.~Zhuo, C.~Liu, J.~Shi, J.~Zhu, N.~Chen, and B.~Zhang.
\newblock Message passing {Stein} variational gradient descent.
\newblock In \emph{Proceedings of the 35th International Conference on Machine Learning}, 2018.

\end{thebibliography}

\end{document}